\documentclass{article}

\usepackage{microtype}
\usepackage{graphicx}
\usepackage{booktabs} % for professional tables

\usepackage{hyperref}

\usepackage[accepted]{icml2024}

\usepackage{amsmath}
\usepackage{amssymb}
\usepackage{mathtools}
\usepackage{amsthm}

\usepackage[capitalize,noabbrev]{cleveref}

\theoremstyle{plain}
\newtheorem{theorem}{Theorem}[section]
\newtheorem{proposition}[theorem]{Proposition}

\theoremstyle{definition}

\theoremstyle{remark}

\usepackage{bm}  % assumes amsmath package installed
\usepackage{graphics} % for pdf, bitmapped graphics files
\usepackage{epsfig} % for postscript graphics files
\usepackage{subcaption}
\usepackage[font=small,labelfont=small]{caption}
\usepackage{hyperref}
\usepackage{color}
\usepackage{sidecap}
\sidecaptionvpos{figure}{c}
\usepackage{tikz}
\usetikzlibrary{arrows.meta}
\usepackage{tabularx}
\usepackage{colortbl}
\usepackage{booktabs}
\usepackage{multirow}
\usepackage{siunitx}

\graphicspath{{Figures/}}

\definecolor{lightred}{rgb}{0.8, 0.22, 0.29}
\definecolor{turquoise}{rgb}{0.25, 0.89, 0.82}
\definecolor{mediumorchid}{rgb}{0.73, 0.33, 0.83}
\definecolor{darkorange}{rgb}{1.0, 0.65, 0.0}
\definecolor{navy}{rgb}{0.0, 0.0, 0.5}
\definecolor{dodgerblue}{rgb}{0.12, 0.565, 1.0}
\definecolor{crimson}{rgb}{0.86, 0.08, 0.235}
\DeclareRobustCommand{\redGeod}{\raisebox{2pt}{\tikz{\draw[lightred,solid,line width = 1.1pt](0,0) -- (4mm,0);}}}
\DeclareRobustCommand{\turquoisecircle}{\tikz{ \filldraw[color=white, fill=turquoise, thick](0,0) circle (.075);}}
\DeclareRobustCommand{\orchidcircle}{\tikz{ \filldraw[color=white, fill=mediumorchid, thick](0,0) circle (.075);}}
\DeclareRobustCommand{\yellowarrow}{\tikz{\draw[color=darkorange, -{Triangle[width = 4pt, length = 2pt]}, line width = 1.2pt] (0.0, 0.0) -- (.5, 0.0);}}
\DeclareRobustCommand{\crimsonline}{\raisebox{2pt}{\tikz{\draw[crimson,solid,line width = 1.1pt](0,0) -- (4mm,0);}}}
\DeclareRobustCommand{\crimsondashedline}{\raisebox{2pt}{\tikz{\draw[crimson,dashed,line width = 1.1pt](0,0) -- (4mm,0);}}}
\DeclareRobustCommand{\navyline}{\raisebox{2pt}{\tikz{\draw[navy,solid,line width = 1.1pt](0,0) -- (4mm,0);}}}
\DeclareRobustCommand{\navydashedline}{\raisebox{2pt}{\tikz{\draw[navy,dashed,line width = 1.1pt](0,0) -- (4mm,0);}}}
\DeclareRobustCommand{\dodgerblueline}{\raisebox{2pt}{\tikz{\draw[dodgerblue,solid,line width = 1.1pt](0,0) -- (4mm,0);}}}
\DeclareRobustCommand{\dodgerbluedashedline}{\raisebox{2pt}{\tikz{\draw[dodgerblue,dashed,line width = 1.1pt](0,0) -- (4mm,0);}}}

\DeclareRobustCommand{\vertblackline}{\raisebox{-1pt}{\tikz{\draw[black,solid,line width = 1.1pt](0,-1.5mm) -- (0,1mm);}}}

\newcommand{\trsp}{\mathsf{T}}
\newcommand{\ty}[1]{{\scriptscriptstyle{\mathcal{#1}}}}

\DeclareMathOperator*{\argmax}{argmax} % thin space, limits underneath in displays
\DeclareMathOperator{\GP}{GP}
\newcommand{\Rho}{\mathrm{P}}
\DeclareSymbolFont{bbold}{U}{bbold}{m}{n}
\DeclareSymbolFontAlphabet{\mathbbold}{bbold}

\newcommand{\euclideanspace}{\mathbb{R}}
\newcommand{\manifold}{\mathcal{M}}
\newcommand{\tangentspace}[1]{\mathcal{T}_{#1}\mathcal{M}}

\newcommand{\innerprod}[3]{\langle #2, #3 \rangle_{#1}}  % Inner product of #2 and #3 at #1
\newcommand{\norm}[2]{\| #2\|_{#1}}  % Norm of #2 at #1
\newcommand{\expmap}[2]{\text{Exp}_{#1}(#2)}  % Exponential map of #2 at #1
\newcommand{\logmap}[2]{\text{Log}_{#1}(#2)}  % Logarithmic map of #2 at #1
\newcommand{\prltrsp}[3]{\Rho_{#1 \rightarrow #2}\big(#3\big)}  % Parallel transport of #3 from #1 to #2
\newcommand{\hyperbolic}[1]{\mathbb{H}^{#1}} % Notation of hyperbolic manifold of dimensionality #1
\newcommand{\lorentz}[1]{\mathcal{L}^{#1}} % Notation of hyperbolic manifold of dimensionality #1
\newcommand{\tangentspacelorentz}[2]{\mathcal{T}_{#1}\mathcal{L}^{#2}}  % Tangent space for hyperbolic manifold
\newcommand{\poincare}[1]{\mathcal{P}^{#1}} % Notation of Poincaré model of hyperbolic manifold of dimensionality #1
\newcommand{\manifolddist}[2]{d_{\manifold}(#1, #2)}
\newcommand{\hypenormal}[3]{\mathcal{N}_{\lorentz{d}}(#1;#2,#3)} % Hyperbolic Gaussian with sample #1, mean #2 and covariance #3
\newcommand{\hypenormallatent}[3]{\mathcal{N}_{\lorentz{Q}}(#1;#2,#3)} % Hyperbolic Gaussian with sample #1, mean #2 and covariance #3
\newcommand{\hypenormalprior}[2]{\mathcal{N}_{\lorentz{Q}}(#1,#2)} % Hyperbolic prior with mean #1 and covariance #2
\newcommand{\gaussiandist}[3]{\mathcal{N}(#1;#2,#3)} % Gaussian with sample #1, mean #2 and covariance #3
\newcommand{\kl}[2]{\text{KL}\big( #1 || #2\big)}  % KL divergence between #1 and #2
\newcommand{\expectation}[2]{\mathbb{E}_{#1}\left[#2\right]}  % Expectation 

\icmltitlerunning{Bringing Motion Taxonomies to Continuous Domains via GPLVM on Hyperbolic manifolds}

\begin{document}

\twocolumn[
\icmltitle{Bringing Motion Taxonomies to Continuous Domains \\ via GPLVM on Hyperbolic Manifolds}

\begin{icmlauthorlist}
\icmlauthor{No\'emie Jaquier}{kit}
\icmlauthor{Leonel Rozo}{bcai}
\icmlauthor{Miguel Gonz\'alez-Duque}{cop}
\icmlauthor{Viacheslav Borovitskiy}{eth}
\icmlauthor{Tamim Asfour}{kit}
\end{icmlauthorlist}

\icmlaffiliation{kit}{Karlsruhe Institute of Technology}
\icmlaffiliation{bcai}{Bosch Center for Artificial Intelligence}
\icmlaffiliation{cop}{University of Copenhagen}
\icmlaffiliation{eth}{ETH Z\"urich}

\icmlcorrespondingauthor{No\'emie Jaquier}{\href{mailto:noemie.jaquier@kit.edu}{noemie.jaquier@kit.edu}}
\icmlkeywords{GPLVM, hyperbolic manifold, Riemannian geometry, motion taxonomies}

\vskip 0.3in
]

\printAffiliationsAndNotice{}  

\begin{abstract}
Human motion taxonomies serve as high-level hierarchical abstractions that classify how humans move and interact with their environment. 
They have proven useful to analyse grasps, manipulation skills, and whole-body support poses.
Despite substantial efforts devoted to design their hierarchy and underlying categories, their use remains limited.
This may be attributed to the lack of computational models that fill the gap between the discrete hierarchical structure of the taxonomy and the high-dimensional heterogeneous data associated to its categories.
To overcome this problem, we propose to model taxonomy data via hyperbolic embeddings that capture the associated hierarchical structure.
We achieve this by formulating a novel Gaussian process hyperbolic latent variable model that incorporates the taxonomy structure through graph-based priors on the latent space and distance-preserving back constraints.
We validate our model on three different human motion taxonomies to learn hyperbolic embeddings that faithfully preserve the original graph structure.
We show that our model properly encodes unseen data from existing or new taxonomy categories, and outperforms its Euclidean and VAE-based counterparts. 
Finally, through proof-of-concept experiments, we show that our model may be used to generate realistic trajectories between the learned embeddings. 
\vspace{-0.3cm}
\end{abstract}

\section{Introduction}
Robotic systems or virtual characters that exhibit human- or animal-like capabilities are often inspired by biological insights~\citep{SicilianoKhatib16:Handbook}.
In the particular context of motion generation, it is first necessary to understand how humans move and interact with their environment to then generate biologically-inspired motions and behaviors of robotic hands, arms, humanoids, or animated characters.
In this endeavor, researchers have proposed to structure and categorize human hand grasps and body poses into hierarchical classifications known as \emph{taxonomies}. 
Their structure depends on the sensory variables considered to categorize human motions and the interactions with the environment, as well as on associated qualitative measures.

Different taxonomies have been proposed in the area of human and robot grasping~\citep{Cutkosky89:GraspTaxonomy,Feix16:GRASPtaxonomy,Abbasi16:ForceGraspTaxonomy,Stival19:HumanGraspTaxonomy}. 
\citet{Feix16:GRASPtaxonomy} introduced a hand grasp taxonomy whose structure was mainly defined by the hand pose and the type of contact with the object.
As such taxonomy heavily depends on subjective qualitative measures, \citet{Stival19:HumanGraspTaxonomy} proposed a quantitative tree-like hand grasp taxonomy based on muscular and kinematic patterns.
A similar data-driven approach was used to design a grasp taxonomy based on contact forces in~\citep{Abbasi16:ForceGraspTaxonomy}.
\citet{Bullock13:HandTaxonomy} introduced a hand-centric manipulation taxonomy that classifies manipulation skills according to the type of contact with the objects and the object motion imparted by the hand.
A different strategy was developed by~\citet{Paulius19:MotionTaxonomy}, who designed a manipulation taxonomy based on a categorization of contacts and motion trajectories.
Humanoid robotics also made significant efforts to analyze human motions, thus proposing taxonomies as high-level abstractions of human motion configurations.  
For example, \citet{Borras17:WholeBodyTaxonomy} analyzed the contacts between the human limbs and the environment to design a whole-body support pose taxonomy. 

Besides their analytical purpose in biomechanics or robotics, some of the aforementioned taxonomies were employed for modeling grasp actions~\citep{Romero10:SpatioTempGraspsGPLVM,Lin15:GraspPlanning}, for planning contact-aware whole-body pose sequences~\citep{Mandery16:LanguageWholeBody}, and for learning manipulation skills embeddings~\citep{Paulius20:ManipulationTaxonomy}. 
However, despite most of these taxonomies carry a well-defined hierarchical structure, it was often overlooked. 
First, these taxonomies were usually employed for classification tasks where only the tree leaves were used to define target classes, disregarding the full taxonomy structure~\citep{Feix16:GRASPtaxonomy,Abbasi16:ForceGraspTaxonomy}.
Second, the discrete representation of the taxonomy categories hindered their use for motion generation~\citep{Romero10:SpatioTempGraspsGPLVM}.

Arguably the main difficulty of leveraging human motion taxonomies is the lack of computational models that exploit \emph{(i)} the domain knowledge encoded in the hierarchy, and \emph{(ii)} the information of the high-dimensional data associated to the taxonomy categories. 
We tackle this problem from a representation learning perspective by modeling taxonomy data as embeddings that capture the associated hierarchical structure.
Inspired by the pioneer work of~\citet{Krioukov10:HyperbolicComplexNet} on the use of hyperbolic geometry on complex hierarchies, and by recent advances on hierarchical representation learning~\citep{Nickel2017:Poincare,Nickel2018:Lorentz,Mathieu19:HyperbolicVAE,Montanaro22:HyperbolicPointCloud}, we propose to leverage the \emph{hyperbolic manifold}~\citep{Ratcliffe19:HyperbolicManifold} to learn such embeddings. 
An important property of the hyperbolic manifold is that distances grow exponentially when moving away from the origin, and shortest paths between distant points tend to pass through it, resembling a \emph{continuous hierarchical structure}.
Therefore, we hypothesize that the geometry of the hyperbolic manifold allows us to learn embeddings that comply with the hierarchical structure of human motion taxonomies. 

\textbf{In this paper} we propose a Gaussian process hyperbolic latent variable model (GPHLVM) to learn embeddings of taxonomy data on the hyperbolic manifold. 
Our \emph{first contribution} tackles the challenges that arise when imposing a hyperbolic geometry to the latent space of the well-known GPLVM~\citep{Lawrence03:GPLVM,Titsias10:BayesGPVLM}, a model that has been successfully applied in human pose estimation and motion generation~\citep{Lawrence06:BackConstrGPLVM,Urtasun08:TopologicalGPLVM,Gupta08:GPLVMhumanpose,Ding15:GPLVMhumangait,Lalchand22:GPLVMstochasticVI}, and in complex settings such as robotic dressing assistants~\cite{Koganti19:ClothRobotGPLVM}. Specifically, we reformulate the Gaussian distribution, the kernel, and the optimization process of the GPLVM to account for the geometry of the hyperbolic latent space.
To do so, we leverage the hyperbolic wrapped Gaussian distribution~\citep{Nagano19:HyperbolicNormal}, and provide a positive-definite-guaranteed approximation of the hyperbolic kernel proposed by~\citet{McKean70:HyperbolicKernel}.
Moreover, we resort to Riemannian optimization~\citep{Absil07:RiemannOpt,Boumal22:RiemannOpt} to optimize the GPHLVM embeddings.
Our GPHLVM is conceptually similar to the GPLVM for Lie groups~\citep{Jensen20:ManifoldGPLVM}, which also imposes geometric properties to the GPLVM latent space. 
However, our formulation is specifically designed for the hyperbolic manifold and fully built on tools from Riemannian geometry. 
Moreover, unlike~\citep{Tosi14:RiemannianGPLVM} and~\citep{Jorgensen21:RiemannianGPLVM}, where the latent space was endowed with a pullback Riemannian metric learned via the GPLVM mapping, we impose the hyperbolic geometry to the GPHLVM latent space as an inductive bias to comply with the hierarchical structure of taxonomy data.
As a \emph{second contribution}, we introduce mechanisms to enforce the taxonomy structure in the learned embeddings through graph-based priors on the latent space and via graph-distance-preserving back constraints~\citep{Lawrence06:BackConstrGPLVM,Urtasun08:TopologicalGPLVM}.

We validate our approach on three distinct human motion taxonomies: a bimanual manipulation taxonomy~\citep{Krebs22:BimanualTaxonomy}, a hand grasps taxonomy~\citep{Stival19:HumanGraspTaxonomy}, and a whole-body support pose taxonomy~\citep{Borras17:WholeBodyTaxonomy}.
The proposed GPHLVM successfully learns hyperbolic embeddings that comply with the original graph structure of all the considered taxonomies, and it properly encodes unseen poses from existing or new taxonomy nodes. 
Moreover, we show how we can exploit the continuous geometry of the hyperbolic manifold to generate trajectories between different embeddings pairs via geodesic paths in the latent space.
We leverage this to, for example, generate realistic trajectories that are competitive with state-of-the-art character animation, while being trained at low data regimes.
Our results show that GPHLVM consistently outperforms its Euclidean and VAE-based counterparts. The source code and video accompanying the paper are available at \url{https://sites.google.com/view/gphlvm/}.

\begin{figure*}
	\vspace{-0.1cm}
	\centering
	\includegraphics[trim={0.cm 2.0cm 0.cm 1.8cm},clip,width=0.29\textwidth]{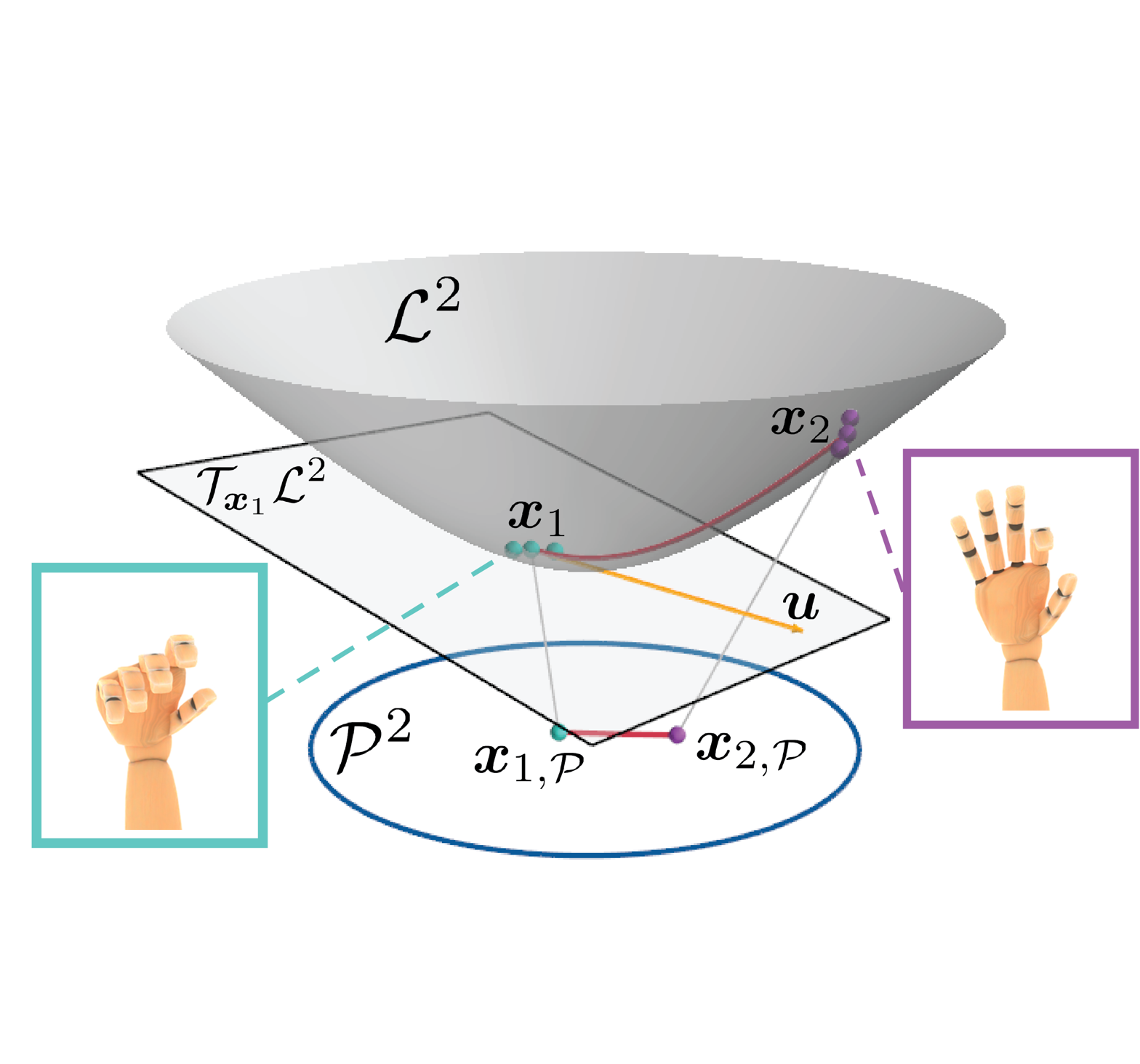}
    \hspace{0.2cm}
	\includegraphics[trim={0.cm 0.0cm 0.cm 0.0cm},clip,width=0.68\textwidth]{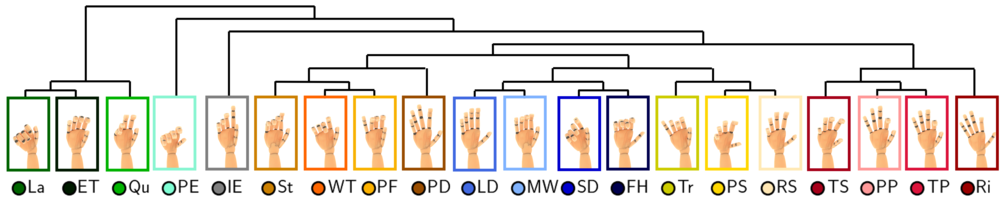}
	\caption{\emph{Left:} Illustration of the Lorentz $\mathcal{L}^2$ and Poincaré $\mathcal{P}^2$ models of the hyperbolic manifold. The former is depicted as the gray hyperboloid, while the latter is represented by the blue circle. Both models show a geodesic (\redGeod) between two points $\bm{x}_1$ (\turquoisecircle) and $\bm{x}_2$ (\orchidcircle). The vector $\bm{u}$ (\yellowarrow) lies on the tangent space of $\bm{x}_1$ such that $\bm{u} = \logmap{\bm{x}_1}{\bm{x}_2}$. \emph{Right:} Hand grasp taxonomy~\citep{Stival19:HumanGraspTaxonomy} used in one of our experiments. Grasp types are organized in a tree structure based on their muscular and kinematic properties. Each leaf node of the tree is a hand grasp type. The lines represent the depth of the leaves, e.g., $\mathsf{PE}$ and $\mathsf{IE}$ are at distance $2$ and $3$ from the root node.}
	\label{fig:HyperbolicAndTaxonomy}
	\vspace{-0.5cm}
\end{figure*}

%===============================================================================
\vspace{-0.3cm}
\section{Background}
\label{sec:background}
\paragraph{Gaussian Process Latent Variable Models:}
A GPLVM defines a generative mapping from latent variables $\{\bm{x}_n\}_{n=1}^N, \bm{x}_n\in\euclideanspace^Q$ to observations $\{\bm{y}_n\}_{n=1}^N, \bm{y}_n\in\euclideanspace^D$ by modeling the corresponding non-linear transformation with Gaussian processes (GPs)~\citep{Lawrence03:GPLVM}.
The GPLVM is described as,
\begin{equation}
\begin{split}
    &y_{n,d} \sim \gaussiandist{y_{n,d}}{f_{n,d}}{\sigma^2_d} \\ & \text{ with }\;\;\; f_{n,d} \sim \GP(m_d (\bm{x}_n), k_{d}(\bm{x}_n,\bm{x}_n)), \\ & \text{ and }\;\;\; \bm{x}_n \sim \mathcal{N}(\bm{0},\bm{I}),
\end{split}
\label{eq:GPLVM}
\end{equation}
where $y_{n,d}$ denotes the $d$-th dimension of the observation $\bm{y}_n$, $m_d(\cdot):\euclideanspace^Q\to \euclideanspace$ and $k_d(\cdot, \cdot):\euclideanspace^Q \times \euclideanspace^Q \to \euclideanspace$ are the GP mean and kernel function, respectively, and $\sigma_d^2$ is a hyperparameter.
Conventionally, the hyperparameters and latent variables of the GPLVM were optimized using \textit{maximum likelihood} or \textit{maximum a posteriori} (MAP) estimates. 
As this does not scale gracefully to large datasets, contemporary methods use inducing points and variational approximations of the evidence~\citep{Titsias10:BayesGPVLM}. 
In contrast to neural-network-based generative models, GPLVMs are data efficient and provide automatic uncertainty quantification.

\vspace{-0.35cm}
\paragraph{Riemannian geometry:}
To understand the hyperbolic manifold, it is necessary to first define some basic Riemannian geometry concepts~\citep{Lee18:RiemannManifold}. 
To begin with, consider a Riemannian manifold $\manifold$, which is a locally Euclidean topological space with a globally-defined differential structure. 
For each point $\bm{x}\!\in\!\manifold$, there exists a tangent space $\tangentspace{\bm{x}}$ that is a vector space consisting of the tangent vectors of all the possible smooth curves passing through $\bm{x}$. 
A Riemannian manifold is equipped with a Riemannian metric, which permits to define curve lengths in $\manifold$. 
Shortest-path curves, called geodesics, can be seen as the generalization of straight lines on the Euclidean space to Riemannian manifolds, as they are minimum-length curves between two points in $\manifold$.
To operate with Riemannian manifolds, it is common practice to exploit the Euclidean tangent spaces. 
To do so, we resort to mappings back and forth between $\tangentspace{\bm{x}}$ and $\manifold$, which are the exponential and logarithmic maps.
The exponential map $\expmap{\bm{x}}{\bm{u}}: \tangentspace{\bm{x}} \to \manifold$ maps a point $\bm{u}$ in the tangent space of $\bm{x}$ to a point $\bm{y}$ on the manifold, so that it lies on the geodesic starting at $\bm{x}$ in the direction $\bm{u}$, and such that the geodesic distance $d_{\manifold}$ between $\bm{x}$ and $\bm{y}$ equals the distance between $\bm{x}$ and $\bm{u}$. 
The inverse operation is the logarithmic map $\logmap{\bm{x}}{\bm{u}}: \manifold \to \tangentspace{\bm{x}}$.
Finally, the parallel transport $\prltrsp{\bm{x}}{\bm{y}}{\bm{u}}: \tangentspace{\bm{x}}\to\tangentspace{\bm{y}}$ operates with manifold elements lying on different tangent spaces.

\vspace{-0.35cm}
\paragraph{Hyperbolic manifold:}
The hyperbolic space $\hyperbolic{d}$ is the unique simply-connected complete $d$-dimensional Riemannian manifold with a constant negative sectional curvature~\citep{Ratcliffe19:HyperbolicManifold}.
There are several isometric models for the hyperbolic space, in particular, the Poincar\'{e} ball $\poincare{d}$ and the Lorentz (hyperboloid) model $\lorentz{d}$ (see Fig.~\ref{fig:HyperbolicAndTaxonomy}-\emph{left}).
The latter representation is chosen here as it is numerically more stable than the former, and thus better suited for Riemannian optimization (see App.~\ref{app:hyperbolic-operations} for the principal Riemannian operations and their illustration on the Lorentz model). 
However, the Poincar\'{e} model provides a more intuitive representation and is here used for visualization. 
This is easily achieved by leveraging the isometric mapping between both models (see App.~\ref{app:hyperbolic-isometry} for details).
An important property of the hyperbolic manifold is the exponential rate of the volume growth of a ball with respect to its radius.
In other words, distances in $\hyperbolic{d}$ grow exponentially when moving away from the origin, and shortest paths between distant points on the manifold tend to pass through the origin, resembling a continuous hierarchical structure.
Because of this, the hyperbolic manifold is often exploited to embed hierarchical data such as trees or graphs~\citep{Nickel2017:Poincare,Chami2020:TreesHyperbolic}.
Although its potential to embed discrete data structures into a continuous space is well known in the machine learning community, its application in motion analysis and generation is presently scarce.

\vspace{-0.35cm}
\paragraph{Hyperbolic wrapped Gaussian distribution:}
Probabilistic models on Riemannian manifolds demand to work with probability distributions that consider the manifold geometry.
We use the hyperbolic wrapped distribution~\citep{Nagano19:HyperbolicNormal}, which builds on a Gaussian distribution on the tangent space at the origin $\bm{\mu}_0=(1,0,\ldots,0)^\trsp$ of $\lorentz{d}$, that is then projected onto the hyperbolic space after transporting the tangent space to the desired location.
Intuitively, the construction of this wrapped distribution is as follows (see also Fig.~\ref{fig:appendix:hyperbolic-wrapped}): \emph{(1)} sample a point $\tilde{\bm{v}} \in \euclideanspace^d$ from the Euclidean normal distribution $\mathcal{N}(\bm{0}, \bm{\Sigma})$, \emph{(2)} transform $\tilde{\bm{v}}$ to an element of $\tangentspacelorentz{\bm{\mu}_0}{d} \subset \euclideanspace^{d+1}$ by setting $\bm{v} = (0, \tilde{\bm{v}})^\trsp$, \emph{(3)} apply the parallel transport $\bm{u} = \prltrsp{\bm{\mu}_0}{\bm{\mu}}{\bm{v}}$, and \emph{(4)} project $\bm{u}$ to $\lorentz{d}$ via $\expmap{\bm{\mu}}{\bm{u}}$.  
The resulting probability density function is,
\begin{align}
\log \hypenormal{\bm{x}}{\bm{\mu}}{\bm{\Sigma}} = &\log \gaussiandist{\bm{v}}{\bm{0}}{\bm{\Sigma}} \label{eq:WrappedHypeNormal} \\ &- (d-1) \log \left( \sinh(\norm{\ty{\mathcal{L}}}{\bm{u}}) / \norm{\ty{\mathcal{L}}}{\bm{u}} \right) ,
\nonumber
\end{align}
with $\bm{v} = \prltrsp{\bm{\mu}}{\bm{\mu}_0}{\bm{u}}$, $\bm{u} = \logmap{\bm{\mu}}{\bm{x}}$, and \mbox{$\norm{\ty{\mathcal{L}}}{\bm{u}}^2 = \innerprod{\bm{\mu}}{\bm{u}}{\bm{u}}$}.
The hyperbolic wrapped distribution~\citep{Nagano19:HyperbolicNormal} has a more general expression given in~\citep{Skopek2020:MixedCurvatureVAE}.

%===============================================================================
\vspace{-0.2cm}
\section{The proposed GPHLVM}
\label{sec:GPHLVM}
\vspace{-0.1cm}
We present the GPHLVM, that extends GPLVM to hyperbolic latent spaces.  
A GPHLVM defines a generative mapping from the hyperbolic latent space $\lorentz{Q}$ to the observation space, e.g., the data associated to the taxonomy, based on GPs. 
By considering independent GPs across the observation dimensions, the GPHLVM is formally described as,
\begin{equation}
\begin{split}
    & y_{n,d} \sim \gaussiandist{y_{n,d}}{f_{n,d}}{\sigma^2_d} \\ & \text{ with }\;\; f_{n,d} \sim \GP(m_d (\bm{x}_n), k_{d}^{\lorentz{Q}}(\bm{x}_n,\bm{x}_n)) \\ & \text{and}\;\; \bm{x}_n \sim \hypenormalprior{\bm{\mu}_0}{\alpha \bm{I}},
\end{split}
\label{eq:GPHLVM}
\end{equation}
where $y_{n,d}$ denotes the $d$-th dimension of the observation $\bm{y}_n\in\mathbb{R}^D$ and $\bm{x}_n\in\lorentz{Q}$ is the corresponding latent variable. 
Our GPHLVM is built on hyperbolic GPs, characterized by a mean function $m_d(\cdot):\lorentz{Q} \to \euclideanspace$ (usually set to $0$), and a kernel $k_d^{\lorentz{Q}}(\cdot, \cdot):\lorentz{Q} \times \lorentz{Q} \to \euclideanspace$. 
These kernels encode similarity information in the latent hyperbolic manifold and should reflect its geometry to perform effectively, as detailed in \S.~\ref{subsec:HyperbolicKernels}. 
Also, the latent variable $\bm{x}\in\lorentz{Q}$ is assigned a hyperbolic wrapped Gaussian prior $\hypenormalprior{\bm{\mu}_0}{\alpha \bm{I}}$ based on~\eqref{eq:WrappedHypeNormal}, where $\bm{\mu}_0$ is the origin of $\lorentz{Q}$, and the parameter $\alpha$ controls the spread of the latent variables in $\lorentz{Q}$. 
As Euclidean GPLVMs, our GPHLVM can be trained by finding a MAP estimate or via variational inference. However, special care must be taken to guarantee that the latent variables belong to the hyperbolic manifold, as explained in \S.~\ref{subsec:trainingGPHLVM}.
\vspace{-0.2cm}

\subsection{Hyperbolic kernels}
\label{subsec:HyperbolicKernels}
\vspace{-0.2cm}
For GPs in Euclidean spaces, the squared exponential (SE) and Mat\'ern kernels are standard choices~\citep{Rasmussen06:GPML}.
In the modern machine learning literature these were generalized to non-Euclidean spaces such as manifolds~\citep{Borovitskiy20:GPManifolds,Jaquier21:GaBOMatern} or graphs~\citep{Borovitskiy21:GPGraph}. 
The generalized SE kernels can be connected to the much studied \emph{heat kernels}. 
These are given (cf.~\citet{GrigoryanNoguchi98:HyperbolicHeatKernel}) by,
\begin{align}
\label{eq:hyp_heat}
\!\!\!\!\!k^{\lorentz{2}}(\bm{x}, \bm{x}')
&=
\frac{\sigma^2}{C_{\infty}}
\int_{\rho}^{\infty}
\frac{s e^{-s^2/(2 \kappa^2)}}{(\cosh(s) - \cosh(\rho))^{1/2}} \mathrm{d} s,
\\
k^{\lorentz{3}}(\bm{x}, \bm{x}')
&=
\frac{\sigma^2}{C_{\infty}}
\frac{\rho}{\sinh{\rho}} e^{-\rho^2/(2 \kappa^2)},
\end{align}
where $\rho = \operatorname{dist}_{\lorentz{d}}(\bm{x}, \bm{x}')$ is the geodesic distance between $\bm{x}, \bm{x}' \in \lorentz{d}$,  $\kappa$ and $\sigma^2$ are the kernel lengthscale and variance, and $C_{\infty}$ is a normalizing constant. 
To the best of our knowledge, no closed form expression for $\lorentz{2}$ is known.
In this case, the kernel is approximated via a discretization of the integral. 
One appealing option is the Monte Carlo approximation based on the truncated Gaussian density.
Unfortunately, such approximations easily fail to be positive semidefinite if the number of samples is not very large.
We address this via an alternative Monte Carlo approximation, 
\begin{equation} 
\label{eq:heat_monte_carlo}
k^{\lorentz{2}}(\bm{x}, \bm{x}')
\approx
\frac{\sigma^2}{C_{\infty}'}
\frac{1}{L} \sum_{l=1}^L s_l \tanh(\pi s_l)  \; w \; \overline{w} ,
\end{equation}
where $\langle \bm{x}_{\mathcal{P}}, \bm{b} \rangle = \frac{1}{2}\log\frac{1-|\bm{x}_{\mathcal{P}}|^2}{|\bm{x}_{\mathcal{P}}-\bm{b}|^2}$ is the hyperbolic outer product with $\bm{x}_{\mathcal{P}}$ being the representation of $\bm{x}$ as a point on the Poincar\'e disk $\mathcal{P}^2 = \mathbb{D}$, $w = e^{(2 s_l i + 1) \langle \bm{x}_{\mathcal{P}}, \bm{b}_l \rangle}$ with $i$, $\overline{z}$ denoting the imaginary unit and complex conjugation, respectively,
${\bm{b}_l \stackrel{\text{i.i.d.}}{\sim} U(\mathbb{T})}$ with $\mathbb{T}$ the unit circle, and $s_l \stackrel{\text{i.i.d.}}{\sim} e^{- s^2 \kappa^2 / 2} \mathbbold{1}_{[0, \infty)}(s)$.
The distributions of $\bm{b}_l$ and $s_l$ are easy to sample from: The former is sampled by applying $x \to e^{2 \pi i x}$ to $x \sim U([0, 1])$ and the latter is (proportional to) a truncated normal distribution.
Importantly, the right-hand side of \eqref{eq:heat_monte_carlo} is easily recognized to be an inner product in the space $\mathbb{C}^L$, which implies its positive semidefiniteness. Notice that we leverage the isometry between Lorentz and Poincaré models (see App.~\ref{app:hyperbolic}) for computing the kernel~\eqref{eq:heat_monte_carlo} (see App.~\ref{app:kernels} for details on~\eqref{eq:heat_monte_carlo}). Note that hyperbolic kernels for $\lorentz{Q}$ with $Q>3$ are generally defined as integrals of the kernels~\eqref{eq:hyp_heat}~\citep{GrigoryanNoguchi98:HyperbolicHeatKernel}. Analogs of Mat\'ern kernels for $\lorentz{Q}$ are obtained as integral of the SE kernel of the same dimension~\citep{Jaquier21:GaBOMatern}.

\vspace{-0.25cm}
\subsection{Model training}
\label{subsec:trainingGPHLVM}
\vspace{-0.2cm}
As in the Euclidean case, training the GPHLVM is equivalent to finding optimal latent variables $\bm{\mathcal{X}} = \{\bm{x}_n\}_{n=1}^N$ and hyperparameters $\bm{\Theta}=\{\theta_d\}_{d=1}^D$ by solving $\argmax_{\bm{\mathcal{X}},\bm{\Theta} } \ell$, with $\bm{x}_n\in\lorentz{Q}$, $\theta_d$ being the hyperparameters of the $d$-th GP, and $\ell$ as a loss function. We introduce a GPHLVM trained via MAP estimation for small datasets and a variational GPHLVM that handles larger datasets, providing users with the most appropriate tool for their needs. Note that recent extensions of GPLVM~\citep{Lalchand22:scalableGPLVM,Lalchand22:GPLVMstochasticVI} scale to massively large datasets via stochastic variational inference. 
For small datasets, the GPHLVM can be trained by maximizing the log posterior, i.e., $\ell_{\text{MAP}} = \log \big( p(\bm{Y} | \bm{X}) p(\bm{X}) \big)$ with $\bm{Y}=(\bm{y}_1 \ldots \bm{y}_N)^\trsp$ and $\bm{X}=(\bm{x}_1 \ldots \bm{x}_N)^\trsp$. For large datasets, the GPHLVM can be trained, similarly to the so-called Bayesian GPLVM~\citep{Titsias10:BayesGPVLM}, by maximizing the marginal likelihood of the data, i.e., $\ell_{\text{VA}} =\log p(\bm{Y}) = \log \int p(\bm{Y} | \bm{X}) p(\bm{X}) d\bm{X}$.
As this quantity is intractable, it is approximated via variational inference by adapting the methodology of~\citet{Titsias10:BayesGPVLM} to hyperbolic latent spaces, as explained next. Corresponding algorithms are provided in App.~\ref{app:Algorithms}.
\vspace{-0.2cm}

\paragraph{Variational inference:} We approximate the posterior $ p(\bm{X} | \bm{Y})$ by a variational distribution defined as a hyperbolic wrapped normal distribution over the latent variables, 
\begin{equation}
q_\phi(\bm{X}) = \prod_{n=1}^N \hypenormallatent{\bm{x}_n}{\bm{\mu}_n}{\bm{\Sigma_n}},
\label{eq:hyperbolic_variational_distribution}
\end{equation}
with variational parameters $\phi = \{\bm{\mu}_n, \bm{\Sigma}_n\}_{n=1}^N$, with $\bm{\mu}_n\in\lorentz{Q}$ and $\bm{\Sigma}_n\in\tangentspacelorentz{\bm{\mu}_n}{Q}$. 
Similarly to the Euclidean case~\citep{Titsias10:BayesGPVLM}, this variational distribution allows the formulation of a lower bound, 
\begin{equation}
\log p(\bm{Y}) \geq 
\expectation{q_\phi(\bm{X})}{\log p(\bm{Y} | \bm{X})} - \kl{q_\phi(\bm{X})}{p(\bm{X})}.
\label{eq:hyperbolicVariational}
\end{equation}
The KL divergence $\kl{q_\phi(\bm{X})}{p(\bm{X})}$ between two hyperbolic wrapped normal distributions can easily be evaluated via Monte-Carlo sampling (see App.~\ref{app:KLdivergenceWrappedGP} for details).
Moreover, the expectation $\expectation{q_\phi(\bm{X})}{\log p(\bm{Y} | \bm{X})}$ can be decomposed into individual terms for each observation dimension as $\sum_{d=1}^D \expectation{q_\phi(\bm{X})}{\log p(\bm{y}_d | \bm{X})}$, where $\bm{y}_d$ is the $d$-th column of $\bm{Y}$.
For large datasets, each term can be evaluated via a variational sparse GP approximation~\citep{Titsias09:SparseGP,Hensman15:ScalableVariationalGPs}. To do so, we introduce $M$ inducing inputs $\{\bm{z}_{d,m}\}_{m=1}^M$ with $\bm{z}_{d,m}\in\lorentz{Q}$ for each observation dimension $d$, whose corresponding inducing variables $\{u_{d,m}\}_{m=1}^M$ are defined as noiseless observations of the GP in~\eqref{eq:GPHLVM}, i.e, $u_d \sim \GP(m_d (\bm{z}_d), k_{d}^{\lorentz{Q}}(\bm{z}_d,\bm{z}_d))$. Similar to~\citep{Hensman15:ScalableVariationalGPs}, we can write,
\begin{equation}
\begin{split}
\log p(\bm{y}_d | \bm{X}) \geq &\expectation{q_\lambda(\bm{f}_{d})}{\log \gaussiandist{\bm{y}_{d}}{\bm{f}_{d}(\bm{X})}{\sigma^2_d}} \\ & - \kl{q_\lambda(\bm{u}_d)}{p(\bm{u}_d | \bm{Z}_d)} ,
\end{split}
\label{eq:ELBOsparse}
\end{equation}
where we defined $q_\lambda(\bm{f}_{d}) = \int p(\bm{f}_d|\bm{u}_d) q_\lambda(\bm{u}_{d})d\bm{u}_{d}$ with the variational distribution $q_\lambda(\bm{u}_{d})=\gaussiandist{\bm{u}_{d}}{\tilde{\bm{\mu}}_{d}}{\tilde{\bm{\Sigma}}_{d}}$, and variational parameters $\lambda = \{\tilde{\bm{\mu}}_d, \tilde{\bm{\Sigma}}_d\}_{d=1}^D$. Remember that the inducing variables $u_{d,m}$ are Euclidean, i.e., the variational distribution $q_\lambda(\bm{u}_{d})$ is a Euclidean Gaussian and the KL divergence in~\eqref{eq:ELBOsparse} has a closed-form solution.
In this case, the training parameters of the GPHLVM are the set of inducing inputs $\{\bm{z}_{d,m}\}_{m=1}^M$, the variational parameters $\phi$ and $\lambda$, and the hyperparameters $\bm{\Theta}$ (see App.~\ref{app:GPHLVMvariational} for the derivation of the GPHLVM variational inference process).
\vspace{-0.3cm}

\paragraph{Optimization:}
As several training parameters of the GPHLVM belong to $\lorentz{Q}$, i.e., the latent variables $\bm{x}_n$ for the MAP estimation, or the inducing inputs $\bm{z}_{d,m}$ and means $\bm{\mu}_n$ for variational inference, we need to account for their hyperbolic geometry during optimization.
To do so, we leverage Riemannian optimization methods~\citep{Absil07:RiemannOpt,Boumal22:RiemannOpt} to train the GPHLVM. 
Each step of first order (stochastic) Riemannian optimization methods is generally of the form,
\begin{equation}
\begin{split}
&\bm{\eta}_t \gets h\big(\text{grad}\: \ell(\bm{x}_t), \bm{\tau}_{t-1}\big), \\
&\bm{x}_{t+1} \gets \expmap{\bm{x}_t}{-\alpha_t \bm{\eta}_t}, \\
&\bm{\tau}_t \gets \prltrsp{\bm{x}_t}{\bm{x}_{t+1}}{\bm{\eta}_t}.
\end{split}
\label{eq:RiemannianGD}
\end{equation}
The update $\bm{\eta}_t\in\tangentspace{\bm{x}_t}$ is first computed as a function $h$ of the Riemannian gradient $\text{grad}$ of the loss $\ell(\bm{x}_t)$ and of $\bm{\tau}_{t-1}$, the previous update that is parallel-transported to the tangent space of the new estimate $\bm{x}_t$. The estimate $\bm{x}_t$ is then updated by projecting the update $\bm{\eta}_t$ scaled by a learning rate $\alpha_t$ onto the manifold using the exponential map. The function $h$ is equivalent to computing the update of the Euclidean algorithm, e.g., $\bm{\eta}_t \gets \text{grad}\: \ell(\bm{x}_t)$ for a simple gradient descent. Notice that~\eqref{eq:RiemannianGD} is applied on a product of manifolds when optimizing several parameters.
In this paper, we used the Riemannian Adam~\citep{Becigneul19:RiemannianAdaptiveOpt} implemented in Geoopt~\citep{Kochurov20:geoopt} to optimize the GPHLVM parameters.
\vspace{-0.2cm}

\section{Incorporating Taxonomy Knowledge} 
\label{sec:taxonomyKnowledge}
\vspace{-0.15cm}
While we are now able to learn hyperbolic embeddings of the data associated to a taxonomy using our GPHLVM, they do not necessarily follow the taxonomy graph structure. 
In other words, the manifold distances between pairs of embeddings do not necessarily match the graph distances.
To overcome this, we introduce graph-distance information as inductive bias to learn the embeddings.  
To do so, we leverage two well-known techniques in the GPLVM literature: priors on the embeddings and back constraints~\citep{Lawrence06:BackConstrGPLVM,Urtasun08:TopologicalGPLVM}.
Both are reformulated to preserve the taxonomy graph structure in the hyperbolic latent space as a function of the node-to-node shortest paths.

\vspace{-0.3cm}
\paragraph{Graph-distance priors:}
As shown by~\citet{Urtasun08:TopologicalGPLVM}, the structure of the latent space can be modified by adding priors of the form $p(\bm{X})\propto e^{-\phi(\bm{X})/\sigma^2_{\phi}}$ to the GPLVM, where $\phi(\bm{X})$ is a function that we aim at minimizing. 
Incorporating such a prior may also be alternatively understood as augmenting the GPLVM loss $\ell$ with a regularization term $-\phi(\bm{X})$.
Therefore, we propose to augment the loss of the GPHLVM with a distance-preserving graph-based regularizer. 
Several such losses have been proposed in the literature, see~\citep{Cruceru21:GraphEmbeddings} for a review. Specifically, we define $\phi(\bm{X})$ as the stress loss,
\begin{equation}
\ell_{\text{stress}}(\bm{X}) = \sum_{i<j}
\big( \operatorname{dist}_{\mathbb{G}}(c_i, c_j) -\operatorname{dist}_{\lorentz{Q}}(\bm{x}_i, \bm{x}_j) \big)^2 ,
\label{eq:stressLoss}
\end{equation}
where $c_i$ denotes the taxonomy node to which the observation $\bm{y}_i$ belongs, and $\operatorname{dist}_{\mathbb{G}}$, $\operatorname{dist}_{\lorentz{Q}}$ are the taxonomy graph distance and the geodesic distance on $\lorentz{Q}$, respectively. 
The loss~\eqref{eq:stressLoss} encourages the preservation of all distances of the taxonomy graph in $\lorentz{Q}$. 
It therefore acts \emph{globally}, thus allowing the complete taxonomy structure to be reflected by the GPHLVM.
Notice that~\citet{Cruceru21:GraphEmbeddings} also survey a distortion loss that encourages the distance of the embeddings to match the graph distance by considering their ratio.
However, this distortion loss is only properly defined when the embeddings $\bm{x}_i$ and $\bm{x}_j$ correspond to different classes $c_i \neq c_j$.
Interestingly, our empirical results using this loss were lackluster and numerically unstable (see App.~\ref{app:distortion}).

\vspace{-0.3cm}
\paragraph{Back-constraints:}
The back-constrained GPLVM \citep{Lawrence06:BackConstrGPLVM}  defines the latent variables as a function of the observations, i.e., $x_{n,q} = g_q(\bm{y}_1 \ldots, \bm{y}_n; \bm{w}_q)$ with parameters $\{\bm{w}_q\}_{q=1}^Q$. This allows us to incorporate new observations in the latent space after training, while preserving local similarities between observations in the embeddings. To incorporate graph-distance information into the GPHLVM and ensure that latent variables lie on the hyperbolic manifold, we propose the back-constraints mapping,
\begin{equation}
\begin{split}
    &\bm{x}_{n} = \expmap{\bm{\mu}_0}{\tilde{\bm{x}}_n} \\ & \text{ with } \;\;\;\; \tilde{x}_{n,q} = \sum_{m=1}^N w_{q,m} k^{\mathbb{R}^D}(\bm{y}_n,\bm{y}_m) k^{\mathbb{G}}(c_n, c_m).
\end{split}
\label{eq:backconstraints}
\end{equation}
The mapping~\eqref{eq:backconstraints} not only expresses the similarities between data in the observation space via the kernel $k^{\mathbb{R}^J}$, but encodes the relationships between data belonging to nearby taxonomy nodes via $k^{\mathbb{G}}$. 
In other words, similar observations associated to the same (or near) taxonomy nodes will be close to each other in the resulting latent space. 
The kernel $k^{\mathbb{G}}$ is a Mat\'ern kernel on the taxonomy graph following the formulation introduced in~\citep{Borovitskiy21:GPGraph}, which accounts for the graph geometry (see also App.~\ref{app:graphKernels}). 
We also use a Euclidean SE kernel for $k^{\mathbb{R}^D}$. 
Notice that the back constraints only incorporate \emph{local} information into the latent embedding. 
Therefore, to preserve the \emph{global} graph structure, we pair the proposed back-constrained GPHLVM with the stress prior~\eqref{eq:stressLoss}.
Note that both kernels are required in~\eqref{eq:backconstraints}: By defining the mapping as a function of the graph kernel only, the observations of each taxonomy node would be encoded by a single latent point. 
When using the observation kernel only, dissimilar observations of the same taxonomy node would be distant in the latent space, despite the additional stress prior, as $k^{\mathbb{R}^D}(\bm{y}_n,\bm{y}_m) \approx 0$. 
\vspace{-0.2cm}
%===============================================================================

\section{Experiments}
\label{sec:experiments}
\vspace{-0.1cm}
\begin{figure*}
	\centering
	\includegraphics[trim={5.8cm 2.2cm 4.3cm 2.2cm},clip,width=0.8\textwidth]{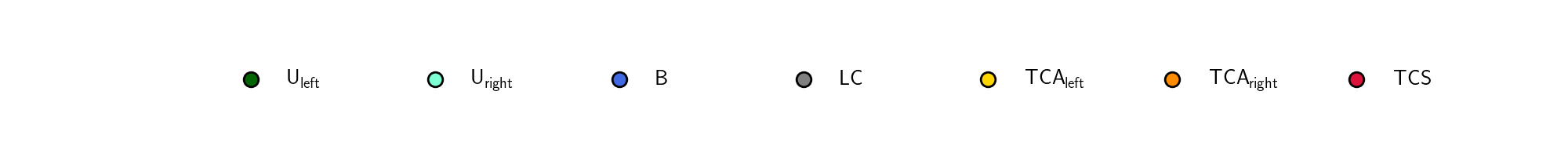}
	\begin{subfigure}[b]{0.15\textwidth}
		\centering
		\includegraphics[trim={2.5cm 2.5cm 2.5cm 2.5cm},clip,width=\textwidth]{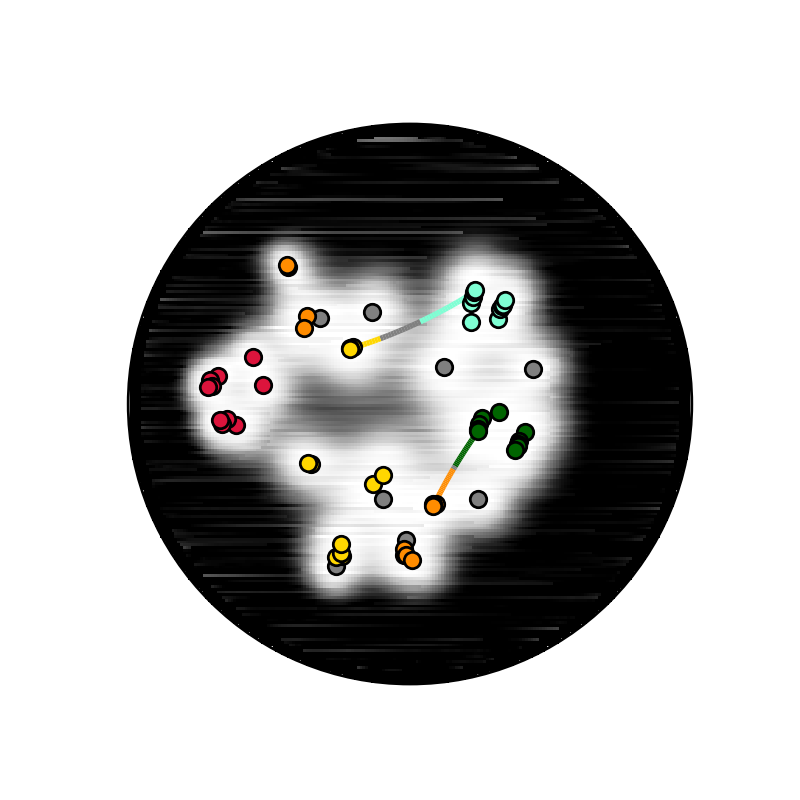}
		\includegraphics[trim={2.0cm 2.0cm 0.5cm 2.0cm},clip,width=0.9\textwidth]{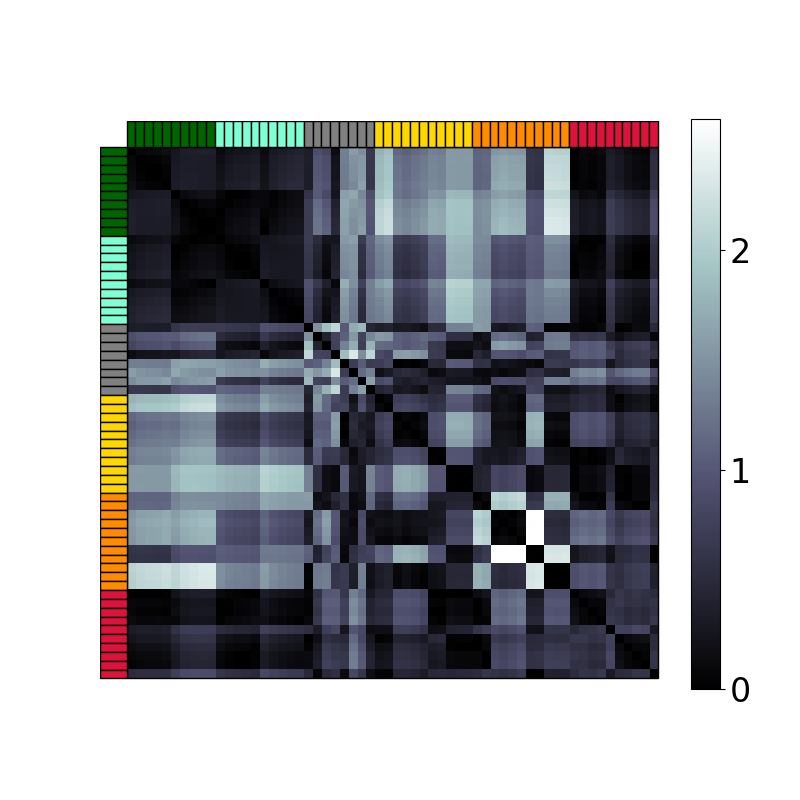}
		\includegraphics[trim={2.0cm 2.0cm 2.0cm 2.0cm},clip,width=0.9\textwidth]{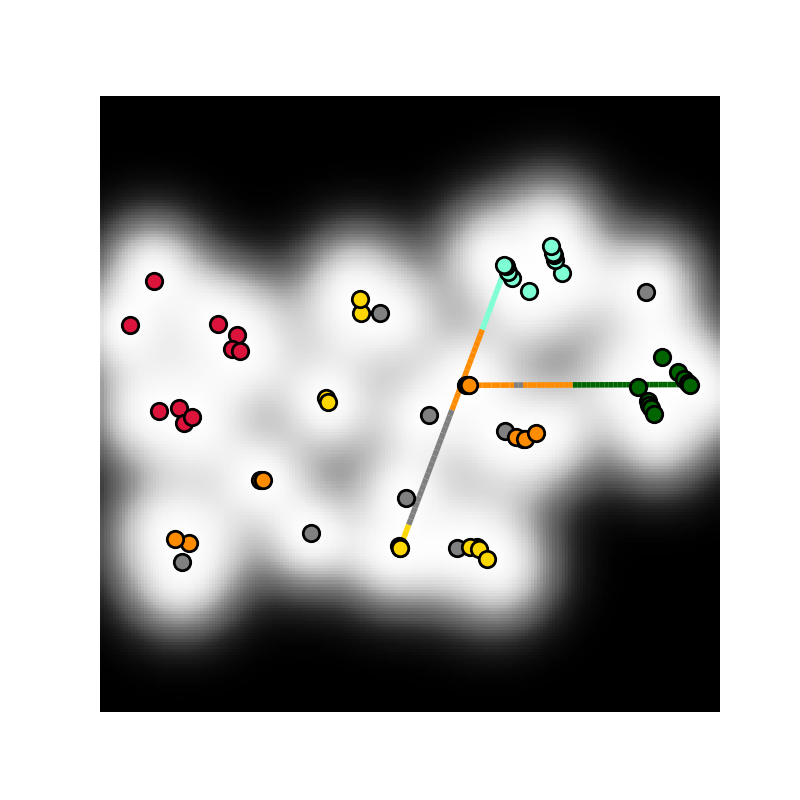}
		\includegraphics[trim={2.0cm 2.0cm 0.5cm 2.0cm},clip,width=0.9\textwidth]{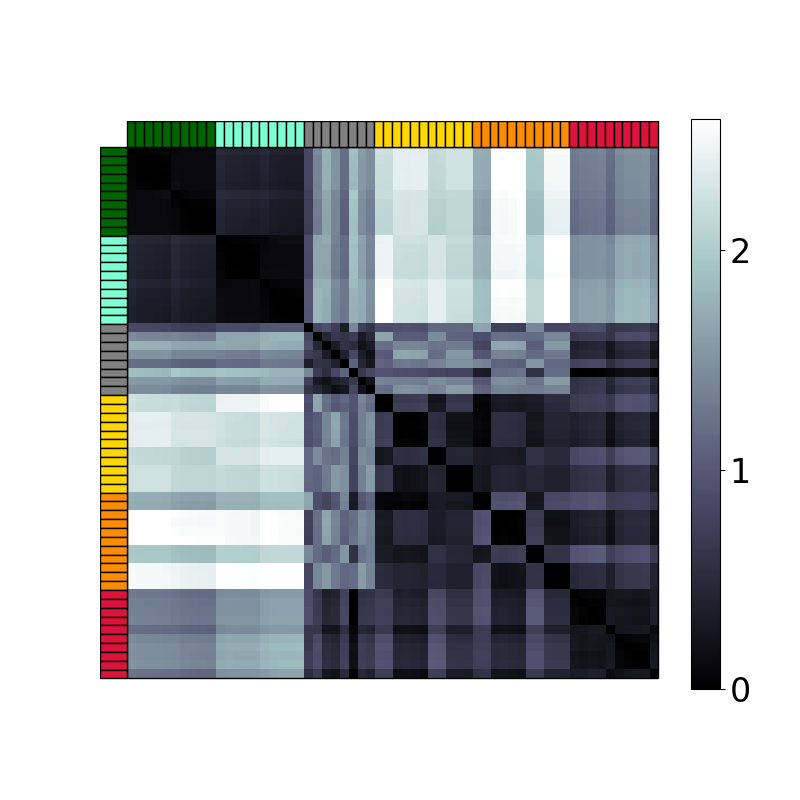}
		\caption{Vanilla}
		\label{fig:GPHLVM-bimanual:vanilla}
	\end{subfigure}%
	\begin{subfigure}[b]{0.15\textwidth}
		\centering
		\includegraphics[trim={2.5cm 2.5cm 2.5cm 2.5cm},clip,width=\textwidth]{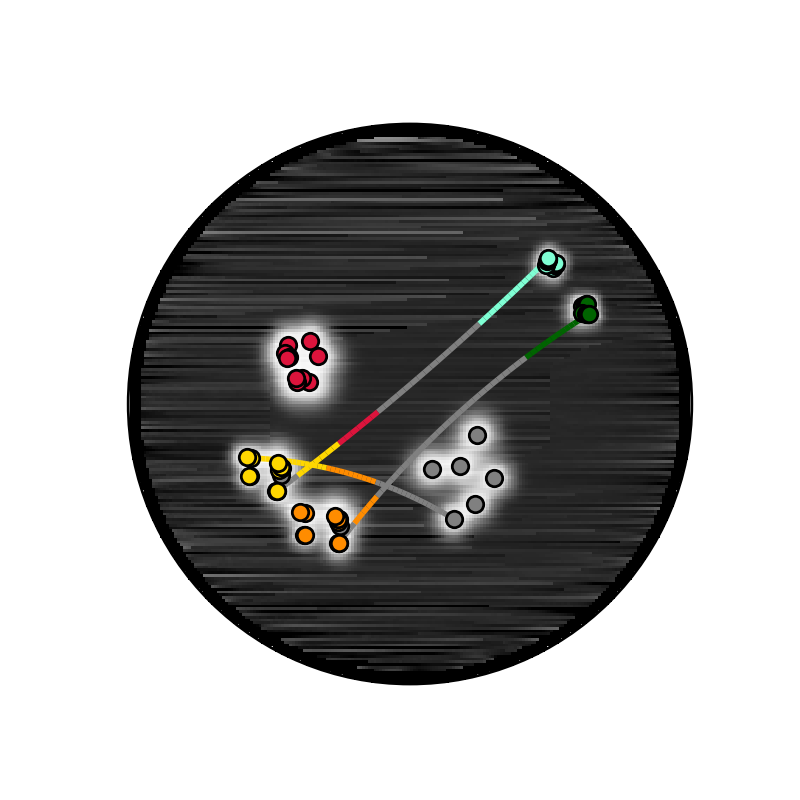}
		\includegraphics[trim={2.0cm 2.0cm 0.5cm 2.0cm},clip,width=0.9\textwidth]{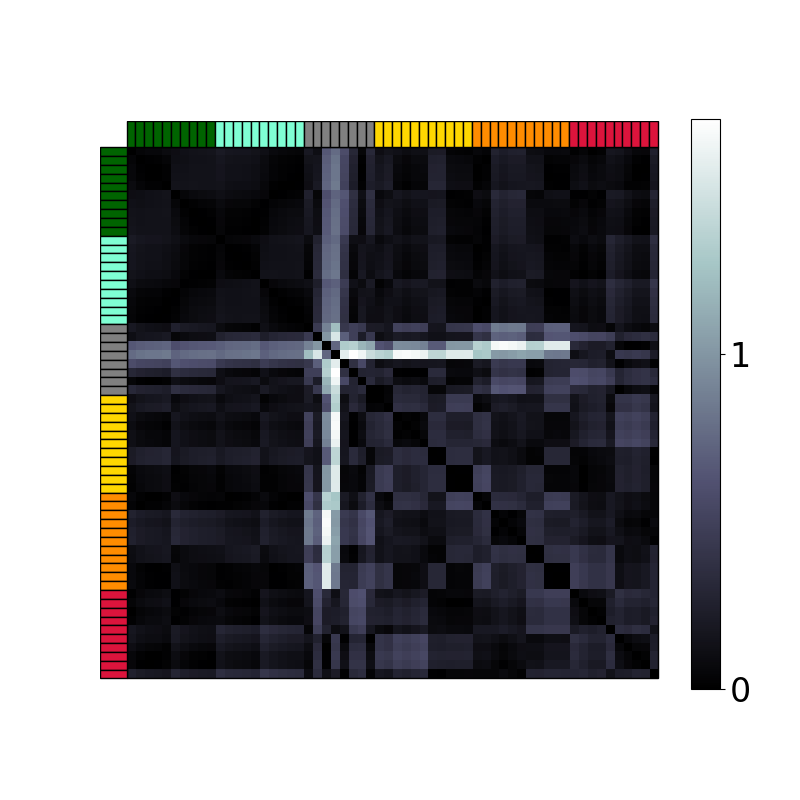}
		\includegraphics[trim={2.0cm 2.0cm 2.0cm 2.0cm},clip,width=0.9\textwidth]{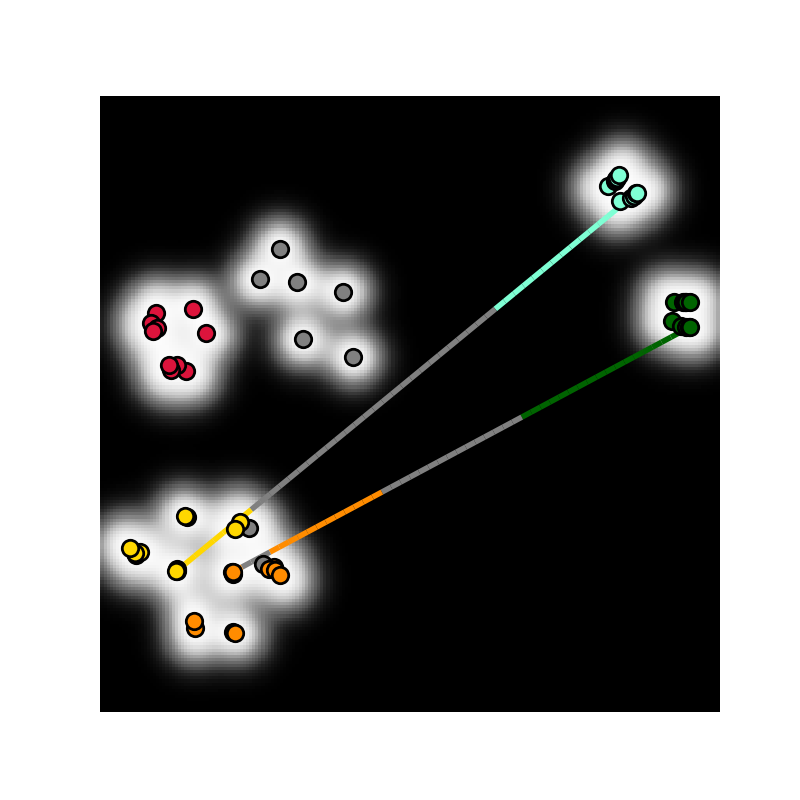}
		\includegraphics[trim={2.0cm 2.0cm 0.5cm 2.0cm},clip,width=0.9\textwidth]{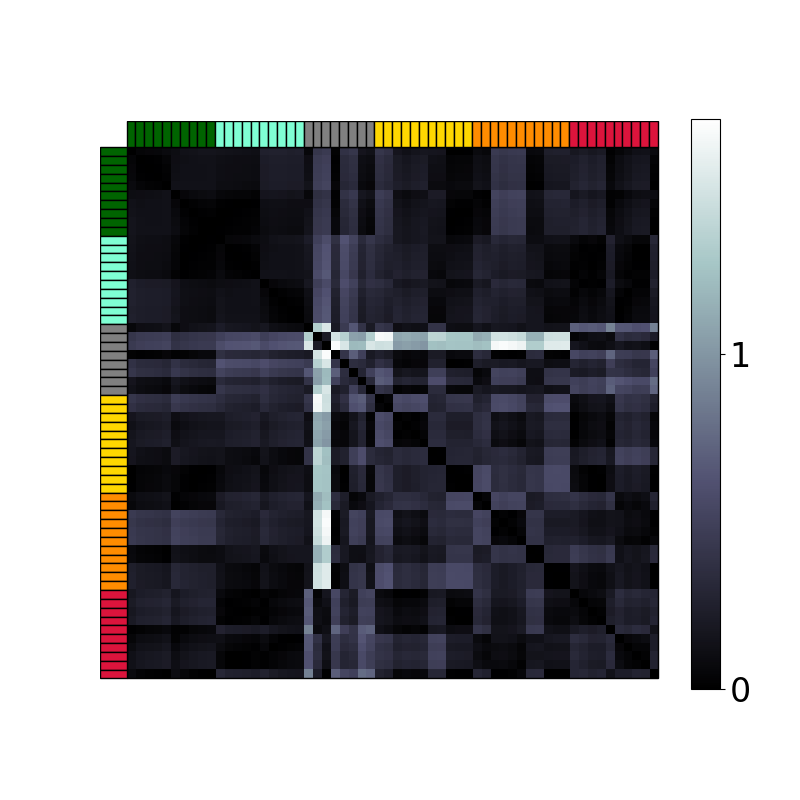}
		\caption{Stress prior}
		\label{fig:GPHLVM-bimanual:stress_prior}
	\end{subfigure}%
	\begin{subfigure}[b]{0.15\textwidth}
		\centering
		\includegraphics[trim={2.5cm 2.5cm 2.5cm 2.5cm},clip,width=\textwidth]{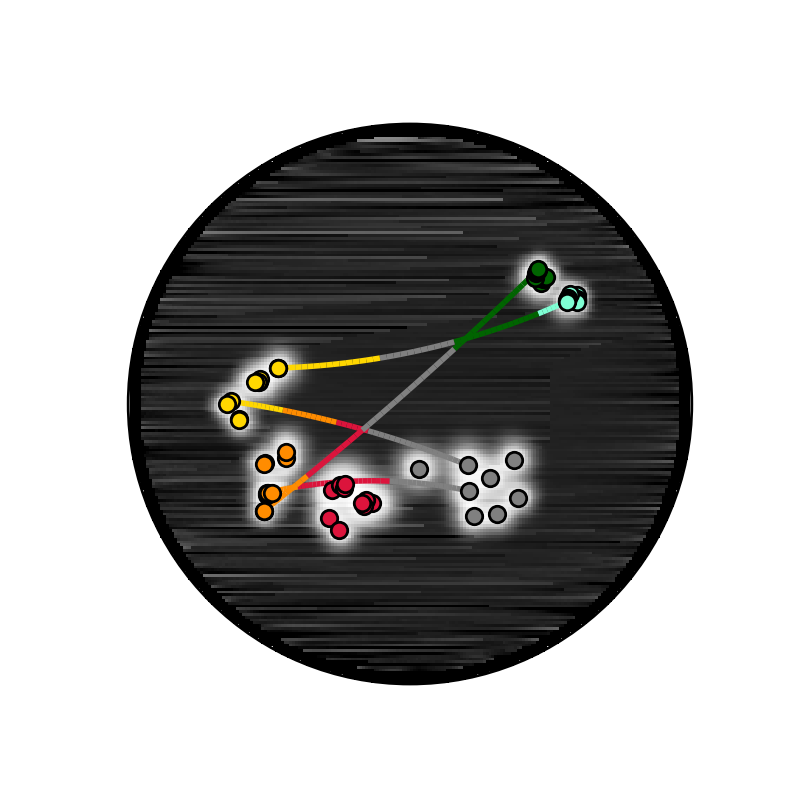}
		\includegraphics[trim={2.0cm 2.0cm 0.5cm 2.0cm},clip,width=0.9\textwidth]{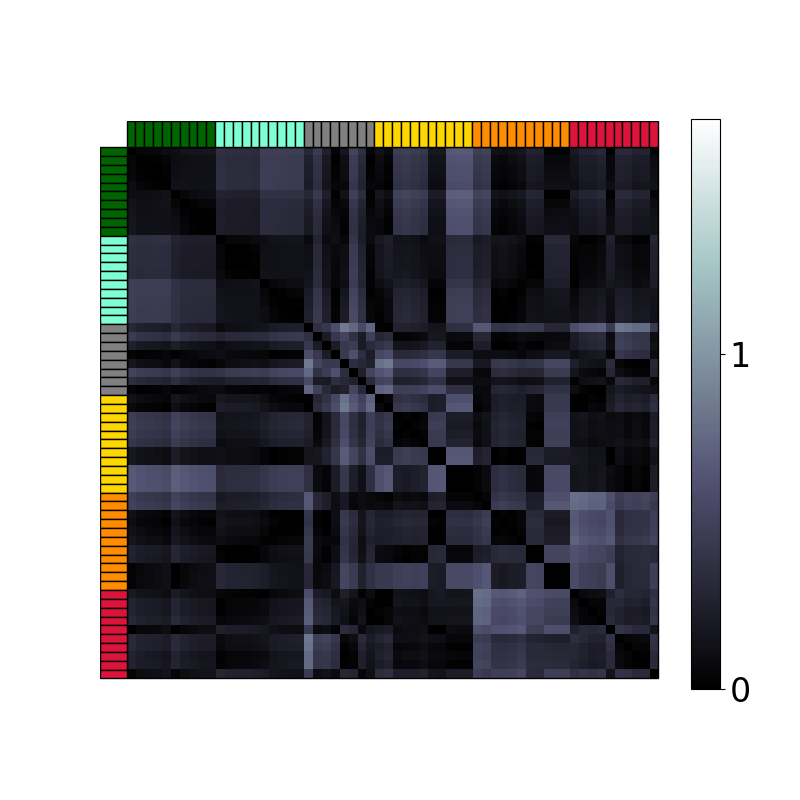}
		\includegraphics[trim={2.0cm 2.0cm 2.0cm 2.0cm},clip,width=0.9\textwidth]{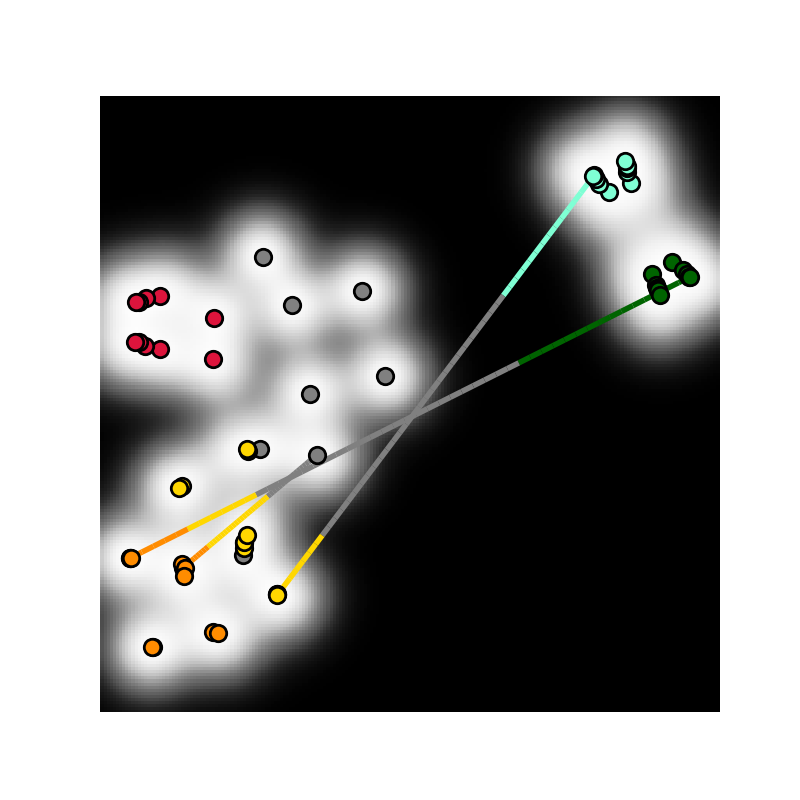}
		\includegraphics[trim={2.0cm 2.0cm 0.5cm 2.0cm},clip,width=0.9\textwidth]{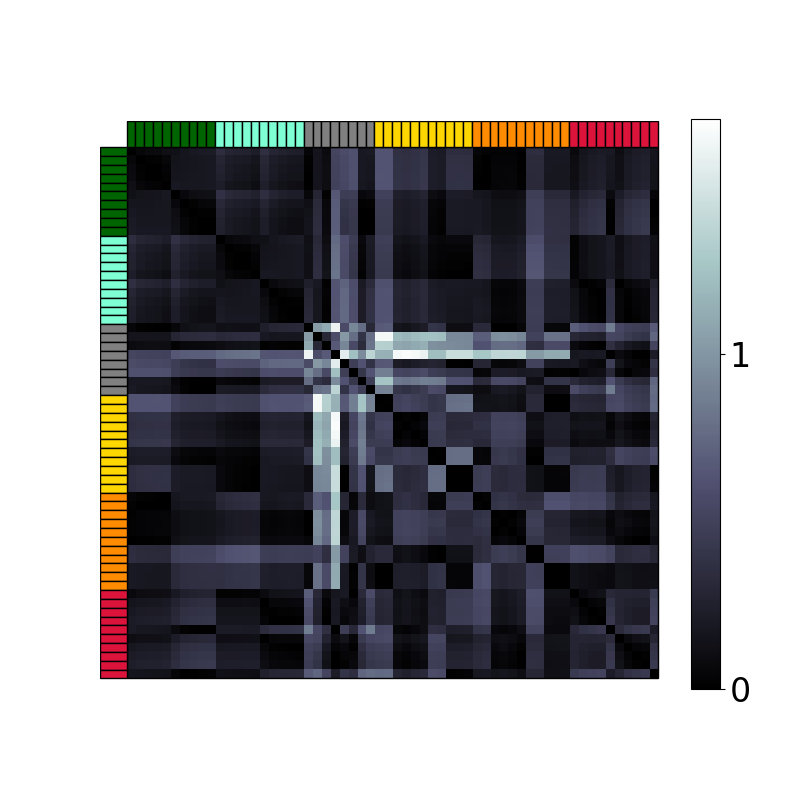}
		\caption{BC + stress prior}
		\label{fig:GPHLVM-bimanual:backconstrained_and_stress}
	\end{subfigure}%
	\begin{subfigure}[b]{0.15\textwidth}
		\centering
		\includegraphics[trim={2.5cm 2.5cm 2.5cm 2.5cm},clip,width=\textwidth]{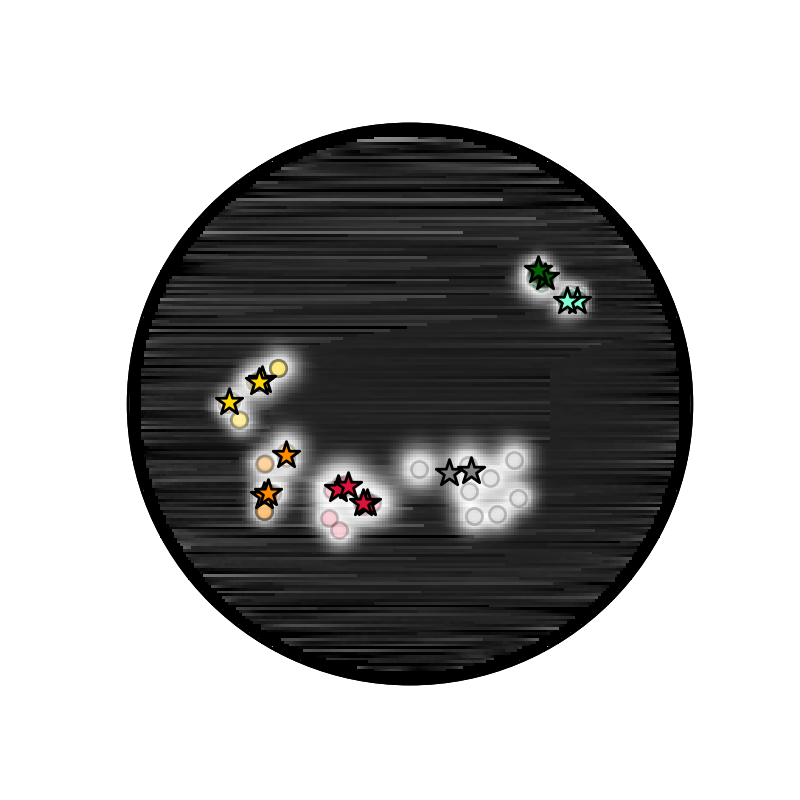}
		\includegraphics[trim={2.0cm 2.0cm 0.5cm 2.0cm},clip,width=0.9\textwidth]{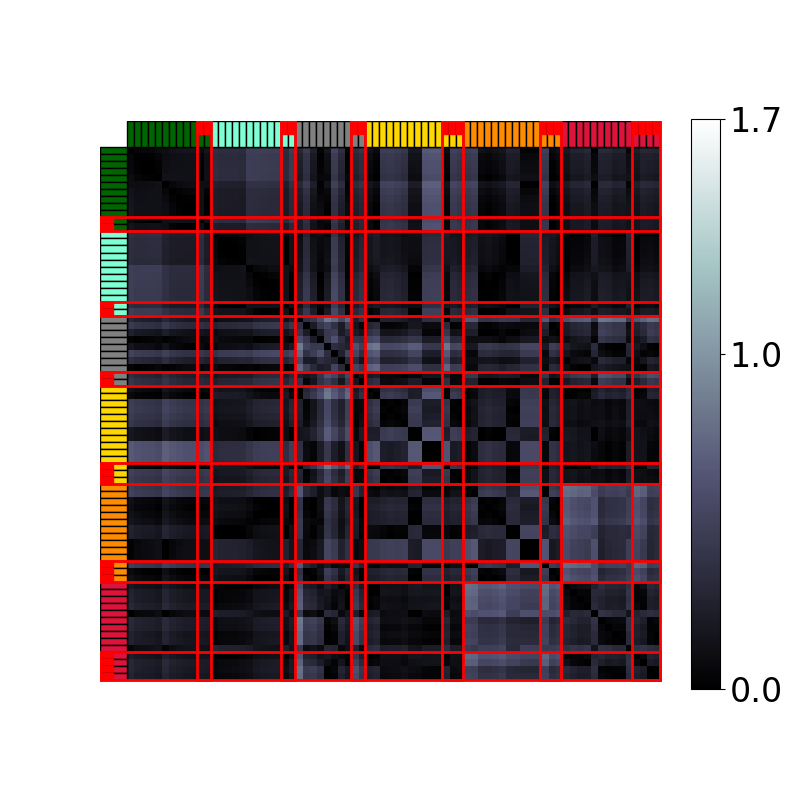}
		\includegraphics[trim={2.0cm 2.0cm 2.0cm 2.0cm},clip,width=0.9\textwidth]{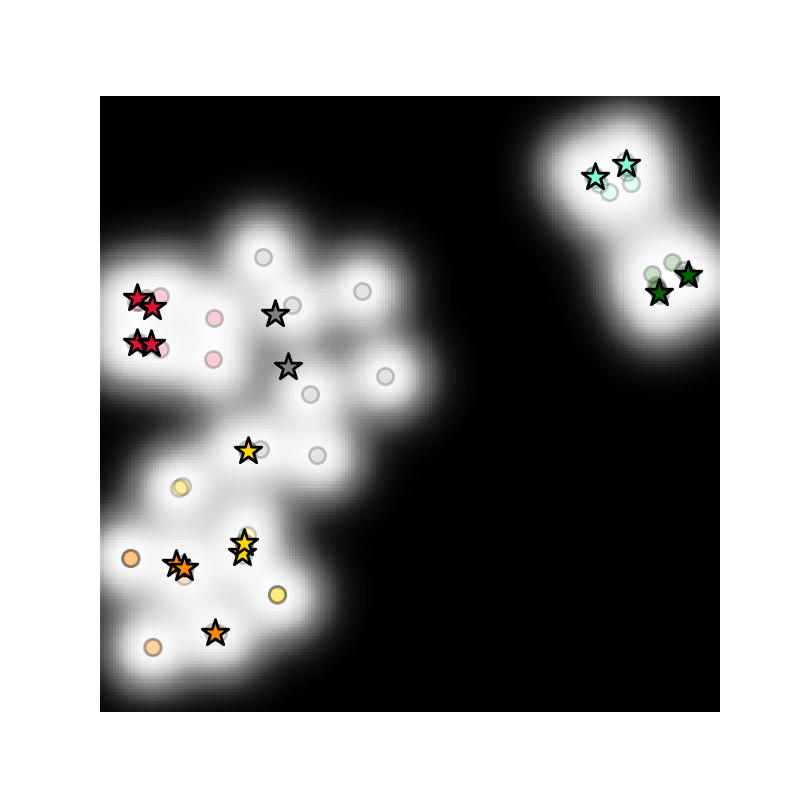}
		\includegraphics[trim={2.0cm 2.0cm 0.5cm 2.0cm},clip,width=0.9\textwidth]{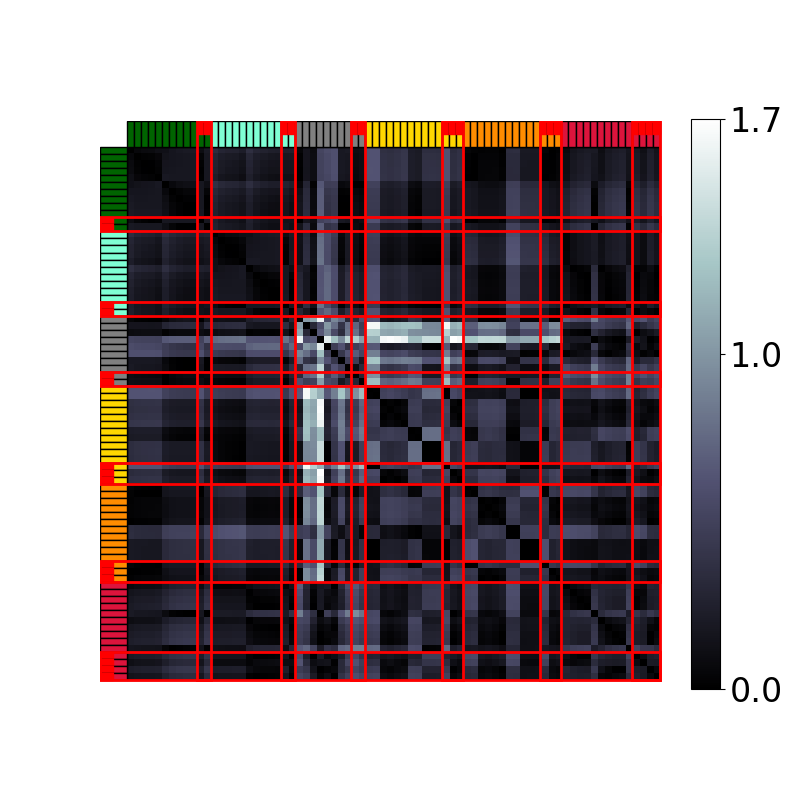}
		\caption{Adding poses}
		\label{fig:GPHLVM-bimanual:added_poses}
	\end{subfigure}%
	\begin{subfigure}[b]{0.15\textwidth}
		\centering
		\includegraphics[trim={2.5cm 2.5cm 2.5cm 2.5cm},clip,width=\textwidth]{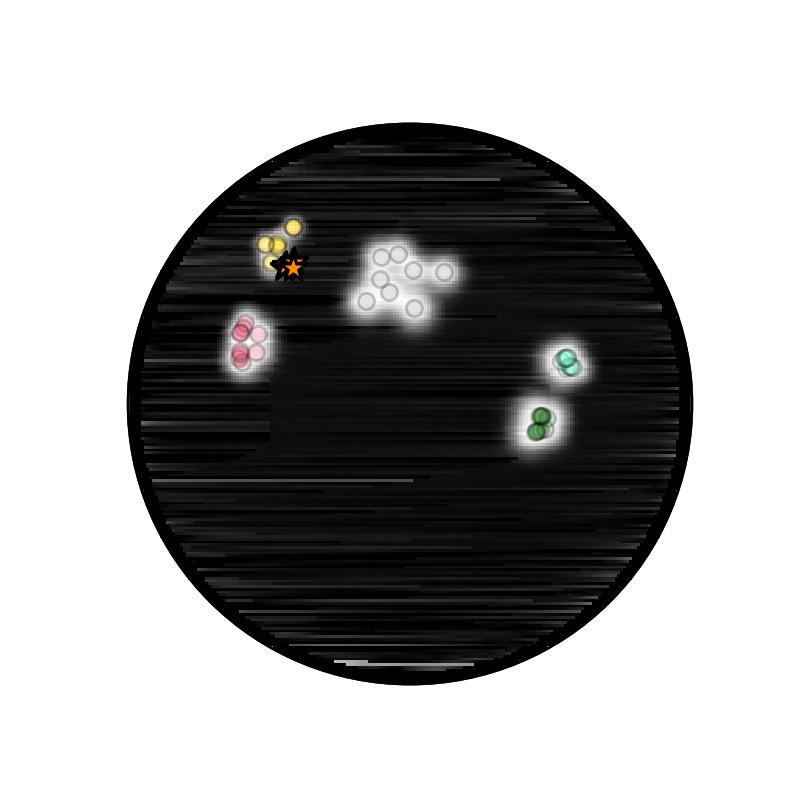}
		\includegraphics[trim={2.0cm 2.0cm 0.5cm 2.0cm},clip,width=0.9\textwidth]{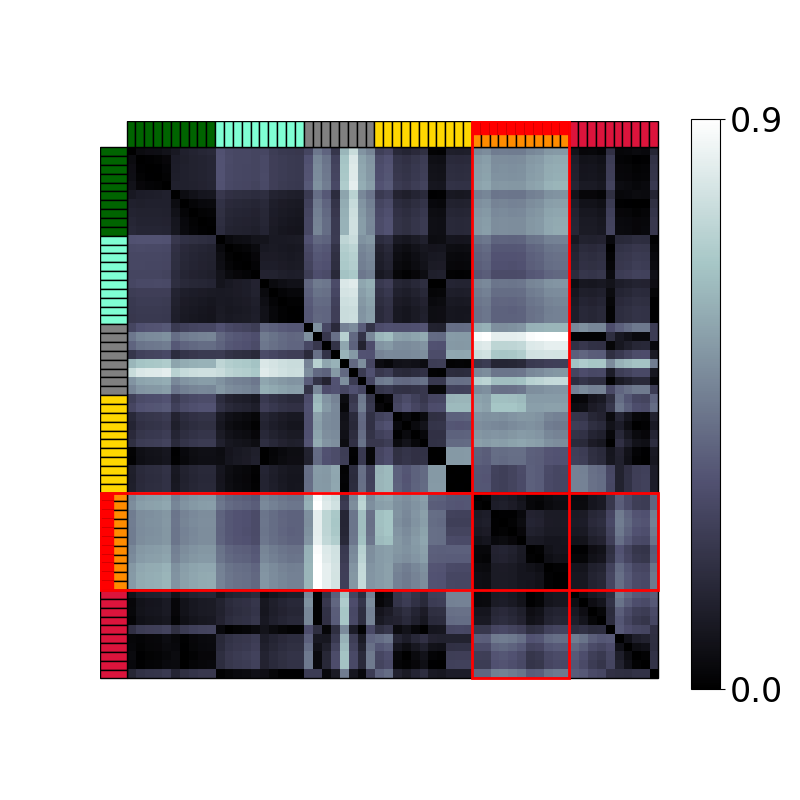}
		\includegraphics[trim={2.0cm 2.0cm 2.0cm 2.0cm},clip,width=0.9\textwidth]{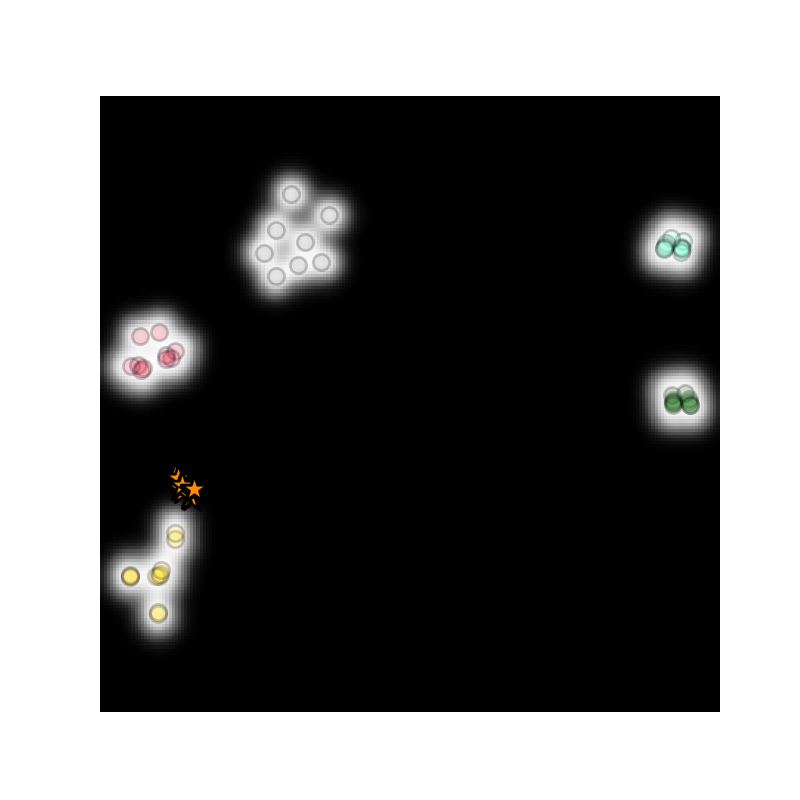}
		\includegraphics[trim={2.0cm 2.0cm 0.5cm 2.0cm},clip,width=0.9\textwidth]{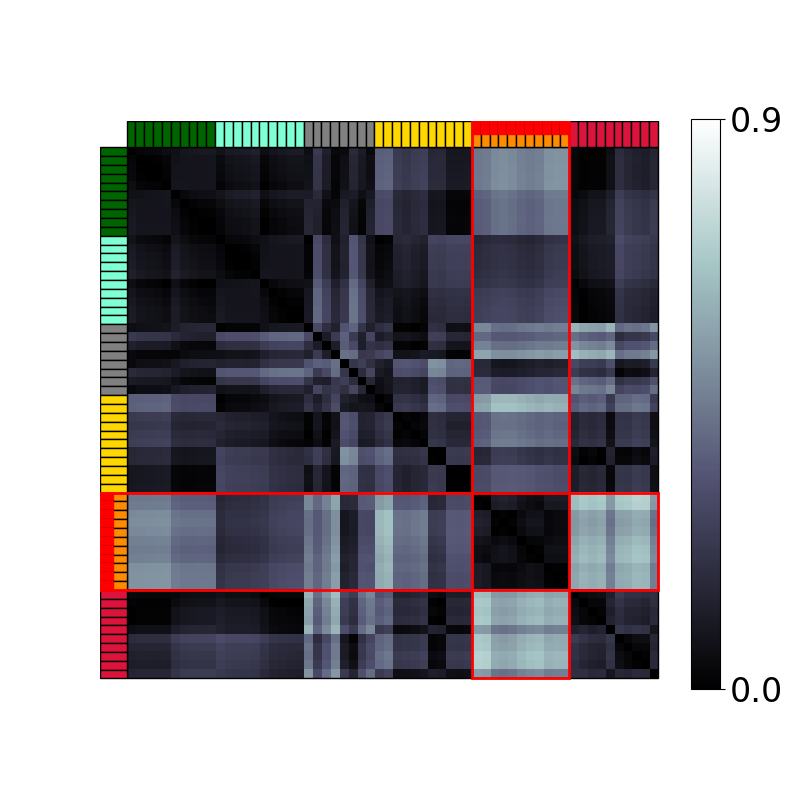}
		\caption{Adding a class}
		\label{fig:GPHLVM-bimanual:added_class}
	\end{subfigure}
    \vspace{-0.2cm}
	\caption{Bimanual manipulation categories: The first and last two rows show the latent embeddings and examples of interpolating geodesics in $\mathcal{P}^2$ and $\mathbb{R}^2$, followed by pairwise error matrices between geodesic and taxonomy graph distances. Background colors indicate the GPLVM uncertainty. Added poses \emph{(d)} and classes $\mathsf{TCA}_{\text{right}}$ \emph{(e)} are marked with stars and highlighted with red in the error matrices.}
	\label{fig:GPHLVM_trained_models-bimanual}
\end{figure*}
We test the proposed GPHLVM on three distinct robotics taxonomies. 
First, we model the data from the bimanual manipulation taxonomy~\citep{Krebs22:BimanualTaxonomy}, that is a simple binary tree whose nodes represent coordination patterns of human bimanual manipulation skills. We use a balanced dataset of $60$ whole-body poses extracted from recordings of bimanual household activities, as in~\citep{Krebs22:BimanualTaxonomy}. Each pose is a vector of joint angles $\bm{y}_n \in \mathbb{R}^{86}$.
Second, we consider a hand grasp taxonomy~\cite{Stival19:HumanGraspTaxonomy} that organizes common grasp types into a tree structure based on their muscular and kinematic properties (see Fig.~\ref{fig:HyperbolicAndTaxonomy}-\emph{right}). We use $94$ grasps of $19$ types obtained from recordings of humans grasping different objects. Each grasp is encoded by a vector of wrist and finger joint angles $\bm{y}_n \in \mathbb{R}^{24}$.
Third, we model data from the whole-body support pose taxonomy~\citep{Borras17:WholeBodyTaxonomy}. 
Each node of this taxonomy graph is a support pose defined by its contacts, so that the distance between nodes can be viewed as the number of contact changes required to go from a support pose to another.
We use standing and kneeling poses of the datasets in~\citep{Mandery16:LanguageWholeBody} and~\citep{Langenstein20:ThesisWholeBodyMotion}. 
The former were extracted from recordings of a human walking without hand support, or using supports from a handrail or from a table on one side or on both sides. 
The latter were obtained from a human standing up from a kneeling position.
Each pose is identified with a node of the graph of Fig.~\ref{fig:support-poses-taxonomy}.
We test our approach on an unbalanced dataset of $100$ poses ($72$ standing and $28$ kneeling poses), where each pose is represented by a vector of joint angles $\bm{y}_n \in \mathbb{R}^{44}$. Note that we augment the taxonomy to explicitly distinguish between left and right contacts.
The main results are analyzed in the sequel, while additional experimental details, results, and comparisons are given in App.~\ref{app:experiment_details} and~\ref{app:comparisons}.

\vspace{-0.4cm}
\paragraph{Implementation details:}
App.~\ref{app:experiment_data_details} and App.~\ref{app:experiment_training_details} describe the data and hyperparameters used for all experiments. 
We used the hyperbolic SE kernels of \S~\ref{subsec:HyperbolicKernels} for the GPHLVMs, and the classical SE kernel for the Euclidean models. 
As GPLVMs are generally prone to local optima during training, they benefit from a good initialization~\citep{Bitzer10:KickStartingGPLVM,Ko11:GPBayesFilters}.
Here, we initialize the embeddings of all GPLVMs by minimizing the stress associated with their taxonomy nodes, so that $\bm{X} = \min_{\bm{X}} \ell_{\text{stress}}$ with $\ell_{\text{stress}}$ as in~\eqref{eq:stressLoss}, using the hyperbolic and Euclidean distance for the GPHLVMs and GPLVMs, respectively (see App.~\ref{app:initialization}).
Since our experiments deal with low-data scenarios, all models were trained via MAP estimation by maximizing the loss $\ell = \ell_{\text{MAP}} - \gamma \ell_{\text{stress}}$, where $\gamma$ is a parameter balancing the two losses (see App.~\ref{app:LossScale} for details). 

\begin{table*}[t]
    \caption{Average stress per geometry and regularization. The stress is computed using~\eqref{eq:stressLoss} and averaged over all pairs of training embeddings. For models with unseen poses and classes, the stress is computed over all pairs of training and additional embeddings. Lower stress values indicate better compliance with the taxonomy structure.}
    \vspace{-0.4cm}
    \label{table:mean_stress_of_models}
    %\centering
    \begin{center}
    \begin{small}
    \begin{sc}
    \resizebox{\textwidth}{!}{
		\begin{tabular}{lllllll}
			\toprule
			\textbf{Taxonomy} & \textbf{Model} & \textbf{No regularizer} & \textbf{Stress} & \textbf{BC + Stress} & \textbf{Unseen poses} & \textbf{Unseen class} \\
            \toprule
			\multirow{4}{*}{\parbox{2.1cm}{Bimanual manipulation categories}} & GPLVM, $\mathbb{R}^2$ &  $2.03\pm2.15$ &  $0.13\pm0.33$ &  $0.15\pm0.31$&  $0.15\pm 0.29$ &  $\bm{0.08\pm0.11}$\\
			& GPHLVM, $\lorentz{2}$ &  $\bm{0.98\pm1.26}$ &  $\bm{0.11\pm0.33}$ &  $\bm{0.09\pm0.12}$&  $\bm{0.09\pm 0.11}$ &  $0.13\pm2.15$ \\
            \cmidrule(lr){2-7}
			& GPLVM, $ \mathbb{R}^3$ &  $2.39\pm2.36$ &  $\bm{0.01\pm0.01}$ &  $0.20\pm0.38$ &  $0.20\pm 0.38$ &  $\bm{0.05\pm0.07}$ \\
            & GPHLVM, $ \lorentz{3}$ &  $\bm{1.18\pm1.35}$ &  $\bm{0.01\pm0.03}$ &  $\bm{0.04\pm0.08}$&  $\bm{0.03\pm 0.07}$ &  $\bm{0.05\pm0.07}$ \\
            \toprule
			\multirow{4}{*}{Grasps} & GPLVM, $\mathbb{R}^2$ &  $7.25\pm 5.40$ &  $0.39\pm 0.41$ &  $0.40\pm 0.44$&  $0.53\pm 0.77$ &  $0.60\pm 0.73$ \\
			& GPHLVM, $ \lorentz{2}$ &  $\bm{5.47\pm 4.07}$ &  $\bm{0.14\pm 0.16}$ &  $\bm{0.18\pm 0.29}$&  $\bm{0.35\pm 0.78}$ &  $\bm{0.48\pm 0.76}$ \\
            \cmidrule(lr){2-7}
			& GPLVM, $ \mathbb{R}^3$ &  $\bm{8.15\pm 5.85}$ &  $0.14\pm 0.18$ &  $0.15\pm 0.19$&  $0.29\pm 0.64$ &  $0.38\pm 0.66$ \\
            & GPHLVM, $ \lorentz{3}$ &  $8.37\pm 5.71$ &  $\bm{0.04\pm 0.08}$ &  $\bm{0.07\pm 0.18}$&  $\bm{0.23\pm 0.68}$ &  $\bm{0.37\pm0.72}$ \\
			\toprule
			\multirow{4}{*}{\parbox{1.5cm}{Support poses}} & GPLVM, $\mathbb{R}^2$ &  $3.93\pm 3.97$ &  $0.58\pm 0.94$ &  $0.63\pm 0.94$&  $0.66\pm 0.99$ &  $\bm{0.85\pm 1.73}$ \\
			& GPHLVM, $ \lorentz{2}$ &  $\bm{2.05\pm 2.50}$ &  $\bm{0.51\pm 0.82}$ &  $\bm{0.53\pm 0.83}$&  $\bm{0.56\pm 0.86}$ & $0.86\pm 1.70$ \\
            \cmidrule(lr){2-7}
			& GPLVM, $ \mathbb{R}^3$ &  $\bm{3.76\pm 3.74}$ &  $\bm{0.24\pm 0.40}$ &  $\bm{0.29\pm 0.39}$&  $\bm{0.30\pm 0.43}$ &  $\bm{0.55\pm 1.26}$ \\
            & GPHLVM, $ \lorentz{3}$ &  $3.78\pm 3.71$ &  $0.30\pm 0.38$ &  $0.35\pm 0.45$&  $0.37\pm 0.50$ &  $0.69\pm 1.36$ \\
			\bottomrule
	\end{tabular}
 }
 \end{sc}
 \end{small}
 \end{center}
 \vspace{-0.5cm}
\end{table*}
\begin{figure*}
	\centering
	\includegraphics[trim={5.8cm 2.2cm 4.3cm 2.2cm},clip,width=.8\textwidth]{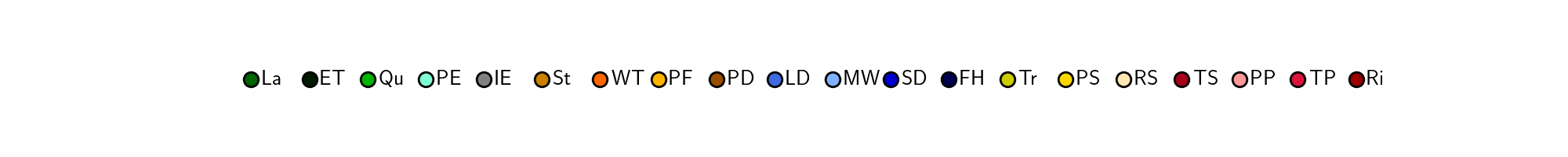}
	\begin{subfigure}[b]{0.15\textwidth}
		\centering
		\includegraphics[trim={2.5cm 2.5cm 2.5cm 2.5cm},clip,width=\textwidth]{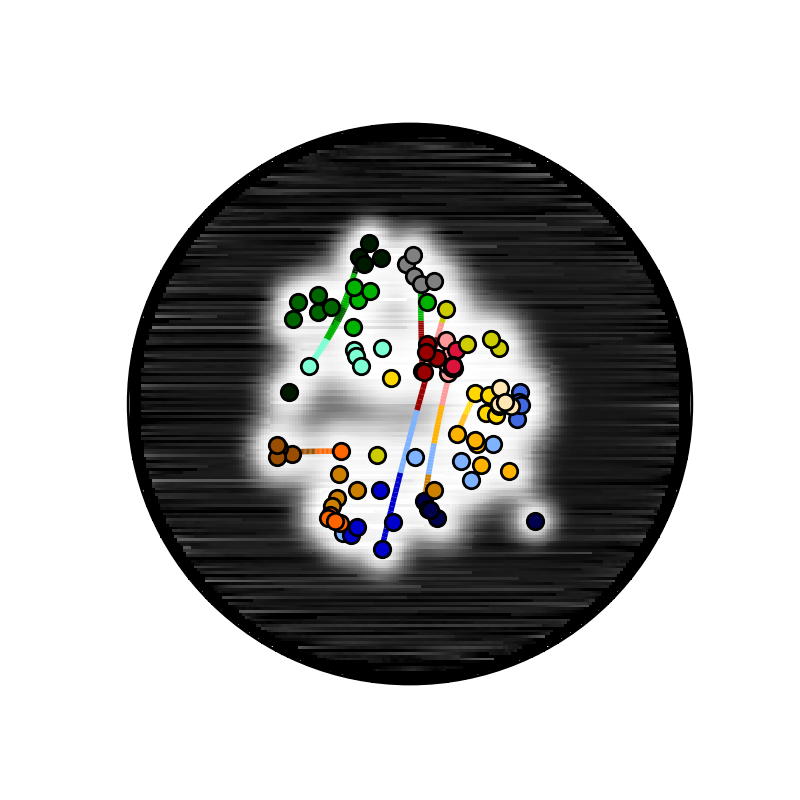}
		\includegraphics[trim={2.0cm 2.0cm 0.5cm 2.0cm},clip,width=0.9\textwidth]{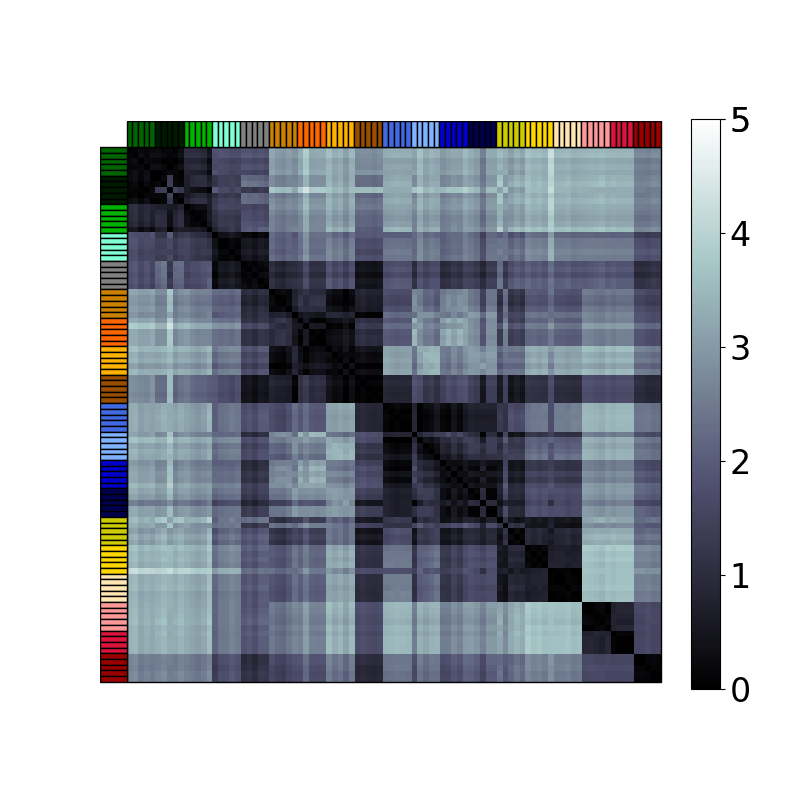}
		\includegraphics[trim={2.0cm 2.0cm 2.0cm 2.0cm},clip,width=0.9\textwidth]{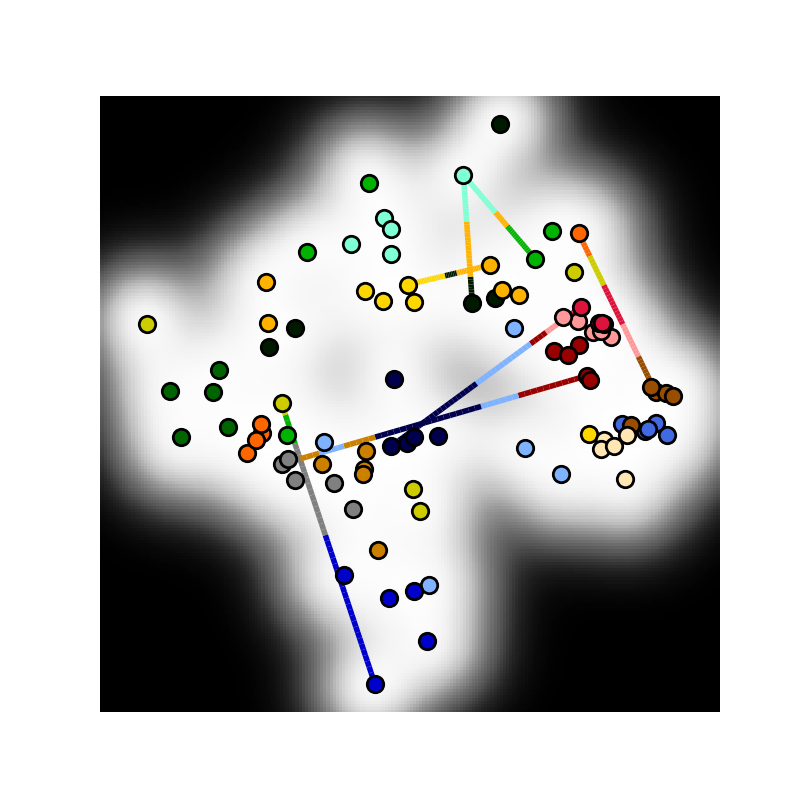}
		\includegraphics[trim={2.0cm 2.0cm 0.5cm 2.0cm},clip,width=0.9\textwidth]{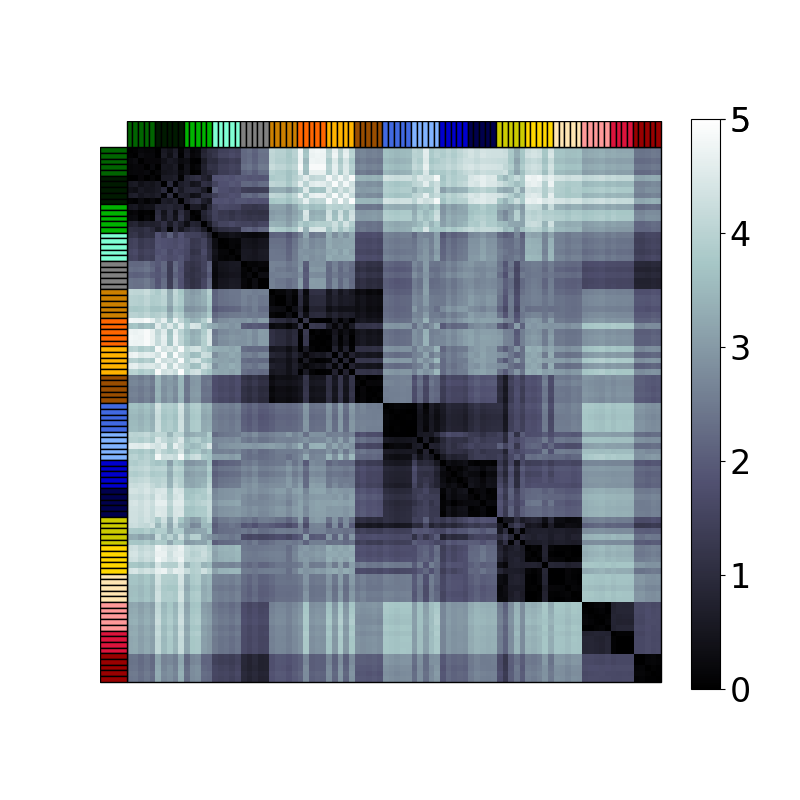}
		\caption{Vanilla}
		\label{fig:GPHLVM-grasps:vanilla}
	\end{subfigure}%
	\begin{subfigure}[b]{0.15\textwidth}
		\centering
		\includegraphics[trim={2.5cm 2.5cm 2.5cm 2.5cm},clip,width=\textwidth]{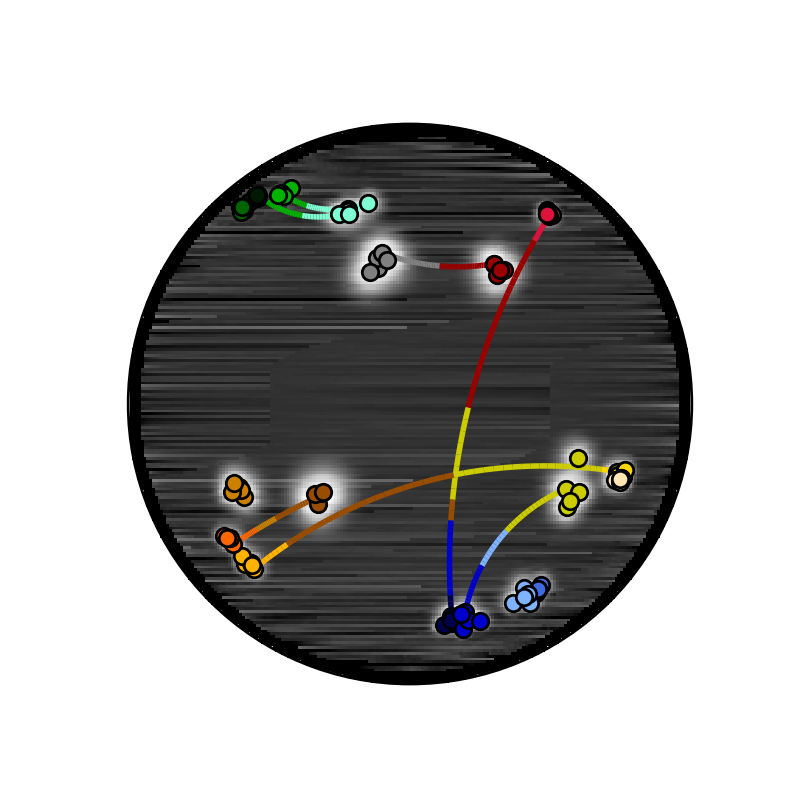}
		\includegraphics[trim={2.0cm 2.0cm 0.5cm 2.0cm},clip,width=0.9\textwidth]{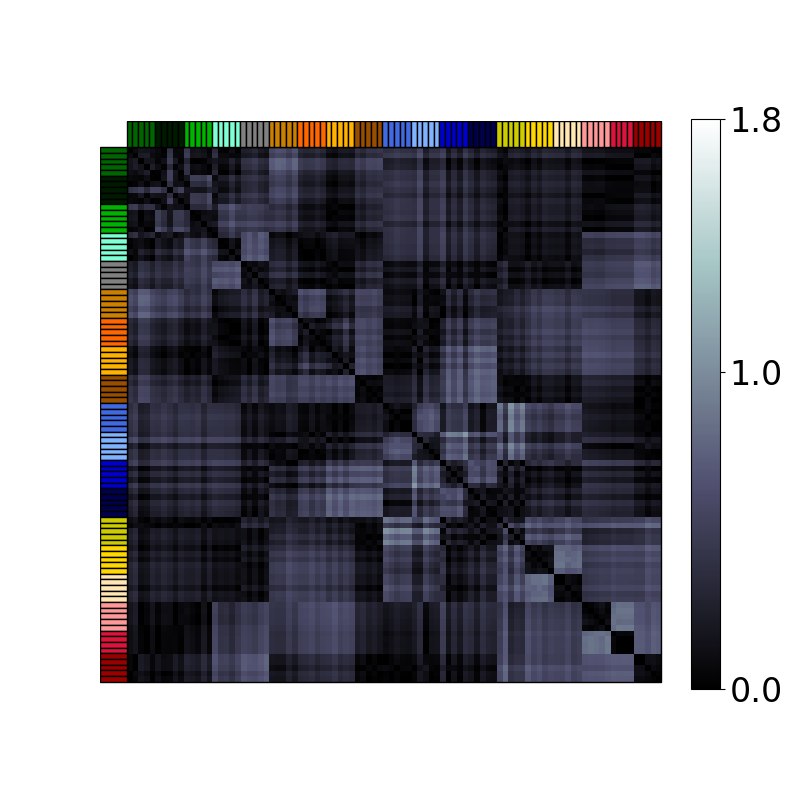}
		\includegraphics[trim={2.0cm 2.0cm 2.0cm 2.0cm},clip,width=0.9\textwidth]{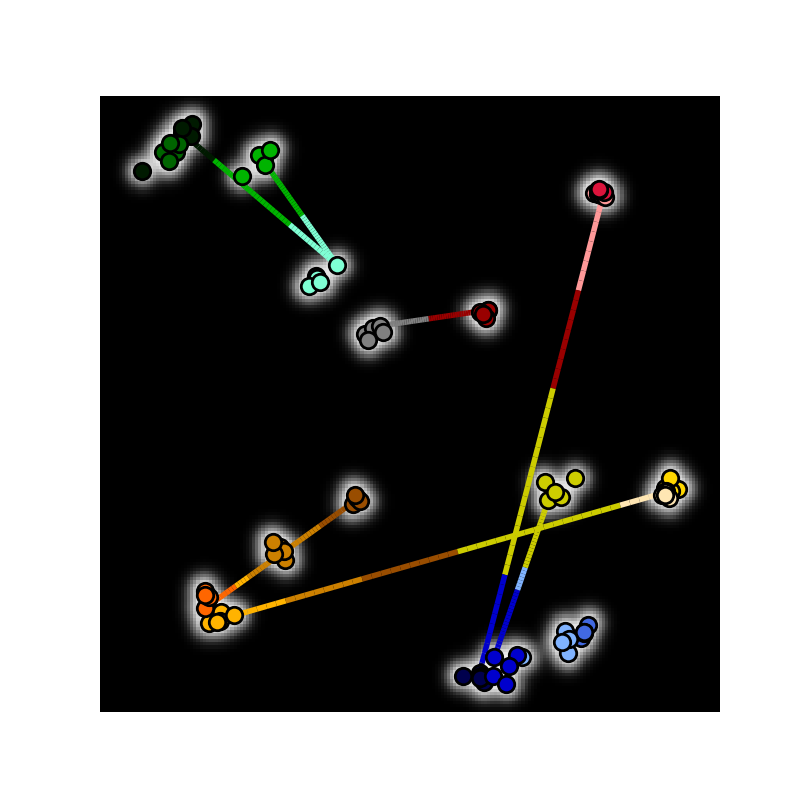}
		\includegraphics[trim={2.0cm 2.0cm 0.5cm 2.0cm},clip,width=0.9\textwidth]{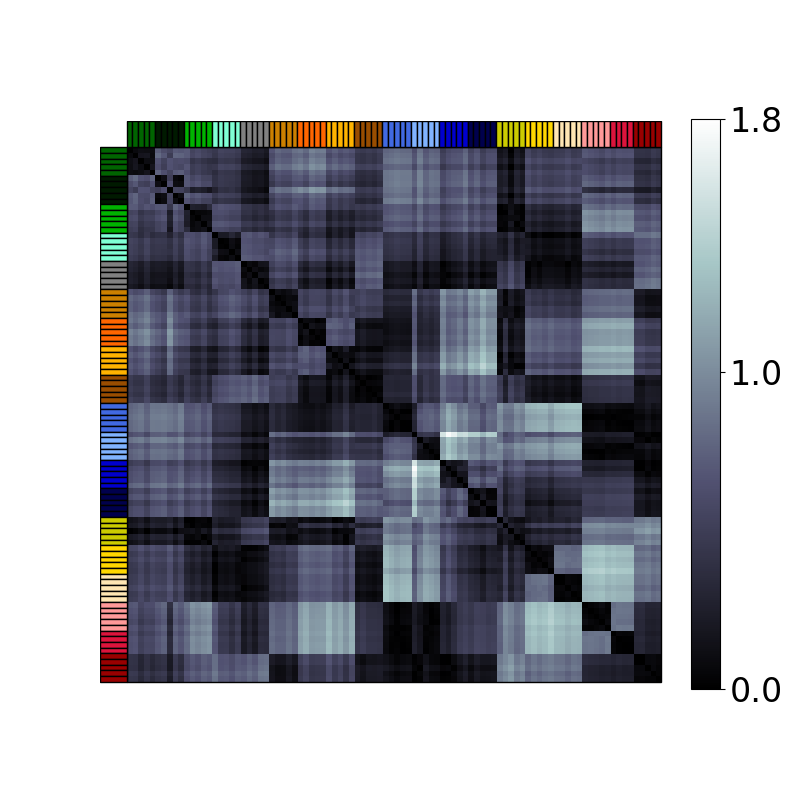}
		\caption{Stress prior}
		\label{fig:GPHLVM-grasps:stress_prior}
	\end{subfigure}%
	\begin{subfigure}[b]{0.15\textwidth}
		\centering
		\includegraphics[trim={2.5cm 2.5cm 2.5cm 2.5cm},clip,width=\textwidth]{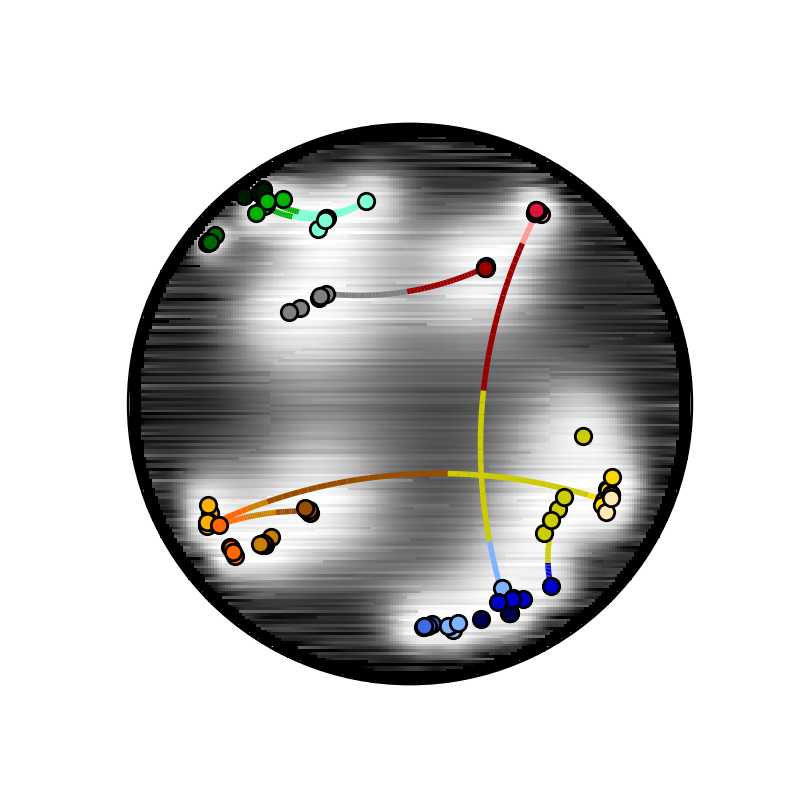}
		\includegraphics[trim={2.0cm 2.0cm 0.5cm 2.0cm},clip,width=0.9\textwidth]{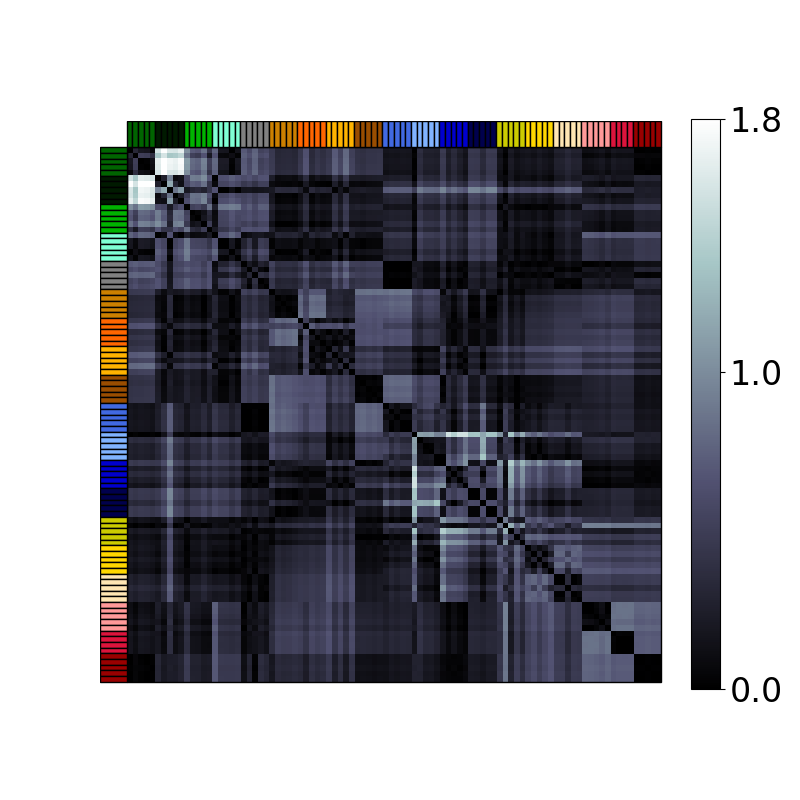}
		\includegraphics[trim={2.0cm 2.0cm 2.0cm 2.0cm},clip,width=0.9\textwidth]{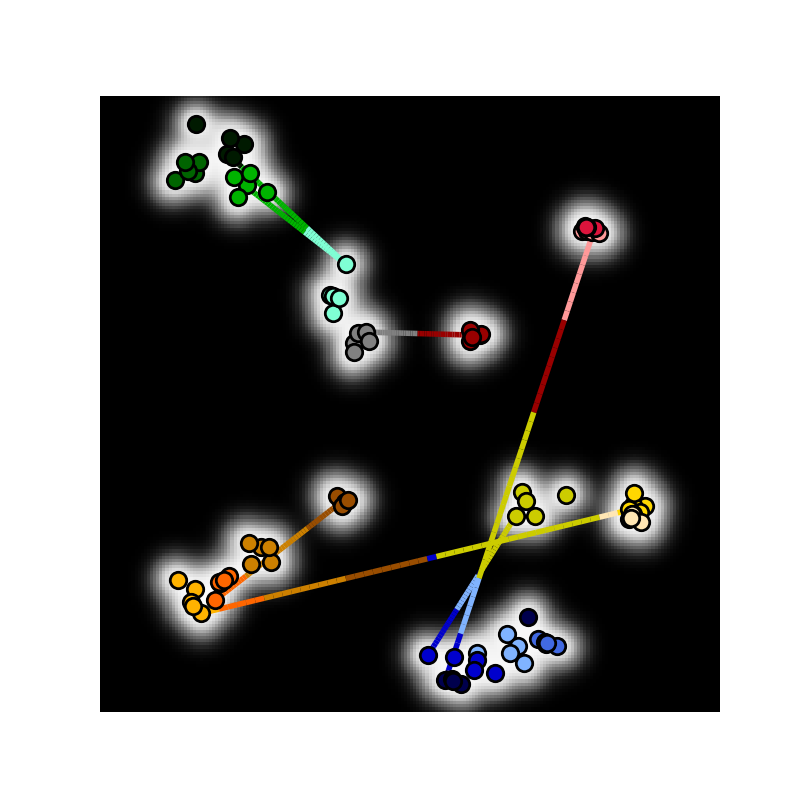}
		\includegraphics[trim={2.0cm 2.0cm 0.5cm 2.0cm},clip,width=0.9\textwidth]{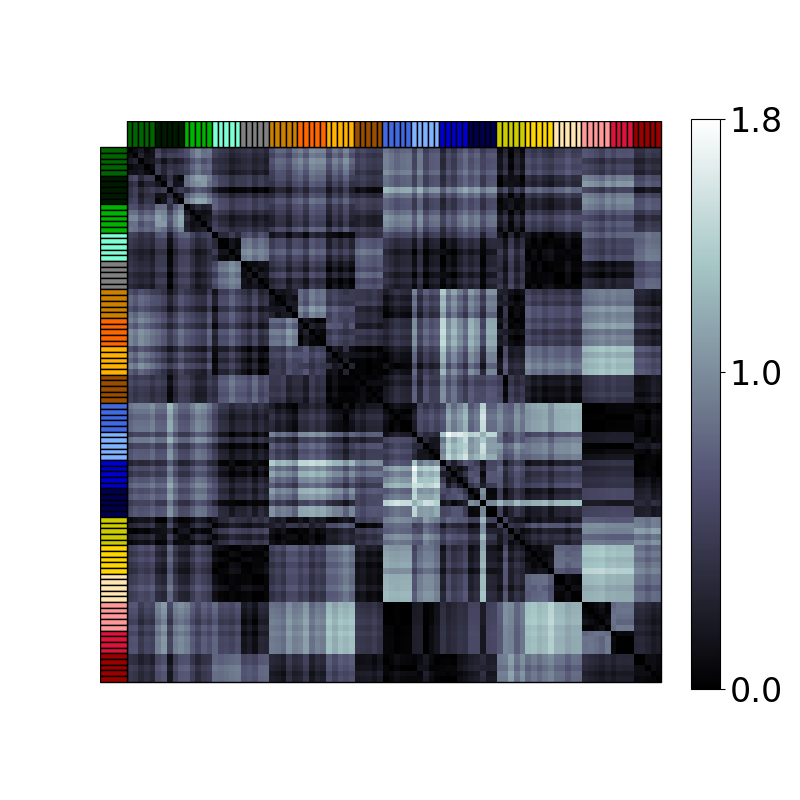}
		\caption{BC + stress prior}
		\label{fig:GPHLVM-grasps:backconstrained_and_stress}
	\end{subfigure}%
	\begin{subfigure}[b]{0.15\textwidth}
		\centering
		\includegraphics[trim={2.5cm 2.5cm 2.5cm 2.5cm},clip,width=\textwidth]{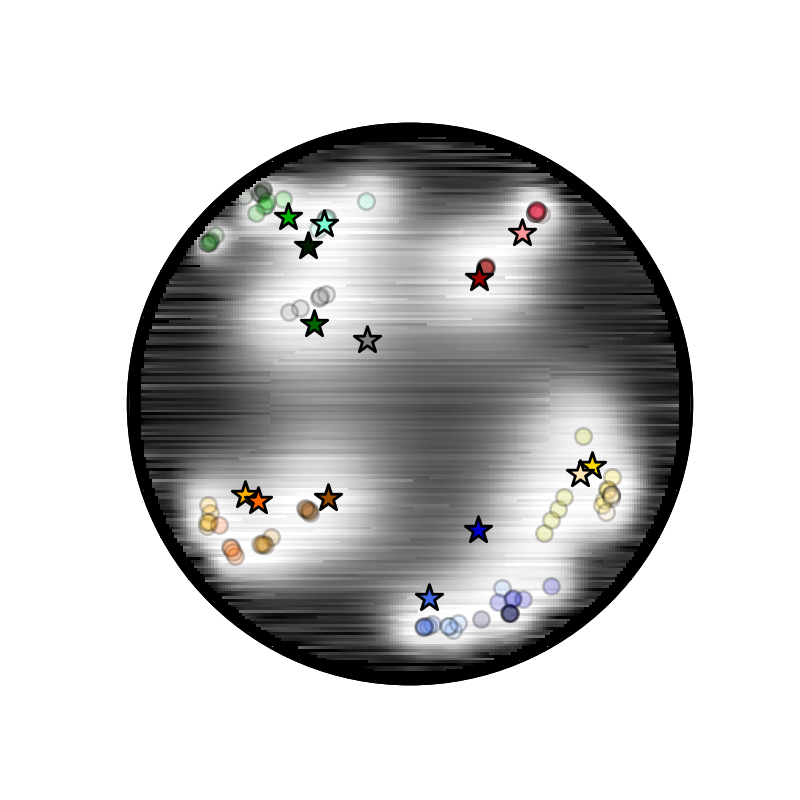}
		\includegraphics[trim={2.0cm 2.0cm 0.5cm 2.0cm},clip,width=0.9\textwidth]{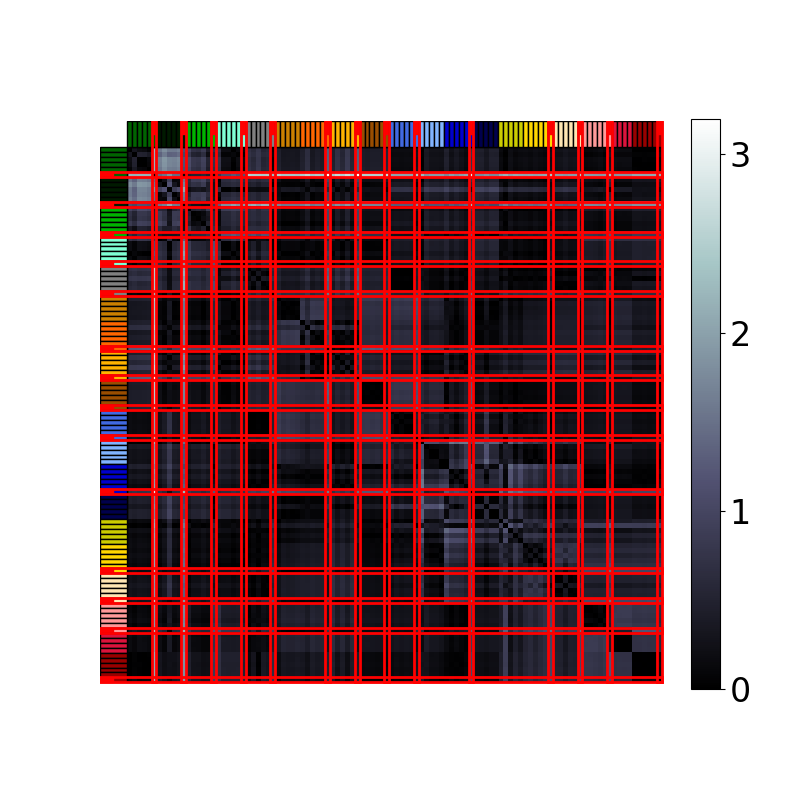}
		\includegraphics[trim={2.0cm 2.0cm 2.0cm 2.0cm},clip,width=0.9\textwidth]{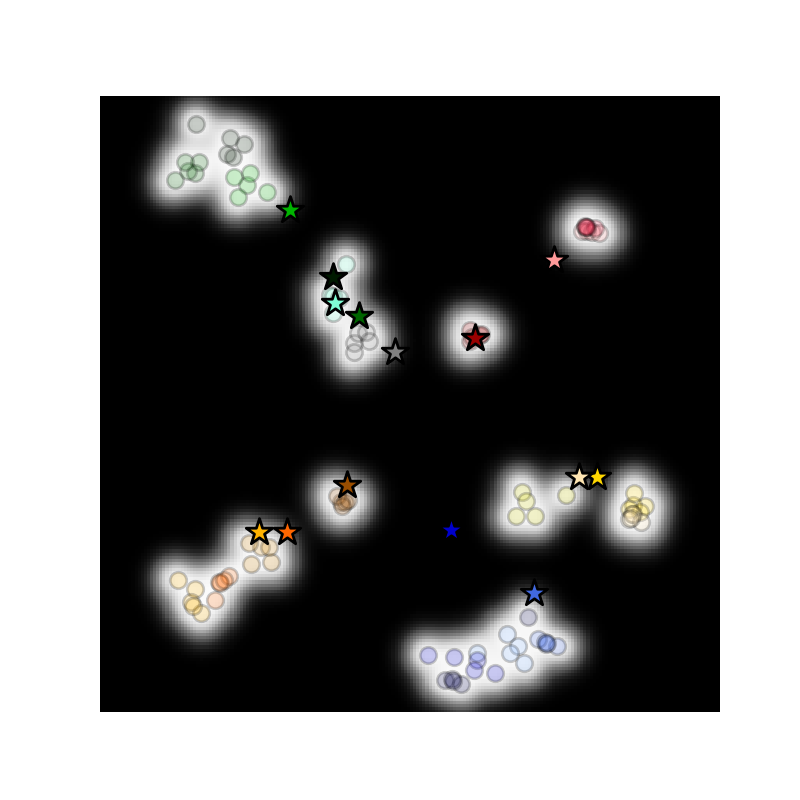}
		\includegraphics[trim={2.0cm 2.0cm 0.5cm 2.0cm},clip,width=0.9\textwidth]{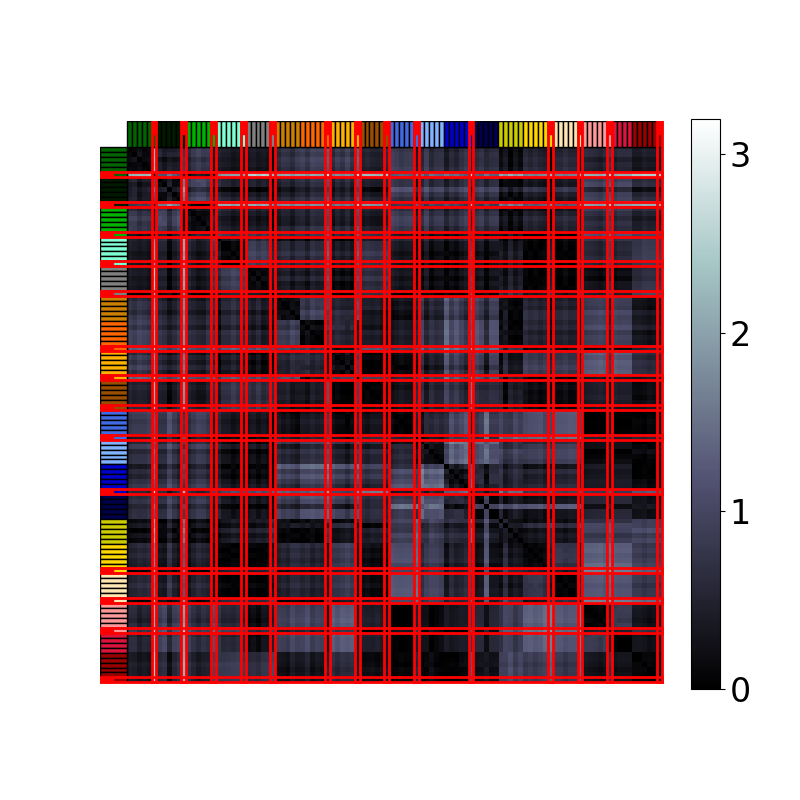}
		\caption{Adding poses}
		\label{fig:GPHLVM-grasps:added_poses}
	\end{subfigure}%
	\begin{subfigure}[b]{0.15\textwidth}
		\centering
		\includegraphics[trim={2.5cm 2.5cm 2.5cm 2.5cm},clip,width=\textwidth]{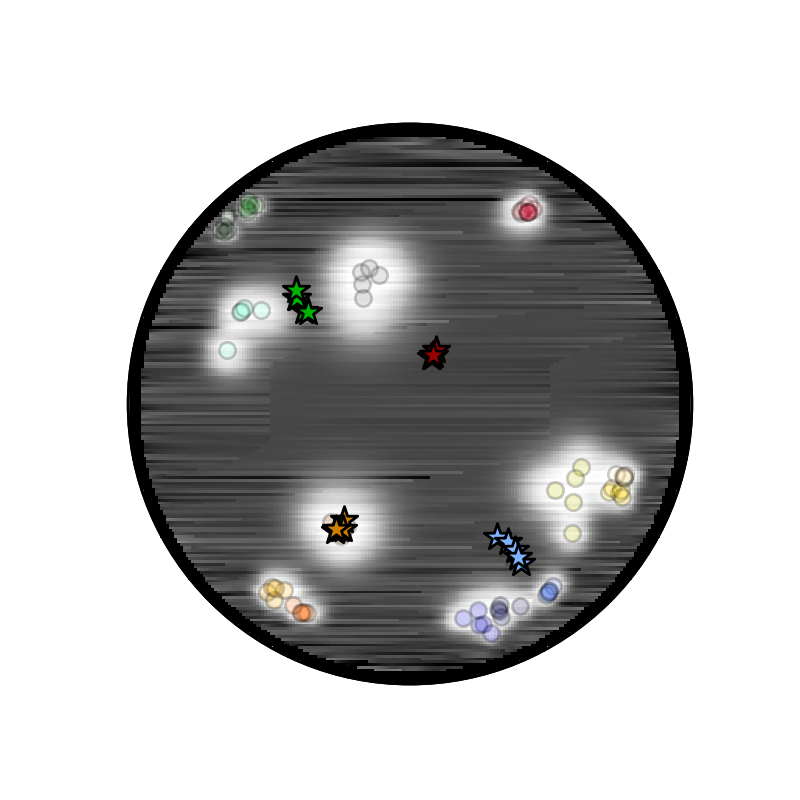}
		\includegraphics[trim={2.0cm 2.0cm 0.5cm 2.0cm},clip,width=0.9\textwidth]{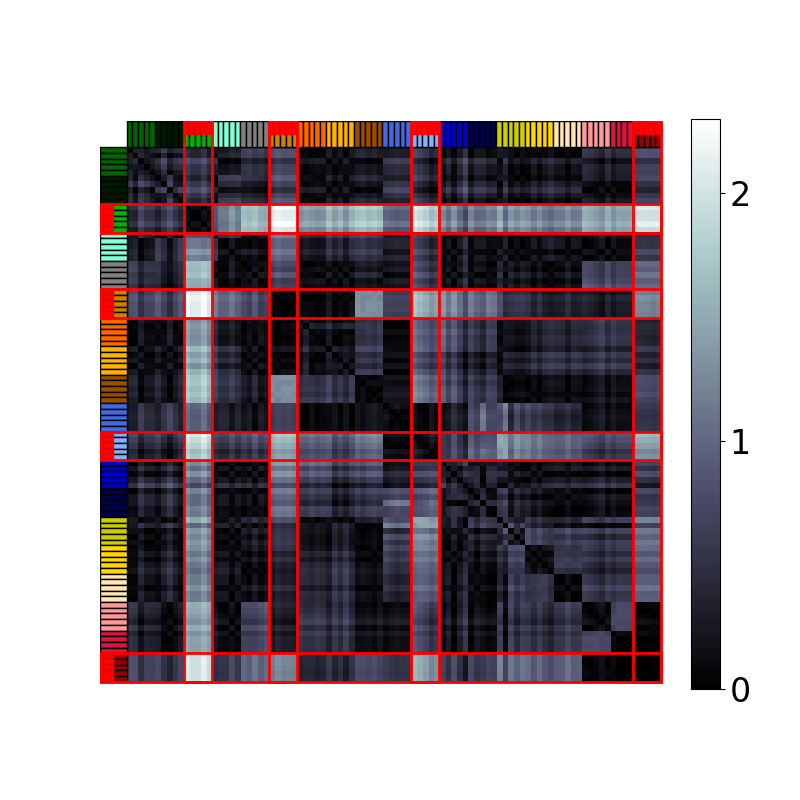}
		\includegraphics[trim={2.0cm 2.0cm 2.0cm 2.0cm},clip,width=0.9\textwidth]{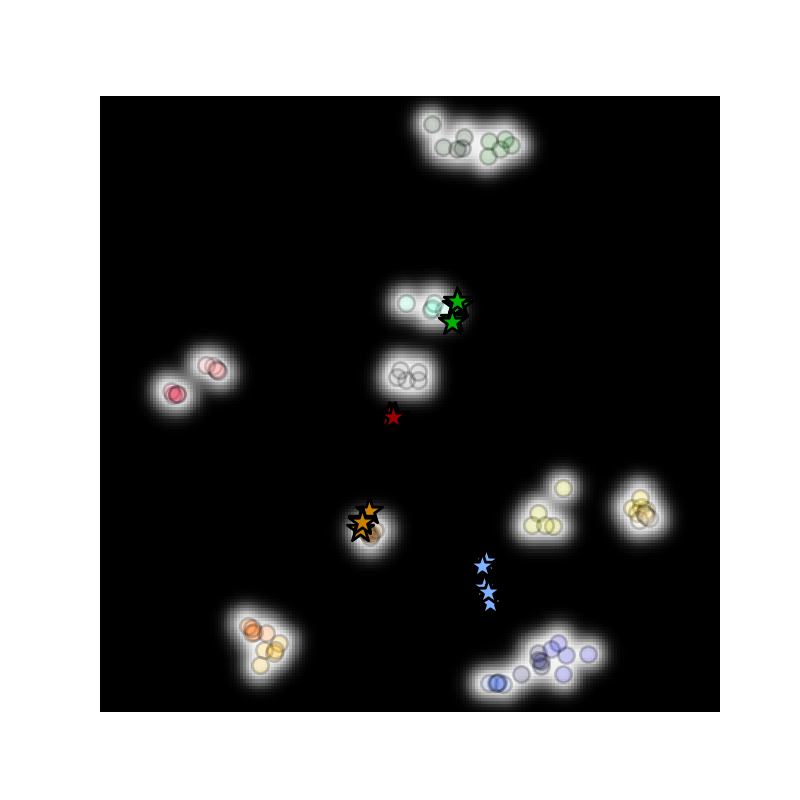}
		\includegraphics[trim={2.0cm 2.0cm 0.5cm 2.0cm},clip,width=0.9\textwidth]{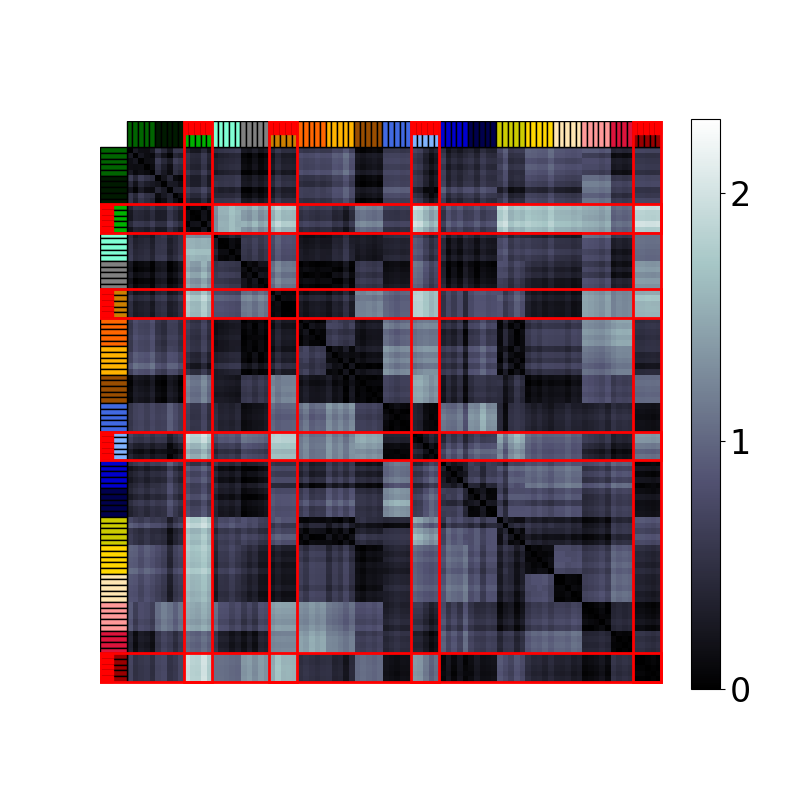}
		\caption{Adding a class}
		\label{fig:GPHLVM-grasps:added_class}
	\end{subfigure}
    \vspace{-0.25cm}
	\caption{Grasps: The first and last two rows show the latent embeddings and examples of interpolating geodesics in $\mathcal{P}^2$ and $\mathbb{R}^2$, followed by pairwise error matrices between geodesic and graph distances. Embeddings colors match those of Fig.~\ref{fig:HyperbolicAndTaxonomy}-\emph{right}, and background colors indicate the GPLVM uncertainty. Added poses \emph{(d)} and classes $\mathsf{Qu}, \mathsf{St}, \mathsf{MW},$ and $\mathsf{Ri}$ \emph{(e)} are marked with stars and highlighted with red in the error matrices.}
	\label{fig:GPHLVM_trained_models-grasps}
	\vspace{-0.55cm}
\end{figure*}

\vspace{-0.4cm}
\paragraph{Hyperbolic embeddings of taxonomy data:}
We embed the taxonomy data of the aforementioned taxonomies into $2$-dimensional hyperbolic and Euclidean spaces using GPHLVM and GPLVM.
For each, we test the model without regularization, with stress prior, and with back-constraints coupled with stress prior.
Figures~\ref{fig:GPHLVM-bimanual:vanilla}-\ref{fig:GPHLVM-bimanual:backconstrained_and_stress},~\ref{fig:GPHLVM-grasps:vanilla}-\ref{fig:GPHLVM-grasps:backconstrained_and_stress}, and~\ref{fig:GPHLVM:vanilla}-\ref{fig:GPHLVM:backconstrained_and_stress} show the learned embeddings alongside error matrices depicting the difference between geodesic and taxonomy graph distances for the bimanual manipulation, hand grasps, and support pose taxonomies, respectively. 
As shown in Figs.~\ref{fig:GPHLVM-bimanual:vanilla},~\ref{fig:GPHLVM-grasps:vanilla},~\ref{fig:GPHLVM:vanilla}, the models without regularization do not encode any meaningful distance structure in latent space. 
In contrast, the models with stress prior result in embeddings that comply with the taxonomy graph structure: The embeddings are grouped and organized according to the taxonomy nodes, the geodesic distances match the graph ones, and arguably more so in the hyperbolic case (see error matrices in Figs.~\ref{fig:GPHLVM-bimanual:stress_prior}-\ref{fig:GPHLVM-bimanual:backconstrained_and_stress},~\ref{fig:GPHLVM-grasps:stress_prior}-\ref{fig:GPHLVM-grasps:backconstrained_and_stress},~\ref{fig:GPHLVM:stress_prior}-\ref{fig:GPHLVM:backconstrained_and_stress}). 
Moreover, the GPHLVM with back constraints further organizes the embeddings inside a class according to the similarity between their observations (see Figs.~\ref{fig:GPHLVM-bimanual:backconstrained_and_stress},~\ref{fig:GPHLVM-grasps:backconstrained_and_stress},~\ref{fig:GPHLVM:backconstrained_and_stress}).
Note that augmenting the support pose taxonomy leads to several groups of the same support pose in Figs.~\ref{fig:GPHLVM:stress_prior}-\ref{fig:GPHLVM:backconstrained_and_stress}, e.g., $\mathsf{F}$ splits into $\mathsf{LF}$ and $\mathsf{RF}$. % for left-foot and right-foot contacts, respectively.

A quantitative comparison of the stress values of the latent embeddings with respect to the graph distances confirms that a hyperbolic geometry captures better the data structure (see Table~\ref{table:mean_stress_of_models}). 
All regularized GPHLVMs with $2$-dimensional latent spaces outperform their Euclidean counterparts.
In general, we observe a prominent stress reduction for the Euclidean and hyperbolic $3$-dimensional latent spaces compared to the $2$-dimensional ones. 
This is due to the increase of volume available to match the graph structure in $3$-dimensional spaces relative to $2$-dimensional ones. 
Interestingly, the hyperbolic models of the bimanual manipulation and hand grasps taxonomies also outperform the Euclidean models with $3$-dimensional latent spaces (see models in App.~\ref{appendix:3d-embeddings}).
This is due to the fact that the volume of balls in hyperbolic space increases exponentially with respect to the radius of the ball rather than polynomially as in Euclidean space, which translates to significantly more space for embedding tree-like data with minimal distortion. This inherent property makes hyperbolic spaces a natural fit to embed hierarchical taxonomies with a tree-like structure, as the bimanual manipulation and hand grasps taxonomies.
In the case of the support pose taxonomy, the Euclidean models with $3$-dimensional latent space slightly outperform the $3$-dimensional hyperbolic embeddings. We attribute this to the cyclic graph structure of the taxonomy. Such type of structure has been shown to be better embedded in spherical or Euclidean spaces~\citep{Gu19:MixedCurvatureEmbeddings}. Interestingly, despite the cyclic graph structure of the support pose taxonomy, the Euclidean models are still outperformed by the hyperbolic embeddings in the $2$-dimensional case (see Table~\ref{table:mean_stress_of_models}). This suggests that the increase of volume available to match the graph structure in hyperbolic spaces compared to Euclidean spaces leads to better low-dimensional representations of taxonomy data, including those with cyclic graph structure.  

Importantly, a comparative study reported in App.~\ref{app:comparison_VAEs} shows that the GPHLVM also outperformed vanilla and hyperbolic versions of a VAE to encode meaningful taxonomy information in the latent space.
For all taxonomies, adding the stress regularization~\eqref{eq:stressLoss} to the VAEs helps to preserve the graph distance structure, although the embeddings of different taxonomy nodes are not as clearly separated as in the GPHLVMs. This is illustrated by the higher average stress of the VAEs' latent embeddings and their higher reconstruction error compared to the GPHLVMs’ (see Table~\ref{table:mean_stress_and_reconstruction_of_models}).
Finally, we also tested a GPLVM for learning a Riemannian manifold~\citep{Tosi14:RiemannianGPLVM} of the taxonomy data, reported in App.~\ref{app:comparison_Tosi}, which is unable to capture the local and global data structure as this model was not originally designed for hierarchical discrete data. 

\vspace{-0.35cm}
\paragraph{Runtimes:} 
Table~\ref{tab:ApproachesRuntime} shows the runtime measurements for the training and decoding phases of GPHLVM and GPLVM. 
The main computational burden arises in the GPHLVM with a $2$-dimensional latent space, which is in sharp contrast with the experiments using a $3$-dimensional latent space. 
This increase in computational cost is mainly attributed to the $2$-dimensional hyperbolic kernel (see Table~\ref{tab:app:kernel_opt_times} in App.~\ref{app:runtimes}). This may be alleviated by reducing the number of samples or via more efficient sampling strategies.

\begin{table}[t]
    \caption{Average runtime for training and decoding phases over $10$ experiments of the hand grasps taxonomy. Training time was measured over $500$ iterations for both models. The implementations are fully developed on Python, and the runtime measurements were taken using a standard laptop with $32$ GB RAM, Intel Xeon CPU E3-1505M v6 processor, and Ubuntu 20.04 LTS.}
    \vspace{-0.25cm}
	\label{tab:ApproachesRuntime}
	\begin{center}
    \begin{small}
    \begin{sc}
	\begin{tabular}{lll}
        \toprule
        \textbf{Model} & \textbf{Training} & \textbf{Decoding} \\
        \toprule
        GPLVM, $\mathbb{R}^2$ & $2.978 \si{\second} \pm 0.082$ & $6.256 \si{\milli \second} \pm 0.314$ \\
        GPHLVM, $\lorentz{2}$ & $414.67 \si{\second} \pm 30.87$ & $2.74 \si{\second} \pm 0.487$ \\
        \midrule
        GPLVM, $\mathbb{R}^3$ & $3.148 \si{\second} \pm 0.171$ & $6.774 \si{\milli \second} \pm 0.545$ \\
        GPHLVM, $\lorentz{3}$ & $6.887 \si{\second} \pm 0.307$ & $10.34 \si{\milli \second} \pm 1.05$ \\
        \bottomrule
	\end{tabular}
    \end{sc}
    \end{small}
    \end{center}
    \vspace{-0.4cm}
\end{table}

\vspace{-0.35cm}
\paragraph{Taxonomy expansion and unseen poses encoding:}
An advantage of back-constrained GPLVMs is their affordance to ``embed'' new observations into the latent space. We test the GPHLVM ability to place unseen class instances or unobserved taxonomy classes into the latent space, hypothesizing that their respective embeddings should be positioned to preserve the relative distances within the taxonomy graph compared to the other latent points.
First, we consider back-constrained GPHLVMs with stress prior previously trained on a subset of the taxonomies data (i.e., the models in Figs.~\ref{fig:GPHLVM-bimanual:backconstrained_and_stress},~\ref{fig:GPHLVM-grasps:backconstrained_and_stress},~\ref{fig:GPHLVM:backconstrained_and_stress}) and embedded unseen class instances. Figures~\ref{fig:GPHLVM-bimanual:added_poses},~\ref{fig:GPHLVM-grasps:added_poses} and~\ref{fig:GPHLVM:added_poses} show how the new data land close to their respective class cluster. 
Second, we train new GPHLVMs for the three taxonomies while withholding all data instance from one or several classes (see App.~\ref{app:UnseenClassesAndPoses}).
We then encode these data and find that they preserve the relative taxonomy graph distances when compared to the model trained on the full dataset.
Although this is accomplished by both models, our GPHLVMs display lower stress values (see Table~\ref{table:mean_stress_of_models}).

\vspace{-0.35cm}
\paragraph{Trajectory generation via geodesics:}
The geometry of the GPHLVM latent space can also be exploited to generate trajectories in the latent space by following the geodesic, i.e., the shortest path, between two embeddings. 
In other words, our GPHLVM intrinsically provides a mechanism to plan motions via geodesics in the low-dimensional latent space.
Examples of geodesics between two embeddings for the three taxonomies are shown in Figs.~\ref{fig:GPHLVM-bimanual:stress_prior}-\ref{fig:GPHLVM-bimanual:backconstrained_and_stress},~\ref{fig:GPHLVM-grasps:stress_prior}-\ref{fig:GPHLVM-grasps:backconstrained_and_stress}, and~\ref{fig:GPHLVM:stress_prior}-\ref{fig:GPHLVM:backconstrained_and_stress}, with the colors along the trajectory matching the class corresponding to the closest hyperbolic latent point. 
Importantly, the geodesics in GPHLVMs latent space follow the transitions between classes defined in the taxonomy. 
In other words, the shortest paths in the hyperbolic embedding correspond to the shortest paths in the taxonomy graph. 
For instance, in the case of the support pose taxonomy, the geodesic from $\mathsf{LF}$ to $\mathsf{F}_2\mathsf{RH}$ follows $\mathsf{LF} \to \mathsf{F}_2 \to \mathsf{F}_2\mathsf{RH}$. Straight lines in the Euclidean embeddings are more likely to deviate from the graph shortest path, resulting in transitions that do not exist in the taxonomy, e.g., $\mathsf{RF}\mathsf{RH}\to \mathsf{F}_2$ in the Euclidean latent space of Figs.~\ref{fig:GPHLVM:stress_prior}-\ref{fig:GPHLVM:backconstrained_and_stress}.
Figure~\ref{fig:geodesic_trajectories} and App.~\ref{app:experiment_geodesic_motions} show motions resulting from geodesic interpolation in the GPHLVM latent space. The obtained motions are more realistic than those obtained via linear interpolation in the GPLVM latent space and as realistic as those obtained via VPoser~\citep{Pavlakos19:VPoser} (see Figs.~\ref{fig:GPHLVM:trajectories_grasps2}-\ref{fig:GPHLVM:trajectories_support_poses1}).

\begin{figure}[t]
	\centering
    \begin{subfigure}[b]{0.48\textwidth}
    \centering
        \includegraphics[trim={0.14cm 0.2cm 0.0cm 0.3cm},clip,width=0.95\textwidth]{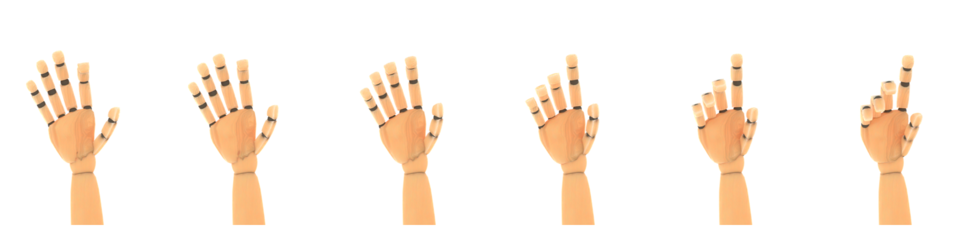}
        \includegraphics[trim={0.15cm 0.85cm 0.0cm 0.25cm},clip,width=0.95\textwidth]{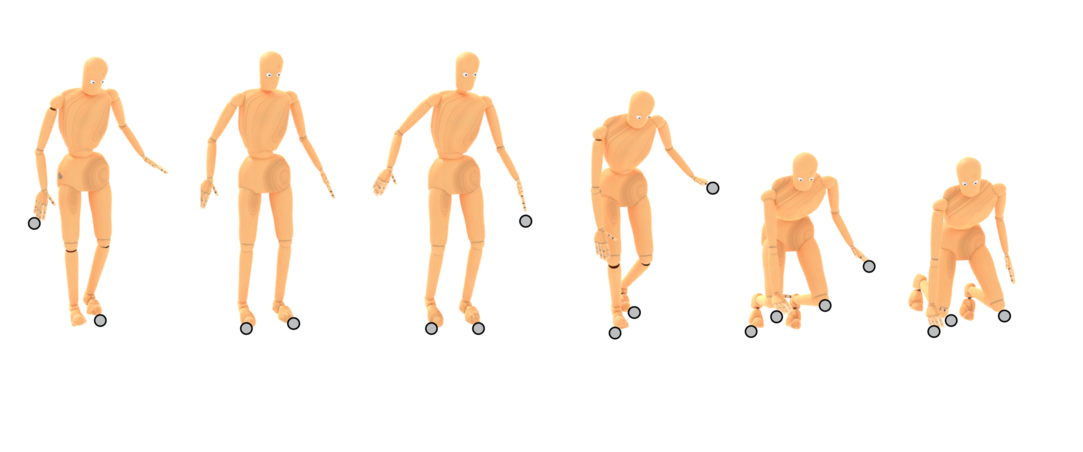}
    \end{subfigure}
	\caption{Motions obtained via geodesic interpolation in the back-constrained GPHLVM latent space. \emph{Top}: Grasp taxonomy from ring ($\mathsf{Ri}$) to index finger extension ($\mathsf{IE}$). \emph{Bottom}: Support pose taxonomy from $\mathsf{LF}\mathsf{RH}$ to $\mathsf{K}_2\mathsf{RH}$. Gray circles denote contacts. }
	\label{fig:geodesic_trajectories}
	\vspace{-0.6cm}
\end{figure}

\vspace{-0.2cm}
\section{Conclusions}
\label{sec:conclusion}
\vspace{-0.2cm}
Inspired by the recent developments of human motion taxonomies, we proposed a computational model GPHLVM that leveraged two types of domain knowledge: the structure of a human-designed taxonomy and a hyperbolic geometry on the latent space which complies with the intrinsic taxonomy's hierarchical structure.
Our GPHLVM allows us to learn hyperbolic embeddings of the features of the taxonomy nodes while capturing the associated hierarchical structure.
To achieve this, our model exploited the curvature of the hyperbolic manifold and the graph-distance information as inductive bias. 
We showed that these two forms of inductive biases are essential to learn taxonomy-aware embeddings, encode unseen data, and potentially expand the learned taxonomy. 
Moreover, we reported that vanilla Euclidean approaches underperformed on all the foregoing cases.
Finally, we introduced a mechanism to generate taxonomy-aware motions in the hyperbolic latent space. 

Note that we assumed that the desired hierarchy is mainly provided by the given taxonomy, which we use as inductive bias in our model. Although our assumption is that the provided taxonomy structure is accurate, our model may also be encouraged to discover additional hierarchical structure by adjusting the scale $\gamma$ of the stress loss function. This is particularly interesting for cases where the provided taxonomy is incomplete or inexact, and thus a lower scale may allow the model to prioritize unsupervised discovery of a hierarchical structure from the dataset itself, mitigating the impact of potential taxonomy errors. 
An interesting extension to our work would be incorporating uncertainty measures for the taxonomy graph. If quantifiable measures of uncertainty for specific nodes or relationships exist, we could integrate them into the stress loss. This would allow us to down-weight the influence of unreliable nodes of the taxonomy, further improving robustness.

Our proposed GPHLVM opens the door to potential applications in fields like bioinformatics. For instance, GPHLVM may uncover hierarchical structures associated to protein interactions~\citep{AlanisLobato18:HyperbolicProteins} or within biological sequences~\citep{Corso21:HyperbolicBioSequence,Macaulay23:HyperbolicPhylogenetic}.
Moreover, the availability of a motion taxonomy structure empowers the GPHLVM to impact various downstream tasks, including robot motion generation, robotic grasping and manipulation, human motion prediction, and character animation. 
In particular, the taxonomy prior may compensate for the lack of data in some of the foregoing applications. 
Unlike other LVMs such as VPoser~\citep{Pavlakos19:VPoser}, GAN-S~\citep{Davydov22:AdversarialParametricPosePrior}, and TEACH~\citep{Athanasiou22:TEACH}, which are trained on full human motion trajectories and thousands of datapoints, our model leverages the taxonomies as inductive bias to better structure the learned embeddings, and uses geodesics as a simple and effective motion generator between single poses.
However, as other models, our geodesic motion generation does not use explicit knowledge on how physically feasible the generated trajectories are. We plan to investigate how to include physics constraints or explicit contact data into the GPHLVM to obtain physically-feasible motions. % that can be executed on real robots.
We will also work on alleviating the computational cost of the hyperbolic kernel by using more efficient sampling strategies. 
For example, instead of sampling from a Gaussian distribution for the approximation~\eqref{eq:heat_monte_carlo}, we could sample from the Rayleigh distribution. This is because complex numbers, whose real and imaginary components are i.i.d. Gaussian, have absolute value that is Rayleigh-distributed.
Finally, we will investigate other manifold geometries to accommodate more complex structures coming from highly-heterogeneous graphs~\citep{Giovanni2022:heterogeneous}.

\clearpage
\section*{Impact Statement}
This paper presents work whose goal is to advance the field of Machine Learning by incorporating different types of inductive bias in latent variable models. The introduction of such inductive bias --- particularly those related to taxonomy structures --- may lead to more explainable Machine Learning models. 

\section*{Acknowledgements}
NJ and TA were supported by the Carl Zeiss Foundation through the JuBot project and by the European Union's Horizon Europe Framework Programme under grant agreement No 101070596 (euROBIN).
MGD collaborated in this work during his PhD sabbatical at the Bosch Center for Artificial Intelligence (BCAI). 
VB was supported by an ETH Z\"urich Postdoctoral Fellowship.

\bibliography{References}
\bibliographystyle{icml2024}

%%%%%%%%%%%%%%%%%%%%%%%%%%%%%%%%%%%%%%%%%%%%%%%%%%%%%%%%%%%%%%%%%%%%%%%%%%%%%%%
%%%%%%%%%%%%%%%%%%%%%%%%%%%%%%%%%%%%%%%%%%%%%%%%%%%%%%%%%%%%%%%%%%%%%%%%%%%%%%%
% APPENDIX
%%%%%%%%%%%%%%%%%%%%%%%%%%%%%%%%%%%%%%%%%%%%%%%%%%%%%%%%%%%%%%%%%%%%%%%%%%%%%%%
%%%%%%%%%%%%%%%%%%%%%%%%%%%%%%%%%%%%%%%%%%%%%%%%%%%%%%%%%%%%%%%%%%%%%%%%%%%%%%%
\newpage
\appendix
\onecolumn

\section{Hyperbolic manifold}
\label{app:hyperbolic}
\subsection{Manifold operations}
\label{app:hyperbolic-operations}
As mentioned in the main text (\S~\ref{sec:background}), we resort to the exponential and logarithmic maps to operate with Riemannian manifold data. 
The exponential map $\expmap{\bm{x}}{\bm{u}}: \tangentspace{\bm{x}} \to \manifold$ maps a point $\bm{u}$ in the tangent space of $\bm{x}$ to a point $\bm{y}$ on the manifold, while the logarithmic map $\logmap{\bm{x}}{\bm{u}}: \manifold \to \tangentspace{\bm{x}}$ performs the corresponding inverse operation. 
In some settings, it is necessary to work with data lying on different tangent spaces of the manifold.
In this case, one needs to operate with all data on a single tangent space, which can be achieved by leveraging the parallel transport $\prltrsp{\bm{x}}{\bm{y}}{\bm{u}}: \tangentspace{\bm{x}}\to\tangentspace{\bm{y}}$.
All the aforementioned operators are defined in Table~\ref{tab:HyperbolicOperations} for the Lorentz model $\lorentz{d}$ and illustrated in Fig.~\ref{fig:appendix:hyperbolic-basics} for $\mathcal{L}^2$.
Moreover, we introduce the inner product $\innerprod{\bm{x}}{\bm{u}}{\bm{v}}$ between two points on $\lorentz{d}$, which is used to compute the geodesic distance $d_{\manifold}(\bm{u}, \bm{v})$ and all the foregoing operations in the Lorentz model, as shown in Table~\ref{tab:HyperbolicOperations}. 
\begin{figure}
	\centering
	\begin{subfigure}[b]{0.45\textwidth}
		\centering
		\includegraphics[trim={0.0cm 0.0cm 0.0cm 0.0cm},clip, width=.8\textwidth]{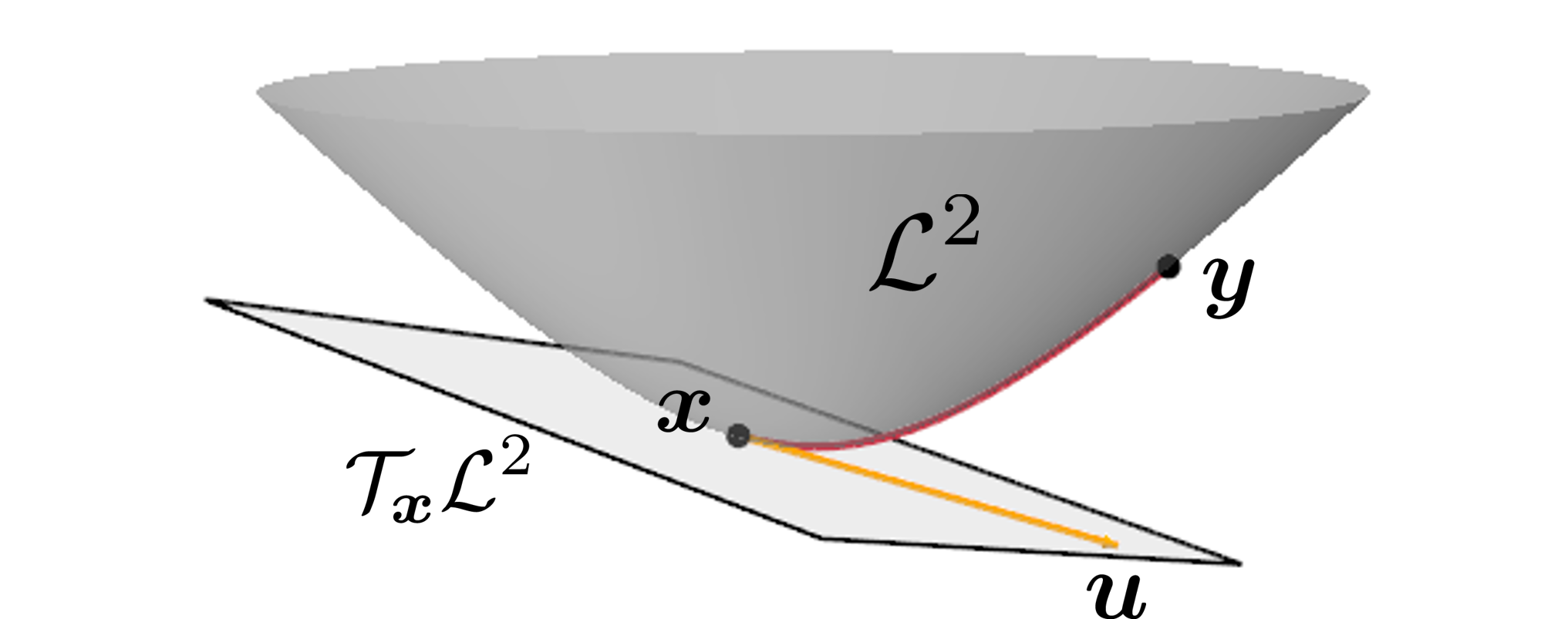}
		\caption{Exponential and logarithmic maps.}
		\label{fig:appendix:explog_maps}
	\end{subfigure}%
	\begin{subfigure}[b]{0.45\textwidth}
		\centering
		\includegraphics[trim={0.0cm 0.0cm 0.0cm 0.0cm},clip, width=.8\textwidth]{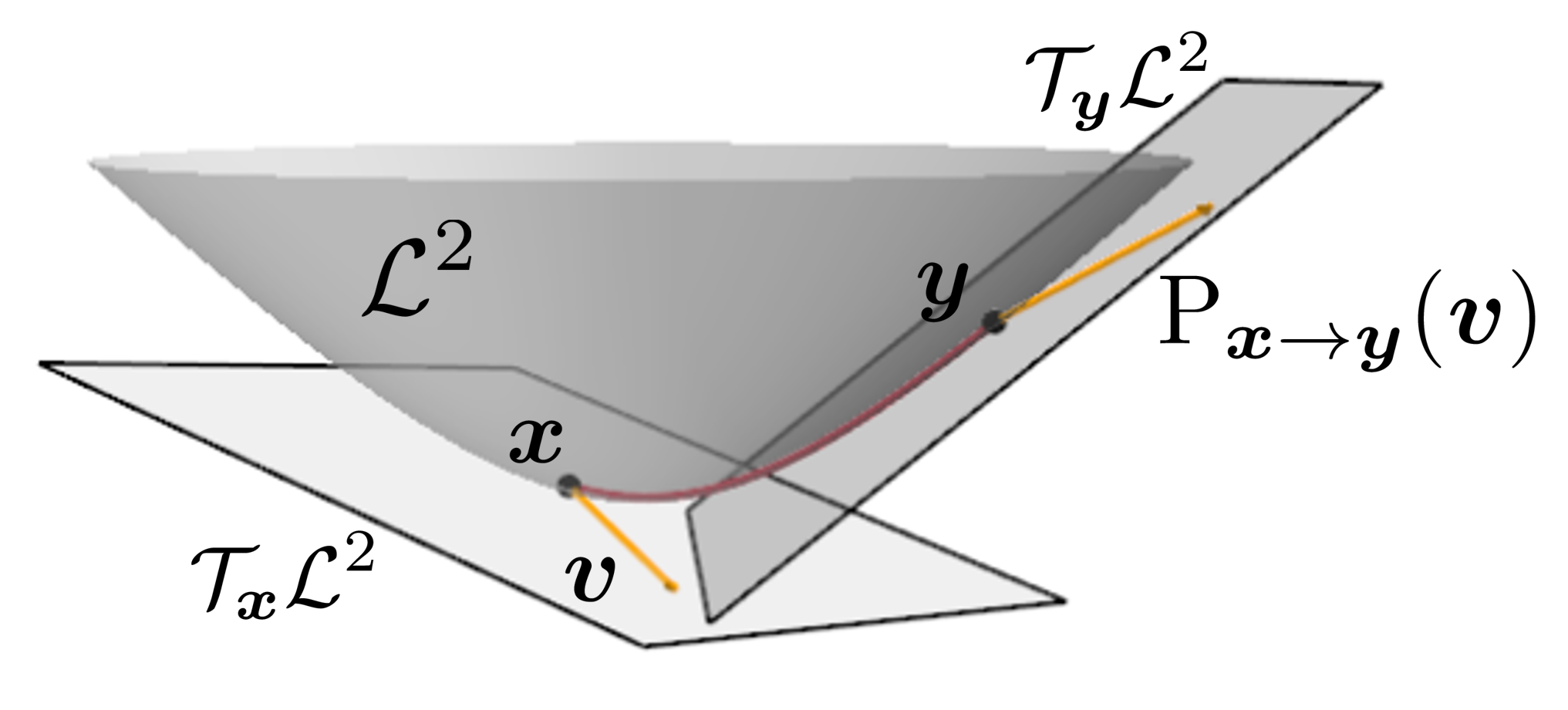}
		\caption{Parallel transport.}
		\label{fig:appendix:prl_trsp}
	\end{subfigure}
	\caption{Principal Riemannians operation on the Lorentz model $\lorentz{2}$. \emph{(a)} The geodesic (\crimsonline) is the shortest path between the two points $\bm{x}$ to $\bm{y}$ on the manifold. The vector $\bm{u}$ (\yellowarrow) lies on the tangent space of $\bm{x}$ such that $\bm{y} = \expmap{\bm{x}}{\bm{u}}$. %and $\bm{u} = \logmap{\bm{x}}{\bm{y}}$. 
    \emph{(b)} $\prltrsp{\bm{x}}{\bm{y}}{\bm{v}}$ is the parallel transport of the vector $\bm{v}$ from $\mathcal{T}_{\bm{x}}\lorentz{2}$ to $\mathcal{T}_{\bm{y}}\lorentz{2}$. }
	\label{fig:appendix:hyperbolic-basics}
    \vspace{-0.2cm}
\end{figure}

\setlength{\extrarowheight}{2pt}
\begin{table}[h]
\caption{Principal operations on the Lorentz model $\lorentz{d}$. For more details, see~\citep{Bose20:HyperbolicNFs} and \citep{Peng21:SurveyHperbolic}.}
\vspace{-0.3cm}
\label{tab:HyperbolicOperations}
    \begin{center}
    \begin{small}
	\begin{tabular}{ll}
		\rowcolor{lightgray}
		\hline
		\textbf{\textsc{Operation}} & \textbf{\textsc{Formula}} \\ [0.3ex]
		\hline 
		$\innerprod{\bm{x}}{\bm{u}}{\bm{v}}$ & $-u_0v_0 + \sum_{i=1}^{d} u_iv_i$ \\ [0.3ex] 
		\hline 
		$\manifolddist{\bm{u}}{\bm{v}}$ & $\operatorname{arcosh}(-\innerprod{\bm{x}}{\bm{u}}{\bm{v}})$ \\ [0.3ex] 
		\hline
		$\expmap{\bm{x}}{\bm{u}}$ & $\operatorname{cosh}(\norm{\ty{\mathcal{L}}}{\bm{u}})\bm{x} + \operatorname{sinh}(\norm{\ty{\mathcal{L}}}{\bm{u}}) \frac{\bm{u}}{\norm{\ty{\mathcal{L}}}{\bm{u}}}$ with $\norm{\ty{\mathcal{L}}}{\bm{u}} = \sqrt{\innerprod{\bm{x}}{\bm{u}}{\bm{u}}}$\\ [0.3ex]
		\hline 
		$\logmap{\bm{x}}{\bm{y}}$ & $\frac{\manifolddist{\bm{x}}{\bm{y}}}{\sqrt{\alpha^2-1}}(\bm{y}+\alpha\bm{x})$ with $\alpha=\innerprod{\bm{x}}{\bm{x}}{\bm{y}}$\\ [0.3ex]
		\hline
		$\prltrsp{\bm{x}}{\bm{y}}{\bm{v}}$ & $\bm{v} + \frac{\innerprod{\bm{x}}{\bm{y}}{\bm{v}}}{1-\innerprod{\bm{x}}{\bm{x}}{\bm{y}}}(\bm{x}+\bm{y})$\\ [0.3ex]
		\hline 
	\end{tabular}
    \end{small}
    \end{center}
\end{table}

\subsection{Equivalence of Poincaré and Lorentz models}
\label{app:hyperbolic-isometry}
As pointed out in the main text (\S~\ref{sec:background}), it is possible to map points from the Lorentz model to the Poincaré ball via an isometric mapping. 
Formally, such  an isometry is defined as the mapping function $f:\lorentz{d} \to \poincare{d}$ such that
\begin{equation}
f(\bm{x}) = \frac{\left(x_1, \dots, x_d \right)^\trsp}{x_0 + 1} ,
\label{eq:HypeDiffeomorphism}
\end{equation}
where $\bm{x} \in \lorentz{d}$ with components $x_0, x_1, \dots, x_d$.
The inverse mapping $f^{-1}:\poincare{d} \to \lorentz{d}$ is defined as follows 
\begin{equation}
f^{-1}(\bm{y}) = \frac{\left(1+\norm{}{\bm{y}}^2, 2y_1, \dots, 2y_d\right)^\trsp}{1 - \norm{}{\bm{y}^2}} ,
\end{equation}
with $\bm{y} \in \poincare{d}$ with components $y_1, \dots, y_d$. Notice that we used the mapping~\eqref{eq:HypeDiffeomorphism} to represent the hyperbolic embeddings in the Poincaré disk throughout the paper, as well as in the computation of the kernel $k^{\lorentz{2}}$~\eqref{eq:hyp_heat}.

\subsection{Hyperbolic wrapped Gaussian distribution}
\label{app:hyperbolic-wGD}
Fig.~\ref{fig:appendix:hyperbolic-wrapped} illustrates the hyperbolic wrapped Gaussian distribution~\citep{Nagano19:HyperbolicNormal}, which is introduced in \S~\ref{sec:background} and utilized as prior distribution for the GPHLVM's embeddings.

\begin{figure}
	\centering
	\includegraphics[trim={0.0cm 0.0cm 0.0cm 0.0cm},clip, width=.3\textwidth]{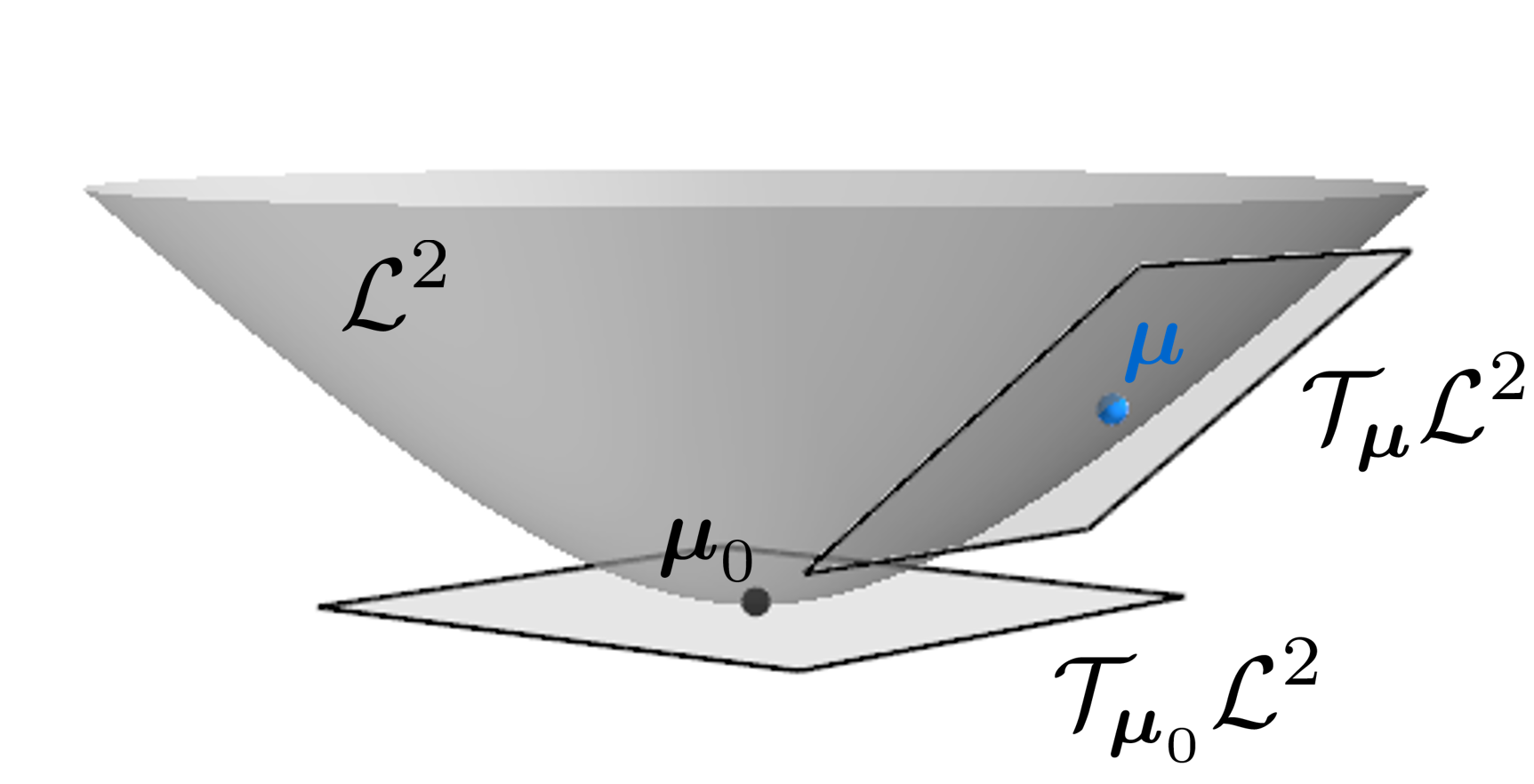}
	\includegraphics[trim={0.0cm 0.0cm 0.0cm 0.0cm},clip, width=.3\textwidth]{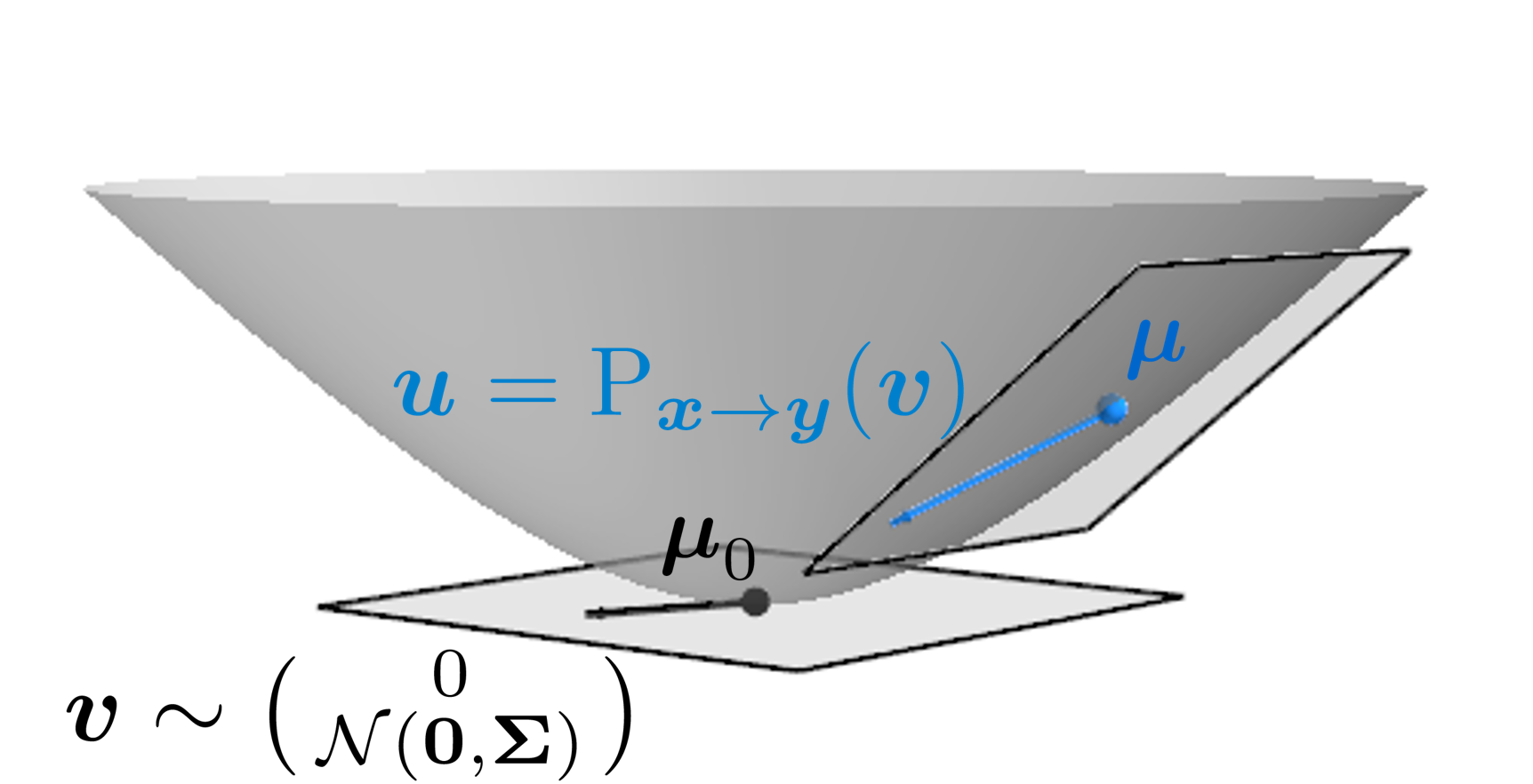}
    \includegraphics[trim={0.0cm 0.0cm 0.0cm 0.0cm},clip, width=.3\textwidth]{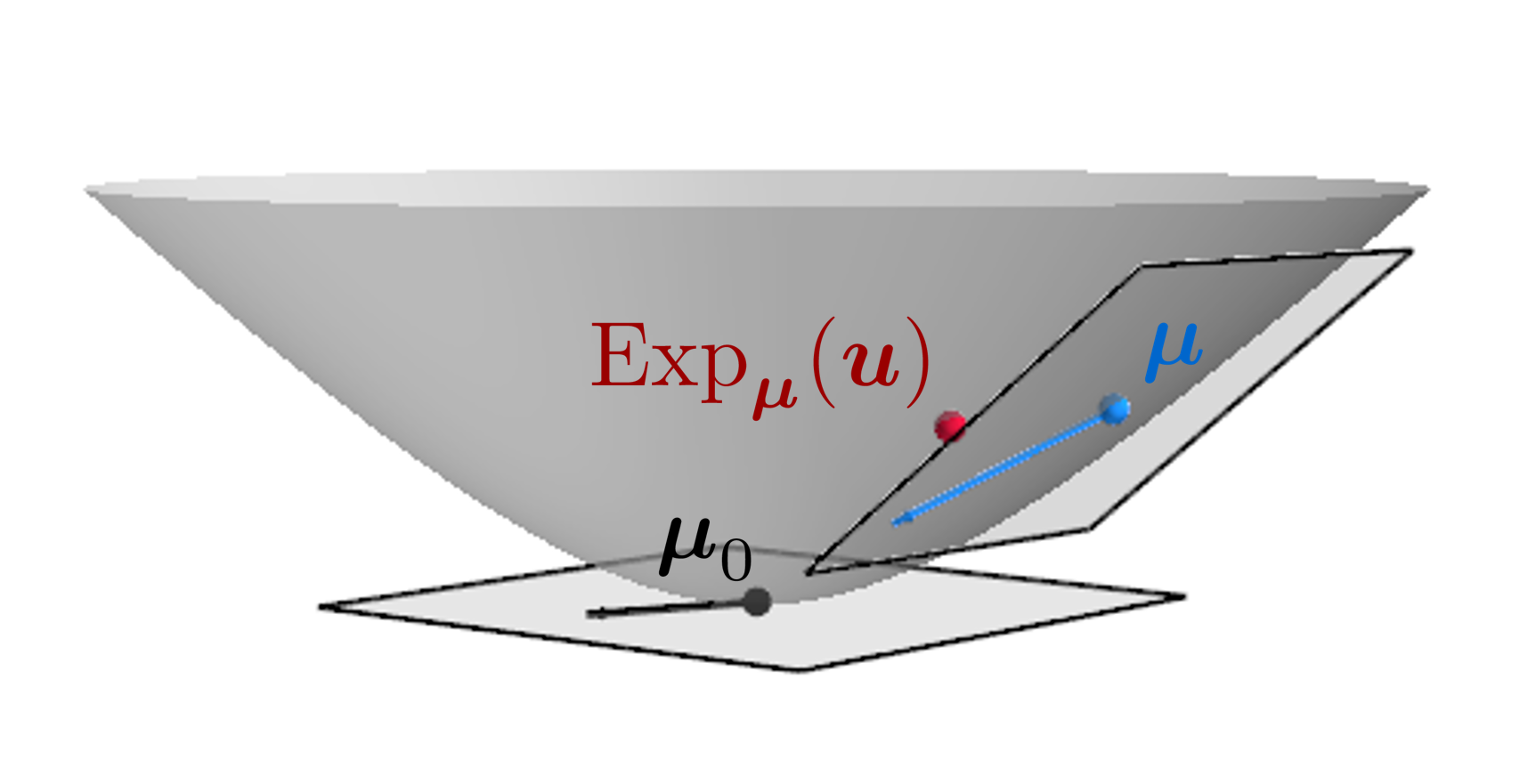}
	\caption{Illustration of the hyperbolic wrapped Gaussian distribution $\hypenormal{\bm{x}}{\bm{\mu}}{\bm{\Sigma}}$ on the Lorenz model $\lorentz{2}$ of the hyperbolic manifold. \emph{Left:} Manifold origin $\bm{\mu}_0$, mean $\bm{\mu}$, and corresponding tangent spaces. \emph{Middle:} A point $\bm{v}$ is sampled from a Euclidean Gaussian distribution in the tangent space of $\bm{\mu}_0$ and moved to $\mathcal{T}_{\bm{\mu}}\lorentz{2}$ via the parallel transport. \emph{Right:} The parallel transported sample $\bm{u}$ is projected onto the manifold using the exponential map. The resulting hyperbolic sample is depicted as a red dot.}
	\label{fig:appendix:hyperbolic-wrapped}
    \vspace{-0.2cm}
\end{figure}

%===============================================================================

\section{Hyperbolic kernels}
\label{app:kernels}
As mentioned in the main text (\S~\ref{subsec:HyperbolicKernels}), following the developments on kernels on manifolds like \citep{Borovitskiy20:GPManifolds,Jaquier21:GaBOMatern}, we may identify the generalized squared exponential kernel with the \emph{heat kernel} --- an important object studied on its own in the mathematical literature.
Due to this, we can obtain the expressions~\eqref{eq:hyp_heat}.
The expression for the case of $\lorentz{2}$ requires discretizing the integral, which may lead to an approximation that is not positive semidefinite.
We address this by suggesting another approximation guaranteed to be positive semidefinite.

\subsection{Alternative Monte Carlo approximation}

Reversing the derivation in \citep[p. 246]{Chavel84:Eigenvalues}, we obtain
\begin{equation} \label{eqn:vanilla_hyperbolic}
k^{\poincare{2}}_{\infty, \kappa, \sigma^2}(\bm{x}, \bm{x}')
=
\frac{\sigma^2}{C_{\infty}'}
\int_0^{\infty}
\exp(-s^2/(2 \kappa^2))
P_{-1/2 + i s}(\cosh(\rho))
s \tanh(\pi s)
\mathrm{d} s ,
\end{equation}
where $\rho = \operatorname{dist}_{\poincare{d}}(\bm{x}, \bm{x}')$ denotes the geodesic distance between $\bm{x}, \bm{x}' \in \poincare{2}$,  $\kappa$ and $\sigma^2$ are the kernel lengthscale and variance, $C_{\infty}'$ is a normalizing constant and $P_{\alpha}$ are Legendre functions \cite{Abramowitz64:Handbook}. Note that we leverage the isometry between the Lorentz and Poincaré models for the computation of the kernel.
Now we prove that these Legendre functions are connected to the \emph{spherical functions} --- special functions closely tied to the geometry of the hyperbolic space and possessing a very important property.

\begin{proposition}
	Assume the disk model of $\poincare{2}$ (i.e. the Poincaré disk).
	Denote the disk by $\mathbb{D}$ and its boundary, the circle, by $\mathbb{T}$.
	Define the hyperbolic outer product by $\langle \bm{z}, \bm{b} \rangle = \frac{1}{2}\log\frac{1-|\bm{z}|^2}{|\bm{z}-\bm{b}|^2}$ for $\bm{z} \in \mathbb{D}, \bm{b} \in \mathbb{T}$.
	Then
	\begin{equation} \label{eqn:magical_property}
	P_{-1/2 + i s}(\cosh(\rho))
	=
	\underbrace{
		\int_{\mathbb{T}}
		e^{(2 s i + 1)\langle \bm{z}, \bm{b}\rangle}
		\mathrm{d} b}_{\text{spherical function } \phi_{2 s}(\bm{z})}
	=
	\int_{\mathbb{T}}
	e^{(2 s i + 1)\langle \bm{z}_1, \bm{b} \rangle}
	\overline{e^{(2 s i + 1)\langle \bm{z}_2, \bm{b} \rangle}}
	\mathrm{d} \bm{b} ,
	\end{equation}
	where $\bm{z} \in \mathbb{D}$ is such that $\rho = \operatorname{dist}_{\poincare{2}}(\bm{z}, \bm{0})$ and $\bm{z}_1, \bm{z}_2 \in \mathbb{D}$ are such that $\rho = \operatorname{dist}_{\poincare{2}}(\bm{z}_1, \bm{z}_2)$.
	Here $i$ denotes the imaginary unit and $\overline{\bm{z}}$ is the complex conjugation.
\end{proposition}

\begin{proof}
	Let $\theta$ denote the angle between $\bm{z}$ and $\bm{b}$, and note the following simple identities
	\begin{align}
	|\bm{z}-\bm{b}|^2
	&=
	|\bm{z}|^2 + 1 - 2 |\bm{z}| \cos(\theta)
	=
	\tanh(\rho)^2 + 1 - 2 \tanh(\rho) \cos(\theta),
	\\
	1 - |\bm{z}|^2
	&=
	1 - \tanh(\rho)^2
	=
	\cosh(\rho)^{-2}.
	\end{align}
	Then, we write
	\begin{align}
	e^{(2 s i + 1) \langle \bm{z}, \bm{b} \rangle}
	=
	\left(
	\frac{|\bm{z}-\bm{b}|^2}{1-|\bm{z}|^2}
	\right)^{-s i - 1/2}
	&=
	\left(\cosh(\rho)^{2} (\tanh(\rho)^2 + 1 - 2 \tanh(\rho) \cos(\theta))\right)^{-s i - 1/2} ,
	\\
	&=
	\left(\sinh(\rho)^2 + \cosh(\rho)^{2} - 2 \sinh(\rho) \cosh(\rho) \cos(\theta)\right)^{-s i - 1/2} ,
	\\
	&=
	\left(\cosh(2 \rho) + \sinh(2 \rho) \cos(\theta) \right)^{-s i - 1/2} .
	\end{align}
	
	On the other hand, by \citet[Eq.~7.4.3]{Lebedev65:Special}, we have
	$
	P_a(\cosh(x))
	=
	\frac{1}{\pi}
	\int_{0}^{\pi}
	(\cosh(x) + \sinh(x)\cos(\theta))^a \mathrm{d} \theta,
	$
	hence
	\begin{align}
	P_{-1/2 + i s}(\cosh(2\rho))
	&=
	\frac{1}{\pi}
	\int_{0}^{\pi}
	(\cosh(2 \rho) + \sinh(2 \rho) \cos(\theta))^{-1/2 + i s} \mathrm{d} \theta ,
	\\
	&=
	\frac{1}{2 \pi}
	\int_{-\pi}^{\pi}
	(\cosh(2 \rho) + \sinh(2 \rho) \cos(\theta))^{-1/2 + i s} \mathrm{d} \theta ,
	\\
	&=
	\int_{\mathbb{T}}
	e^{(-2 s i + 1) \langle \bm{z}, \bm{b} \rangle} \mathrm{d} \bm{b}
	= \phi_{-2 s}(\bm{z}).
	\end{align}
	This computation roughly follows \citet[Section 4.3.4]{Cohen12:Stationary}.
	Now, by \citet[Section 3.5]{Cohen12:Stationary}, we have $\phi_{-2 s}(\bm{z}) = \phi_{2 s}(\bm{z})$ which proves the first identity.
	Finally, Lemma 3.5 from \citet{Cohen12:Stationary} proves the second identity.
\end{proof}

By combining expressions~\eqref{eqn:vanilla_hyperbolic} and \eqref{eqn:magical_property}, we get the following Monte Carlo approximation
\begin{equation} \label{eqn:heat_monte_carlo}
k^{\poincare{2}}_{\infty, \kappa, \sigma^2}(\bm{x}, \bm{x}')
\approx
\frac{\sigma^2}{C_{\infty}'}
\frac{1}{L} \sum_{l=1}^L s_l \tanh(\pi s_l) e^{(2 s_l i + 1) \langle \bm{x}_{\mathcal{P}}, \bm{b}_l \rangle} \overline{e^{(2 s_l i + 1) \langle \bm{x}_{\mathcal{P}}', \bm{b}_l \rangle}} ,
\end{equation}
where $\bm{b}_l \stackrel{\text{i.i.d.}}{\sim} U(\mathbb{T})$ and $s_l \stackrel{\text{i.i.d.}}{\sim} e^{- s^2 \kappa^2 / 2} \mathbbold{1}_{[0, \infty)}(s)$.
This gives the approximation used in the main text (see \S~\ref{subsec:HyperbolicKernels}).

Having established a way to evaluate or approximate the heat kernel, analogs of Mat\'ern kernels can be defined by
\begin{equation} \label{eqn:matern_integral_formula}
k_{\nu, \kappa, \sigma^2}(\bm{x}, \bm{x}')
=
\frac{\sigma^2}{C_{\nu}}
\int_0^{\infty}
u^{\nu - 1}
e^{-\frac{2 \nu}{\kappa^2} u}
\tilde{k}_{\infty, \sqrt{2 u}, \sigma^2}(\bm{x}, \bm{x}')
\mathrm{d} u ,
\end{equation}
where $\tilde{k}_{\infty, \sqrt{2 u}, \sigma^2}$ is the same as $k_{\infty, \sqrt{2 u}, \sigma^2}$ but with the normalizing constant $\sigma^2/C_{\infty}$ dropped for simplicity.
Here $C_\nu$ is the normalizing constant ensuring that $k_{\nu, \kappa, \sigma^2}(\bm{x}, \bm{x}) = \sigma^2$ for all $\bm{x}$.

\subsection{Influence of the number of Monte Carlo samples}
The number of Monte Carlo samples influences the quality of the hyperbolic kernel which is used to evaluate the relationship between the latent embeddings in the GPHLVM with $2$-dimensional latent space. Fig.~\ref{fig:KernelSamplesAnalysis} displays the hyperbolic kernel value $k^{\poincare{2}}_{\infty, \kappa, \sigma^2}(\bm{x}, \bm{x}')$ with $\kappa=\sigma=1$ as a function of the distance $\operatorname{dist}_{\lorentz{2}}(\bm{x}, \bm{x}')$ between two embeddings $\bm{x},\bm{x}'$ for different number of Monte Carlo samples. The expected behavior of the hyperbolic kernel is similar to that of the Euclidean SE kernel, i.e., \emph{(1)} $k^{\poincare{2}}_{\infty, \kappa, \sigma^2}(\bm{x}, \bm{x}')=1$ when $\operatorname{dist}_{\hyperbolic{2}}(\bm{x}, \bm{x}')=0$, and (2) $k^{\poincare{2}}_{\infty, \kappa, \sigma^2}(\bm{x}, \bm{x}')$ decreases monotonically when $\operatorname{dist}_{\hyperbolic{2}}(\bm{x}, \bm{x}')$ increases. We observe that this second property is not respected for low number of Monte Carlo samples ($<1000$), as the kernel value oscillates when the distance increases. This would result in inconsistent behaviors of the GPHLVM, as kernel values may be higher for distant embedding pairs than for closer embedding pairs. The kernel generally achieves the expected behavior when the number of Monte Carlo samples is fixed above $1000$. Higher number of samples lead to higher precision and repeatability in the computation of the kernel at the expense of computation time.
For our experiments, we traded-off between kernel quality and computation time by using $3000$ samples for the kernel computation.

\begin{SCfigure}[50][t]
	\centering
    \includegraphics[trim={0.0cm 0.0cm 0.0cm 0.0cm},clip,width=.46\textwidth]{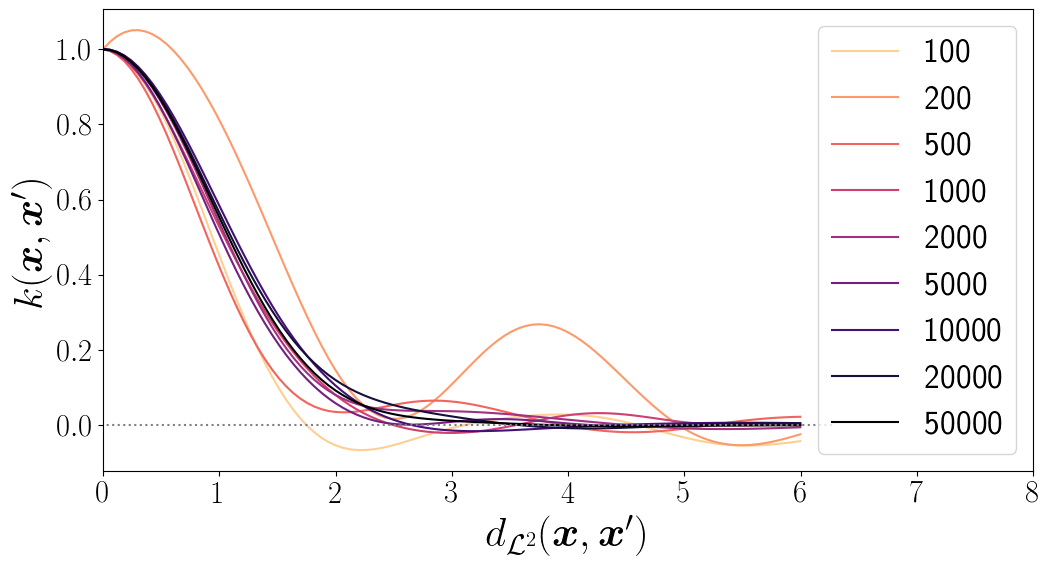}
	\caption{Hyperbolic kernel values as a function of the distance between two embeddings for different number of Monte Carlo samples.}
	\label{fig:KernelSamplesAnalysis}
	\vspace{-0.4cm}
\end{SCfigure}

\section{GPHLVM variational inference}
As mentioned in \S~\ref{subsec:trainingGPHLVM}, when training our GPHLVM on large datasets, we resort to variational inference as originally proposed in~\citep{Titsias10:BayesGPVLM}. 
Here we provide the mathematical details about the changes that are needed to train our model via variational inference.

\label{app:GPHLVMtraining}
\subsection{Computing the KL divergence between two hyperbolic wrapped normal distributions}
\label{app:KLdivergenceWrappedGP}
As mentioned in \S~\ref{subsec:trainingGPHLVM}, we approximate the KL divergence between two hyperbolic wrapped distributions via Monte-Carlo sampling. Namely, given two hyperbolic wrapped distributions $q_\phi(\bm{x})$ and $p(\bm{x})$, we write
\begin{equation}
\kl{q_\phi(\bm{x})}{p(\bm{x})} = \int q_\phi(\bm{x}) \log \frac{q_\phi(\bm{x})}{p(\bm{x})} d\bm{x} \approx \frac{1}{K} \sum_{k=1}^K \log \frac{q_\phi(\bm{x}_k)}{p(\bm{x}_k)},
\end{equation}
where we used $K$ independent Monte-Carlo samples drawn from $q_\phi(\bm{x})$ to approximate the KL divergence. 
These samples are obtained via the procedure described in \S~\ref{sec:background}, i.e., by sampling an element on the tangent space of the origin $\bm{\mu}_0=(1,0,\ldots,0)^\trsp$ of $\lorentz{d}$, via a Euclidean normal distribution, and then applying the parallel transport operation and the exponential map to project it onto $\lorentz{d}$.

\subsection{Details of the variational process}
\label{app:GPHLVMvariational}
As mentioned in the main text (\S~\ref{subsec:trainingGPHLVM}), the marginal likelihood $p(\bm{Y})$ is approximated via variational inference by approximating the posterior $p(\bm{X}|\bm{Y})$ with the hyperbolic variational distribution $q_{\phi}(\bm{X})$ as defined by~\eqref{eq:hyperbolic_variational_distribution}.
The lower bound~\eqref{eq:hyperbolicVariational} is then obtained, similarly as in~\citep{Titsias10:BayesGPVLM}, as
\begin{align}
\log p(\bm{Y}) &= \log \int p(\bm{Y} | \bm{X}) p(\bm{X}) d\bm{X} \\
&= \log \int p(\bm{Y} | \bm{X}) p(\bm{X}) \frac{q_{\phi}(\bm{X})}{q_{\phi}(\bm{X})} d\bm{X} 
= \log \expectation{q_\phi(\bm{X})}{ \frac{p(\bm{Y} | \bm{X}) p(\bm{X})}{q_{\phi}(\bm{X})} } \\
&\geq \expectation{q_\phi(\bm{X})} {\log \frac{p(\bm{Y} | \bm{X}) p(\bm{X})}{q_{\phi}(\bm{X})} } 
= \int q_{\phi}(\bm{X}) \log \frac{p(\bm{Y} | \bm{X}) p(\bm{X})}{q_{\phi}(\bm{X})}  d\bm{X}  \label{eq:ELBOmarginal}\\
&= \int q_{\phi}(\bm{X}) \log p(\bm{Y} | \bm{X}) d\bm{X} - \int q_{\phi}(\bm{X}) \log  \frac{q_{\phi}(\bm{X})}{p(\bm{X})} d\bm{X} \\
&= \expectation{q_\phi(\bm{X})}{\log p(\bm{Y} | \bm{X})} - \kl{q_\phi(\bm{X})}{p(\bm{X})}, \label{eq:ELBOmarginal_final}
\end{align}
following Jensen's inequality in~\eqref{eq:ELBOmarginal}.
As mentioned in \S~\ref{subsec:trainingGPHLVM}, the expectation $\expectation{q_\phi(\bm{X})}{\log p(\bm{Y} | \bm{X})}$ can be decomposed into individual terms for each observation dimension as $\sum_{d=1}^D \expectation{q_\phi(\bm{X})}{\log p(\bm{y}_d | \bm{X})}$, where $\bm{y}_d$ is the $d$-th column of $\bm{Y}$.
We then define the inducing inputs $\bm{Z}_d$ and inducing variables $\bm{u}_d$ the same way as the noiseless observations $\bm{f}_d$, so that the joint distribution of $\bm{f}_d$ and $\bm{u}_d$ can be written as
\begin{equation}
p(\bm{f}_d, \bm{u}_d) = \left( \begin{matrix} \bm{f}_{d} \\ \bm{u}_{d} \end{matrix} \right) = \mathcal{N} \left( \left(\begin{matrix} \bm{m}_d (\bm{X}) \\ \bm{m}_d (\bm{Z}_d) \end{matrix}\right),  \left(\begin{matrix}k_{d}(\bm{X},\bm{X}) & k_{d}(\bm{X},\bm{Z}_d) \\ k_{d}(\bm{Z}_d,\bm{X}) & k_{d}(\bm{Z}_d,\bm{Z}_d) \end{matrix}\right) \right).
\end{equation}
The lower bound~\eqref{eq:ELBOsparse} is then obtained for each dimension, similarly as in~\citep{Hensman15:ScalableVariationalGPs}, as
\begin{align}
\log p(\bm{y}_d | \bm{X}) &= \int \log p(\bm{y}_d | \bm{X}, \bm{u}_d) p(\bm{u}_d) d\bm{u}_d \\
&= \log \int p(\bm{y}_d | \bm{X}, \bm{u}_d) p(\bm{u}_d) \frac{q_{\lambda}(\bm{u}_d)}{q_{\lambda}(\bm{u}_d)} d\bm{u}_d 
= \log \expectation{q_\lambda(\bm{u}_d)}{ \frac{p(\bm{y}_d | \bm{X}, \bm{u}_d) p(\bm{u}_d)}{q_{\lambda}(\bm{u}_d)} } \\
&\geq \expectation{q_\lambda(\bm{u}_d)} {\log \frac{p(\bm{y}_d | \bm{X},\bm{u}_d) p(\bm{u}_d)}{q_{\lambda}(\bm{u}_d)} } 
= \int q_{\lambda}(\bm{u}_d) \log \frac{p(\bm{y}_d | \bm{X},\bm{u}_d) p(\bm{u}_d)}{q_{\lambda}(\bm{u}_d)}  d\bm{u}_d \label{eq:ELBOinducing}\\
&= \int q_{\lambda}(\bm{u}_d) \log p(\bm{y}_d | \bm{X}, \bm{u}_d) d\bm{u}_d - \int q_{\lambda}(\bm{u}_d) \log  \frac{q_{\lambda}(\bm{u}_d)}{p(\bm{u}_d)} d\bm{u}_d \\
&= \expectation{q_\lambda(\bm{u}_d)}{\log p(\bm{y}_d | \bm{X}, \bm{u}_d)} - \kl{q_\lambda(\bm{u}_d)}{p(\bm{u}_d)} \\
&\geq \expectation{q_\lambda(\bm{u}_d)}{\expectation{p(\bm{f}_d | \bm{u}_d)}{\log p(\bm{y}_d | \bm{f}_d(\bm{X}))}} - \kl{q_\lambda(\bm{u}_d)}{p(\bm{u}_d)} \label{eq:approxTitsias}\\
&= \expectation{q_\lambda(\bm{f}_{d})}{\log p(\bm{y}_d | \bm{f}_d(\bm{X}))} - \kl{q_\lambda(\bm{u}_d)}{p(\bm{u}_d | \bm{Z}_d)} \\
&= \expectation{q_\lambda(\bm{f}_{d})}{\log \gaussiandist{\bm{y}_{d}}{\bm{f}_{d}(\bm{X})}{\sigma^2_d}} - \kl{q_\lambda(\bm{u}_d)}{p(\bm{u}_d | \bm{Z}_d)}, \label{eq:ELBOinducing_final}
\end{align}
where we defined $q_\lambda(\bm{f}_{d}) = \int p(\bm{f}_d|\bm{u}_d) q_\lambda(\bm{u}_{d})d\bm{u}_{d}$ with the Euclidean variational distribution $q_\lambda(\bm{u}_{d})=\gaussiandist{\bm{u}_{d}}{\tilde{\bm{\mu}}_{d}}{\tilde{\bm{\Sigma}}_{d}}$, and wrote $p(\bm{u}_d | \bm{Z}_d)=p(\bm{u}_d)$ for simplicity.
The inequality~\eqref{eq:ELBOinducing} corresponds to Jensen's inequality, while~\eqref{eq:approxTitsias} is shown in~\citep{Titsias09:SparseGP}. 

Finally, substituting~\eqref{eq:ELBOinducing_final} in~\eqref{eq:ELBOmarginal_final} results in the following bound on the marginal likelihood
\begin{align}
\log p(\bm{Y}) \geq & \sum_{n=1}^N \sum_{d=1}^D \expectation{q_\phi(\bm{x}_n)}{ \expectation{q_\lambda(f_{n,d})}{\log \gaussiandist{y_{n,d}}{f_{n,d}(\bm{x}_n)}{\sigma^2_d}}} \nonumber \\
& - \sum_{d=1}^D \kl{q_\lambda(\bm{u}_d)}{p(\bm{u}_d | \bm{Z}_d)} - \sum_{n=1}^N  \kl{q_\phi(\bm{x}_n)}{p(\bm{x}_n)}. 
\end{align}

\section{GPHLVM algorithms}
\label{app:Algorithms}
In \S~\ref{subsec:trainingGPHLVM}, we introduced a GPHLVM trained via MAP estimation for small datasets and a variational GPHLVM that handles larger datasets with the aim of providing users with the most appropriate model for their specific problems.
Algorithms~\ref{algo:GPHLVM_MAP} and~\ref{algo:BC-GPHLVM_MAP} summarize the training process of the GPHLVM and back-constrained GPHLVM via MAP estimation. Algorithm~\ref{algo:BayesianGPHLVM} summarizes the training process of the variational GPHLVM.

\begin{algorithm}[h!]
   \caption{GPHLVM training via MAP.}
   \label{algo:GPHLVM_MAP}
\begin{algorithmic}
   \STATE {\bfseries Input:} 
   \STATE Observations $\{\bm{y}_n\}_{n=1}^N$ with $\bm{y}_n\in\euclideanspace^D$, associated taxonomy classes $\{c_n\}_{n=1}^N$, prior on hyperparameters $p(\bm{\Theta})$.
   \STATE {\bfseries Output:} 
   \STATE Latent variables $\{\bm{x}_n\}_{n=1}^N$ with $\bm{x}_n\in\lorentz{Q}$, hyperparameters $\bm{\Theta}=\{\theta_d\}_{d=1}^D$.
   \STATE {\bfseries Initialization:}
   \STATE Set the prior distribution $p(\bm{x})=\hypenormal{\bm{x}}{\bm{\mu}_0}{\alpha \bm{I}}$.
   \STATE Initialize the latent variables $\{\bm{x}_n\}_{n=1}^N$.
   \STATE {\bfseries Training:}
   \REPEAT
   \STATE Compute the MAP loss $\ell_{\text{MAP}}(\bm{X},\bm{\Theta})$.
   \STATE Compute additional losses, e.g., $\ell_{\text{stress}}(\bm{X})$~\eqref{eq:stressLoss}.
   \STATE $\bm{X},\bm{\Theta} \gets \mathsf{RiemannianOptStep}(\ell_{\text{MAP}} + \ell_{\text{stress}})~\eqref{eq:RiemannianGD}.$
   \UNTIL{convergence}
\end{algorithmic}
\end{algorithm}

\begin{algorithm}[tb]
   \caption{Back-constrained GPHLVM training via MAP.}
   \label{algo:BC-GPHLVM_MAP}
\begin{algorithmic}
   \STATE {\bfseries Input:} 
   \STATE Observations $\{\bm{y}_n\}_{n=1}^N$ with $\bm{y}_n\in\euclideanspace^D$, associated taxonomy classes $\{c_n\}_{n=1}^N$, prior on hyperparameters $p(\bm{\Theta})$.
   \STATE {\bfseries Output:} 
   \STATE Back-constraints weights $\{w_{q,n}\}_{n,q=1}^{N,Q}$, hyperparameters $\bm{\Theta}=\{\theta_d\}_{d=1}^D$.
   \STATE {\bfseries Initialization:}
   \STATE Initialize the back-constraints weights $\{w_{q,n}\}_{n,q=1}^{N,Q}$.
   \STATE {\bfseries Training:}
   \REPEAT
   \STATE Compute the latent variables $\{\bm{x}_n\}_{n=1}^N$ from back constraints~\eqref{eq:backconstraints}.
   \STATE Compute the MAP loss $\ell_{\text{MAP}}(\bm{X},\bm{\Theta})$.
   \STATE Compute additional losses, e.g., $\ell_{\text{stress}}(\bm{X})$~\eqref{eq:stressLoss}.
   \STATE $\bm{X},\bm{\Theta} \gets \mathsf{RiemannianOptStep}(\ell_{\text{MAP}} + \ell_{\text{stress}})~\eqref{eq:RiemannianGD}.$
   \UNTIL{convergence}
\end{algorithmic}
\end{algorithm}

\begin{algorithm}[tb]
   \caption{GPHLVM training via variational inference.}
   \label{algo:BayesianGPHLVM}
\begin{algorithmic}
   \STATE {\bfseries Inputs:} 
   \STATE Observations $\{\bm{y}_n\}_{n=1}^N$ with $\bm{y}_n\in\euclideanspace^D$, associated taxonomy classes $\{c_n\}_{n=1}^N$, prior on hyperparameters $p(\bm{\Theta})$.
   \STATE {\bfseries Outputs:} 
   \STATE Inducing inputs $\{\bm{z}_{d,m}\}_{m=1}^M$ with $\bm{z}_{d,m}\in\lorentz{Q}$, hyperbolic variational parameters $\phi = \{\bm{\mu}_n, \bm{\Sigma}_n\}_{n=1}^N$, with $\bm{\mu}_n\in\lorentz{Q}$ and $\bm{\Sigma}_n\in\tangentspacelorentz{\bm{\mu}_n}{Q}$, Euclidean variational parameters $\lambda = \{\tilde{\bm{\mu}}_d, \tilde{\bm{\Sigma}}_d\}_{d=1}^D$, hyperparameters $\bm{\Theta}=\{\theta_d\}_{d=1}^D$.
   \STATE {\bfseries Initialization:}
   \STATE Set the prior distribution $p(\bm{x})=\hypenormal{\bm{x}}{\bm{\mu}_0}{\alpha \bm{I}}$.
   \STATE Initialize the inducing inputs $\{\bm{z}_{d,m}\}_{m=1}^M$.
   \STATE Initialize the hyperbolic variational distribution over the latent variables $q_\phi(\bm{X})$~\eqref{eq:hyperbolic_variational_distribution}.
   \STATE Initialize the Euclidean variational distribution over the inducing variables $q_\lambda(\bm{u}_d)$.
   \STATE {\bfseries Training:}
   \REPEAT
   \STATE Compute the variational loss $\ell_{\text{VA}}$ as the lower bound~\eqref{eq:ELBOsparse}.
   \STATE Compute additional losses, e.g., $\ell_{\text{stress}}(\bm{X})$~\eqref{eq:stressLoss}.
   \STATE $\bm{Z},\phi, \lambda,\bm{\Theta} \gets \mathsf{RiemannianOptStep}(\ell_{\text{VA}} + \ell_{\text{stress}})~\eqref{eq:RiemannianGD}.$
   \UNTIL{convergence}
\end{algorithmic}
\end{algorithm}

\section{Mat\'ern kernels on taxonomy graphs}
\label{app:graphKernels}
As explained in \S~\ref{sec:taxonomyKnowledge} of the main paper, we leverage the Matérn kernel on graphs proposed by~\citet{Borovitskiy21:GPGraph} to design a kernel for our back-constrained GPHLVM that accounts for the geometry of the taxonomy graph. 
Here we provide the main equations of such a kernel, and refer the reader to~\citep{Borovitskiy21:GPGraph} for further details.
Formally, let us define a graph $G=(V,E)$ with vertices $V$ and edges $E$ and the \emph{graph Laplacian} as $\bm{\Delta} = \bm{D} - \bm{W}$, where $\bm{W}$ is the graph adjacency matrix and $\bm{D}$ its corresponding diagonal degree matrix, with $\bm{D}_{ii} = \sum_j \bm{W}_{ij}$. 
The eigendecomposition $\bm{U}\bm{\Lambda}\bm{U}^\trsp$ of the Laplacian $\bm{\Delta}$ is then used to formulate both the SE and Matérn kernels on graphs, as follows,
\begin{equation}
k^{\mathbb{G}}_{\infty, \kappa}(c_n, c_m) = \bm{U} \left( e^{-\frac{\kappa^2}{2}\bm{\Lambda}}\right)\bm{U}^\trsp, \quad \text{and} \quad k^{\mathbb{G}}_{\nu,\kappa}(c_n, c_m) = \bm{U}\left(\frac{2\nu}{\kappa^2} + \bm{\Lambda} \right)^{-\nu}\bm{U}^\trsp ,
\label{eq:GraphKernels}
\end{equation}
where $\kappa$ is the lengthscale (i.e., it controls how distances are measured) and $\nu$ is the smoothness parameter determining mean-squared differentiability of the associated Gaussian process (GP).
Note that the graph kernel expressions in~\eqref{eq:GraphKernels} are obtained by considering the connection between Matérn kernel GPs and stochastic partial differential equations, originally proposed by~\citet{whittle1963:SPDEs} and later extended to Riemannian manifolds in~\citep{Borovitskiy20:GPManifolds}. 
This connection establishes that SE and Matérn GPs satisfy
\begin{equation}
e^{-\frac{\kappa^2}{4}\bm{\Delta}}\bm{f} = \bm{\mathcal{W}}, \quad \text{and} \quad \left(\frac{2\nu}{\kappa^2} + \bm{\Delta} \right)^{\frac{\nu}{2}}\bm{f} = \bm{\mathcal{W}} ,
\end{equation}
where $\bm{\mathcal{W}} \sim \mathcal{N}(\bm{0}, \bm{I})$ and $\bm{f}: V \to \mathbb{R}$, which lead to definition of graph GPs~\citep{Borovitskiy21:GPGraph}.

\section{Distortion loss}
\label{app:distortion}
As explained in the main paper, we focus on two ways of embedding the graph in the hyperbolic space: a global approach using a stress regularization which matches graph distances with geodesic distances, and a combination between this stress regularization and the use of back constraints (see \S~\ref{sec:taxonomyKnowledge}). 
However, the literature on graph embeddings also surveys a \emph{distortion loss}~\citep{Cruceru21:GraphEmbeddings} given by
\begin{equation}
\label{eq:vanilla_distortion}
\ell_{\text{distortion}}(\bm{X}) = \sum_{i<j}\left|
\frac{\operatorname{dist}_{\lorentz{Q}}(\bm{x}_i, \bm{x}_j)^2}{\operatorname{dist}_{\mathbb{G}}(c_i, c_j)^2} - 1\right|^2 ,
\end{equation}
which tries to match the graph and manifold distances by minimizing their ratio's distance to $1$.

\begin{figure}
	\centering
	\includegraphics[trim={5.3cm 2.2cm 4.3cm 2.2cm},clip,width=\textwidth]{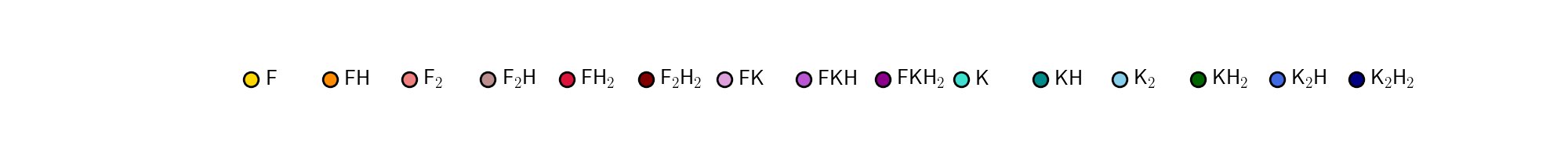}
	\begin{subfigure}[b]{0.45\textwidth}
		\centering
		\includegraphics[trim={2.5cm 2.5cm 2.5cm 2.5cm},clip, width=.45\textwidth]{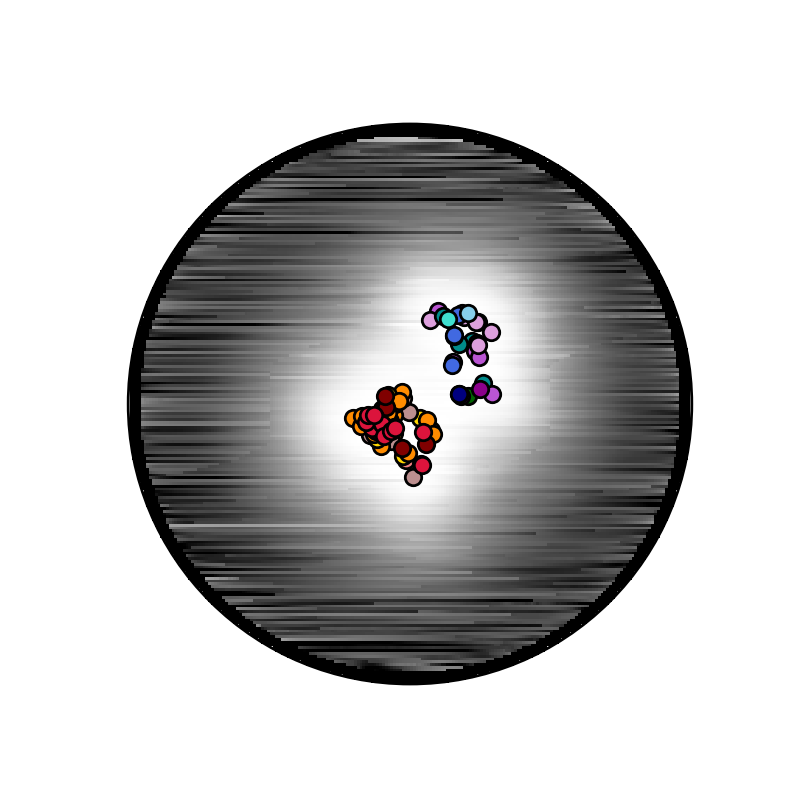}
		\includegraphics[trim={2.0cm 2.0cm 1.0cm 2.0cm},clip, width=.45\textwidth]{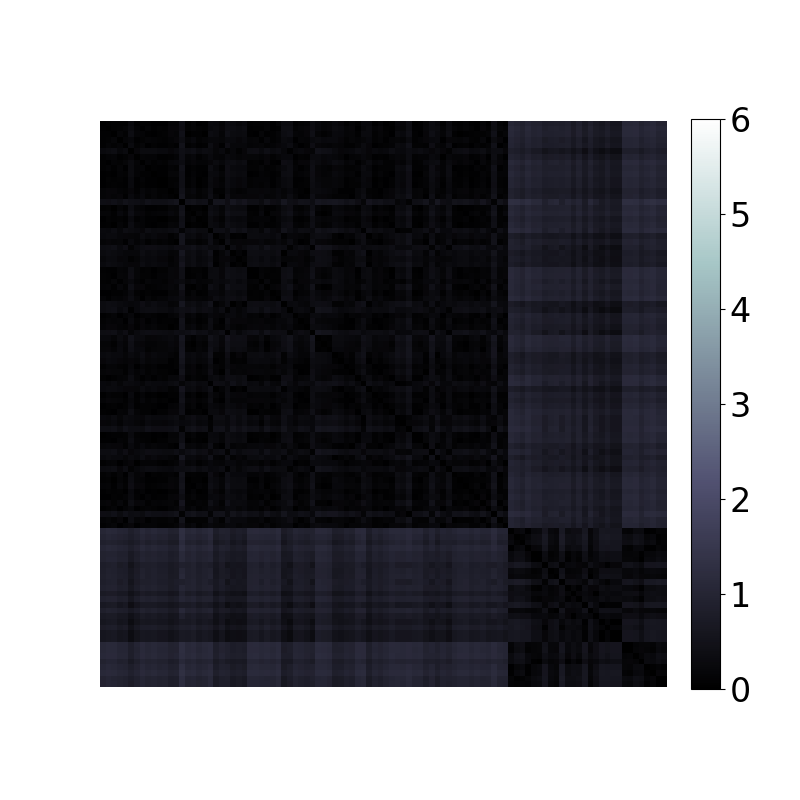}
		\caption{Regularization with $\ell_\text{distortion}$}
		\label{fig:appendix:vanilla_distortion_embeddings}
	\end{subfigure}%
	\begin{subfigure}[b]{0.45\textwidth}
		\centering
		\includegraphics[trim={2.5cm 2.5cm 2.5cm 2.5cm},clip, width=.45\textwidth]{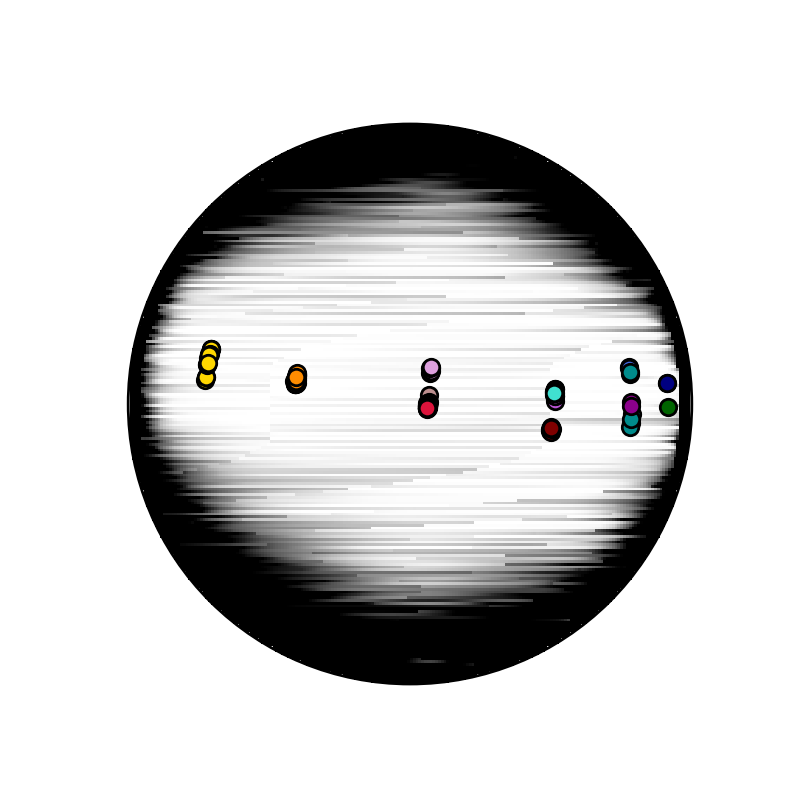}
		\includegraphics[trim={2.0cm 2.0cm 1.0cm 2.0cm},clip, width=.45\textwidth]{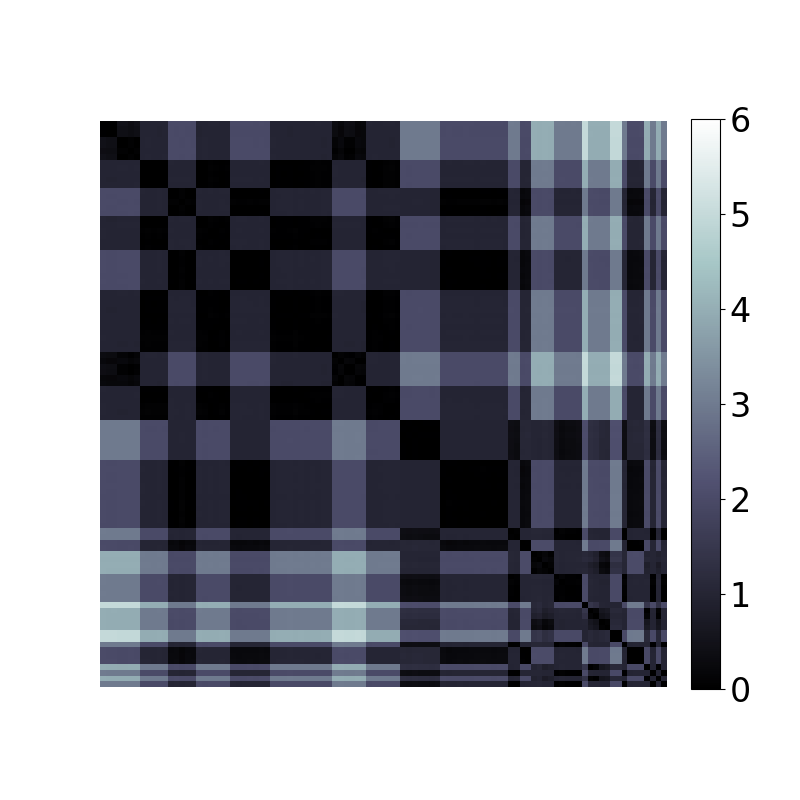}
		\caption{Reg. with $\widetilde{\ell}_\text{distortion}$}
		\label{fig:appendix:distortion_embeddings}
	\end{subfigure}
	\caption{Embeddings learned with distortion regularization. \emph{(a)} and \emph{(b)} display the latent embeddings alongside distance matrices after training our GPHLVM model with an added distortion loss $\ell_{\text{distortion}}$ as it was originally defined, and with our modified distortion loss $\widetilde{\ell}_\text{distortion}$, respectively. These embeddings indeed show that our regularizations failed to encode the distances in the graph.}
	\label{fig:appendix:distortion_results}
    \vspace{-0.3cm}
\end{figure}

We found that our problem is more subtle than usual graph embeddings, given that several points in our dataset may correspond to the same graph node (e.g., two different poses in which the left foot is the only limb in contact). 
Indeed, notice that~\eqref{eq:app:distortion_loss} is ill-defined for the case $i = j$, or equivalently, $\operatorname{dist}_{\mathbb{G}}(c_i, c_j)^2 = 0$. 
This is because all nodes $\bm{x}_i$ are assumed to be different from each other. However, in our setup, several $\bm{x}_i$ may correspond to the exact same class in the taxonomy.

Our first attempt to remediate this was to add a simple regularizer $\varepsilon = 10^{-1}$ to the denominator. However, this caused the loss to give more weight to the points where $\operatorname{dist}_{\mathbb{G}}(c_i, c_j)^2 = 0$ (see Fig.~\ref{fig:appendix:vanilla_distortion_embeddings} for the outcome of training a GPHLVM with this type of regularization).
We then considered an alternate definition of distortion in which the term inside the sum is given by
\begin{equation}
\label{eq:app:distortion_loss}
\widetilde{\ell}_\text{distortion}(\bm{x}_i, \bm{x}_j) = \left\{ \begin{array}{ll}
\lambda_1 \operatorname{dist}_{\lorentz{Q}}(\bm{x}_i, \bm{x}_j) & \text{if }\bm{x}_i\text{ and }\bm{x}_j\text{'s classes are identical} \\
\lambda_2 \ell_\text{distortion}(\bm{x}_i, \bm{x}_j)& \text{otherwise} 
\end{array}\right.
\end{equation}
where $\lambda_1, \lambda_2 \in \mathbb{R}^+$ are hyperparameters. $\lambda_1$ governs how much we encourage latent codes of the same class to collapse into a single point, while $\lambda_2$ weights how much the geodesic distance should match the graph distance. 
After manual hyperparameter tuning, we obtained the latent space and distance matrix portrayed in Figs.~\ref{fig:appendix:vanilla_distortion_embeddings}~\ref{fig:appendix:distortion_embeddings}. 
As can be seen in both accounts, the distortion loss produced lackluster results and failed to properly match the latent space distances with that of the graph. 
For these experiments, we used a loss scale of $50$, $\lambda_1 = 0.01$ and $\lambda_2=10$, meaning that we strongly encouraged the distances between non-identical classes to match in ratio.

\clearpage
\section{Additional details on the experiments of \S~\ref{sec:experiments}}
\label{app:experiment_details}
\subsection{Data}
\label{app:experiment_data_details}
For all experiments, we used humans recordings from the KIT Whole-Body Human Motion Database\footnote{\url{https://motion-database.humanoids.kit.edu/}}~\citep{Mandery16:KITmotionDatabase}. Additional details on the data of each experiments are described in the sequel.

\subsubsection{Bimanual manipulation}
Table~\ref{tab:DataBimanual} describes the data of the bimanual manipulation taxonomy used in the experiments reported in \S~\ref{sec:experiments}. We use data from subject 1723 executing five different bimanual household activities, namely $\mathsf{cut}$ and $\mathsf{peel}$ a cucumber, $\mathsf{roll}$ dough, $\mathsf{stir}$, and $\mathsf{wipe}$. The taxonomy categories are obtained using the annotations provided in~\citep{Krebs22:BimanualTaxonomy}. We obtain data for two bimanual categories that do not appear in the dataset (unimanual right and tightly-coupled asymmetric left dominant) by mirroring the motions. 

\begin{table}[h!]
	\caption{Description of the bimanual manipulation patterns extracted from the bimanual manipulation taxonomy~\citep{Krebs22:BimanualTaxonomy} used in our experiments.}
    \vspace{-0.3cm}
	\label{tab:DataBimanual}
  	\begin{center}
    \begin{small}
    \begin{sc}
	\resizebox{0.85\textwidth}{!}{
		\begin{tabular}{|c|c|c|c|}
			\rowcolor{lightgray}
			\hline
			\textbf{Bimanual category} & \textbf{Abbreviation} & \textbf{Bimanual activity} & \textbf{Number}\\ 
			\hline 
			\hline
			\multirow{2}*{Unimanual left} & \multirow{2}*{$\mathsf{U_{left}}$} & Peel & 5 \\
			& & Wipe & 5 \\
			\hline
			\multirow{2}*{Unimanual right} & \multirow{2}*{$\mathsf{U_{right}}$} & Peel & 5 \\
			& & Wipe & 5 \\
			\hline
            Uncoordinated bimanual & $\mathsf{B}$ &  -  & 0 \\
            \hline
            \multirow{4}*{Loosely coupled} & \multirow{4}*{$\mathsf{LC}$} & Cut & 2 \\
            & & Peel & 4 \\
            & & Stir & 2 \\
            & & Wipe & 2 \\
			\hline
            \multirow{4}*{Tightly-coupled asymmetric left dominant} & \multirow{4}*{$\mathsf{TCA_{left}}$} & Cut & 2 \\
            &  & Peel & 3 \\
            &  & Stir & 2 \\
            &  & Wipe & 3 \\
            \hline
            \multirow{4}*{Tightly-coupled asymmetric right dominant} & \multirow{4}*{$\mathsf{TCA_{right}}$} & Cut & 2 \\
            &  & Peel & 3 \\
            &  & Stir & 2 \\
            &  & Wipe & 3 \\
            \hline
            Tightly-coupled symmetric & $\mathsf{TCS}$ & Roll & 10 \\
            \hline
		\end{tabular}
	}
    \end{sc}
    \end{small}
    \end{center}
\end{table}

\subsubsection{Hand grasps}
Fig.~\ref{fig:HyperbolicAndTaxonomy}-\emph{right} shows the hand grasps taxonomy~\citep{Stival19:HumanGraspTaxonomy} and
Table~\ref{tab:DataGrasps} describes the data used in \S~\ref{sec:experiments}. We use grasp data\footnote{\url{https://motion-database.humanoids.kit.edu/list/motions/?datasets=3534}} from subjects 2122, 2123, 2125, 2177. The considered human recordings consist of a human grasping an object on a table, lifting it, and placing it back. We consider a single object per grasp type and extract the wrist and finger joint angles of the human when the object is at the highest position. Each grasp is identified with a leaf node of the taxonomy tree. Notice that no data was available for the three-fingers-sphere grasp type.

\begin{table}[h!]
	\caption{Description of the grasps extracted from the quantitative grasp taxonomy~\citep{Stival19:HumanGraspTaxonomy} used in our experiments.}
    \vspace{-0.3cm}
	\label{tab:DataGrasps}
 	\begin{center}
    \begin{small}
    \begin{sc}
	\resizebox{0.85\textwidth}{!}{
		\begin{tabular}{|c|c|c|c|c|}
			\rowcolor{lightgray}
			\hline
			\textbf{Category} & \textbf{Grasp type} & \textbf{Abbreviation} & \textbf{Grasped object} & \textbf{Number}\\ 
			\hline 
			\hline
			\multirow{5}*{Flat grasps} & Lateral & $\mathsf{La}$ & Padlock & 5 \\
			& Extension type & $\mathsf{ET}$ & Fruit bars & 5 \\
            & Quadpod & $\mathsf{Qu}$ & Lever red & 5 \\
            & Parallel extension & $\mathsf{PE}$ & Fruit bars & 5 \\
            & Index finger extension & $\mathsf{IE}$ & Knife & 5 \\
			\hline
			\multirow{4}*{Distal grasps} & Stick & $\mathsf{St}$ & Fizzy tablets & 5 \\
            & Writing tripod & $\mathsf{WT}$ & Syringe & 5 \\
            & Prismatic four fingers & $\mathsf{PF}$ & Mixing bowl & 5 \\
            & Power disk & $\mathsf{PD}$ & Dog & 5 \\
			\hline
            \multirow{4}*{Cylindrical grasps} & Large diameter & $\mathsf{LD}$ & Dwarf & 5 \\
            & Medium wrap & $\mathsf{MW}$ & Power tool & 5 \\
            & Small diameter & $\mathsf{SD}$ & Clamp & 5 \\
            & Fixed hook & $\mathsf{FH}$ & Fizzy tablets & 5 \\
			\hline
            \multirow{3}*{Spherical grasps} & Tripod & $\mathsf{Tr}$ & Padlock & 5 \\
            & Power sphere & $\mathsf{PS}$ & Dog & 5 \\
            & Precision sphere & $\mathsf{RS}$ & Ball & 5 \\
            \hline
            \multirow{4}*{Ring grasps} & Three fingers sphere & $\mathsf{TS}$ & - & 0 \\
            & Prismatic pinch & $\mathsf{PP}$ & Flower cup & 5 \\
            & Tip pinch & $\mathsf{TP}$ & Chopsticks & 4 \\
            & Ring & $\mathsf{Ri}$ & Cola bottle & 5 \\
            \hline
		\end{tabular}
	}
    \end{sc}
    \end{small}
    \end{center}
\end{table}

 \begin{SCfigure}[50][t]
	\centering
	\includegraphics[trim={0.cm 0.cm 0.cm 0.0cm},clip,width=0.45\textwidth]{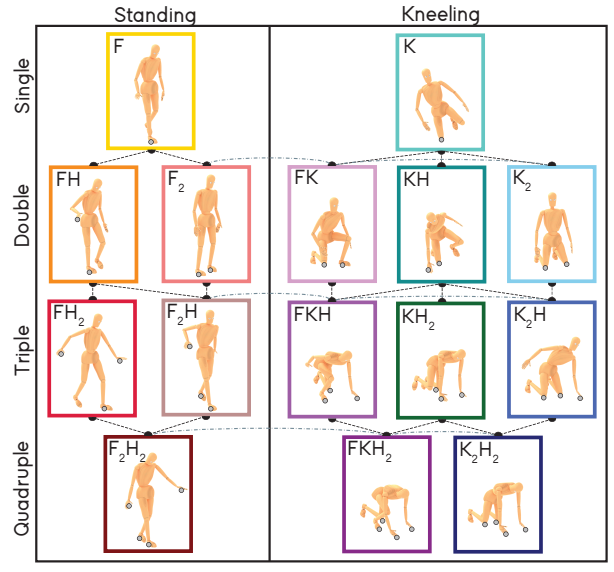}
	\caption{Subset of the whole-body support pose taxonomy~\citep{Borras17:WholeBodyTaxonomy} used in one of our experiments. Each node is a support pose defined by the type of contacts (foot $\mathsf{F}$, hand $\mathsf{H}$, knee $\mathsf{K}$). The lines represent graph transitions between the taxonomy nodes. Contacts are depicted by grey dots.}
	\label{fig:support-poses-taxonomy}
\end{SCfigure}

\subsubsection{Support poses}
Table~\ref{tab:DataSupportPoses} describes the data of the whole-body support pose taxonomy used in the experiments reported in \S~\ref{sec:experiments}.
Each pose is identified with a support pose category, i.e., a node of the graph in Fig.~\ref{fig:support-poses-taxonomy}, and with a set of associated contacts. 
As shown in the table, some support poses include several sets of contacts.
For example, the support pose $\mathsf{F}$ groups all types of support poses where only one foot is in contact with the environment.
In our experiments, we consider an augmented version of the taxonomy that explicitly distinguishes between left and right contacts.
 Notice that some sets of contacts are not represented in the data and thus do not appear in Table~\ref{tab:DataSupportPoses}.

\begin{table}[h!]
    \caption{Description of the support poses extracted from the whole-body support pose taxonomy~\citep{Borras17:WholeBodyTaxonomy} used in our experiments.}
    \vspace{-0.3cm}
	\label{tab:DataSupportPoses}
	\begin{center}
    \begin{small}
    \begin{sc}
	\resizebox{0.85\textwidth}{!}{
		\begin{tabular}{|c|c|c|c|}
			\rowcolor{lightgray}
			\hline
			\textbf{Support pose} & \textbf{Augmented support pose} & \textbf{Contacts} & \textbf{Number}\\ 
			\hline 
			\hline
			\multirow{2}*{$\mathsf{F}$} & $\mathsf{LF}$ & Left foot & 7 \\
			& $\mathsf{RF}$ & Right foot & 6 \\
			\hline
			\multirow{4}*{$\mathsf{FH}$} & $\mathsf{LF LH}$ & Left foot, left hand & 5 \\
			& $\mathsf{RF RH}$& Right foot, right hand & 6 \\
			& $\mathsf{LF RH}$& Left foot, right hand & 5 \\
			& $\mathsf{RFLH}$& Right foot, left hand & 6 \\
			\hline
			$\mathsf{F}_2$ & $\mathsf{F}_2$ & Left foot, right foot & 6 \\
			\hline
			\multirow{2}*{$\mathsf{FH}_2$} & $\mathsf{LF H}_2$ & Left foot, left hand, right hand & 6 \\
			& $\mathsf{RF H}_2$ & Right foot, left hand, right hand & 6 \\
			\hline
			\multirow{2}*{$\mathsf{F}_2\mathsf{H}$} & $\mathsf{F}_2\mathsf{H}^l$ & Left foot, right foot, left hand & 5 \\
			& $\mathsf{F}_2\mathsf{RH}$ & Left foot, right foot, right hand & 7 \\
			\hline
			$\mathsf{F}_2\mathsf{H}_2$ & $\mathsf{F}_2\mathsf{H}_2$ & Left foot, right foot, left hand, right hand & 7 \\
			\hline
			\hline
			\multirow{2}*{$\mathsf{K}$} & $\mathsf{LK}$ & Left knee & 1 \\
			& $\mathsf{RK}$& Right knee & 1 \\
			\hline
			\multirow{2}*{$\mathsf{FK}$} & $\mathsf{LF RK}$ & Left foot, right knee & 2 \\
			& $\mathsf{F LK}$ & Right foot, left knee & 3 \\
			\hline
			\multirow{2}*{$\mathsf{KH}$} & $\mathsf{LK LH}$ & Left knee, left hand & 4 \\
			& $\mathsf{RK RH}$ & Right knee, right hand & 1 \\
			\hline
			$\mathsf{K}_2$ & $\mathsf{K}_2$ & Left knee, right knee & 1 \\
			\hline
			\multirow{2}*{$\mathsf{FKH}$} & $\mathsf{RF LK LH}$ & Right foot, left knee, left hand & 5 \\
			& $\mathsf{LF RK RH}$& Left foot, right knee, right hand & 2 \\
			\hline
			$\mathsf{KH}_2$ & $\mathsf{LK H}_2$ & Left knee, left hand, right hand &  1 \\
			\hline
			\multirow{2}*{$\mathsf{K}_2\mathsf{H}$} & $\mathsf{K}_2\mathsf{LH}$ & Left knee, right knee, left hand & 2 \\
			& $\mathsf{K}_2\mathsf{RH}$ & Left knee, right knee, right hand & 1 \\
			\hline
			$\mathsf{FKH}_2$ & $\mathsf{RF LK H}_2$ & Right foot, left knee, left hand, right hand & 2 \\
			\hline
			$\mathsf{K}_2\mathsf{H}_2$ & $\mathsf{K}_2\mathsf{H}_2$ & Left knee, right knee, left hand, right hand & 2 \\
			\hline
		\end{tabular}
	}
    \end{sc}
    \end{small}
    \end{center}
\end{table}

\subsection{Implementation details}
\label{app:experiment_training_details}
\subsubsection{Training parameters and model choices} 
Table~\ref{tab:experiment_hyperparameters} reports the hyperparameters used for the experiments described in \S~\ref{sec:experiments}. 
We used the hyperbolic SE kernels defined in \S~\ref{subsec:HyperbolicKernels} for the GPHLVMs, and the classical SE kernel for the Euclidean models. 
For the back-constraints mapping~\eqref{eq:backconstraints}, we defined $k^{\mathbb{R}^D}(\bm{y}_n,\bm{y}_m)$ as a Euclidean SE kernel with lengthscale $\kappa_{\euclideanspace^D}$, and $k^{\mathbb{G}}(c_n, c_m)$ as a graph Matérn kernel with smoothness $\nu=2.5$ and lengthscale $\kappa_{\mathbb{G}}$. 
We additionally scaled the product of kernels with a variance $\sigma_{\euclideanspace^D, \mathbb{G}}$.
For training the back-constrained GPHLVM and GPLVM, we used a Gamma prior $\text{Gamma}(\alpha, \beta)$ with shape $\alpha$ and rate $\beta$ on the lengthscale $\kappa$ of the kernels.

\begin{table}[h!]
 	\caption{Summary of experiments and list of hyperparameters.}
    \vspace{-0.3cm}
	\label{tab:experiment_hyperparameters}
   	\begin{center}
    \begin{small}
    \begin{sc}
	\resizebox{\textwidth}{!}{
		\begin{tabular}{p{2.8cm}llclcccl}
			\toprule
			\textbf{Taxonomy} & \textbf{Model} & \textbf{Regularization} & \textbf{Loss scale} $\gamma$ & \textbf{Prior on} $\kappa_{\lorentz/\euclideanspace^Q}$ & $\kappa_{\euclideanspace^D}$ & $\kappa_{\mathbb{G}}$ & $\sigma_{\euclideanspace^D, \mathbb{G}}$ & \textbf{Optimizer} (learning rate $\alpha$)\\
            \toprule
            \multirow{12}{*}{\parbox{2.2cm}{Bimanual manipulation categories}} & \multirow{3}{*}{GPLVM, $\mathbb{R}^2$} & No regularizer & 0 & None & - & - & - & \multirow{3}{*}{Adam (0.01)}\\
			&  & Stress & $1500$ & None & - & - & -\\
			&  & BC + Stress & $1000$ & $\text{Gamma}(2, 2)$ & $3.0$ & $1.5$ & $2$ \\
			& \multirow{3}{*}{GPHLVM, $\lorentz{2}$} & No regularizer & $0$ & None & - & - & - & \multirow{3}{*}{Riemannian Adam (0.025)}\\
			&  & Stress & $1500$ & None & - & - & -\\
			&  & BC + Stress & $1000$ & $\text{Gamma}(2, 2)$ & $3.0$ & $1.5$ & $2$ \\
            \cmidrule(lr){2-9}
            & \multirow{3}{*}{GPLVM, $\mathbb{R}^3$} & No regularizer & 0 & None & - & - & - & \multirow{3}{*}{Adam (0.01)}\\
			&  & Stress & $6000$ & None & - & - & - \\
			&  & BC + Stress & $1200$ & $\text{Gamma}(2, 2)$ & $3.0$ & $1.5$ & $2$ \\
			& \multirow{3}{*}{GPHLVM, $\lorentz{3}$} & No regularizer & $0$ & None & - & - & - & \multirow{3}{*}{Riemannian Adam (0.025)}\\
			&  & Stress & $6000$ & None & - & - & - \\
			&  & BC + Stress & $1200$ & $\text{Gamma}(2, 2)$ & $3.0$ & $1.5$ & $2$ \\
			\toprule
			\multirow{12}{*}{\parbox{3.5cm}{Grasps}} & \multirow{3}{*}{GPLVM, $\mathbb{R}^2$} & No regularizer & 0 & None & - & - & - & \multirow{3}{*}{Adam (0.01)}\\
			&  & Stress & $5500$ & None & - & - & - \\
			&  & BC + Stress & $2000$ & $\text{Gamma}(2, 2)$ & $1.8$ & $1.5$ & $2$ \\
			& \multirow{3}{*}{GPHLVM, $\lorentz{2}$} & No regularizer & $0$ & None & - & - & - & \multirow{3}{*}{Riemannian Adam (0.05)}\\
			&  & Stress & $5500$ & None & - & - & - \\
			&  & BC + Stress & $2000$ & $\text{Gamma}(2, 2)$ & $1.8$ & $1.5$ & $2$ \\
            \cmidrule(lr){2-9}
            & \multirow{3}{*}{GPLVM, $\mathbb{R}^3$} & No regularizer & 0 & None & - & - & - & \multirow{3}{*}{Adam (0.01)}\\
			&  & Stress & $6000$ & None & - & - & - \\
			&  & BC + Stress & $3000$ & $\text{Gamma}(2, 2)$ & $1.8$ & $1.5$ & $2$ \\
			& \multirow{3}{*}{GPHLVM, $\lorentz{3}$} & No regularizer & $0$ & None & - & - & - & \multirow{3}{*}{Riemannian Adam (0.05)}\\
			&  & Stress & $6000$ & None & - & - & - \\
			&  & BC + Stress & $3000$ & $\text{Gamma}(2, 2)$ & $1.8$ & $1.5$ & $2$ \\
			\toprule
            \multirow{12}{*}{\parbox{3.5cm}{Support poses}} & \multirow{3}{*}{GPLVM, $\mathbb{R}^2$} & No regularizer & 0 & None & - & - & - & \multirow{3}{*}{Adam (0.01)}\\
			&  & Stress & $7000$ & None & - & - & -\\
			&  & BC + Stress & $5000$ & $\text{Gamma}(2, 2)$ & $2.0$ & $0.8$ & $2$ \\
			& \multirow{3}{*}{GPHLVM, $\lorentz{2}$} & No regularizer & $0$ & None & - & - & - & \multirow{3}{*}{Riemannian Adam (0.05)}\\
			&  & Stress & $7000$ & None & - & - & -\\
			&  & BC + Stress & $5000$ & $\text{Gamma}(2, 2)$ & $2.0$ & $0.8$ & $2$ \\
            \cmidrule(lr){2-9}
            & \multirow{3}{*}{GPLVM, $\mathbb{R}^3$} & No regularizer & 0 & None & - & - & - & \multirow{3}{*}{Adam (0.01)}\\
			&  & Stress & $10000$ & None & - & - & - \\
			&  & BC + Stress & $8000$ & $\text{Gamma}(2, 2)$ & $2.0$ & $0.8$ & $2$ \\
			& \multirow{3}{*}{GPHLVM, $\lorentz{3}$} & No regularizer & $0$ & None & - & - & - & \multirow{3}{*}{Riemannian Adam (0.05)}\\
			&  & Stress & $10000$ & None & - & - & - \\
			&  & BC + Stress & $8000$ & $\text{Gamma}(2, 2)$ & $2.0$ & $0.8$ & $2$ \\
			\bottomrule
		\end{tabular}
	}
     \end{sc}
    \end{small}
    \end{center}
\end{table}

\subsubsection{Model initialization}
\label{app:initialization}
To provide a good starting point for their optimization, the embeddings of all GPLVMs were initialized by minimizing the stress associated with their taxonomy nodes, so that,
\begin{equation}
    \bm{X} = \min_{\bm{X}} \ell_{\text{stress}}, 
\end{equation}
with $\ell_{\text{stress}}$ as in~\eqref{eq:stressLoss}, using the Euclidean and hyperbolic distance between two embeddings for the GPLVMs and GPHLVM, respectively. The oracle stress possible for each system, achieved by the initialization, is reported in Table~\ref{table:oracle_stress_of_models}.

\begin{table}[h!]
    \caption{Oracle stress achieved by the initialization per geometry and regularization. The stress is computed using~\eqref{eq:stressLoss} and averaged over all pairs of training embeddings. Lower stress values indicate better compliance with the taxonomy structure.}
    \vspace{-0.4cm}
    \label{table:oracle_stress_of_models}
    \begin{center}
    \begin{small}
    \begin{sc}
    \resizebox{.42\textwidth}{!}{
		\begin{tabular}{lll}
			\toprule
			\textbf{Taxonomy} & \textbf{Model} & \textbf{Oracle stress} \\
            \toprule
			\multirow{4}{*}{\parbox{2.1cm}{Bimanual manipulation categories}} & GPLVM, $\mathbb{R}^2$ &  $0.034\pm0.044$ \\
			& GPHLVM, $\lorentz{2}$ &  $\bm{0.018\pm0.022}$ \\
            \cmidrule(lr){2-3}
			& GPLVM, $ \mathbb{R}^3$ &  $0.007\pm0.010$ \\
            & GPHLVM, $ \lorentz{3}$ & $\bm{0.002\pm0.004}$  \\
            \toprule
			\multirow{4}{*}{Grasps} & GPLVM, $\mathbb{R}^2$ &  $0.38\pm 0.40$  \\
			& GPHLVM, $ \lorentz{2}$ &  $\bm{0.13\pm 0.14}$  \\
            \cmidrule(lr){2-3}
			& GPLVM, $ \mathbb{R}^3$ &  $0.13\pm 0.16$  \\
            & GPHLVM, $ \lorentz{3}$ &  $\bm{0.03\pm0.04}$ \\
			\toprule
			\multirow{4}{*}{\parbox{1.5cm}{Support poses}} & GPLVM, $\mathbb{R}^2$ &  $0.56\pm 0.96$  \\
			& GPHLVM, $ \lorentz{2}$ &  $\bm{0.49\pm 0.82}$  \\
            \cmidrule(lr){2-3}
			& GPLVM, $ \mathbb{R}^3$ &  $\bm{0.23\pm 0.45}$  \\
            & GPHLVM, $ \lorentz{3}$ &  $0.29\pm 0.39$ \\
			\bottomrule
	\end{tabular}
 }
 \end{sc}
 \end{small}
 \end{center}
\end{table}

\subsubsection{Influence of the stress loss scale $\gamma$}
\label{app:LossScale}
For our experiments, we trained the GPLVMs and GPHLVMs via MAP estimation by maximizing the loss $\ell = \ell_{\text{MAP}} - \gamma \ell_{\text{stress}}$, where $\gamma$ is a parameter trading-off between the log posterior loss $\ell_{\text{MAP}}$ and the stress-based regularization loss $\ell_{\text{stress}}$. The influence of the loss scale $\gamma$ is illustrated in Fig.~\ref{fig:LossScaleAnalysis} for models trained on the hand grasp taxonomy with $3$-dimensional latent spaces. On one hand, we observe that the stress loss steadily decreases as $\gamma$ increases. This trend continues until the embeddings for each node collapse onto a single point in the latent space, achieving the stress that matches the oracle value. On the other hand, the log-likelihood of the model decreases (a.k.a the negative log-likelihood increases) as $\gamma$ increases. For all our experiments, we chose a loss scale $\gamma$ that trades off between log posterior and stress losses, as depicted by the vertical line in Fig.~\ref{fig:LossScaleAnalysis}.

\begin{SCfigure}[50][t]
	\centering
    \includegraphics[trim={0.0cm 0.0cm 0.0cm 0.0cm},clip,width=.55\textwidth]{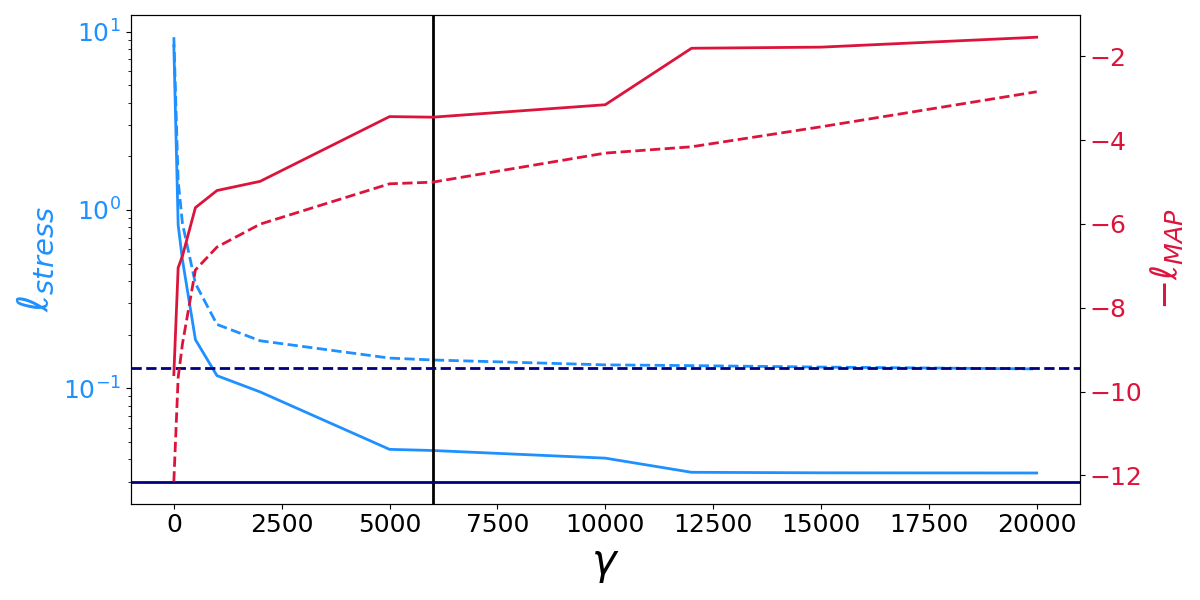}
	\caption{Log-posterior and stress losses for GPHLVMs (\crimsonline, \dodgerblueline) and GPLVMs (\crimsondashedline, \dodgerbluedashedline) with $Q=3$ with stress regularization as a function of the loss scale $\gamma$. The models are trained on the hand grasp taxonomy. When $\gamma$ increases, the models match the oracle stress values (\navyline, \navydashedline). The loss scale $\gamma$ chosen for our experiments ($\;$\vertblackline $\;$) trades off between the two losses.}
	\label{fig:LossScaleAnalysis}
	\vspace{-0.3cm}
\end{SCfigure}

\subsubsection{Taxonomy expansion and unseen poses encoding} 
\label{app:UnseenClassesAndPoses}
For the first part of the experiments on taxonomy expansion, we encoded unseen poses of each class for the back-constrained GPLVM and GPHLVM with a stress regularization using the models presented in Table~\ref{tab:experiment_hyperparameters}. 
For the second part of the experiments, we left one or several classes out during training and we ``embedded'' them using the back-constraints mapping. The left-out classes are: tightly-coupled asymmetric right dominant ($\mathsf{TCA}_{\text{right}}$), $\{\mathsf{Qu}, \mathsf{St}, \mathsf{MW}, \mathsf{Ri}\}$, and $\mathsf{FH}=\{\mathsf{LF LH}, \mathsf{RF LH}, \mathsf{LF RH}, \mathsf{RF RH}\}$, for the bimanual manipulation, hand grasp, and support pose taxonomies, respectively.
The newly-trained models also followed the same hyperparameters presented in Table~\ref{tab:experiment_hyperparameters}.

\subsection{Runtimes}
\label{app:runtimes}
In order to show the computational cost of our approach, we ran a set of experiments to measure the average runtime for the training and decoding phases, using $2$ and $3$-dimensional latent spaces.
As a reference, we added the runtime measurements of Euclidean counterpart, that is, the vanilla GPLVM. 
Table~\ref{tab:ApproachesRuntime} shows the runtime measurements. 
Note that the main computational burden arises in our GPLHVM with a $2$-dimensional latent space, which is in sharp contrast with the experiments using a $3$-dimensional latent space. 
This increase in computational cost is mainly attributed to the $2$-dimensional hyperbolic kernel.
Nevertheless, we also measured the computational cost of evaluating the kernel and the (Riemannian) optimization of the learned embeddings for both GPLVM and GPHLVM in the $2$-dimensional setting. 
Table~\ref{tab:app:kernel_opt_times} shows the average runtimes for both approaches, where it is possible to observe that the highest computational costs comes from the hyperbolic kernel computation.
This may be alleviated by reducing the number of samples or via more efficient sampling strategies.

\begin{table}[h!]
    \caption{Average runtime for kernel evaluation and (Riemannian) optimization of our GPHLVM and vanilla GPLVM over $10$ training iterations of the whole-body support poses taxonomy, using a $2$-dimensional latent space for both models. The implementations are fully developed on Python, and the runtime measurements were taken using a standard laptop with $32$ GB RAM, Intel Xeon CPU E3-1505M v6 processor, and Ubuntu 20.04 LTS. We report the computational cost in miliseconds and percentage w.r.t the total training iteration time.}
	\label{tab:app:kernel_opt_times}
	\begin{center}
    \begin{small}
    \begin{sc}
	\begin{tabular}{lllll}
        \toprule
		\textbf{Model} & \textbf{Kernel comp.} [\si{\milli \second}] & \textbf{Kernel comp.} \% & \textbf{Optimization} [\si{\milli \second}] & \textbf{Optimization} \% \\ [0.3ex]
		\toprule 
		  GPLVM, $\mathbb{R}^2$ & $0.043 \pm 0.009$ & $1\%$ & $0.43 \pm 0.063$ & $10\%$\\ [0.3ex] 
		\midrule
		GPHLVM, $\lorentz{2}$ & $730.69 \pm 75.22$ & $32\%$ & $1.00 \pm 0.15$ & 0.05\% \\ [0.3ex]
		\bottomrule 
	\end{tabular}
    \end{sc}
    \end{small}
    \end{center}
\end{table}

\subsection{Hyperbolic embeddings of support poses}
\label{app:bimanual-taxonomy}
Fig.~\ref{fig:GPHLVM:vanilla}-\ref{fig:GPHLVM:backconstrained_and_stress} show the learned embeddings of the support pose manipulation taxonomy alongside error matrices depicting the difference between geodesic and taxonomy graph distances.
As discussed in \S~\ref{sec:experiments}, the models with stress prior result in embeddings that comply with the taxonomy graph structure, with additional intra-class organizations for the back-constrained models. Note that augmenting the support pose taxonomy leads to several groups of the same support pose in Figs.~\ref{fig:GPHLVM:stress_prior}-\ref{fig:GPHLVM:backconstrained_and_stress}, e.g., $\mathsf{F}$ splits into $\mathsf{LF}$ and $\mathsf{RF}$. It is worth noticing that, despite the cyclic graph structure of the support pose taxonomy, the hyperbolic models outperform the Euclidean models in the $2$-dimensional case
%It is worth noticing that, despite the fact that the bimanual taxonomy graph is smaller than the support pose and grasp taxonomy graphs, all Euclidean GPLVMs remain outperformed by the hyperbolic models, which most closely match the taxonomy structure (see also Table~\ref{table:mean_stress_of_models}). 
As reported in \S~\ref{sec:experiments}, the back-constrained GPHLVM and GPLVM allow us to properly place unseen poses or taxonomy classes into the latent
space (see Figs.~\ref{fig:GPHLVM-bimanual:added_poses}-\ref{fig:GPHLVM-bimanual:added_class}).

\begin{figure}
	\vspace{-0.1cm}
	\centering
	\includegraphics[trim={5.3cm 2.2cm 4.3cm 2.2cm},clip,width=.8\textwidth]{Figures/legend_semifull.png}
	\begin{subfigure}[b]{0.15\textwidth}
		\centering
		\includegraphics[trim={2.5cm 2.5cm 2.5cm 2.5cm},clip,width=\textwidth]{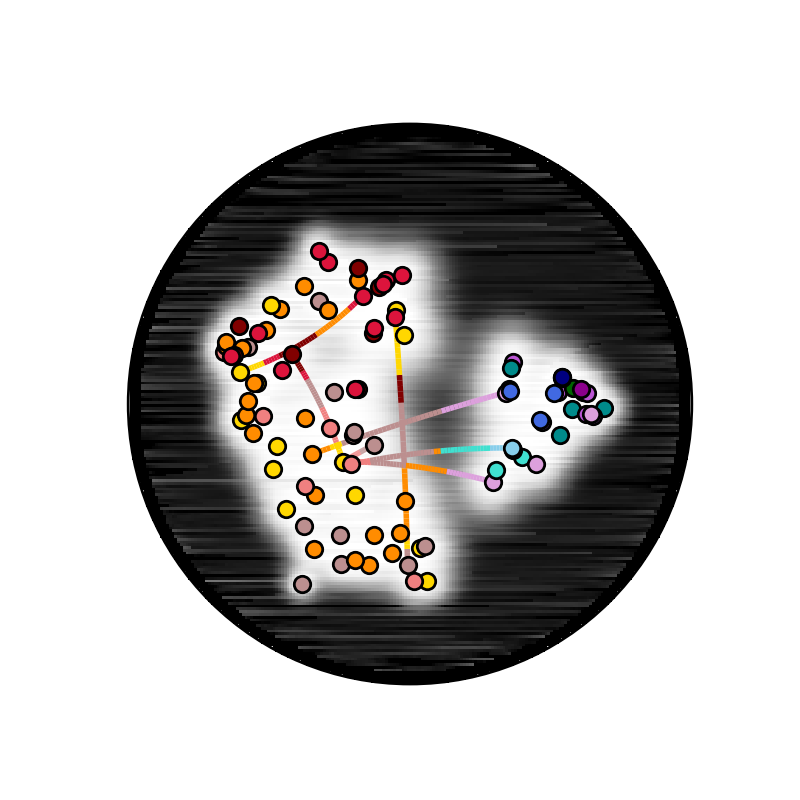}
		\includegraphics[trim={2.0cm 2.0cm 0.5cm 2.0cm},clip,width=0.9\textwidth]{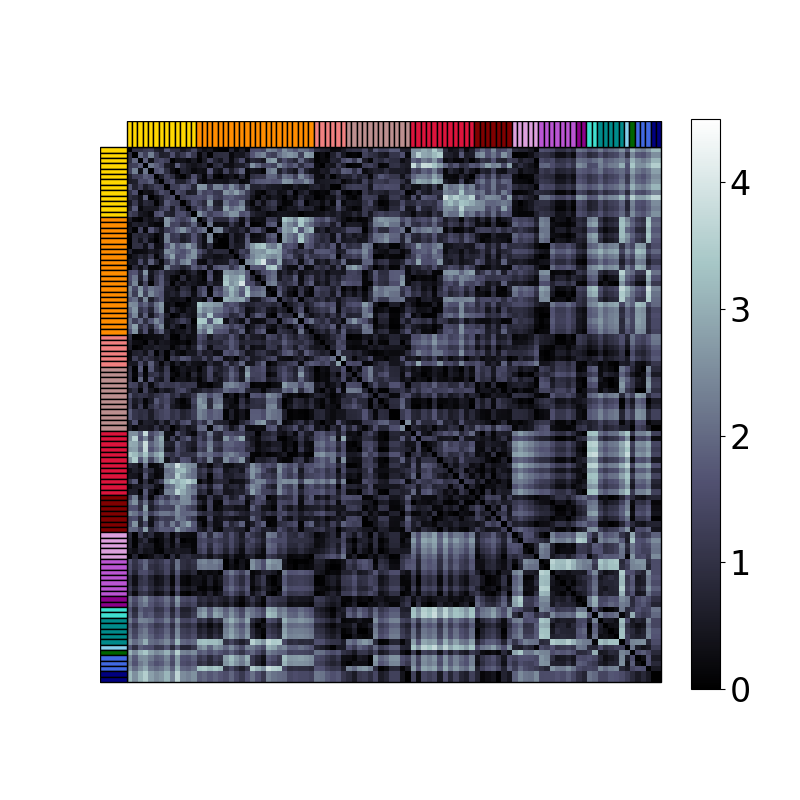}
		\includegraphics[trim={2.0cm 2.0cm 2.0cm 2.0cm},clip,width=0.9\textwidth]{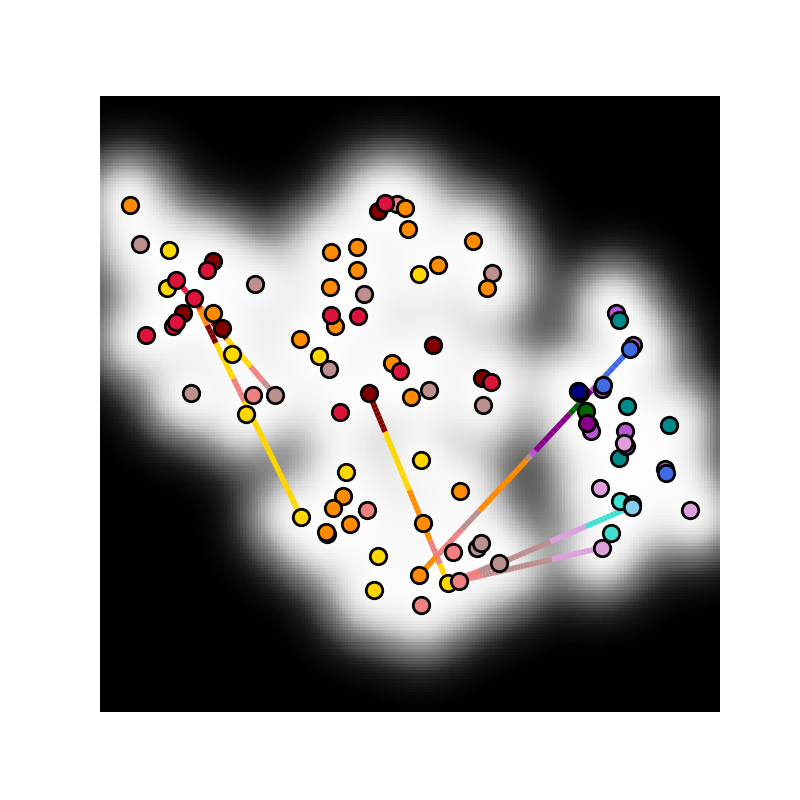}
		\includegraphics[trim={2.0cm 2.0cm 0.5cm 2.0cm},clip,width=0.9\textwidth]{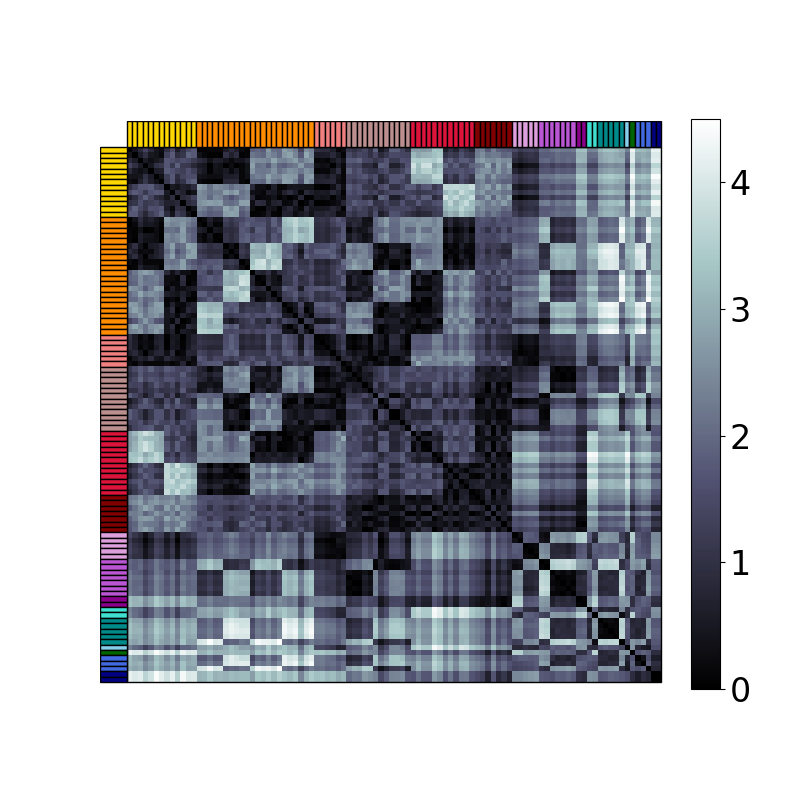}
		\caption{Vanilla}
		\label{fig:GPHLVM:vanilla}
	\end{subfigure}%
	\begin{subfigure}[b]{0.15\textwidth}
		\centering
		\includegraphics[trim={2.5cm 2.5cm 2.5cm 2.5cm},clip,width=\textwidth]{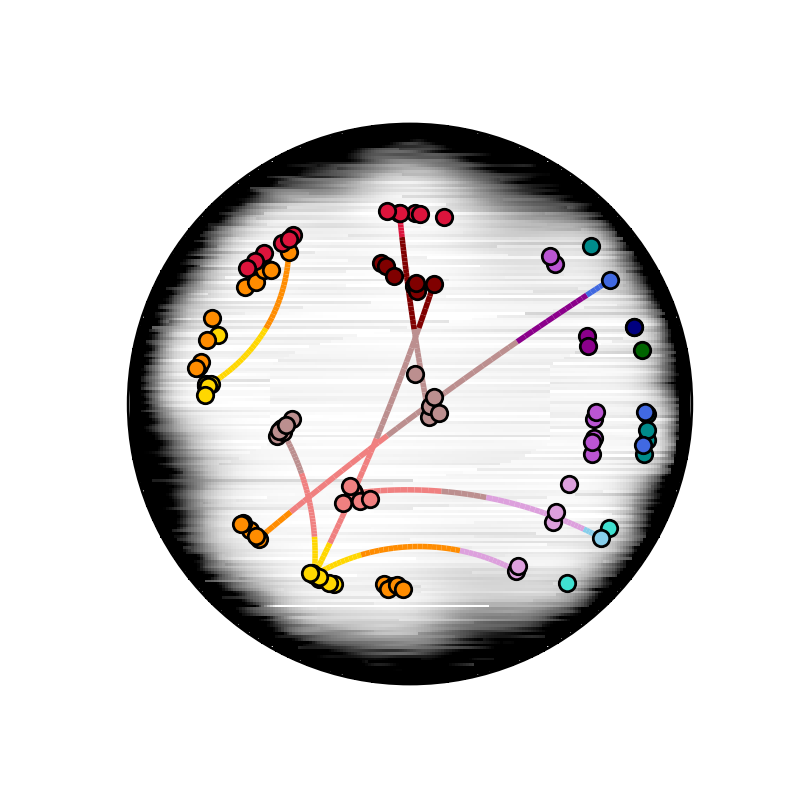}
		\includegraphics[trim={2.0cm 2.0cm 0.5cm 2.0cm},clip,width=0.9\textwidth]{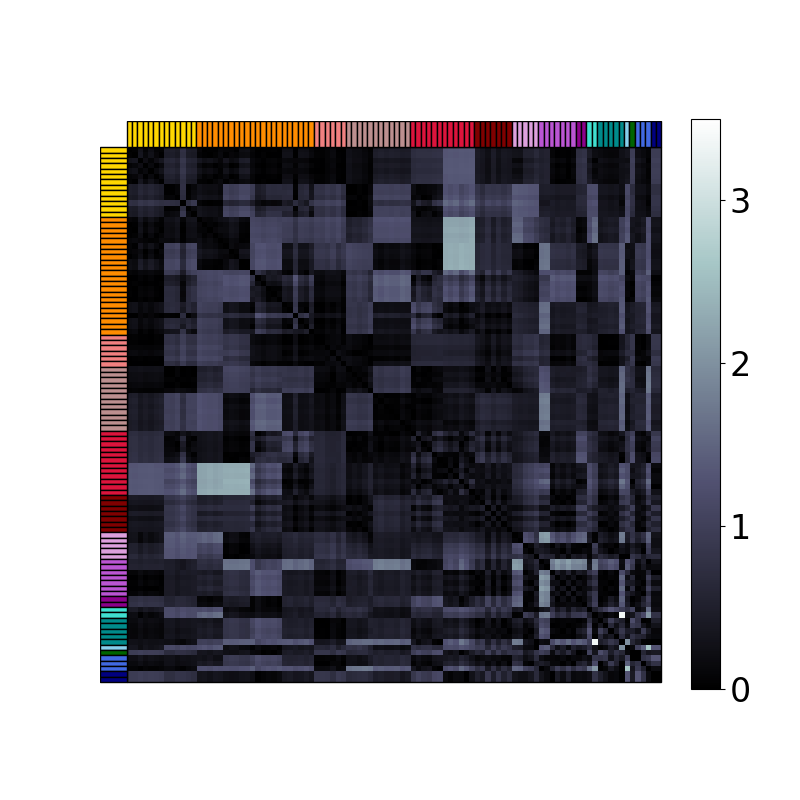}
		\includegraphics[trim={2.0cm 2.0cm 2.0cm 2.0cm},clip,width=0.9\textwidth]{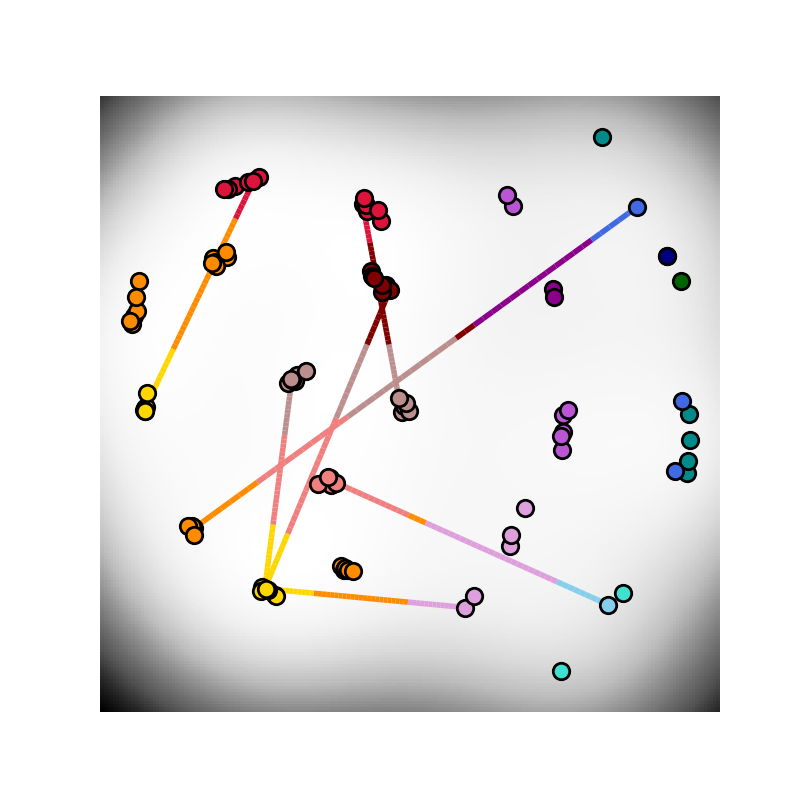}
		\includegraphics[trim={2.0cm 2.0cm 0.5cm 2.0cm},clip,width=0.9\textwidth]{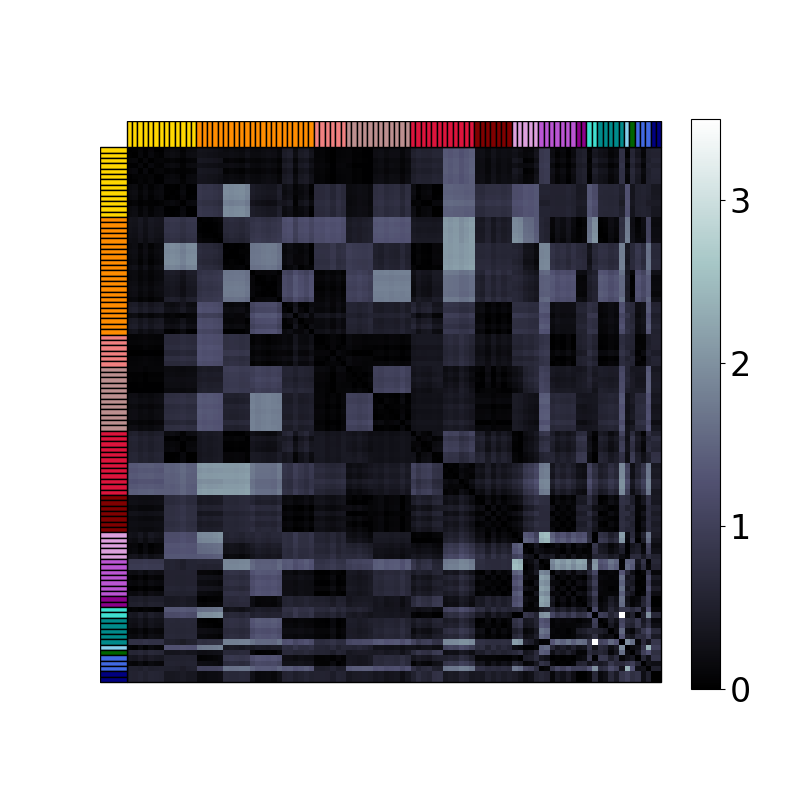}
		\caption{Stress prior}
		\label{fig:GPHLVM:stress_prior}
	\end{subfigure}%
	\begin{subfigure}[b]{0.15\textwidth}
		\centering
		\includegraphics[trim={2.5cm 2.5cm 2.5cm 2.5cm},clip,width=\textwidth]{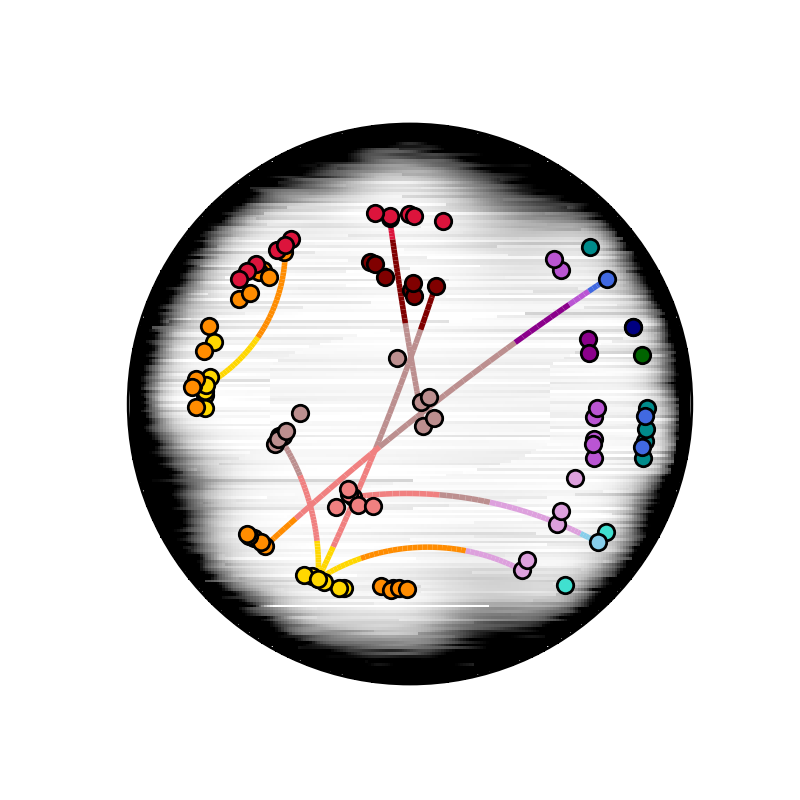}
		\includegraphics[trim={2.0cm 2.0cm 0.5cm 2.0cm},clip,width=0.9\textwidth]{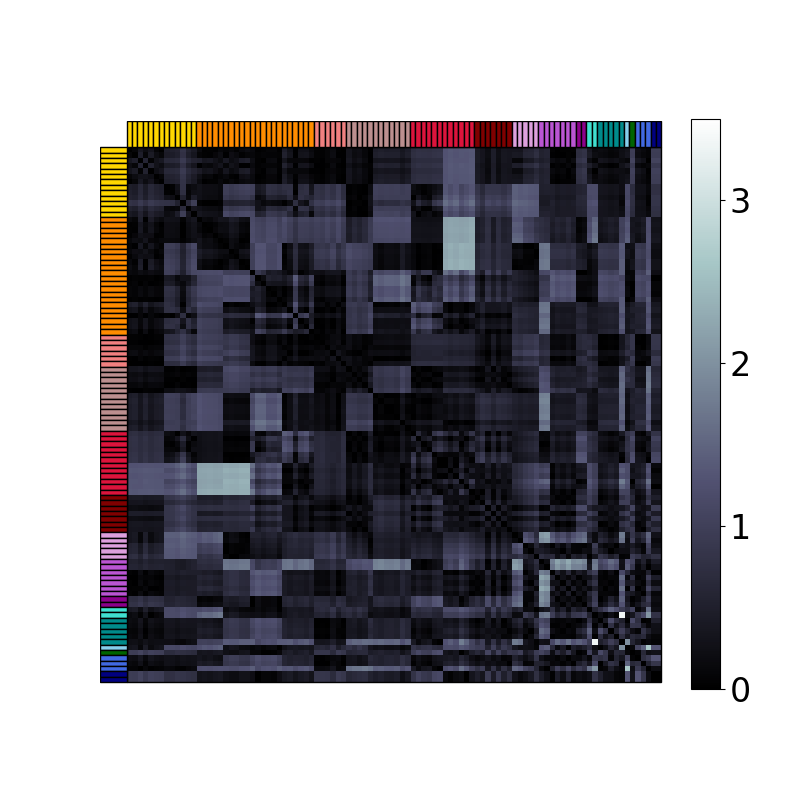}
		\includegraphics[trim={2.0cm 2.0cm 2.0cm 2.0cm},clip,width=0.9\textwidth]{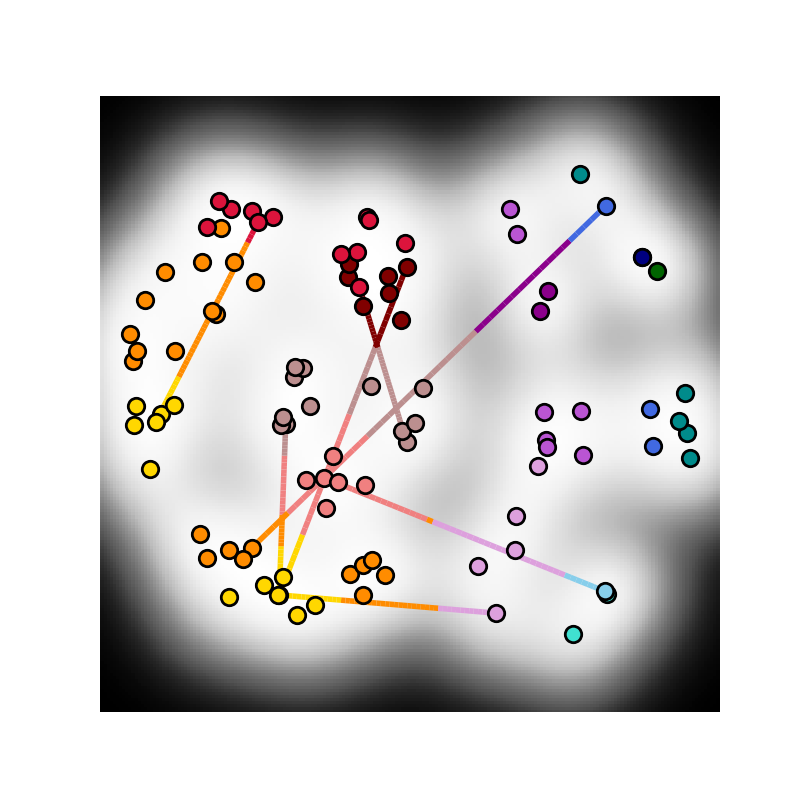}
		\includegraphics[trim={2.0cm 2.0cm 0.5cm 2.0cm},clip,width=0.9\textwidth]{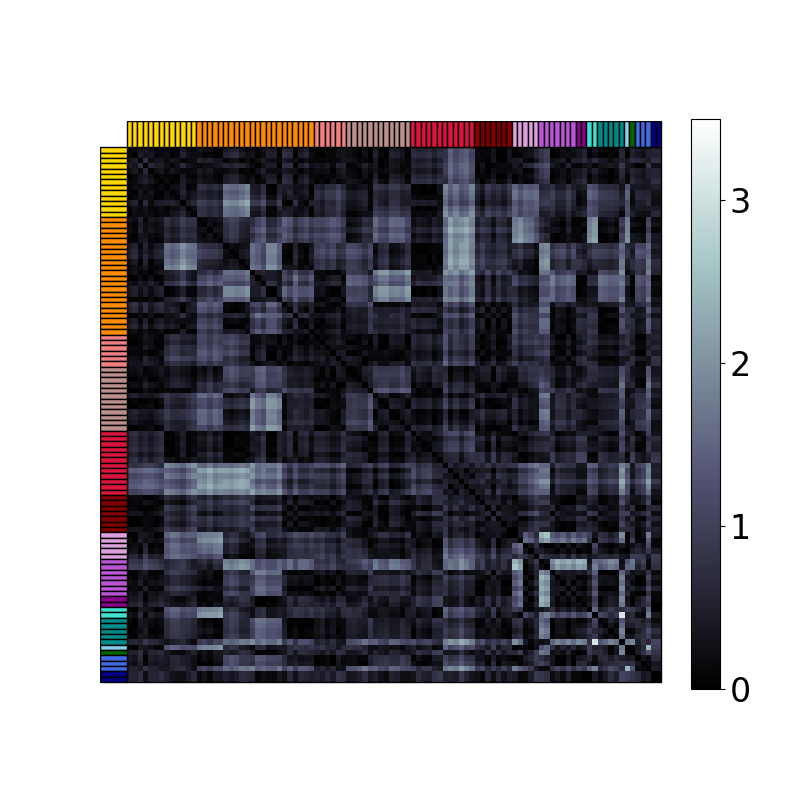}
		\caption{BC + stress prior}
		\label{fig:GPHLVM:backconstrained_and_stress}
	\end{subfigure}%
	\begin{subfigure}[b]{0.15\textwidth}
		\centering
		\includegraphics[trim={2.5cm 2.5cm 2.5cm 2.5cm},clip,width=\textwidth]{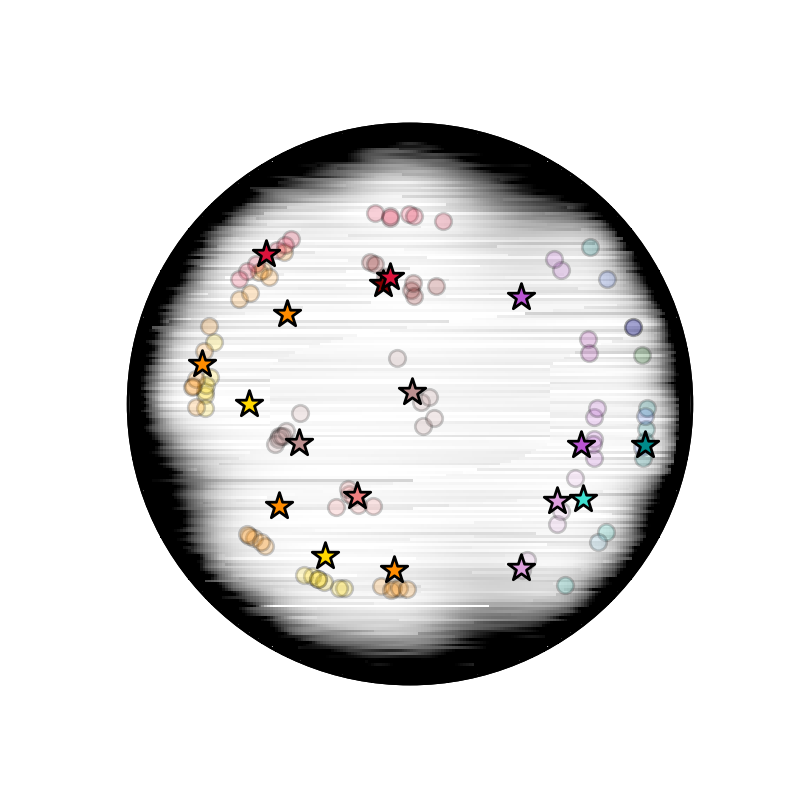}
		\includegraphics[trim={2.0cm 2.0cm 0.5cm 2.0cm},clip,width=0.9\textwidth]{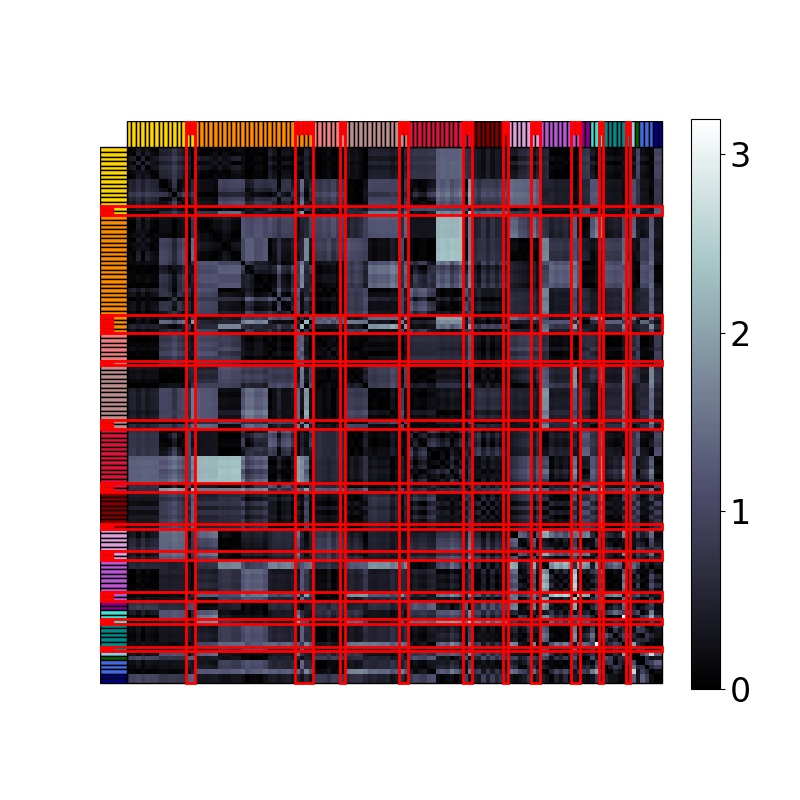}
		\includegraphics[trim={2.0cm 2.0cm 2.0cm 2.0cm},clip,width=0.9\textwidth]{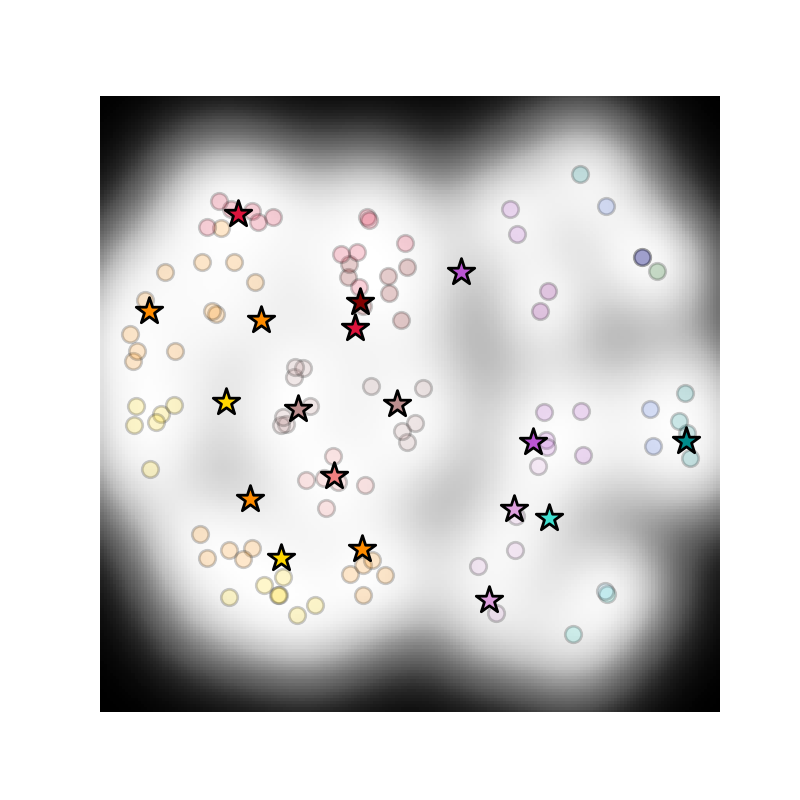}
		\includegraphics[trim={2.0cm 2.0cm 0.5cm 2.0cm},clip,width=0.9\textwidth]{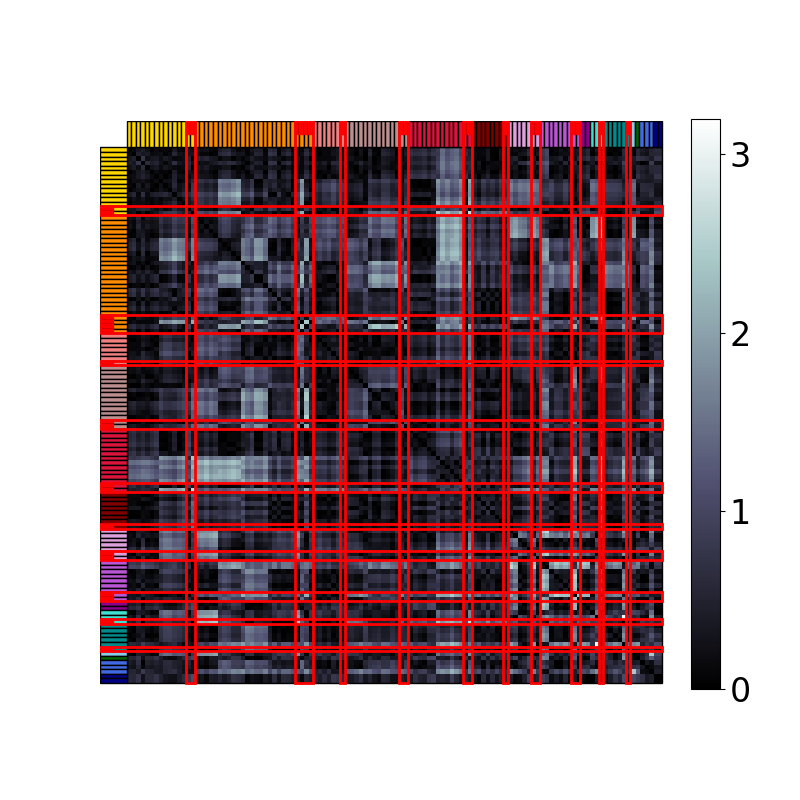}
		\caption{Adding poses}
		\label{fig:GPHLVM:added_poses}
	\end{subfigure}%
	\begin{subfigure}[b]{0.15\textwidth}
		\centering
		\includegraphics[trim={2.5cm 2.5cm 2.5cm 2.5cm},clip,width=\textwidth]{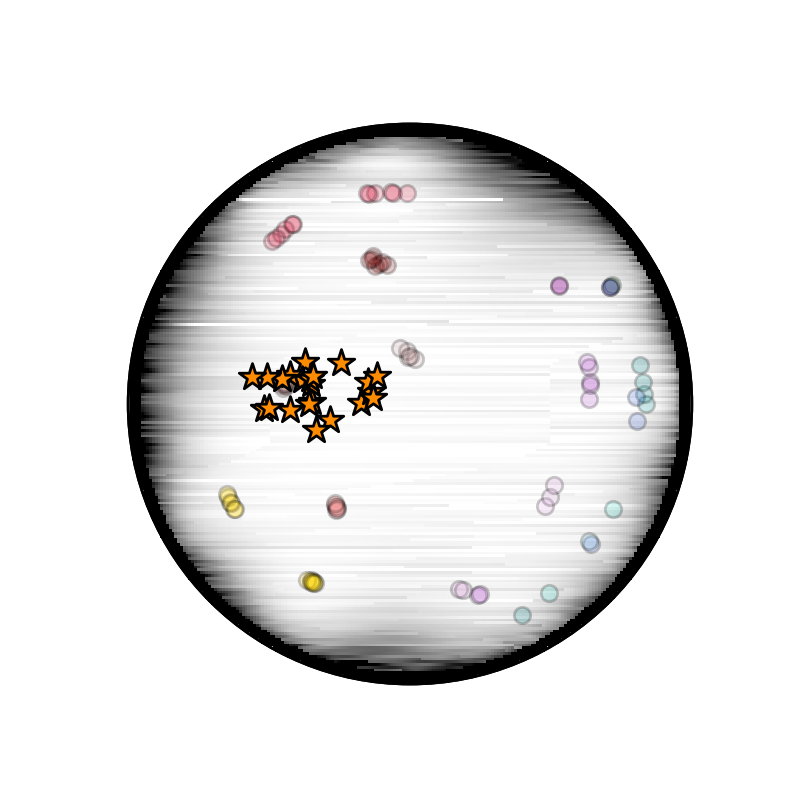}
		\includegraphics[trim={2.0cm 2.0cm 0.5cm 2.0cm},clip,width=0.9\textwidth]{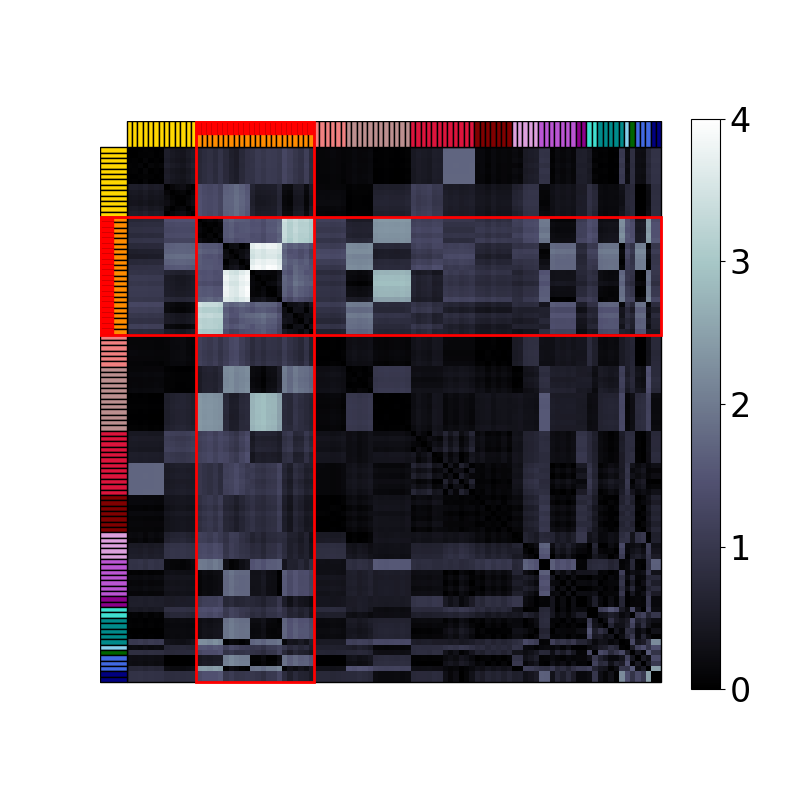}
		\includegraphics[trim={2.0cm 2.0cm 2.0cm 2.0cm},clip,width=0.9\textwidth]{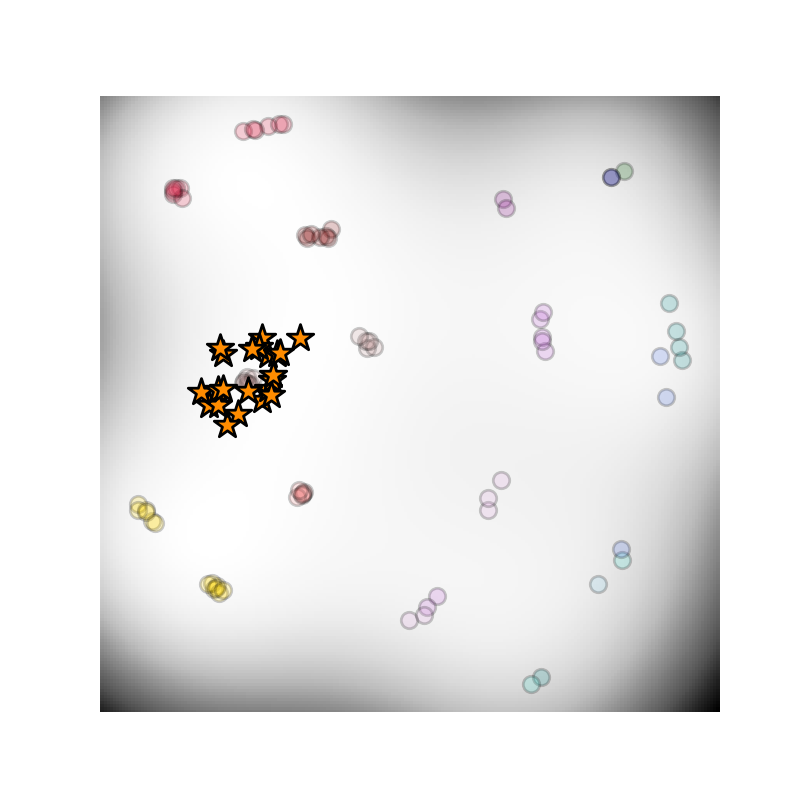}
		\includegraphics[trim={2.0cm 2.0cm 0.5cm 2.0cm},clip,width=0.9\textwidth]{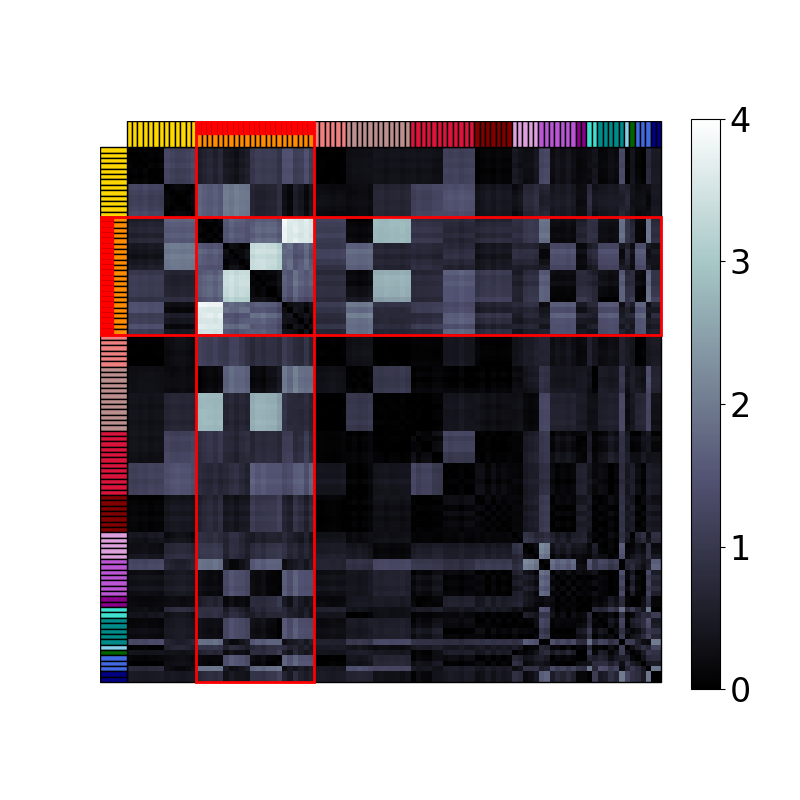}
		\caption{Adding a class}
		\label{fig:GPHLVM:added_class}
	\end{subfigure}
    \vspace{-0.25cm}
	\caption{Support poses: The first and last two rows show the latent embeddings and examples of interpolating geodesics in $\mathcal{P}^2$ and $\mathbb{R}^2$, followed by pairwise error matrices between geodesic and taxonomy graph distances. Embeddings colors match those of Fig.~\ref{fig:support-poses-taxonomy}, and background colors indicate the GPLVM uncertainty. Added poses \emph{(d)} and classes $\mathsf{FH}=\mathsf{F^{\{l,r\}}H^{\{l,r\}}}$ \emph{(e)} are marked with stars and highlighted with red in the error matrices.}
	\label{fig:GPHLVM_trained_models}
	\vspace{-0.3cm}
\end{figure}

\subsection{Hyperbolic embeddings in $\lorentz{3}$}
\label{appendix:3d-embeddings}

In this section, we embed the taxonomy data of the three taxonomies used in \S~\ref{sec:experiments} into $3$-dimensional hyperbolic and Euclidean spaces to analyze the performance of the proposed models in higher-dimensional latent spaces. 
We test the GPHLVM and GPLVM without regularization, with stress prior, and with back-constraints coupled with stress prior, similarly to the experiments on $2$-dimensional latent spaces reported in \S~\ref{sec:experiments} and App.~\ref{app:bimanual-taxonomy}.
Figs.~\ref{fig:GPHLVM:vanilla_3d-bimanual}-\ref{fig:GPHLVM:backconstrained_and_stress_3d-bimanual}, Figs.~\ref{fig:GPHLVM-grasps:vanilla_3d}-\ref{fig:GPHLVM-grasps:backconstrained_and_stress_3d}, and Figs.~\ref{fig:GPHLVM:vanilla_3d}-\ref{fig:GPHLVM:backconstrained_and_stress_3d} show the learned embeddings alongside the corresponding error matrices for the bimanual manipulation taxonomy, the hand grasps taxonomy, and the whole-body support pose taxonomy, respectively. 
As expected, and similarly to the $2$-dimensional embeddings, the models without regularization do not encode any meaningful distance structure in the latent spaces (see Figs.~\ref{fig:GPHLVM:vanilla_3d-bimanual},~\ref{fig:GPHLVM-grasps:vanilla_3d},~\ref{fig:GPHLVM:vanilla_3d}).
In contrast, the models with stress prior result in embeddings that comply with the taxonomy graph structure, and the back constraints further organize the embeddings inside a class according to the similarity between their observations (see Figs.~\ref{fig:GPHLVM:stress_prior_3d-bimanual}-\ref{fig:GPHLVM:backconstrained_and_stress_3d-bimanual},~\ref{fig:GPHLVM-grasps:stress_prior_3d}-\ref{fig:GPHLVM-grasps:backconstrained_and_stress_3d},~\ref{fig:GPHLVM:stress_prior_3d}-\ref{fig:GPHLVM:backconstrained_and_stress_3d}). As discussed in \S~\ref{sec:experiments}, we generally observe a prominent stress reduction for the Euclidean and hyperbolic $3$-dimensional latent spaces compared to the $2$-dimensional ones (see Table~\ref{table:mean_stress_of_models}). 
For taxonomies with a tree structure, such as the bimanual manipulation and hand grasps taxonomy, all Euclidean models are still outperformed by the $3$-dimensional hyperbolic embeddings. This is due to the fact that hyperbolic spaces are ideal to embed such purely-hierarchical taxonomies. For taxonomies with cyclic structure, such as the support pose taxonomy, the Euclidean models with $3$-dimensional latent space slightly outperform the $3$-dimensional hyperbolic embeddings. 
Moreover, similarly to the $2$-dimensional cases, the back-constrained GPHLVM and GPLVM allow us to properly place unseen poses or taxonomy classes into the latent space (see Figs.~\ref{fig:GPHLVM:added_poses_3d-bimanual}-\ref{fig:GPHLVM:added_class_3d-bimanual},~\ref{fig:GPHLVM-grasps:added_poses_3d}-\ref{fig:GPHLVM-grasps:added_class_3d},~\ref{fig:GPHLVM:added_poses_3d}-\ref{fig:GPHLVM:added_class_3d}).

\begin{figure}
	\centering
	\includegraphics[trim={2.8cm 1.2cm 2.0cm 1.5cm},clip,width=0.8\textwidth]{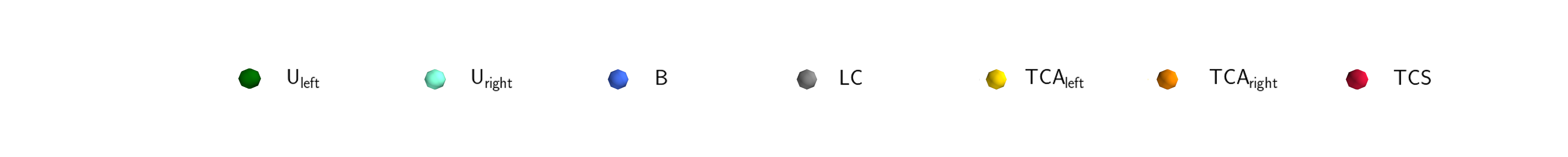}
	\begin{subfigure}[b]{0.15\textwidth}
		\centering
		\includegraphics[trim={2.0cm 1.8cm 2.0cm 1.8cm},clip,width=\textwidth]{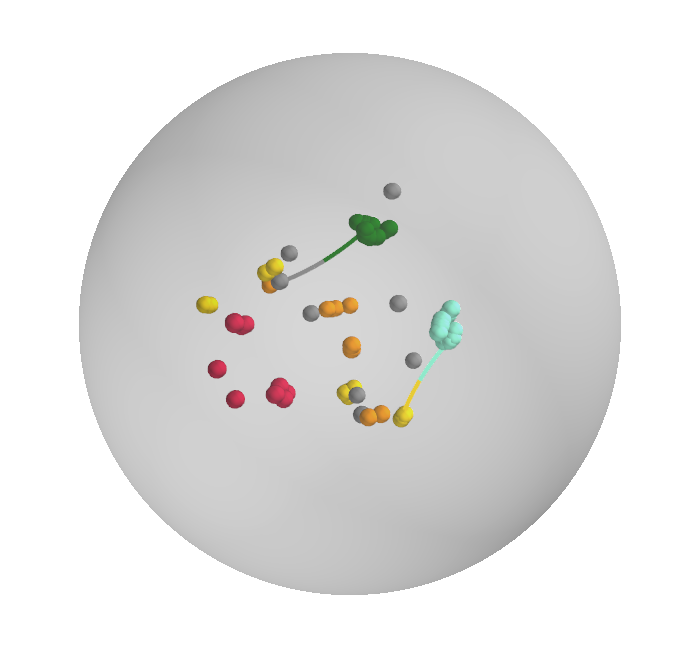}
		\includegraphics[trim={2.0cm 2.0cm 0.5cm 2.0cm},clip,width=0.9\textwidth]{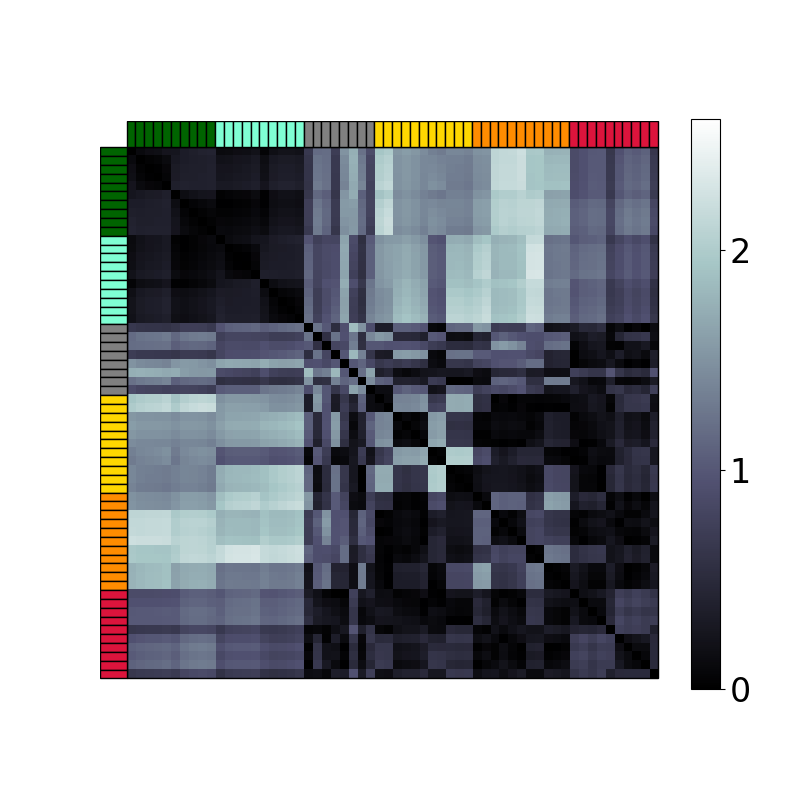}
		\includegraphics[trim={0.0cm 0.0cm 0.0cm 0.0cm},clip,width=\textwidth]{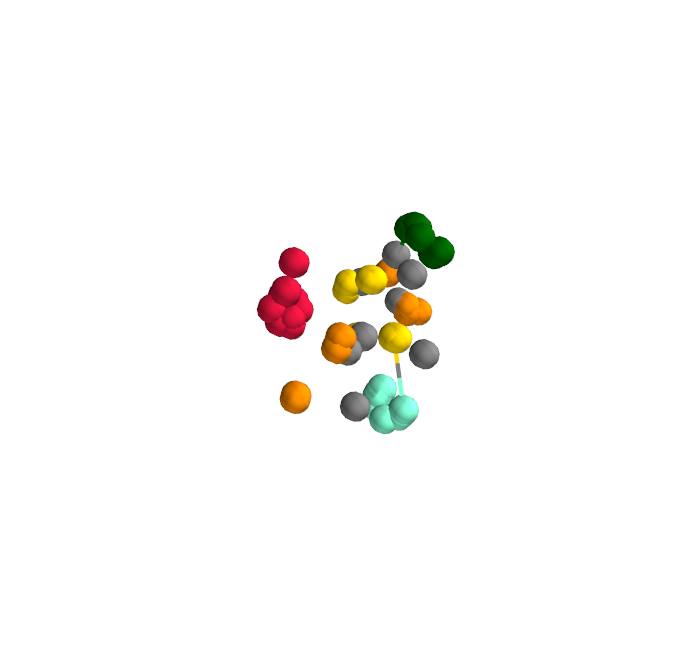}
		\includegraphics[trim={2.0cm 2.0cm 0.5cm 2.0cm},clip,width=0.9\textwidth]{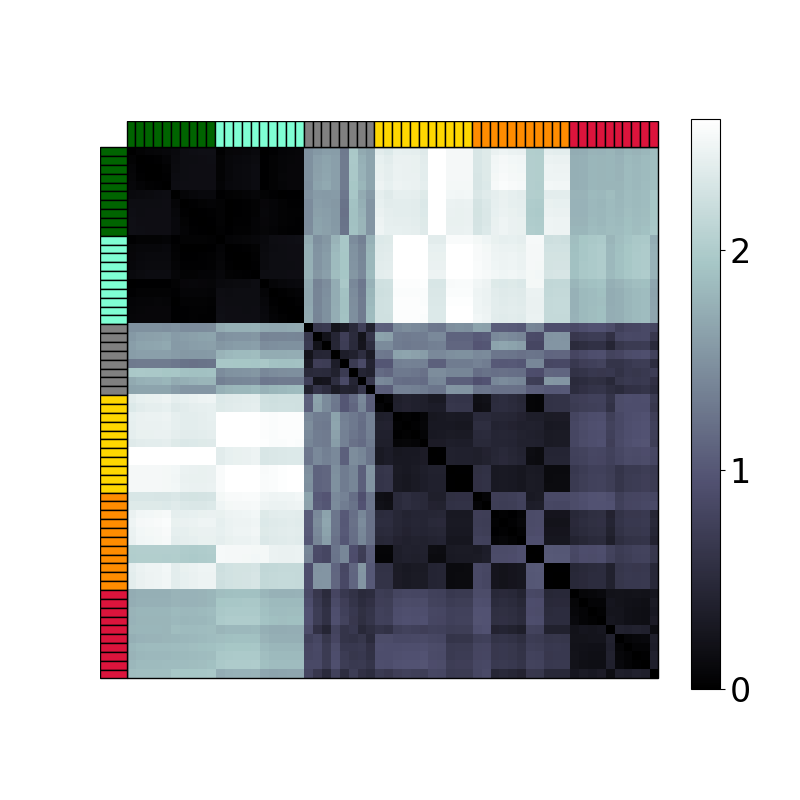}
		\caption{Vanilla}
		\label{fig:GPHLVM:vanilla_3d-bimanual}
	\end{subfigure}%
	\begin{subfigure}[b]{0.15\textwidth}
		\centering
		\includegraphics[trim={2.0cm 1.8cm 2.0cm 1.8cm},clip,width=\textwidth]{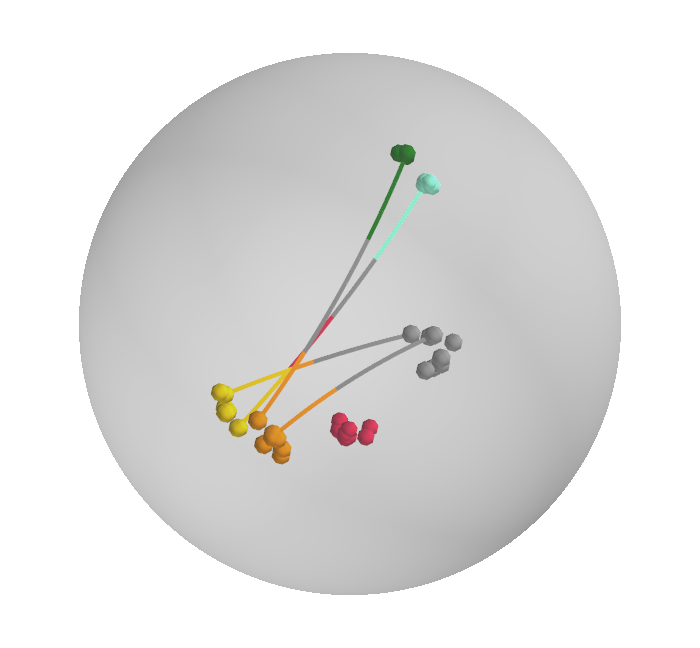}
		\includegraphics[trim={2.0cm 2.0cm 0.5cm 2.0cm},clip,width=0.9\textwidth]{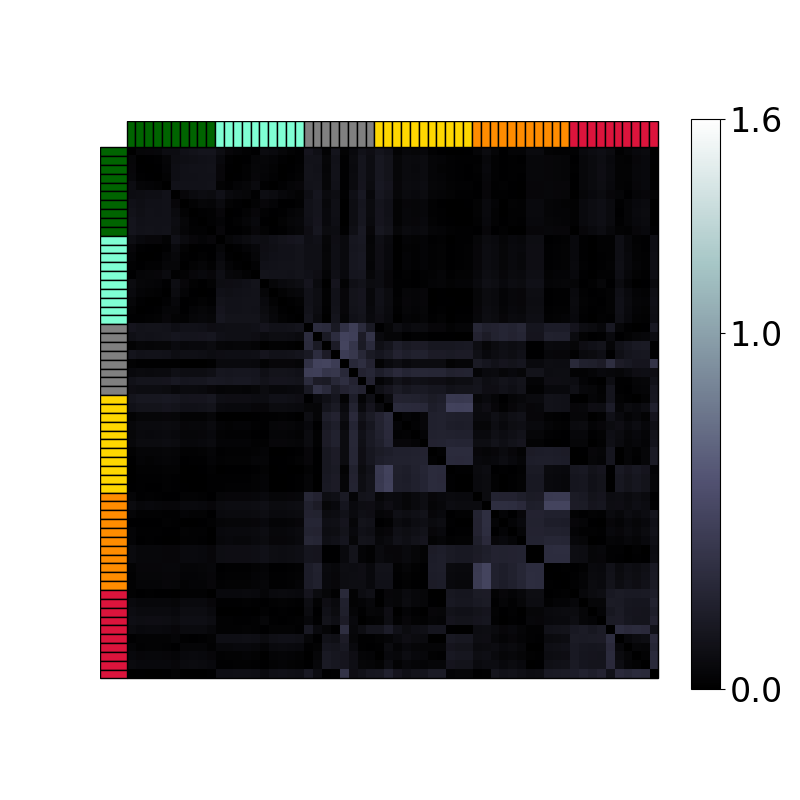}
		\includegraphics[trim={0.0cm 0.0cm 0.0cm 0.0cm},clip,width=\textwidth]{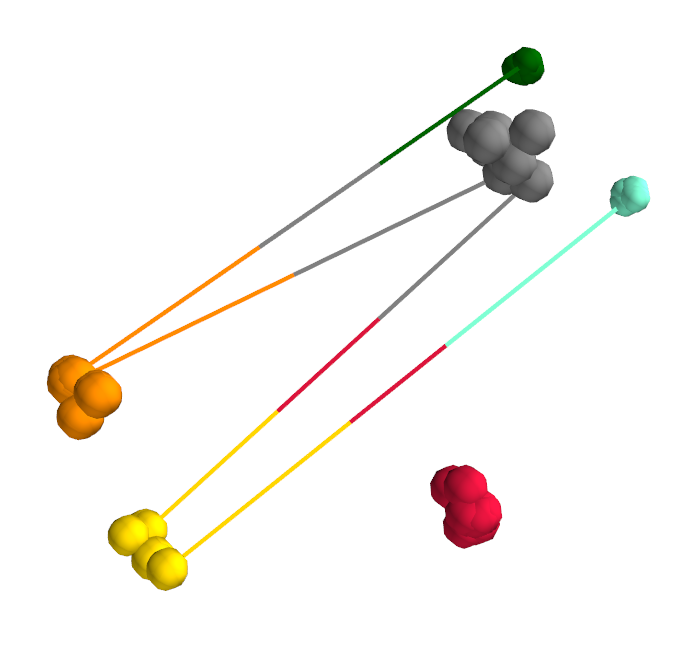}
		\includegraphics[trim={2.0cm 2.0cm 0.5cm 2.0cm},clip,width=0.9\textwidth]{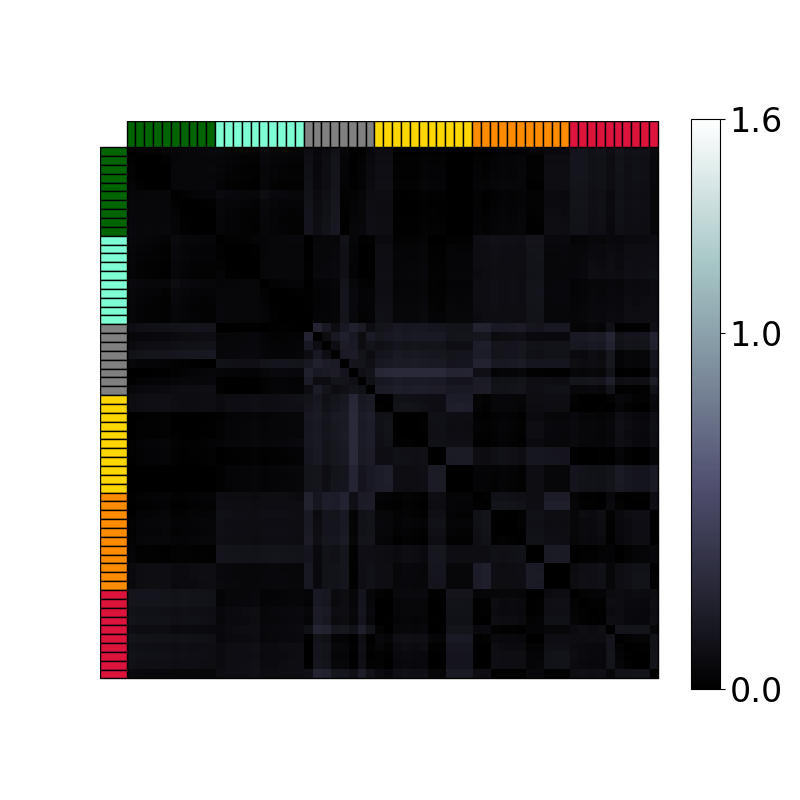}
		\caption{Stress prior}
		\label{fig:GPHLVM:stress_prior_3d-bimanual}
	\end{subfigure}%
	\begin{subfigure}[b]{0.15\textwidth}
		\centering
		\includegraphics[trim={2.0cm 1.8cm 2.0cm 1.8cm},clip,width=\textwidth]{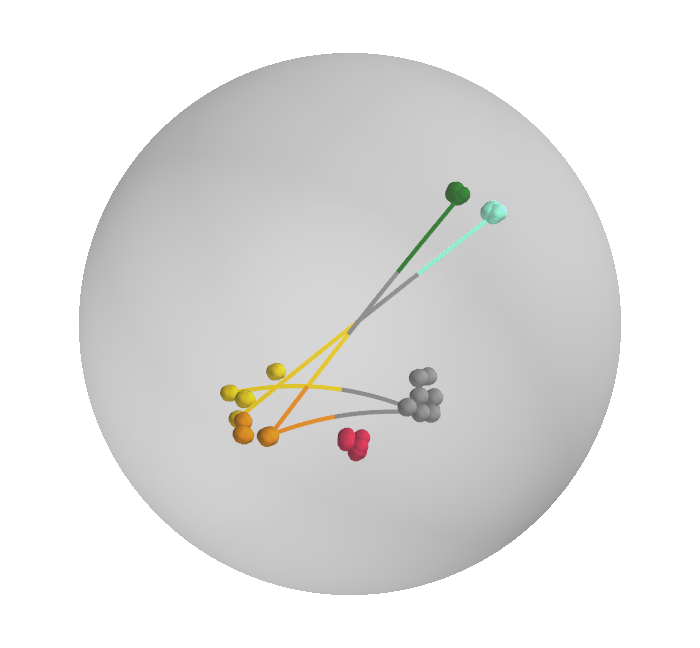}
		\includegraphics[trim={2.0cm 2.0cm 0.5cm 2.0cm},clip,width=0.9\textwidth]{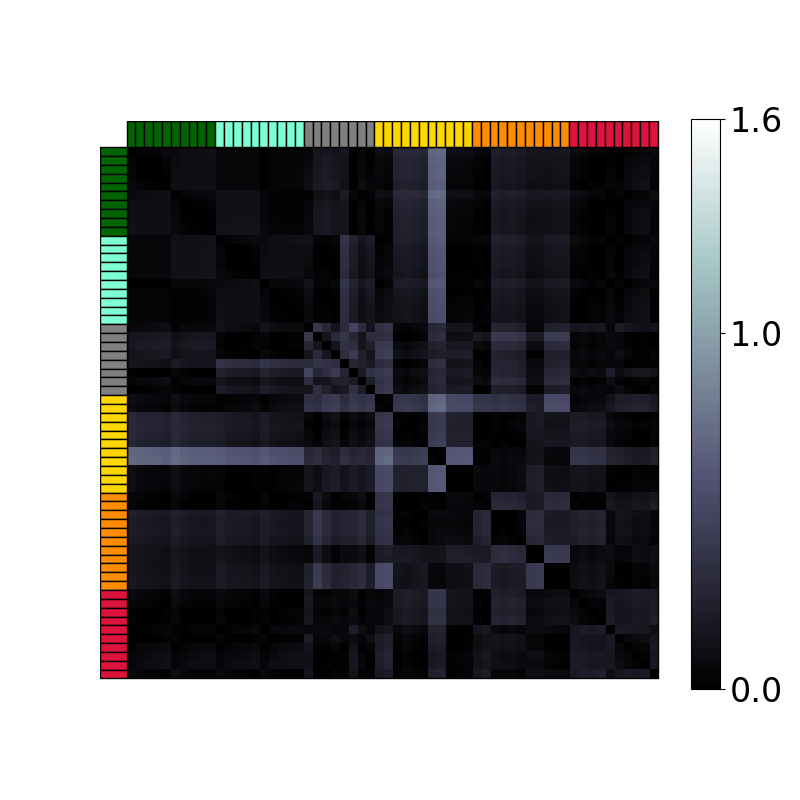}
		\includegraphics[trim={0.0cm 0.0cm 0.0cm 0.0cm},clip,width=\textwidth]{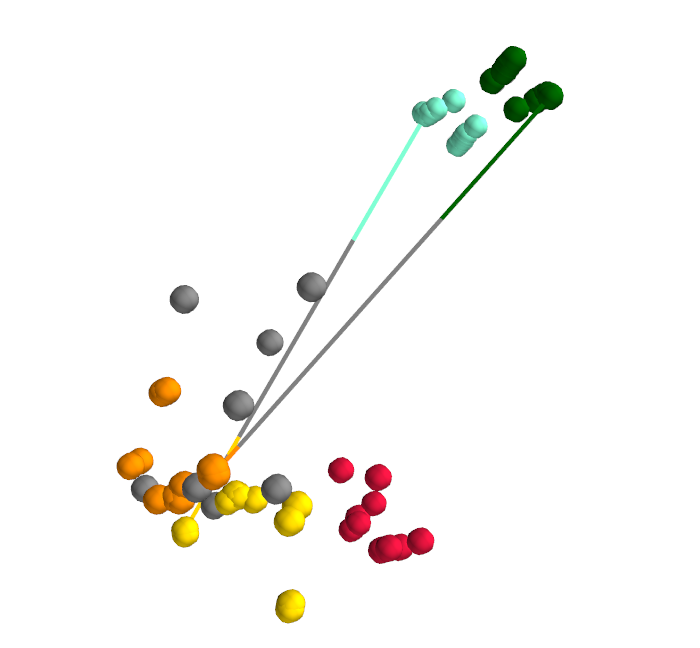}
		\includegraphics[trim={2.0cm 2.0cm 0.5cm 2.0cm},clip,width=0.9\textwidth]{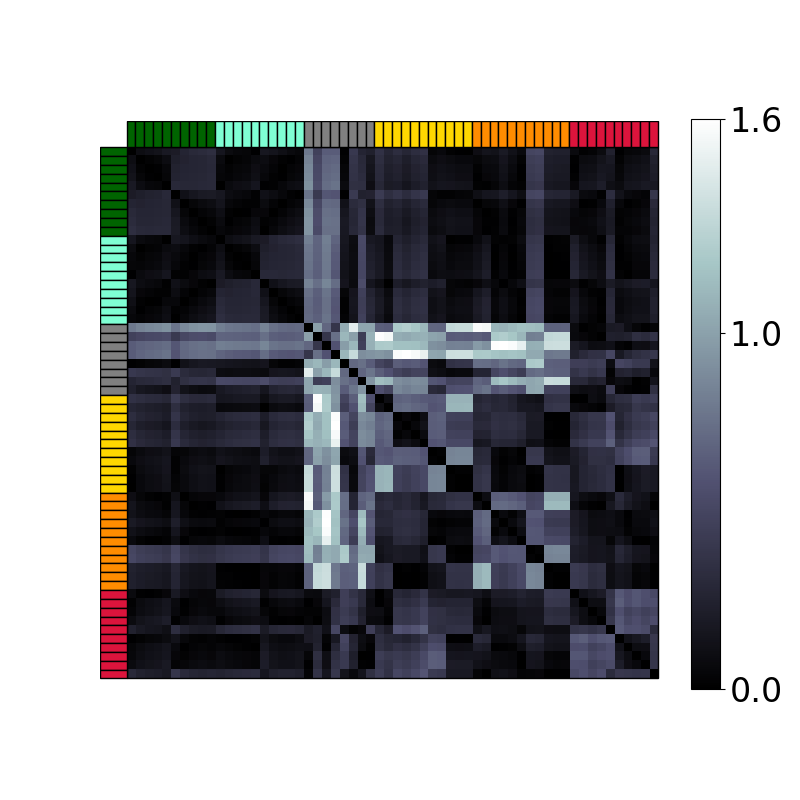}
		\caption{BC + stress prior}
		\label{fig:GPHLVM:backconstrained_and_stress_3d-bimanual}
	\end{subfigure}%
	\begin{subfigure}[b]{0.15\textwidth}
		\centering
        \includegraphics[trim={2.0cm 1.8cm 2.0cm 1.8cm},clip,width=\textwidth]{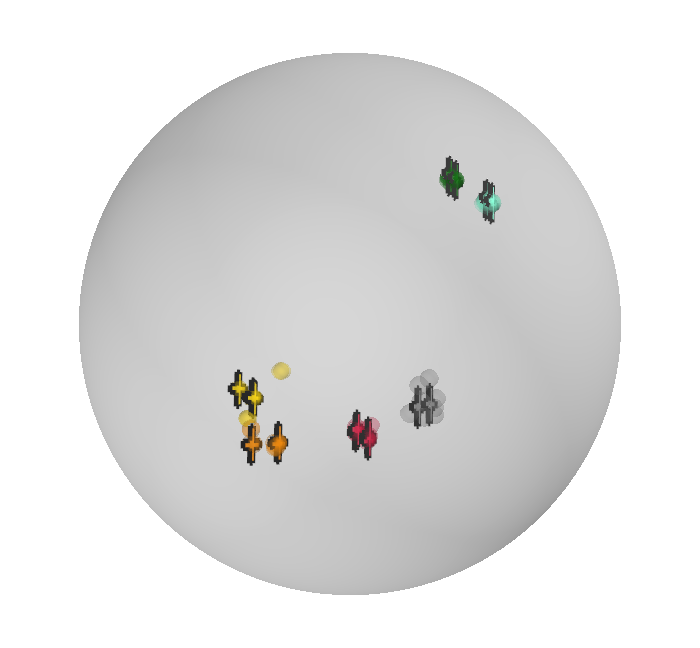}
		\includegraphics[trim={2.0cm 2.0cm 0.5cm 2.0cm},clip,width=0.9\textwidth]{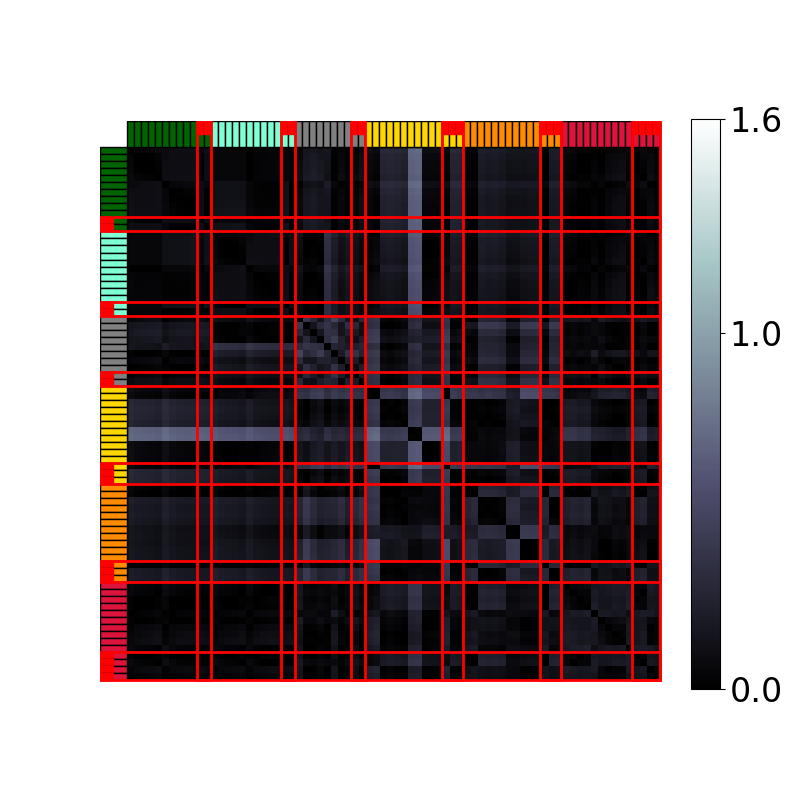}
		\includegraphics[trim={0.0cm 0.0cm 0.0cm 0.0cm},clip,width=\textwidth]{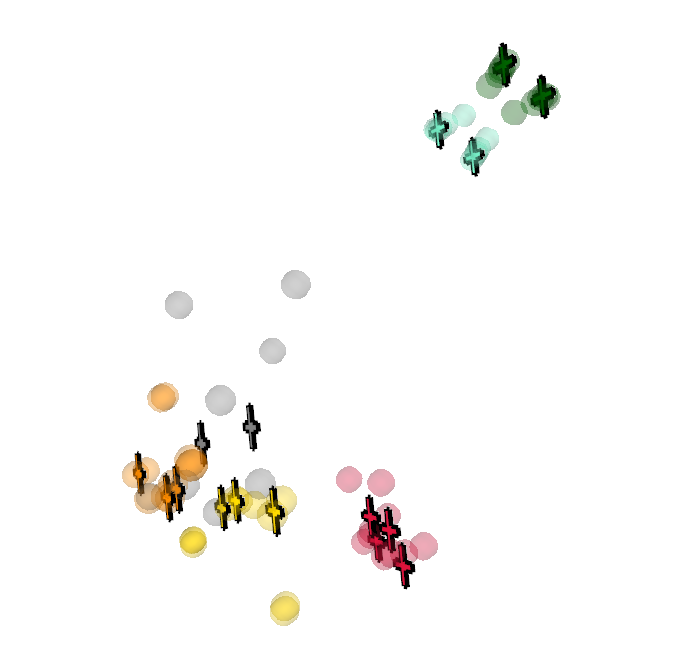}
		\includegraphics[trim={2.0cm 2.0cm 0.5cm 2.0cm},clip,width=0.9\textwidth]{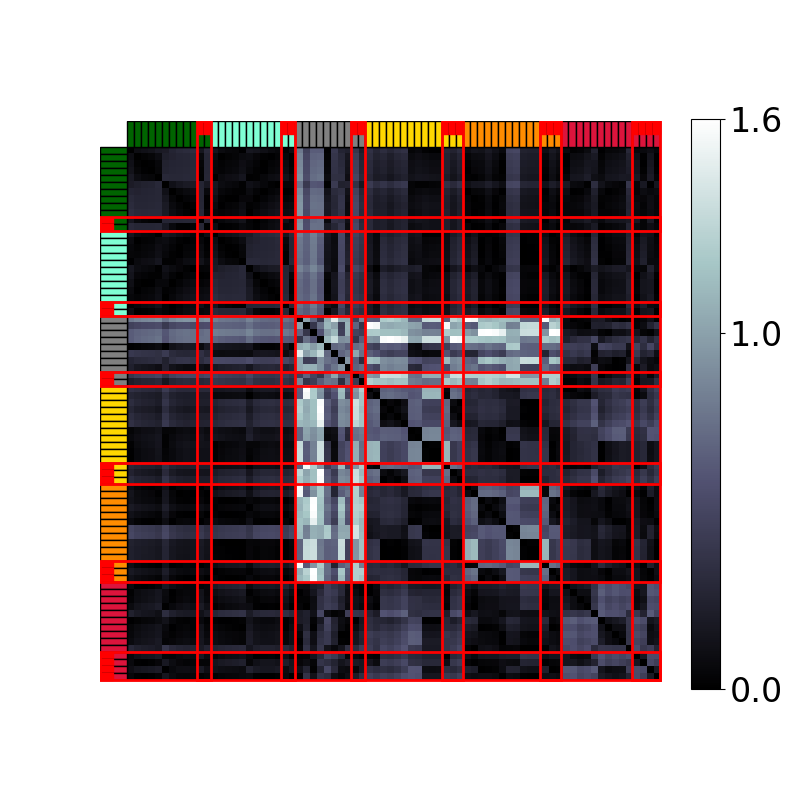}
		\caption{Adding poses}
		\label{fig:GPHLVM:added_poses_3d-bimanual}
	\end{subfigure}%
	\begin{subfigure}[b]{0.15\textwidth}
		\centering
		\includegraphics[trim={2.0cm 1.8cm 2.0cm 1.8cm},clip,width=\textwidth]{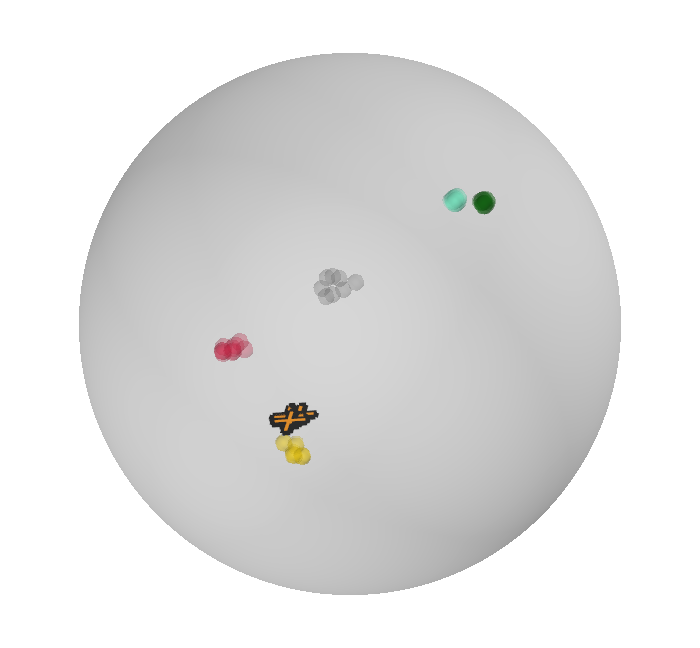}
		\includegraphics[trim={2.0cm 2.0cm 0.5cm 2.0cm},clip,width=0.9\textwidth]{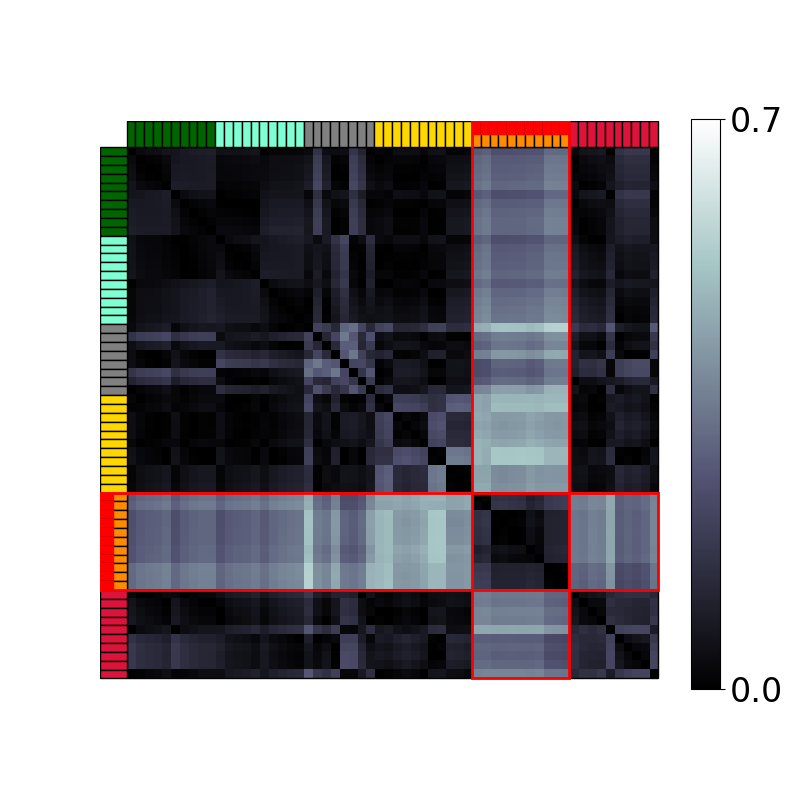}
		\includegraphics[trim={0.0cm 0.0cm 0.0cm 0.0cm},clip,width=\textwidth]{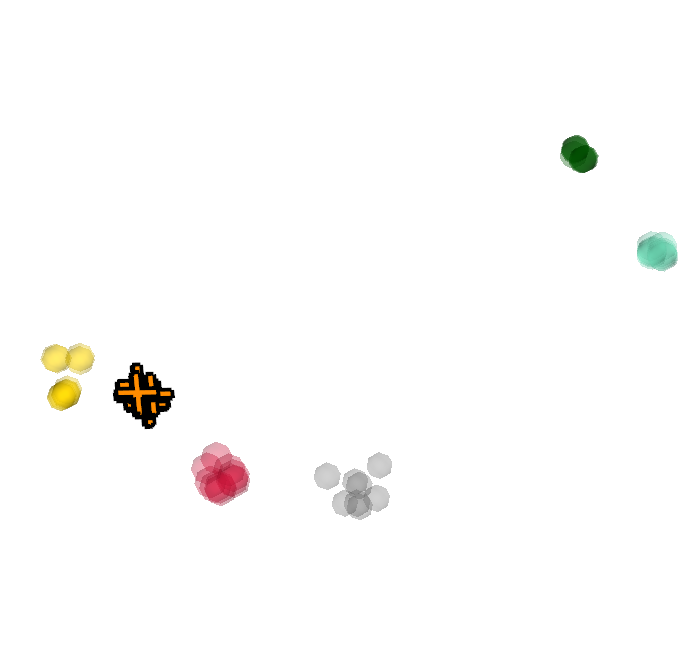}
		\includegraphics[trim={2.0cm 2.0cm 0.5cm 2.0cm},clip,width=0.9\textwidth]{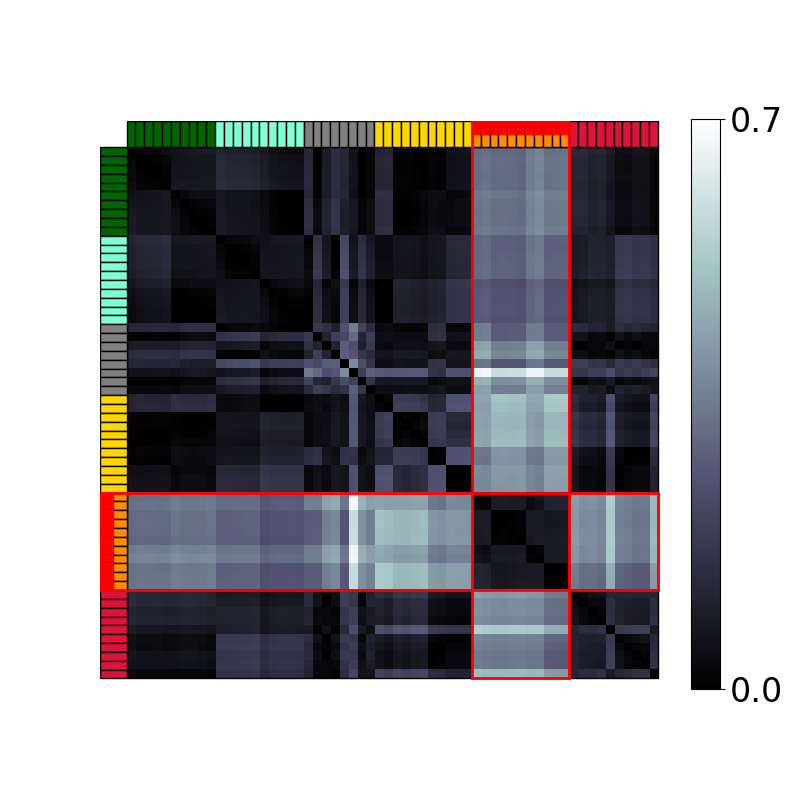}
		\caption{Adding a class}
		\label{fig:GPHLVM:added_class_3d-bimanual}
	\end{subfigure}
    \vspace{-0.2cm}
	\caption{Bimanual manipulation categories: The first and last two rows show the latent embeddings and examples of interpolating geodesics in $\mathcal{P}^3$ and $\mathbb{R}^3$, followed by pairwise error matrices between geodesic and taxonomy graph distances. Added poses \emph{(d)} and classes \emph{(e)} are marked with crosses and highlighted with red in the error matrices.}
	\label{fig:GPHLVM:trained_models_3d-bimanual}
    \vspace{-0.3cm}
\end{figure}

\begin{figure}
	\centering
	\includegraphics[trim={2.8cm 1.2cm 2.0cm 1.5cm},clip,width=0.8\textwidth]{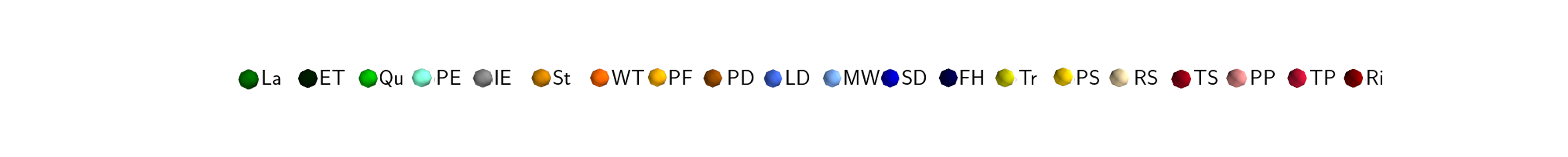}
	\begin{subfigure}[b]{0.15\textwidth}
		\centering
		\includegraphics[trim={2.0cm 1.8cm 2.0cm 1.8cm},clip,width=\textwidth]{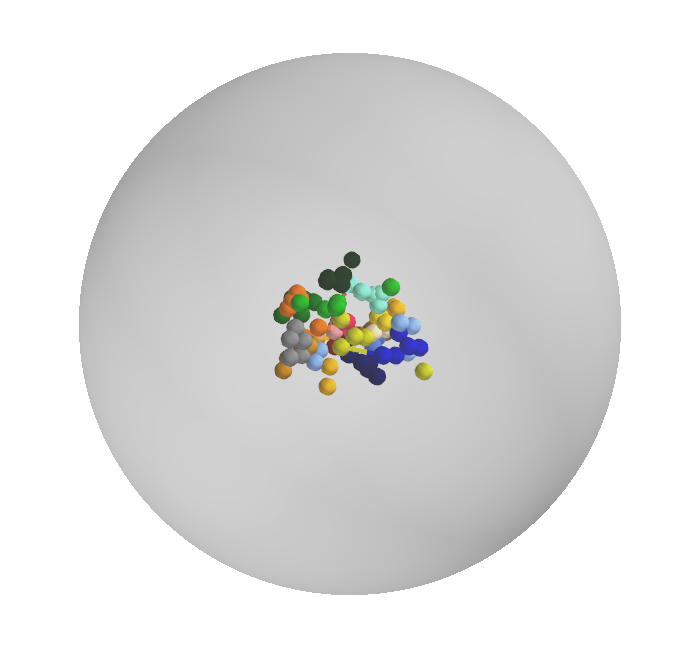}
		\includegraphics[trim={2.0cm 2.0cm 0.5cm 2.0cm},clip,width=0.9\textwidth]{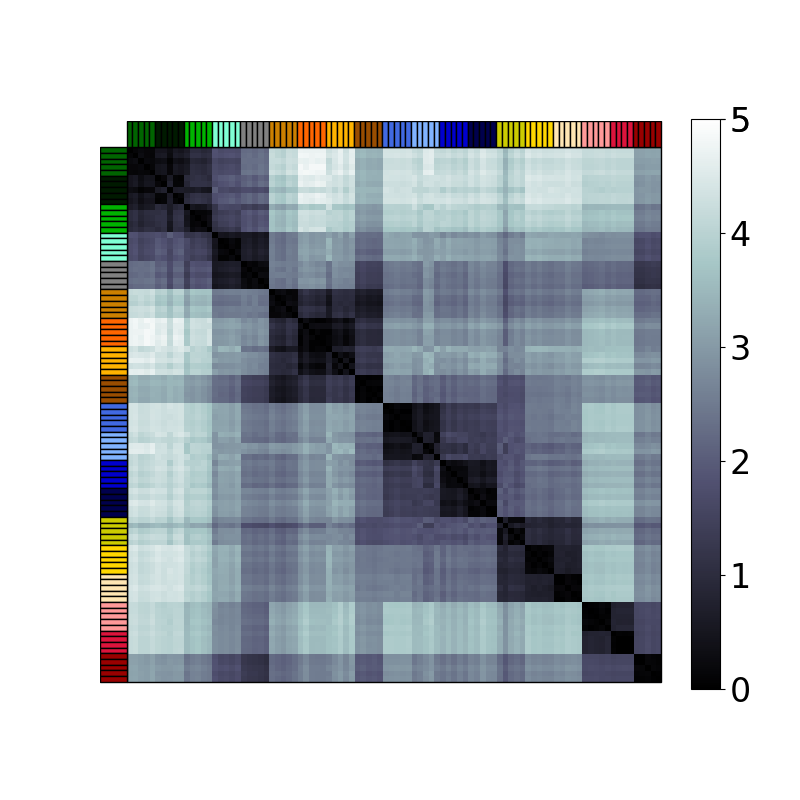}
		\includegraphics[trim={0.0cm 0.0cm 0.0cm 0.0cm},clip,width=\textwidth]{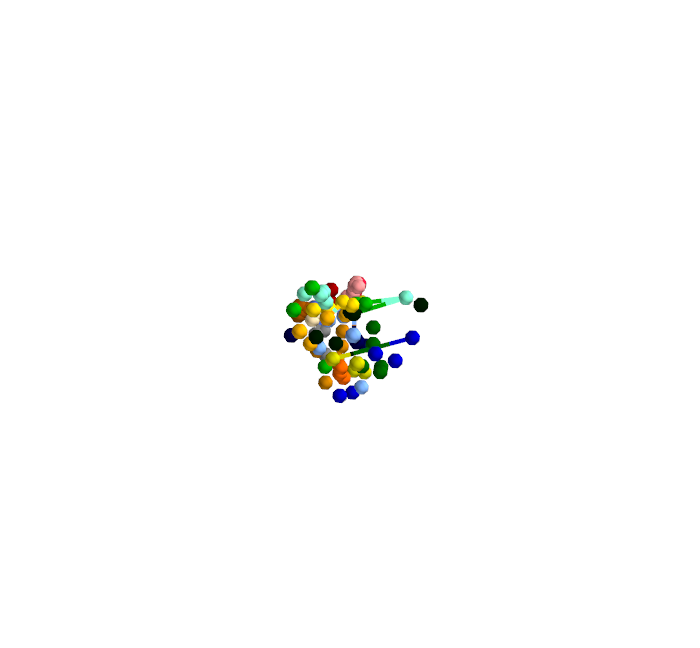}
		\includegraphics[trim={2.0cm 2.0cm 0.5cm 2.0cm},clip,width=0.9\textwidth]{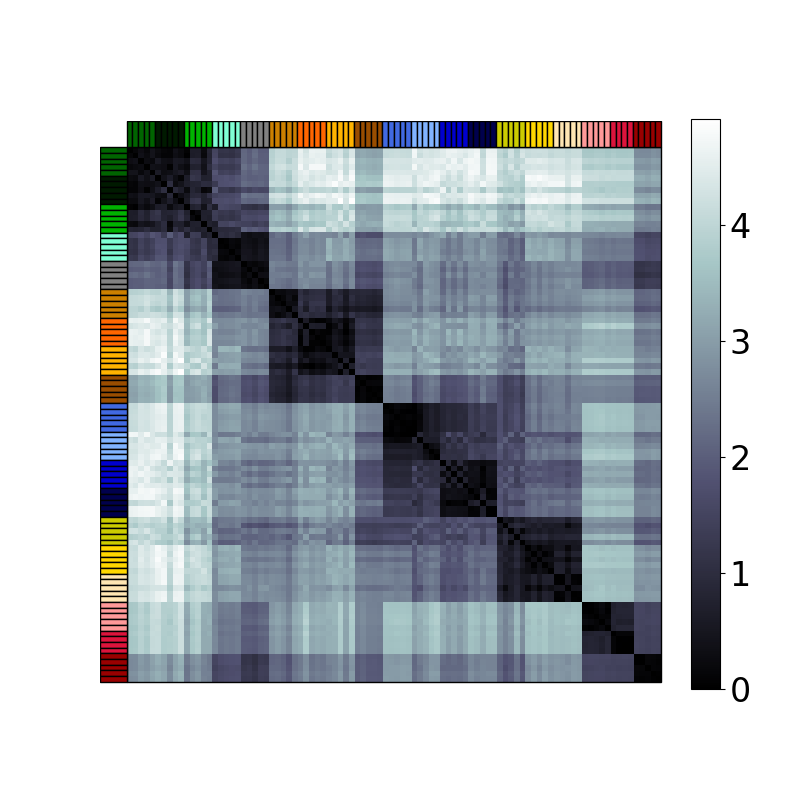}
		\caption{Vanilla}
		\label{fig:GPHLVM-grasps:vanilla_3d}
	\end{subfigure}%
	\begin{subfigure}[b]{0.15\textwidth}
		\centering
		\includegraphics[trim={2.0cm 1.8cm 2.0cm 1.8cm},clip,width=\textwidth]{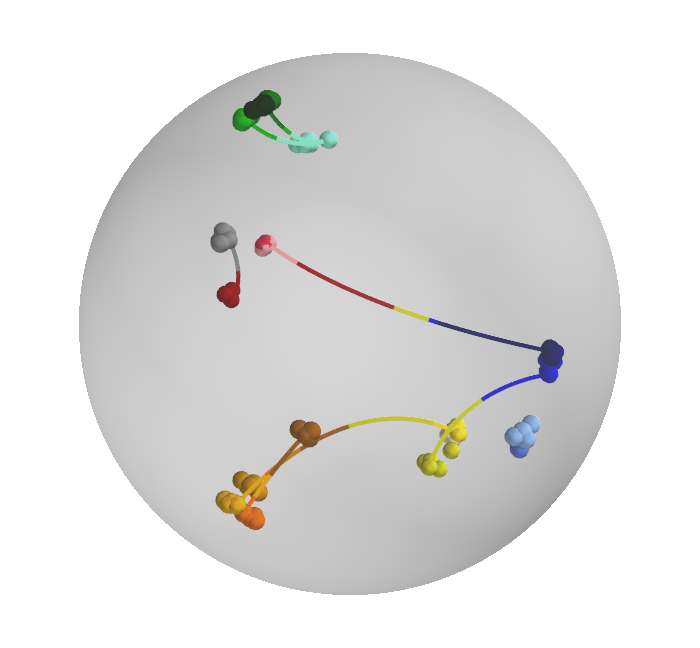}
		\includegraphics[trim={2.0cm 2.0cm 0.5cm 2.0cm},clip,width=0.9\textwidth]{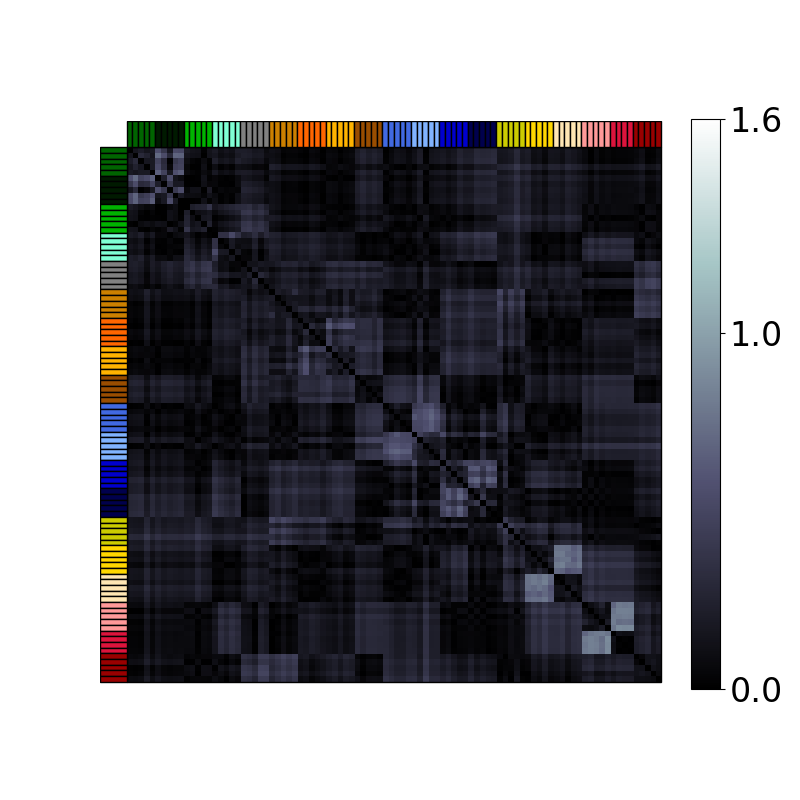}
		\includegraphics[trim={0.0cm 0.0cm 0.0cm 0.0cm},clip,width=\textwidth]{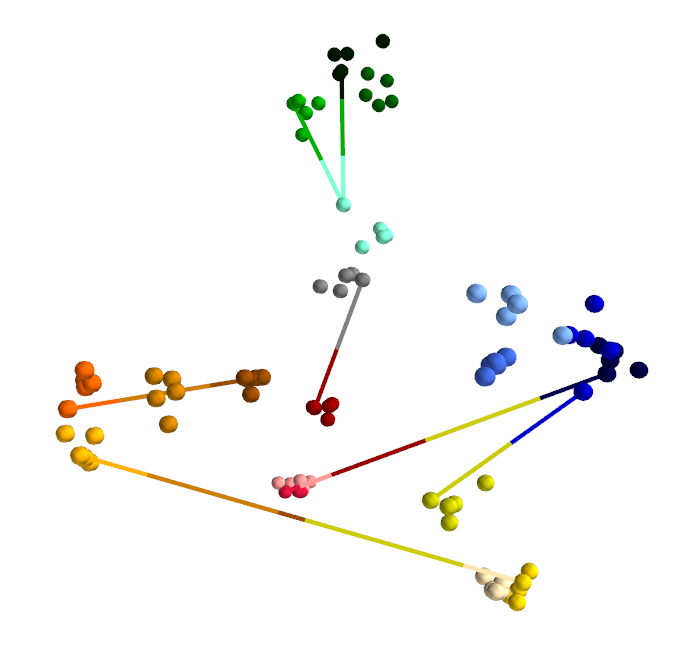}
		\includegraphics[trim={2.0cm 2.0cm 0.5cm 2.0cm},clip,width=0.9\textwidth]{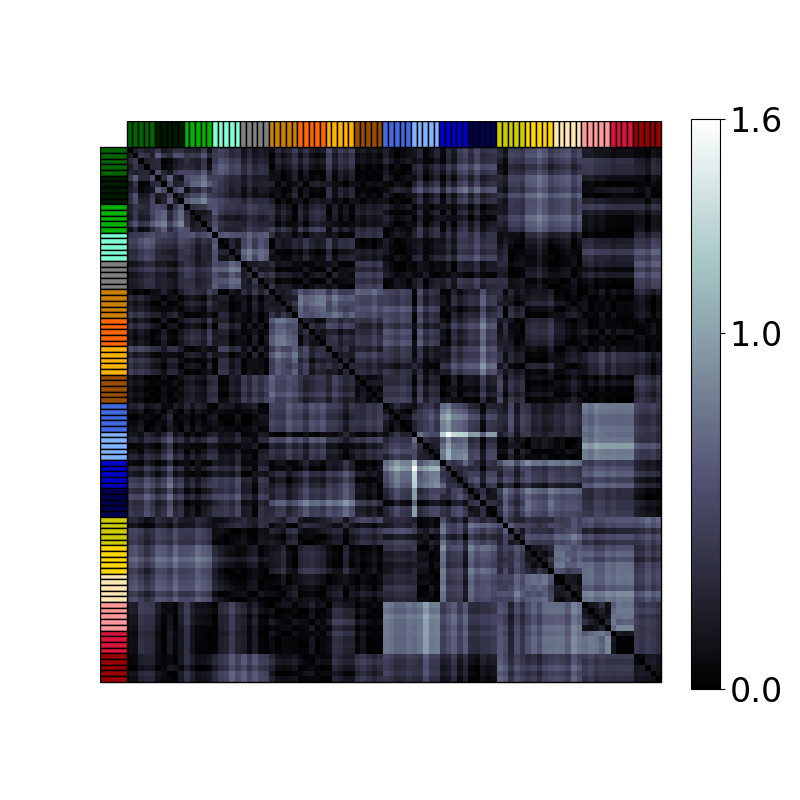}
		\caption{Stress prior}
		\label{fig:GPHLVM-grasps:stress_prior_3d}
	\end{subfigure}%
	\begin{subfigure}[b]{0.15\textwidth}
		\centering
		\includegraphics[trim={2.0cm 1.8cm 2.0cm 1.8cm},clip,width=\textwidth]{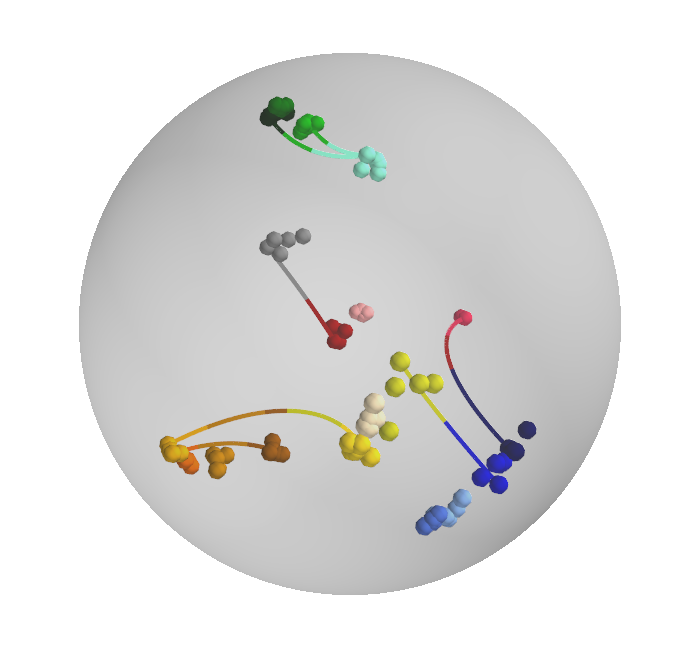}
		\includegraphics[trim={2.0cm 2.0cm 0.5cm 2.0cm},clip,width=0.9\textwidth]{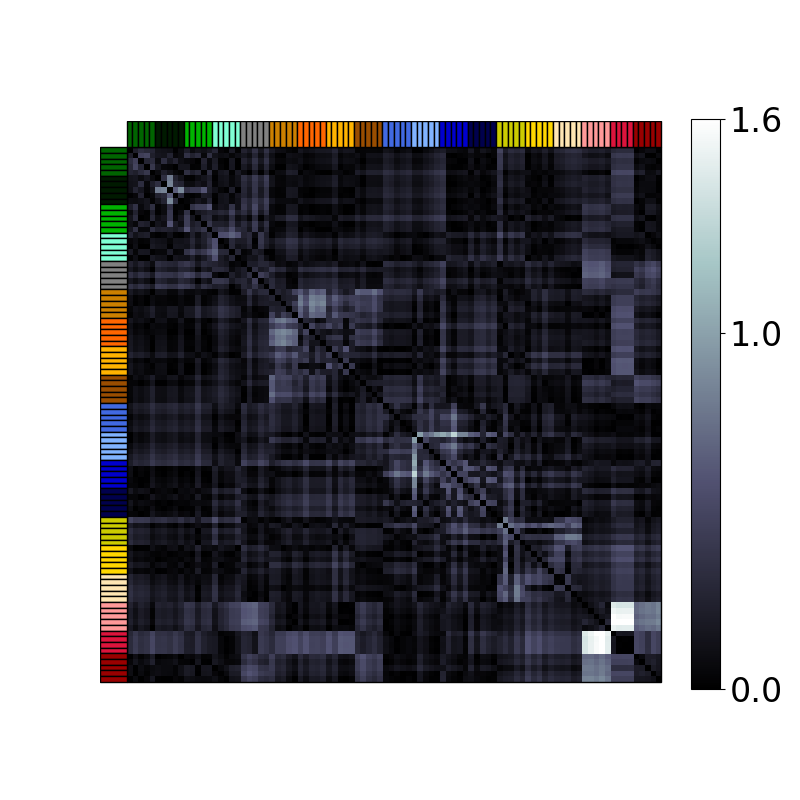}
		\includegraphics[trim={0.0cm 0.0cm 0.0cm 0.0cm},clip,width=\textwidth]{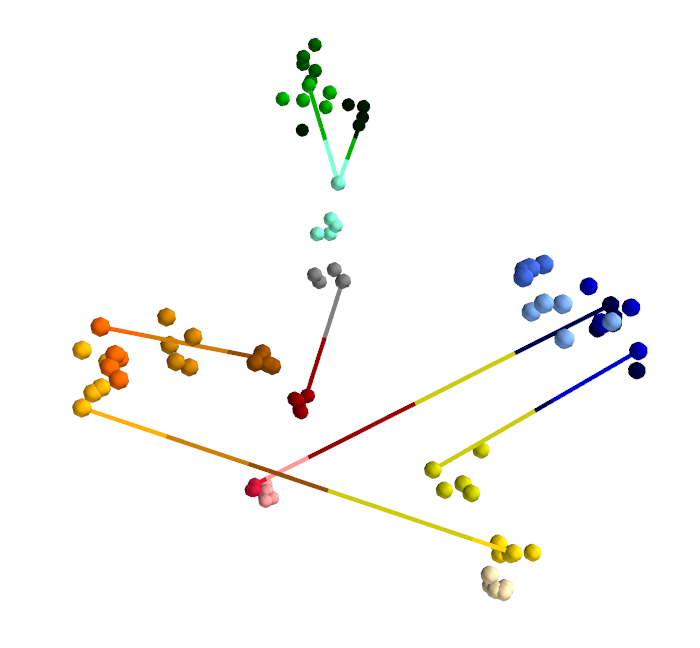}
		\includegraphics[trim={2.0cm 2.0cm 0.5cm 2.0cm},clip,width=0.9\textwidth]{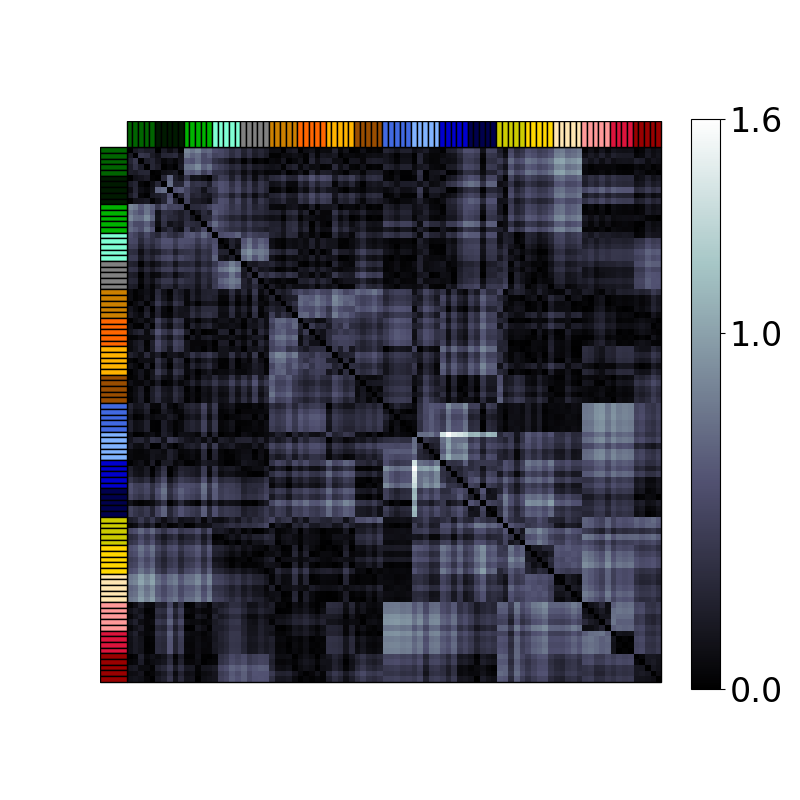}
		\caption{BC + stress prior}
		\label{fig:GPHLVM-grasps:backconstrained_and_stress_3d}
	\end{subfigure}%
	\begin{subfigure}[b]{0.15\textwidth}
		\centering
        \includegraphics[trim={2.0cm 1.8cm 2.0cm 1.8cm},clip,width=\textwidth]{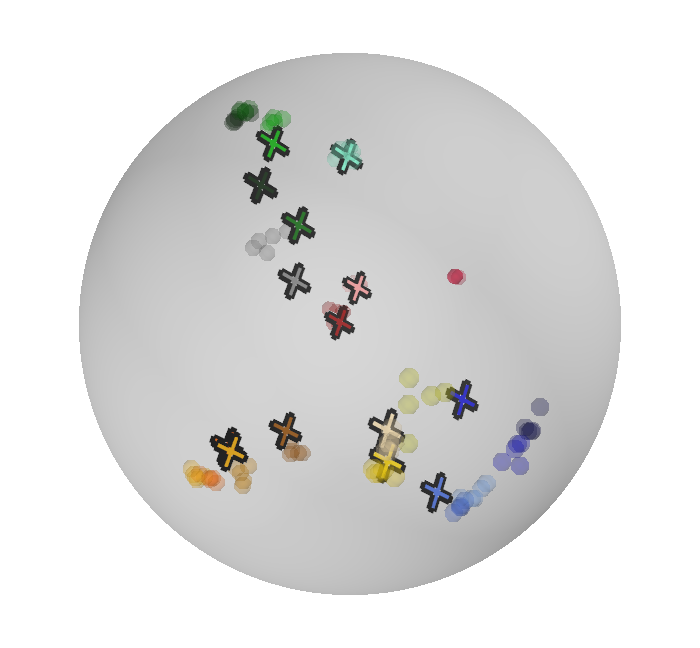}
		\includegraphics[trim={2.0cm 2.0cm 0.5cm 2.0cm},clip,width=0.9\textwidth]{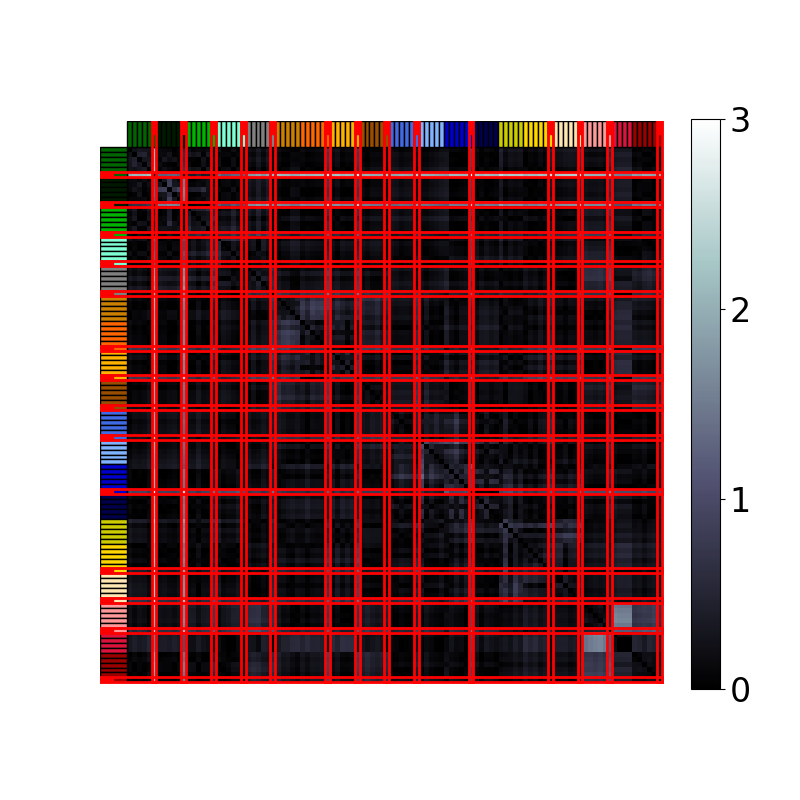}
		\includegraphics[trim={0.0cm 0.0cm 0.0cm 0.0cm},clip,width=\textwidth]{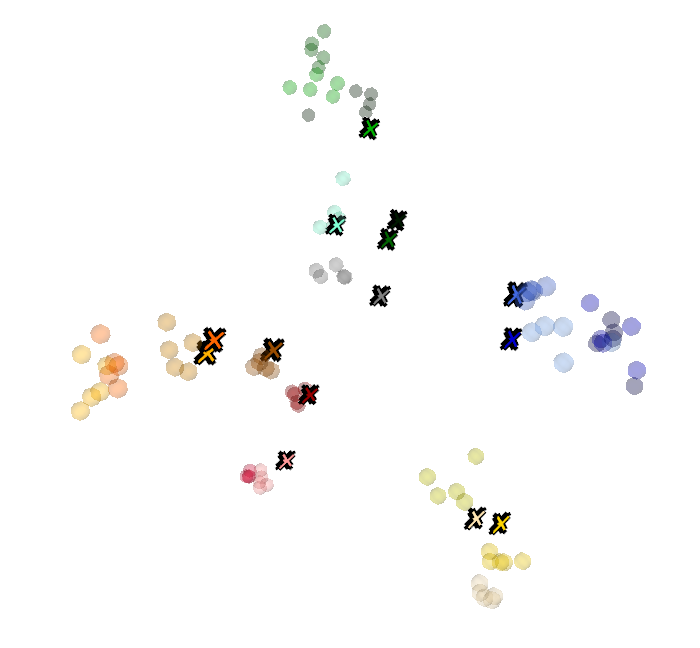}
		\includegraphics[trim={2.0cm 2.0cm 0.5cm 2.0cm},clip,width=0.9\textwidth]{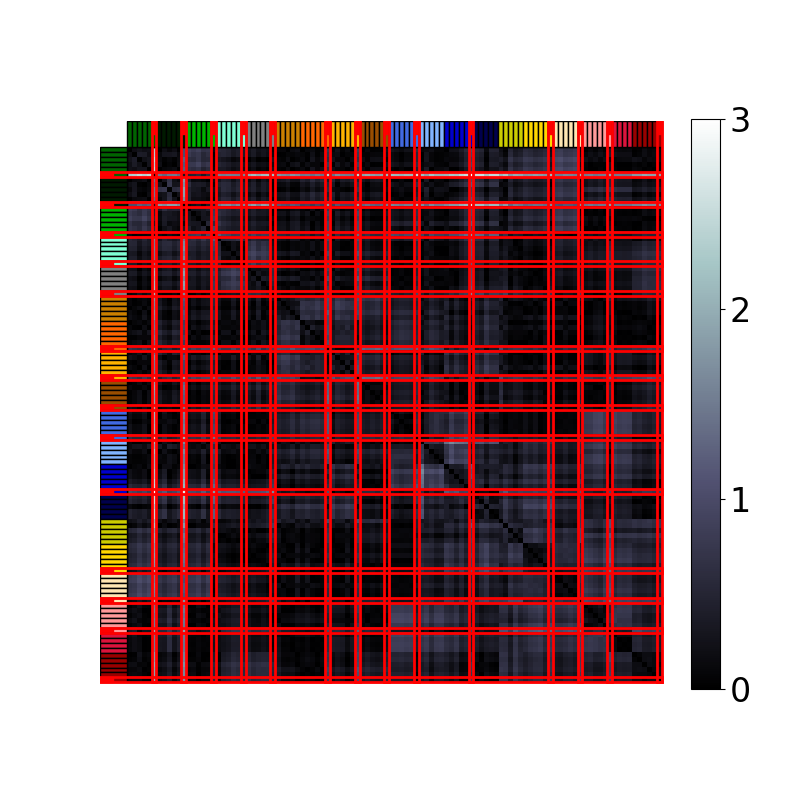}
		\caption{Adding poses}
		\label{fig:GPHLVM-grasps:added_poses_3d}
	\end{subfigure}%
	\begin{subfigure}[b]{0.15\textwidth}
		\centering
        \includegraphics[trim={2.0cm 1.8cm 2.0cm 1.8cm},clip,width=\textwidth]{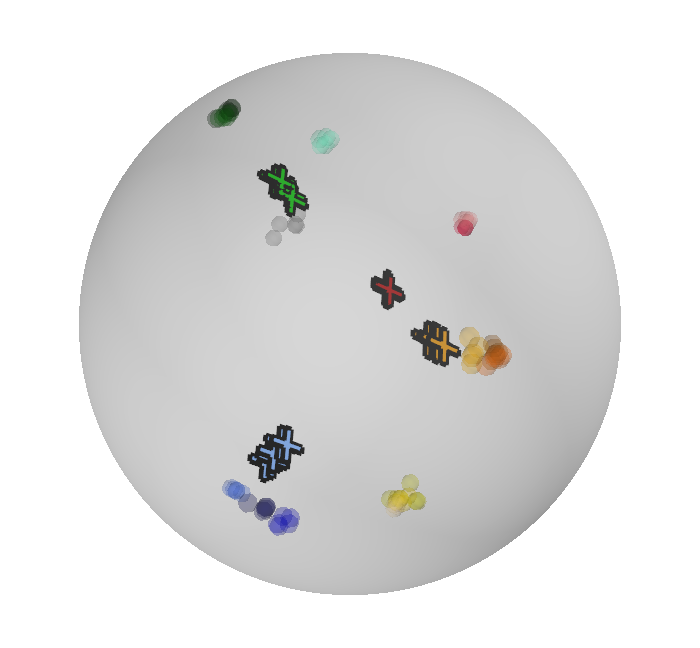}
		\includegraphics[trim={2.0cm 2.0cm 0.5cm 2.0cm},clip,width=0.9\textwidth]{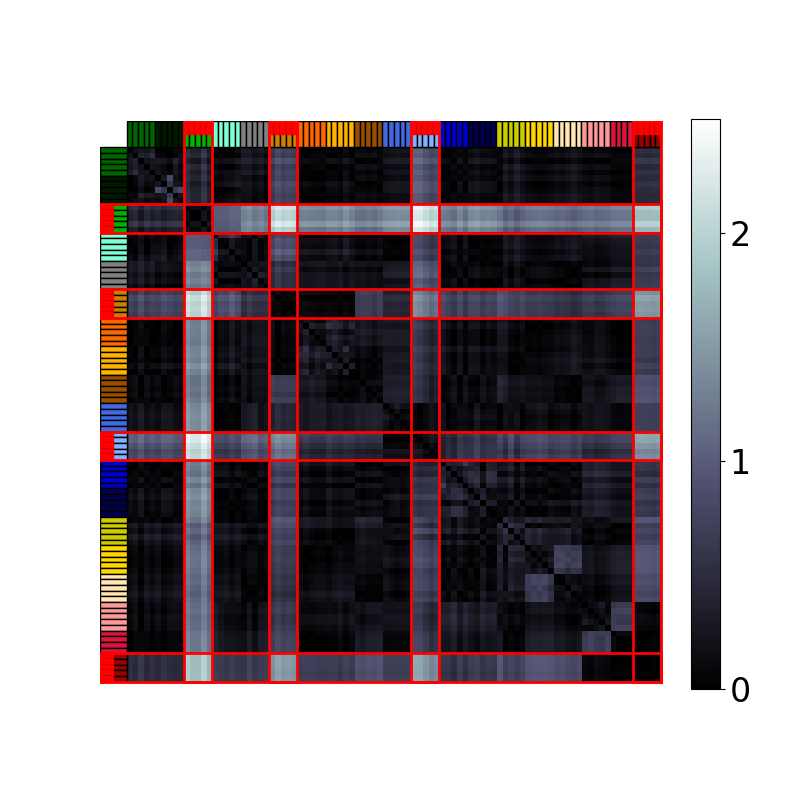}
		\includegraphics[trim={0.0cm 0.0cm 0.0cm 0.0cm},clip,width=\textwidth]{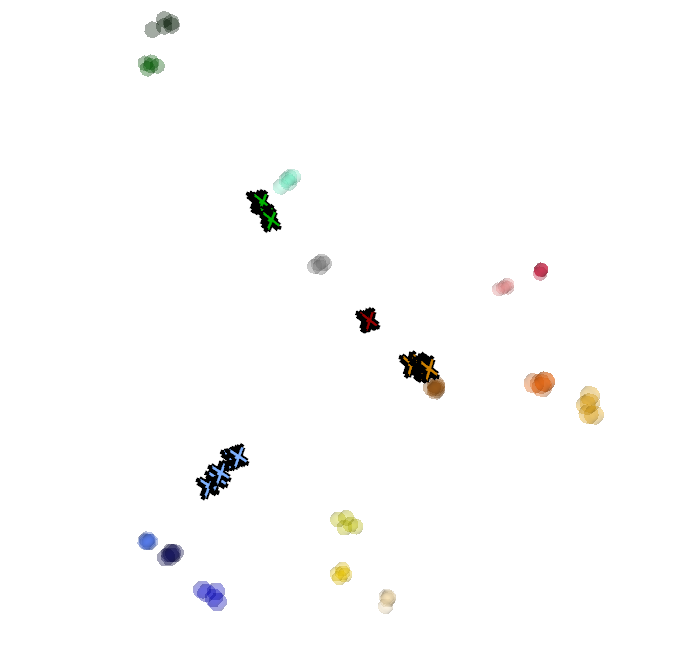}
		\includegraphics[trim={2.0cm 2.0cm 0.5cm 2.0cm},clip,width=0.9\textwidth]{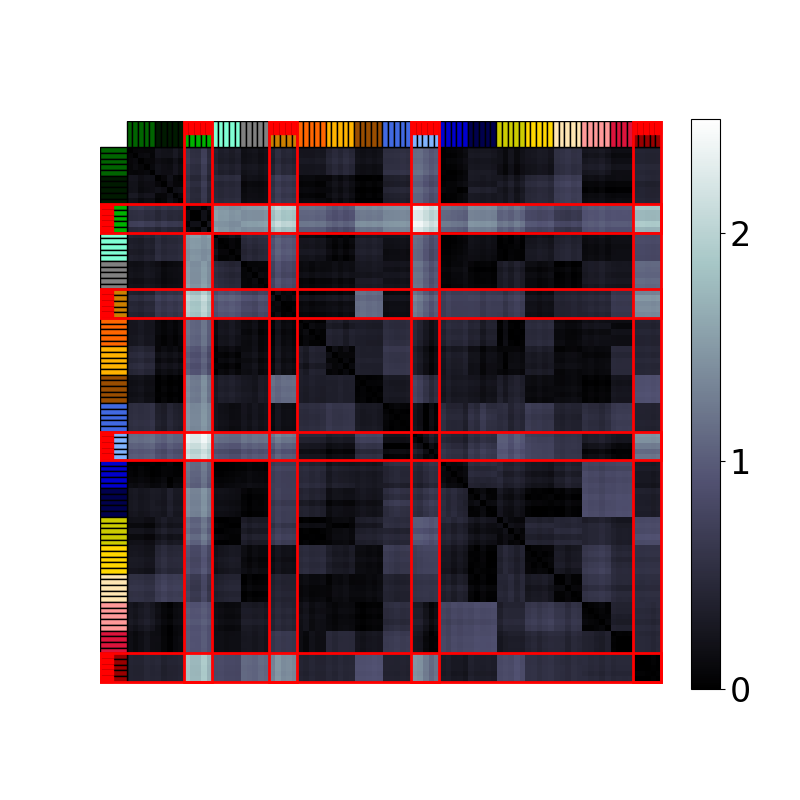}
		\caption{Adding a class}
		\label{fig:GPHLVM-grasps:added_class_3d}
	\end{subfigure}
    \vspace{-0.2cm}
	\caption{Grasps: The first and last two rows show the latent embeddings and examples of interpolating geodesics in $\mathcal{P}^3$ and $\mathbb{R}^3$, followed by pairwise error matrices between geodesic and taxonomy graph distances. Embeddings colors match those of Fig.~\ref{fig:geodesic_trajectories}. Added poses \emph{(d)} and classes \emph{(e)} are marked with crosses and highlighted with red in the error matrices.}
    \vspace{-0.3cm}
	\label{fig:GPHLVM-grasps:trained_models_3d}
\end{figure}

\begin{figure}
	\centering
	\includegraphics[trim={0cm 0cm 0cm 0cm},clip,width=.8\textwidth]{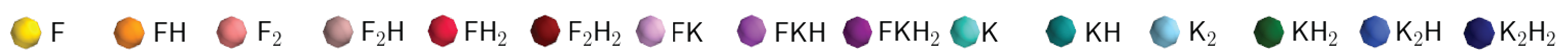}
	\begin{subfigure}[b]{0.15\textwidth}
		\centering
		\includegraphics[trim={2.0cm 1.8cm 2.0cm 1.8cm},clip,width=\textwidth]{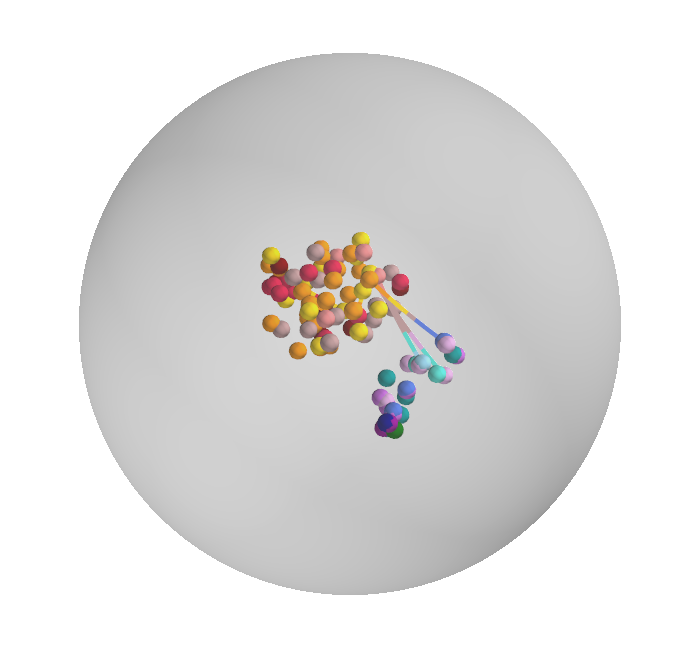}
		\includegraphics[trim={2.0cm 2.0cm 0.5cm 2.0cm},clip,width=0.9\textwidth]{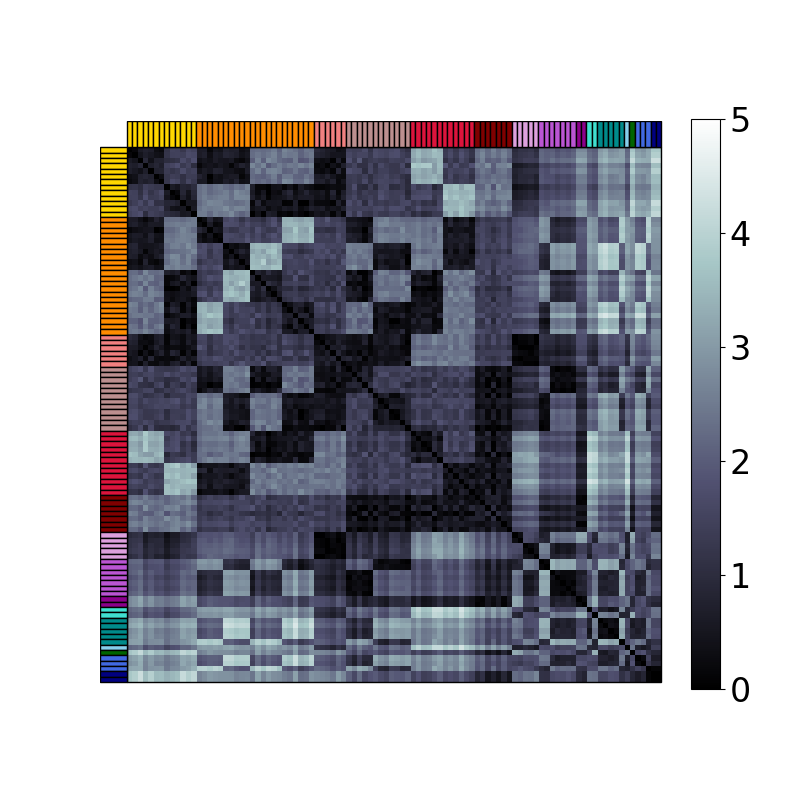}
		\includegraphics[trim={0.0cm 0.0cm 0.0cm 0.0cm},clip,width=\textwidth]{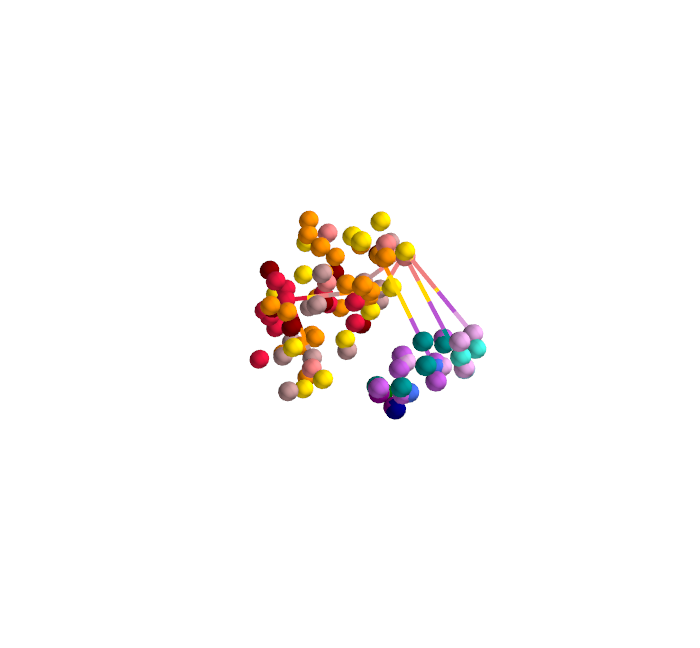}
		\includegraphics[trim={2.0cm 2.0cm 0.5cm 2.0cm},clip,width=0.9\textwidth]{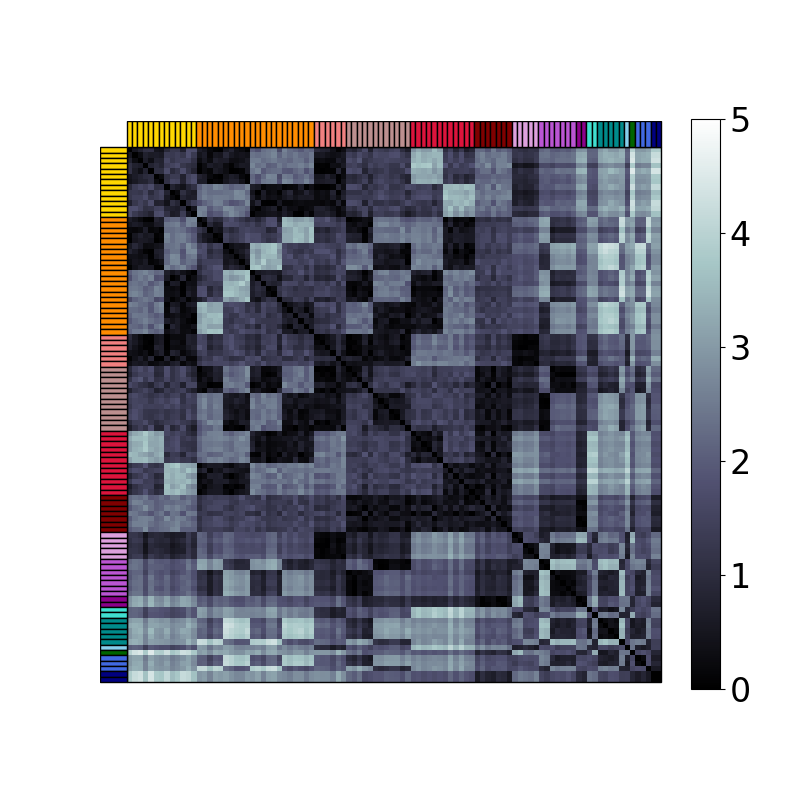}
		\caption{Vanilla}
		\label{fig:GPHLVM:vanilla_3d}
	\end{subfigure}%
	\begin{subfigure}[b]{0.15\textwidth}
		\centering
		\includegraphics[trim={2.0cm 1.8cm 2.0cm 1.8cm},clip,width=\textwidth]{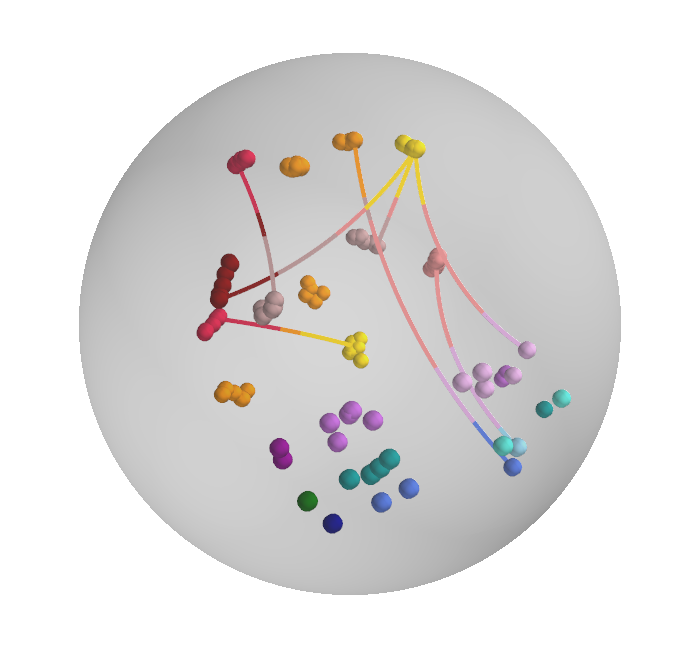}
		\includegraphics[trim={2.0cm 2.0cm 0.5cm 2.0cm},clip,width=0.9\textwidth]{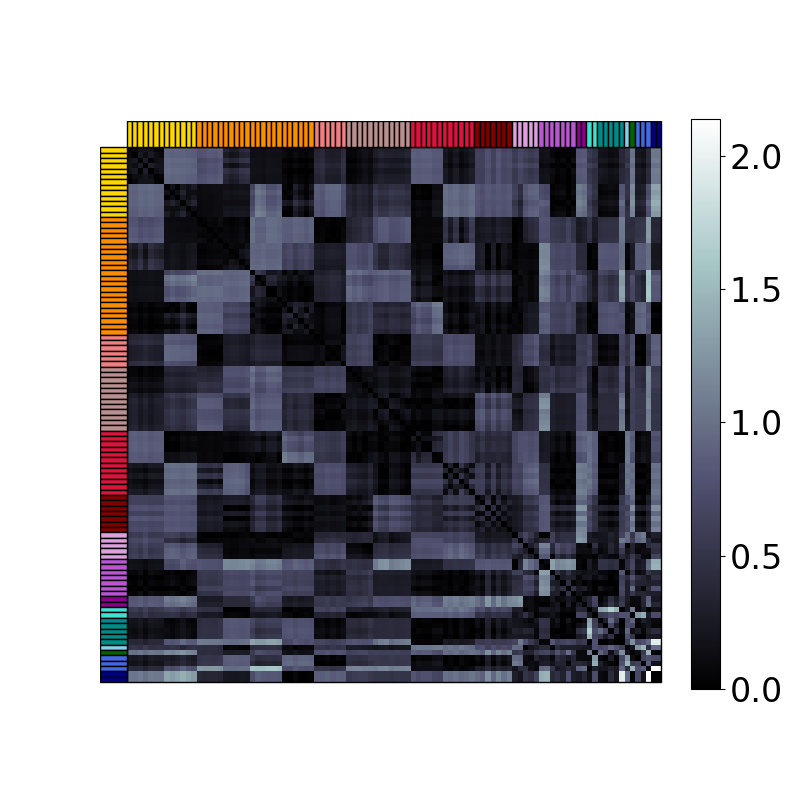}
		\includegraphics[trim={0.0cm 0.0cm 0.0cm 0.0cm},clip,width=\textwidth]{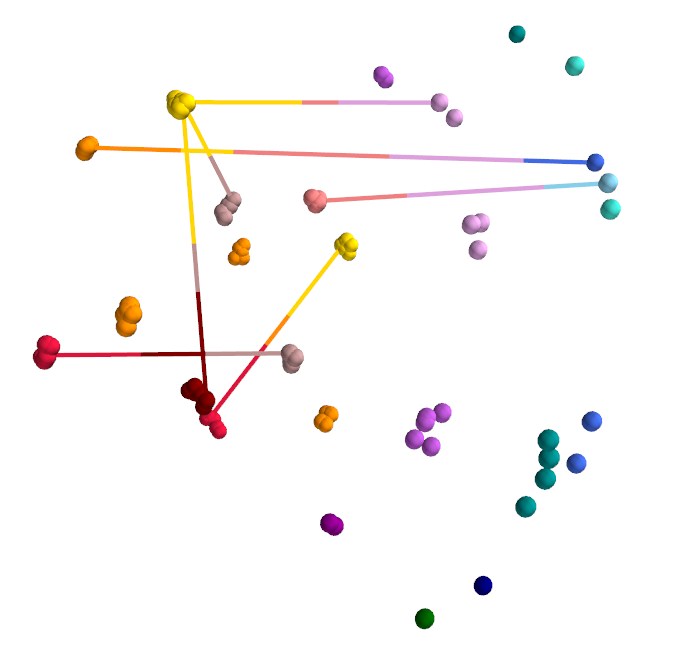}
		\includegraphics[trim={2.0cm 2.0cm 0.5cm 2.0cm},clip,width=0.9\textwidth]{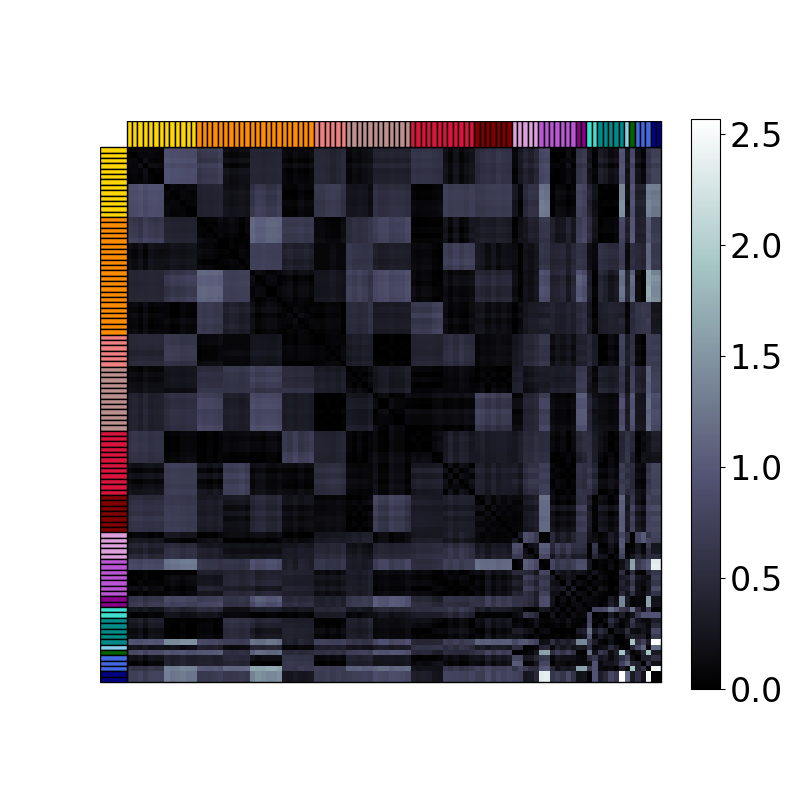}
		\caption{Stress prior}
		\label{fig:GPHLVM:stress_prior_3d}
	\end{subfigure}%
	\begin{subfigure}[b]{0.15\textwidth}
		\centering
		\includegraphics[trim={2.0cm 1.8cm 2.0cm 1.8cm},clip,width=\textwidth]{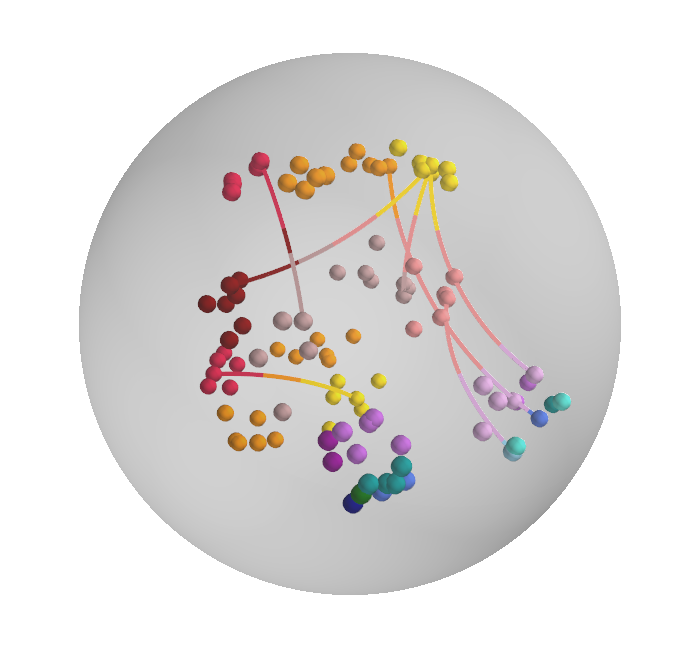}
		\includegraphics[trim={2.0cm 2.0cm 0.5cm 2.0cm},clip,width=0.9\textwidth]{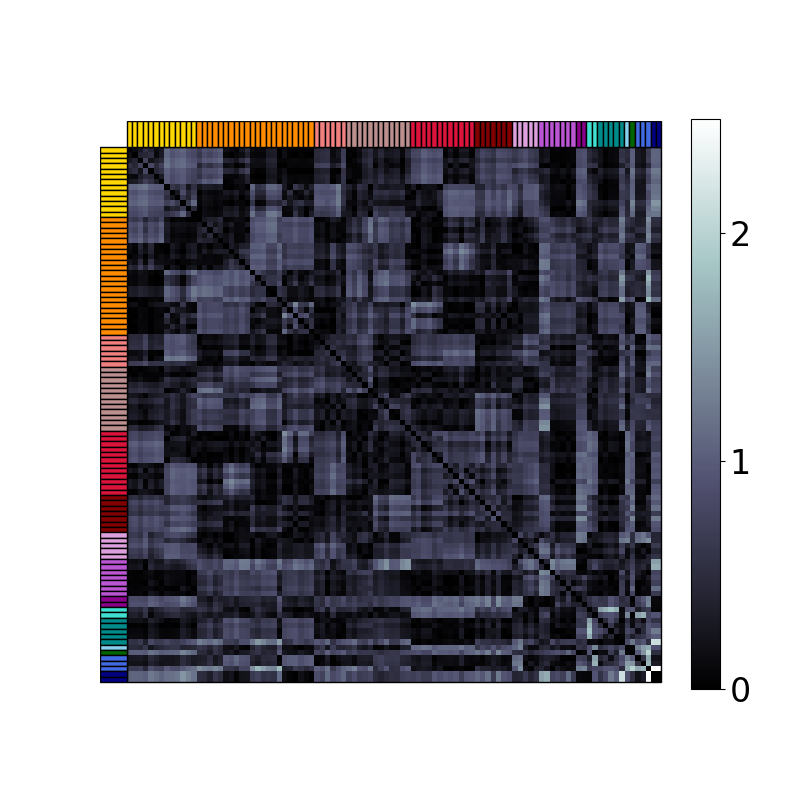}
		\includegraphics[trim={0.0cm 0.0cm 0.0cm 0.0cm},clip,width=\textwidth]{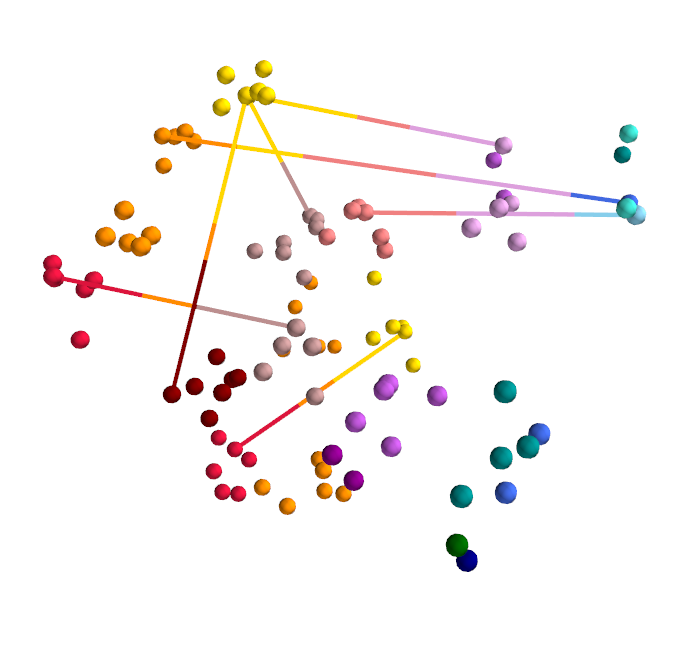}
		\includegraphics[trim={2.0cm 2.0cm 0.5cm 2.0cm},clip,width=0.9\textwidth]{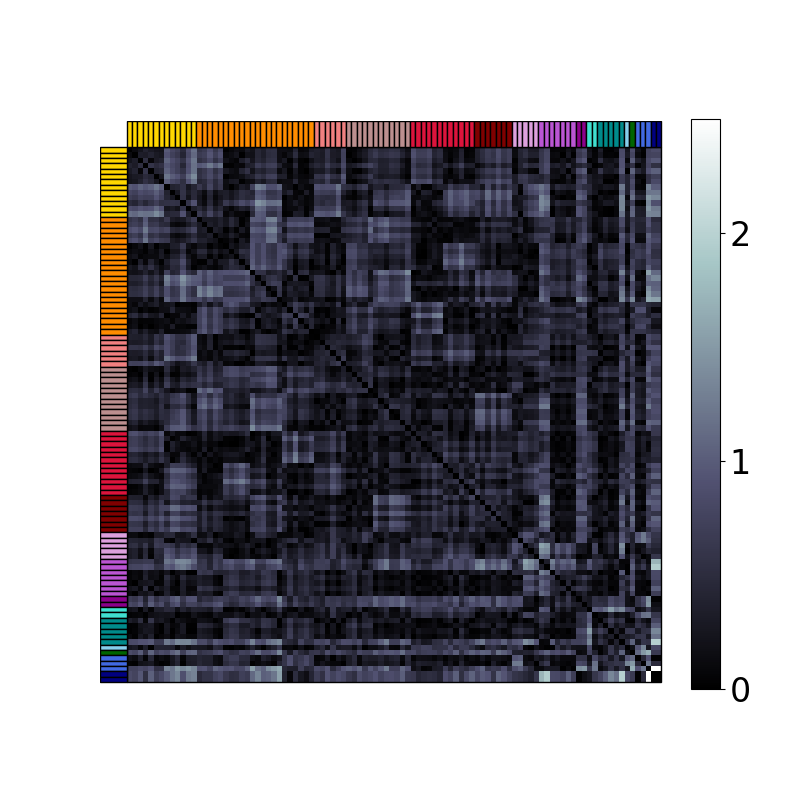}
		\caption{BC + stress prior}
		\label{fig:GPHLVM:backconstrained_and_stress_3d}
	\end{subfigure}%
	\begin{subfigure}[b]{0.15\textwidth}
		\centering
        \includegraphics[trim={2.0cm 1.8cm 2.0cm 1.8cm},clip,width=\textwidth]{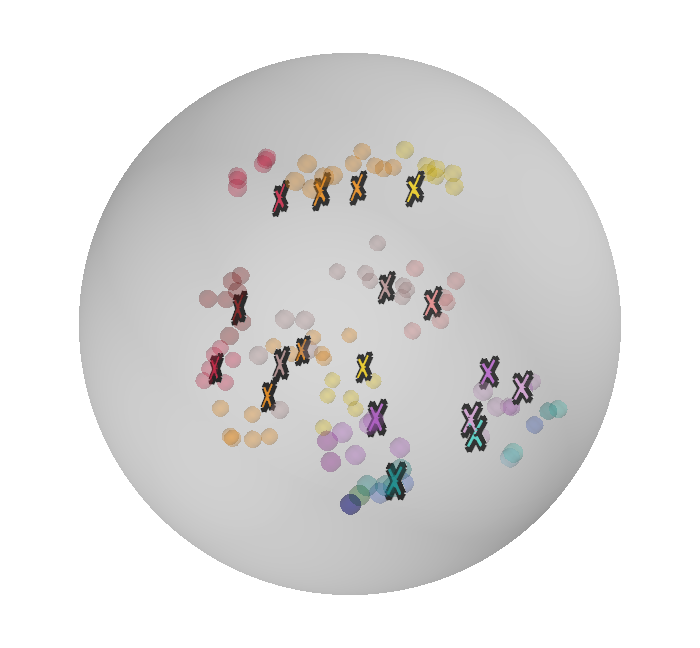}
		\includegraphics[trim={2.0cm 2.0cm 0.5cm 2.0cm},clip,width=0.9\textwidth]{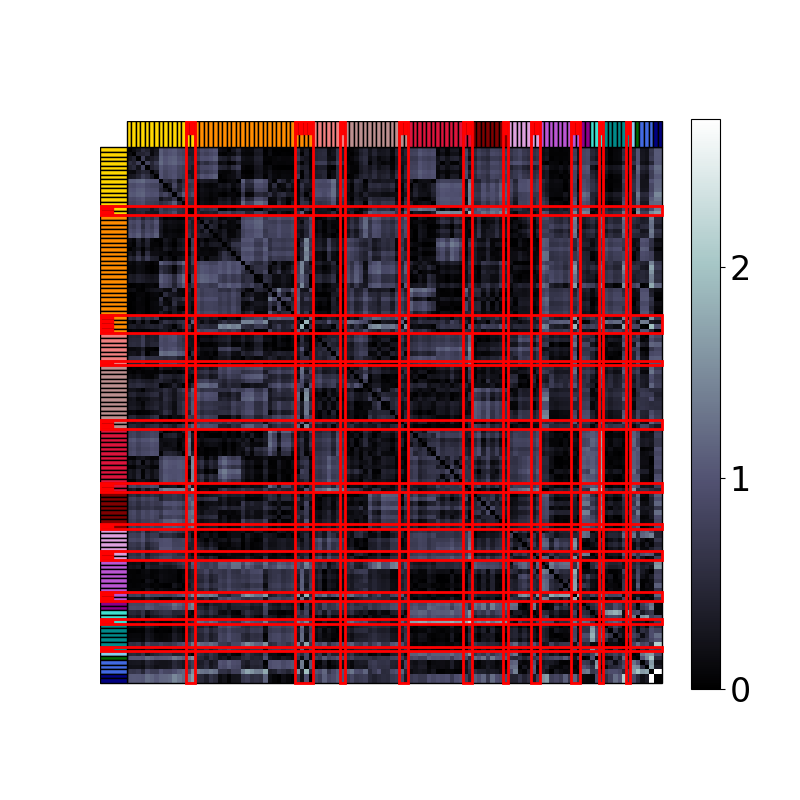}
		\includegraphics[trim={0.0cm 0.0cm 0.0cm 0.0cm},clip,width=\textwidth]{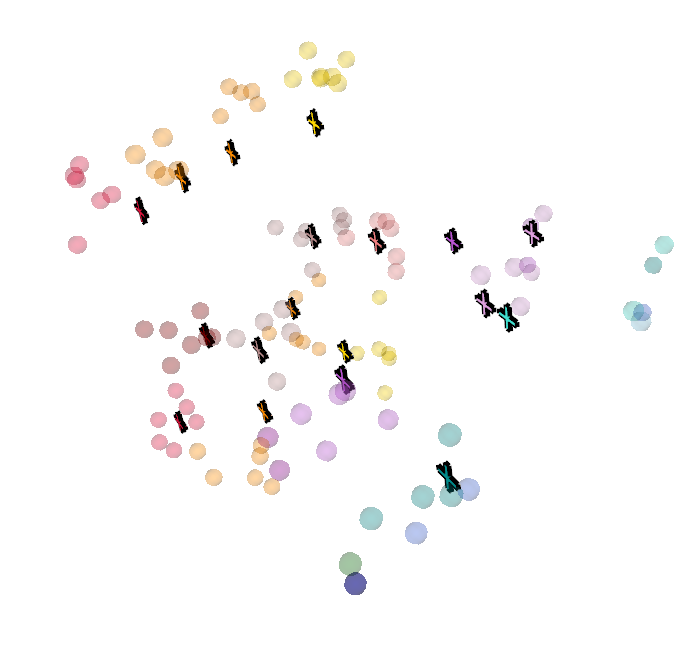}
		\includegraphics[trim={2.0cm 2.0cm 0.5cm 2.0cm},clip,width=0.9\textwidth]{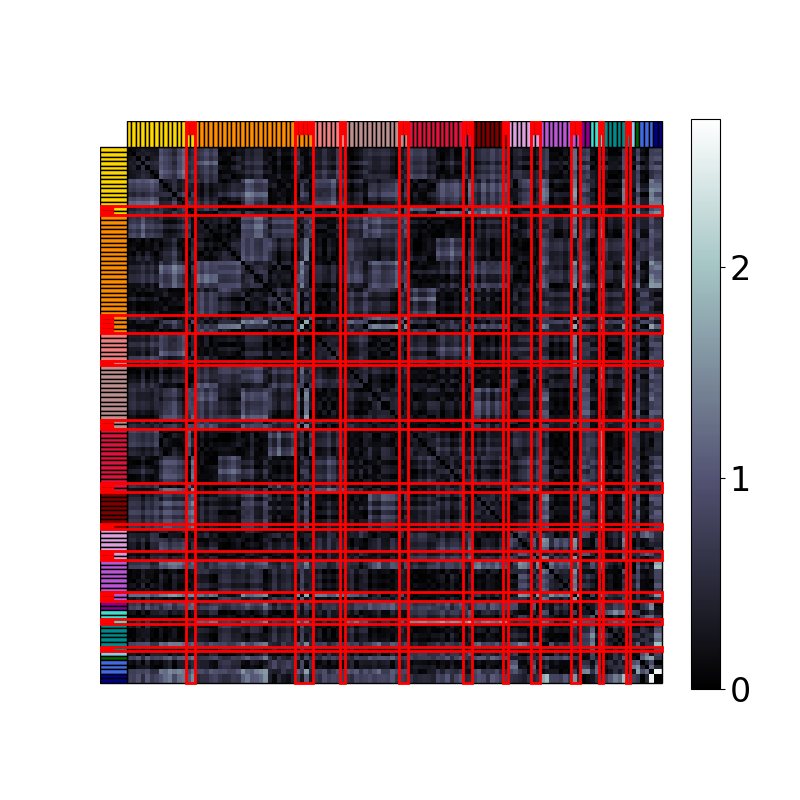}
		\caption{Adding poses}
		\label{fig:GPHLVM:added_poses_3d}
	\end{subfigure}%
	\begin{subfigure}[b]{0.15\textwidth}
		\centering
		\includegraphics[trim={2.0cm 1.8cm 2.0cm 1.8cm},clip,width=\textwidth]{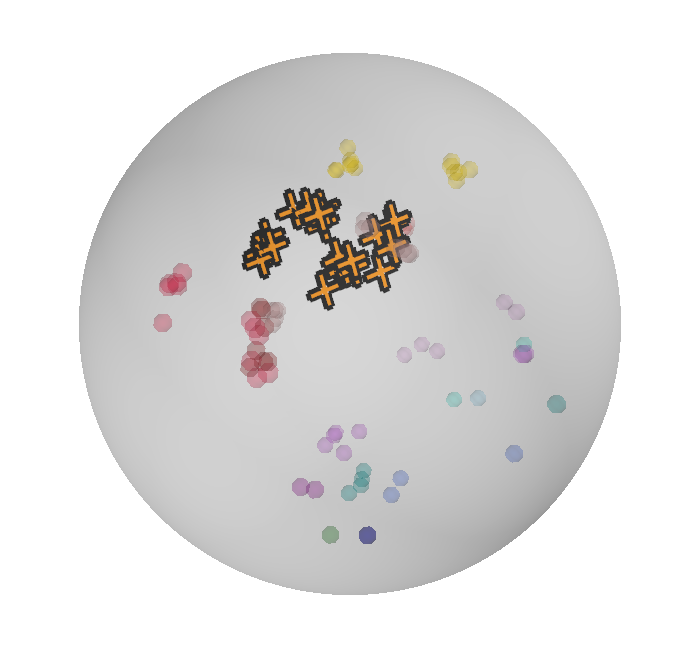}
		\includegraphics[trim={2.0cm 2.0cm 0.5cm 2.0cm},clip,width=0.9\textwidth]{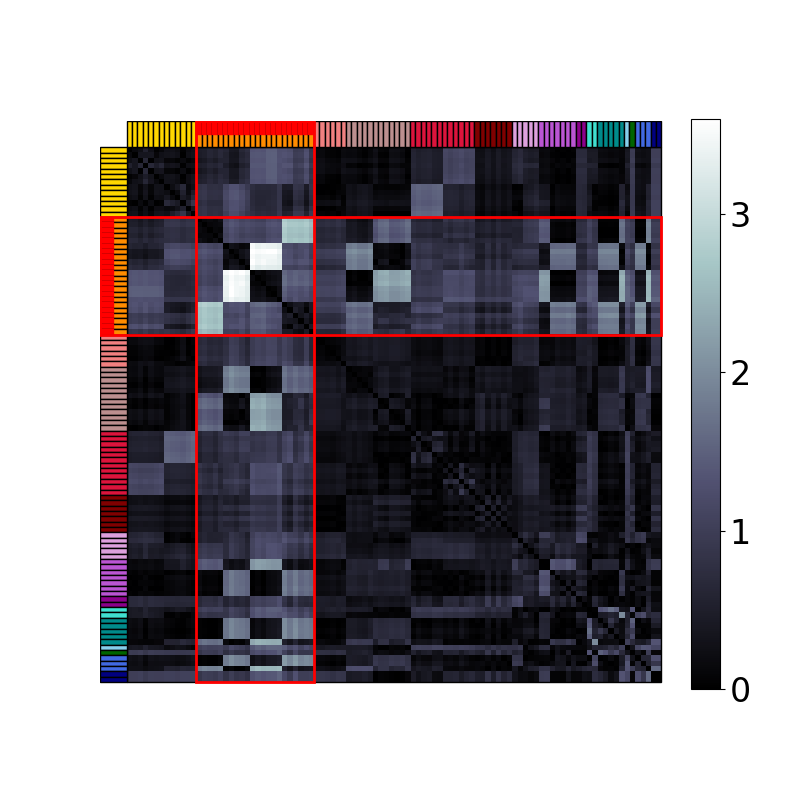}
		\includegraphics[trim={0.0cm 0.0cm 0.0cm 0.0cm},clip,width=\textwidth]{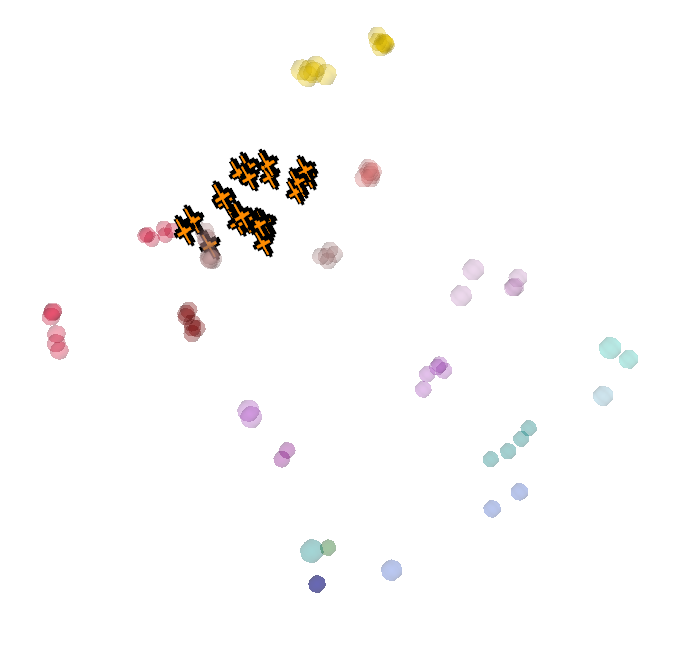}
		\includegraphics[trim={2.0cm 2.0cm 0.5cm 2.0cm},clip,width=0.9\textwidth]{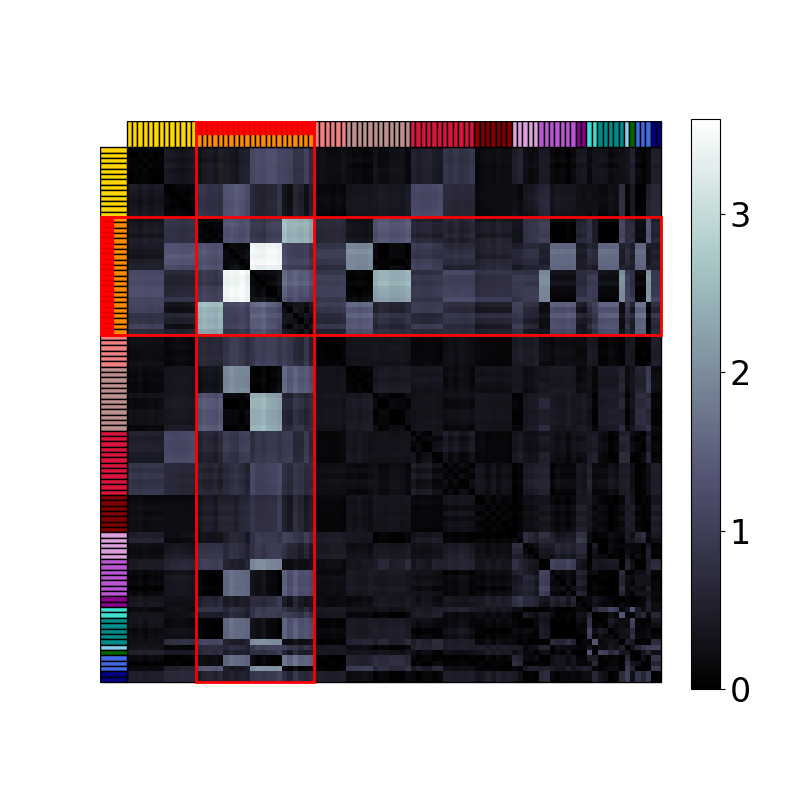}
		\caption{Adding a class}
		\label{fig:GPHLVM:added_class_3d}
	\end{subfigure}
    \vspace{-0.2cm}
	\caption{Support poses: The first and last two rows show the latent embeddings and examples of interpolating geodesics in $\mathcal{P}^3$ and $\mathbb{R}^3$, followed by pairwise error matrices between geodesic and graph distances. Embeddings colors match those of Fig.~\ref{fig:support-poses-taxonomy}. Added poses \emph{(d)} and classes \emph{(e)} are marked with crosses and highlighted with red in the error matrices.}
	\label{fig:GPHLVM:trained_models_3d}
    \vspace{-0.3cm}
\end{figure}

\subsection{Marginal log-likelihoods of trained models}
\begin{table}[h!]
    \caption{Marginal log-likelihood per geometry and regularization.}
    \vspace{-0.2cm}
    \label{table:mll_of_models}
    \begin{center}
    \begin{small}
    \begin{sc}
		\begin{tabular}{lllll}
			\toprule
			    \textbf{Taxonomy} & \textbf{Model}  & \textbf{No reg.} & \textbf{Stress} & \textbf{BC + Stress} \\
			\toprule
			\multirow{4}{*}{\parbox{2.2cm}{Bimanual manipulation categories}} & GPLVM, $\mathbb{R}^2$ &  $79.50$ &  $75.11$ &  $68.76$ \\
			& GPHLVM, $ \lorentz{2}$ & $78.42$ &  $73.38$ & $73.49$ \\
            \cmidrule(lr){2-5}
			& GPLVM, $ \mathbb{R}^3$ &  $83.13$ &  $69.86$ & $83.93$ \\
            & GPHLVM, $\lorentz{3}$ &  $84.55$ &  $68.44$ & $79.77$ \\
            \toprule
			\multirow{4}{*}{Grasps} & GPLVM, $\mathbb{R}^2$ &  $9.97$ &  $4.55$ &  $9.49$\\
			& GPHLVM, $ \lorentz{2}$ &  $7.91$ &  $4.19$ & $5.82$ \\
            \cmidrule(lr){2-5}
			& GPLVM, $ \mathbb{R}^3$ &  $12.15$ &  $5.00$ &  $9.58$ \\
            & GPHLVM, $\lorentz{3}$ &  $9.60$ &  $3.45$ &  $9.15$ \\
            \toprule
            \multirow{4}{*}{\parbox{1.5cm}{Support poses}} & GPLVM, $\mathbb{R}^2$ &  $6.96$ &  $-13.30$ &  $-6.06$\\
			& GPHLVM, $ \lorentz{2}$ &  $5.52$ &  $-12.29$ &  $-7.47$ \\
            \cmidrule(lr){2-5}
			& GPLVM, $ \mathbb{R}^3$ &  $10.63$ &  $-14.35$ &  $-4.90$ \\
            & GPHLVM, $\lorentz{3}$ &  $8.71$ &  $-15.43$ &  $-4.14$\\
			\bottomrule
	\end{tabular}
    \end{sc}
    \end{small}
    \end{center}
\end{table}

Table~\ref{table:mll_of_models} shows the marginal log-likelihood (MLL) 
\begin{equation}
    p(\mathbf{Y}) = p(\mathbf{Y}|\mathbf{X}, \mathbf{\Theta})p(\mathbf{X})p(\mathbf{\Theta})
    \label{eq:MLL}
\end{equation}
of the GPHLVM and GPLVM described in \S~\ref{sec:experiments}. 
We observe that the marginal log-likelihood of the models with regularization is slightly lower than that of the models without regularization. This is due to the combination of the two losses $\ell_{\text{MAP}}$ and $\ell_{\text{stress}}$ when training the regularized models, resulting in a trade-off. In other words, we expect the non-regularized models to achieve the highest MLL. Interestingly, the gap between the MLL of non-regularized and regularized models is reduced for the bimanual manipulation and grasping taxonomies compared to the support pose taxonomy. We hypothesize that this is due to the tree structure of the two former taxonomies, which are ideally embedded in hyperbolic spaces. 
We would like to emphasize that the MLL~\eqref{eq:MLL} depends on the prior distribution $p(\mathbf{X})$, which itself is defined based on the geometry of the manifold, and on the prior $p(\mathbf{\Theta})$ imposed on the model parameters, which also differs across geometries (see Table~\ref{tab:experiment_hyperparameters}). Therefore, comparing the values of the MLL across geometries may generally be misleading.

\subsection{Additional motions obtained via geodesic interpolation and comparisons}
\label{app:experiment_geodesic_motions}
\begin{figure}[h!]
	\centering
    \begin{subfigure}[b]{0.8\textwidth}
		\centering
		\includegraphics[trim={0cm 0cm 0cm 0cm},clip,width=\textwidth]{Figures/Motions/Ri-IE.png}
        \includegraphics[trim={0cm 0cm 0cm 0cm},clip,width=\textwidth]{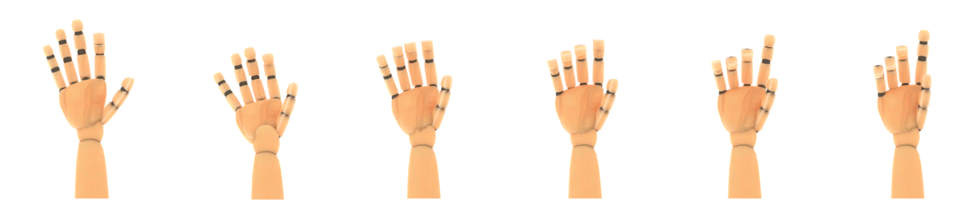}
		\caption{Ring ($\mathsf{Ri}$) to
			index finger extension ($\mathsf{IE}$)}
		\label{fig:grasps_Ri-IE}
	\end{subfigure}
    \\
	\begin{subfigure}[b]{0.8\textwidth}
		\centering
		\includegraphics[trim={0cm 0cm 0cm 0.25cm},clip,width=\textwidth]{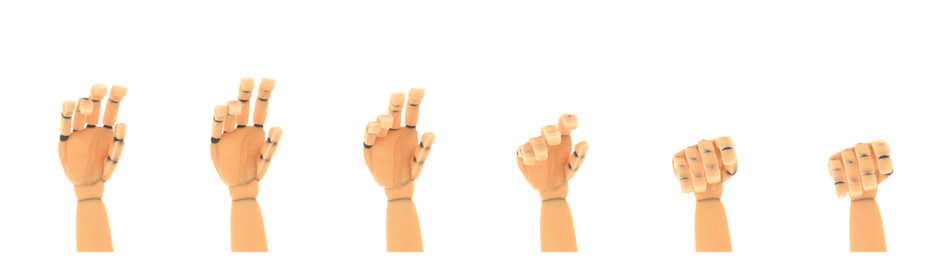}
        \includegraphics[trim={0cm 0cm 0cm 0.5cm},clip,width=\textwidth]{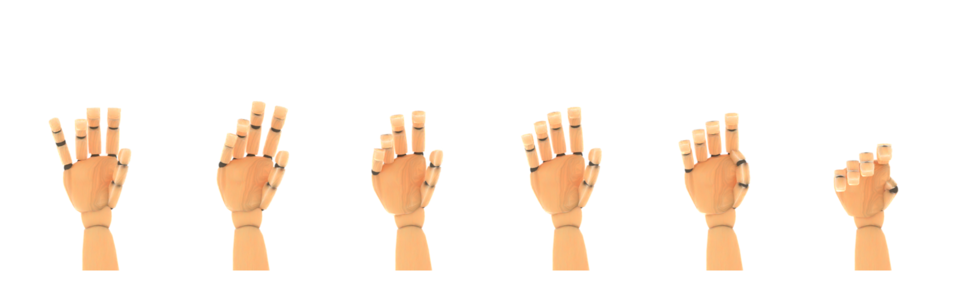}
		\caption{Quadpod ($\mathsf{Qu}$) to
			parallel extension ($\mathsf{PE}$)}
		\label{fig:grasps_Qu-PE}
	\end{subfigure}
	\\
	\begin{subfigure}[b]{0.8\textwidth}
		\centering
		\includegraphics[trim={0cm 0cm 0cm 0.25cm},clip,width=\textwidth]{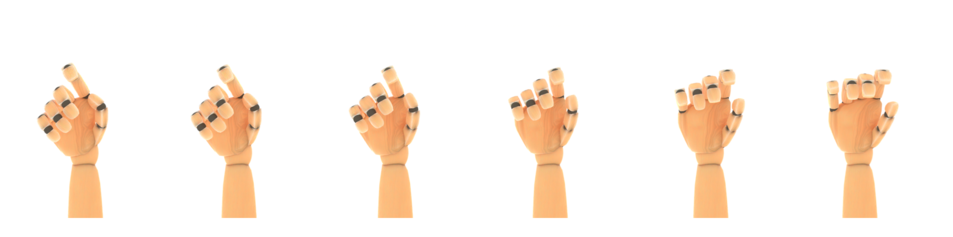}
        \includegraphics[trim={0cm 0cm 0cm 0.4cm},clip,width=\textwidth]{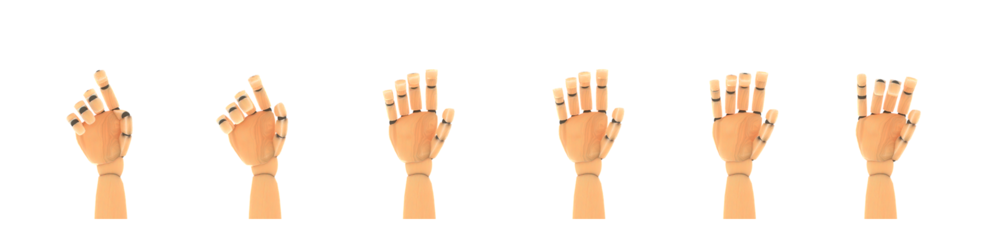}
		\caption{Small diameter ($\mathsf{SD}$) to
			tripod ($\mathsf{Tr}$)}
		\label{fig:grasps_SD-Tr}
	\end{subfigure}
	\caption{Generated motions for grasps. \emph{Top:} Motions obtained via geodesic interpolation in the latent space of the back-constrained GPHLVM trained on the the hand grasp taxonomy (Fig.~\ref{fig:GPHLVM-grasps:backconstrained_and_stress}). \emph{Bottom:} Motions obtained via linear interpolation in the latent space of the corresponding back-constrained GPLVM.}
	\label{fig:GPHLVM:trajectories_grasps1}
\end{figure}

\begin{figure}[h!]
	\centering
	\begin{subfigure}[b]{0.8\textwidth}
		\centering
		\includegraphics[trim={0cm 0cm 0cm 0.5cm},clip,width=\textwidth]{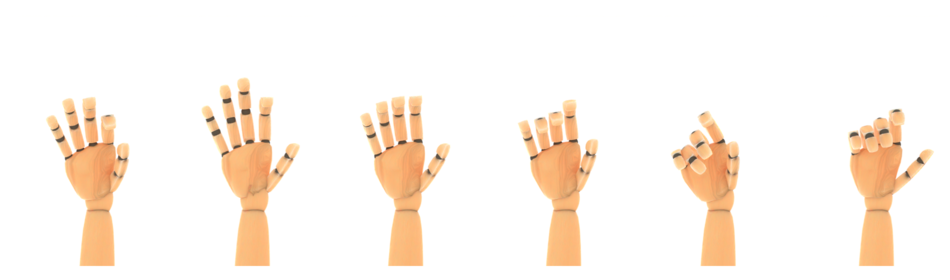}
        \includegraphics[trim={0cm 0cm 0cm 0.5cm},clip,width=\textwidth]{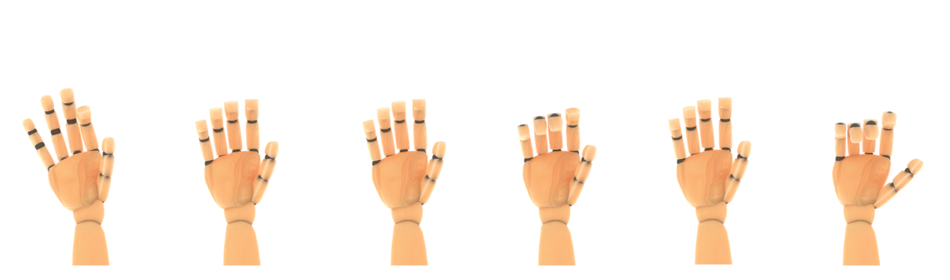}
		\caption{Tip pinch ($\mathsf{TP}$) to
			fixed hook ($\mathsf{FH}$)}
		\label{fig:grasps_TP-FH}
	\end{subfigure}
	\\
	\begin{subfigure}[b]{0.8\textwidth}
		\centering
		\includegraphics[trim={0cm 0cm 0cm 0.25cm},clip,width=\textwidth]{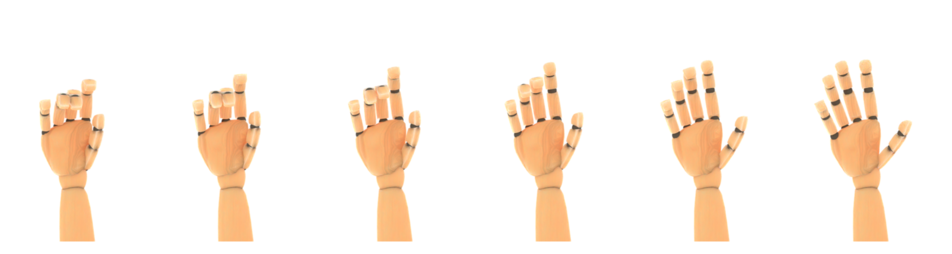}
        \includegraphics[trim={0cm 0cm 0cm 0.25cm},clip,width=\textwidth]{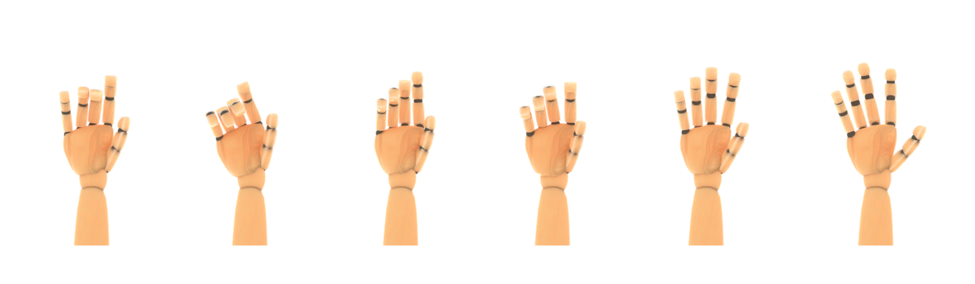}
		\caption{Writing tripod ($\mathsf{WT}$) to
			power disk ($\mathsf{PD}$)}
		\label{fig:grasps_WT-PD}
	\end{subfigure}
	\caption{Generated motions for grasps. \emph{Top:} Motions obtained via geodesic interpolation in the latent space of the back-constrained GPHLVM trained on the the hand grasp taxonomy (Fig.~\ref{fig:GPHLVM-grasps:backconstrained_and_stress}). \emph{Bottom:} Motions obtained via linear interpolation in the latent space of the corresponding back-constrained GPLVM.}
	\label{fig:GPHLVM:trajectories_grasps2}
\end{figure}

Figs.~\ref{fig:GPHLVM:trajectories_grasps1} and~\ref{fig:GPHLVM:trajectories_grasps2} show additional examples of motions obtained via geodesic interpolation between two embeddings of the hand grasps taxonomy in the latent space of the GPHLVM. The generated motions look realistic and smoothly interpolate between the given initial and final grasps. 
In comparison, motions obtained via linear interpolation between two embeddings in the Euclidean latent space of the GPLVM are less realistic. They display less regular interpolation patterns (see Fig.~\ref{fig:grasps_SD-Tr}) and are often noisy, featuring wavering wrist or finger motions (see Figs.~\ref{fig:grasps_Ri-IE},~\ref{fig:grasps_Qu-PE}, and~\ref{fig:grasps_WT-PD}). This is supported by the higher average jerkiness of the motions generated from the GPLVM compared to those generated from the GPHLVM, as reported in Table~\ref{table:jerkiness_of_models}. Moreover, the generated grasps reflect less accurately the taxonomy categories (see, e.g., the parallel extension ($\mathsf{PE}$) grasp in Fig.~\ref{fig:grasps_Qu-PE} or the tip pinch ($\mathsf{TP}$ grasp of Fig.~\ref{fig:grasps_TP-FH}).
Interestingly, the geodesic interpolation between two grasps in the latent space of the GPHLVM allows us to generate unobserved transitions between hand grasps. As such, it offers us a mechanism to generate data that are generally difficult to collect via human motion recordings.

\begin{figure}[h!]
	\centering
    \begin{subfigure}[b]{0.65\textwidth}
        \centering
        \includegraphics[trim={0.0cm 0.55cm 0.0cm 0.25cm},clip,width=\textwidth]{Figures/Motions/LF-LKRKRH.png}
        \includegraphics[trim={0.0cm 0.65cm 0.0cm 0.25cm},clip,width=\textwidth]{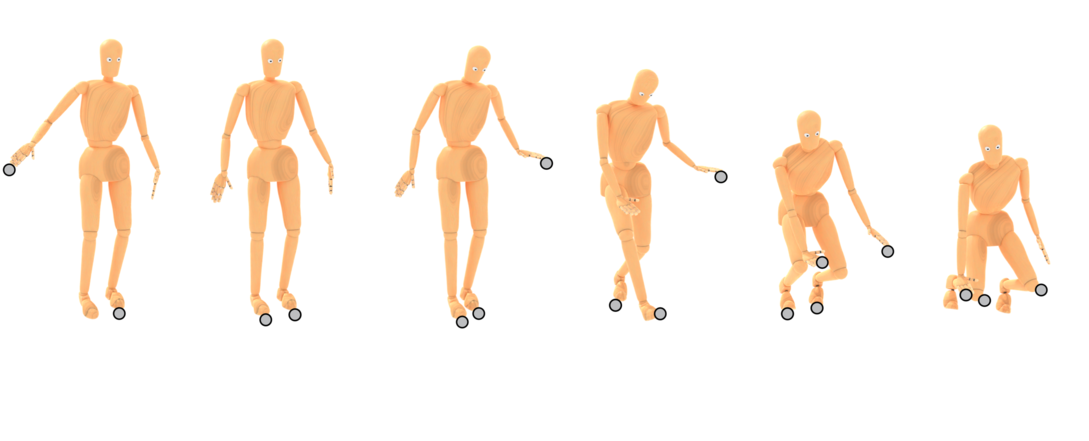}
        \includegraphics[trim={0.0cm 0.0cm 0.0cm 0.0cm},clip,width=\textwidth]{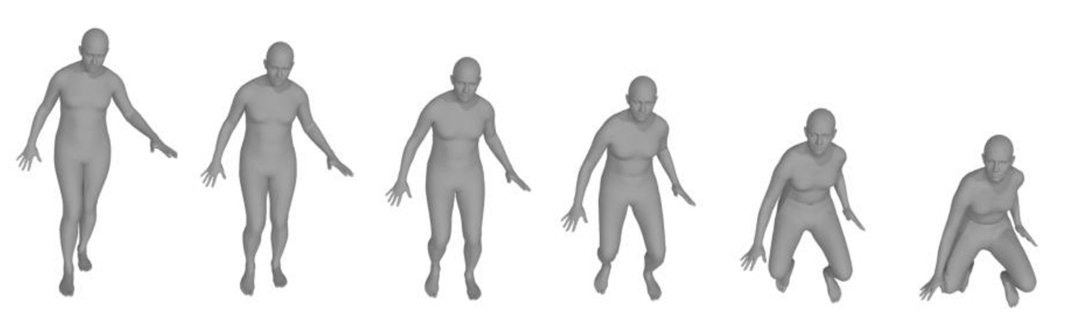}
        \caption{$\mathsf{LF}\mathsf{RH}$ to $\mathsf{K}_2\mathsf{RH}$}
        \label{fig:augmentedfeet_geodesic}
    \end{subfigure}
    \vspace{-0.2cm}
	\caption{Generated motions for support poses. \emph{Top:} Motion obtained via geodesic interpolation in the latent space of the back-constrained GPHLVM trained on the support pose taxonomy (Fig.~\ref{fig:GPHLVM:backconstrained_and_stress}). \emph{Middle:} Motion obtained via linear interpolation in the latent space of the corresponding back-constrained GPLVM. \emph{Bottom:} Motion obtained via linear interpolation in the latent space of VPoser. Contacts are depicted by gray circles in the two first rows. }
	\label{fig:GPHLVM:trajectories_support_poses1}
\end{figure}

\begin{figure}[h!]
	\centering
	\begin{subfigure}[b]{0.65\textwidth}
		\centering
		\includegraphics[trim={0.0cm 0.25cm 0.0cm 0.25cm},clip,width=\textwidth]{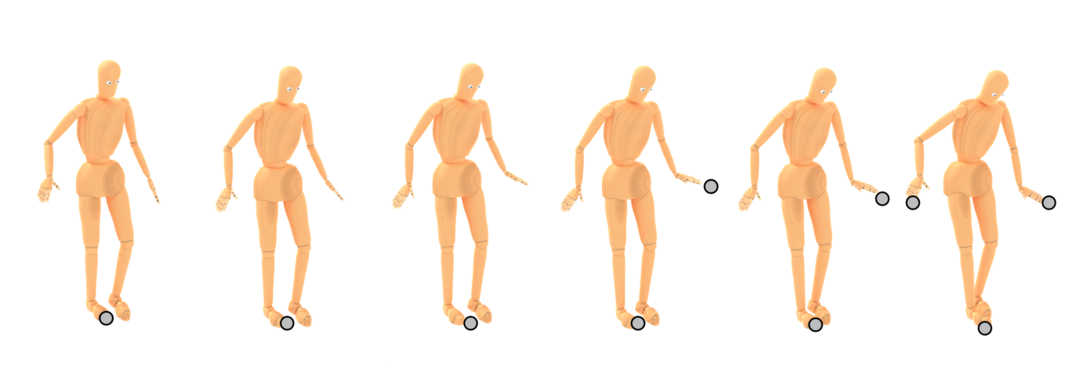}
        \includegraphics[trim={0.0cm 0.1cm 0.0cm 0.4cm},clip,width=\textwidth]{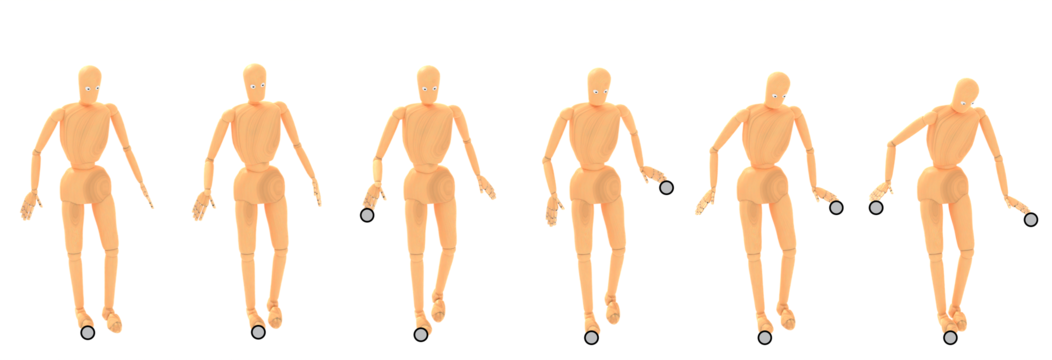}
        \includegraphics[trim={0.0cm 0.0cm 0.0cm 0.0cm},clip,width=\textwidth]{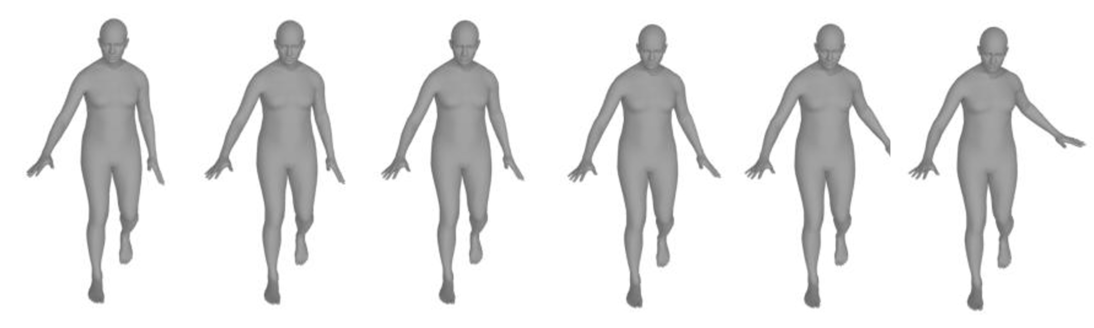}
		\caption{$\mathsf{RF}$ to $\mathsf{RF H}_2$}
		\label{fig:augmentedfeet_geodesic0}
	\end{subfigure}
	\\
	\begin{subfigure}[b]{0.65\textwidth}
		\centering
		\includegraphics[trim={0.0cm 0.4cm 0.0cm 0.1cm},clip,width=\textwidth]{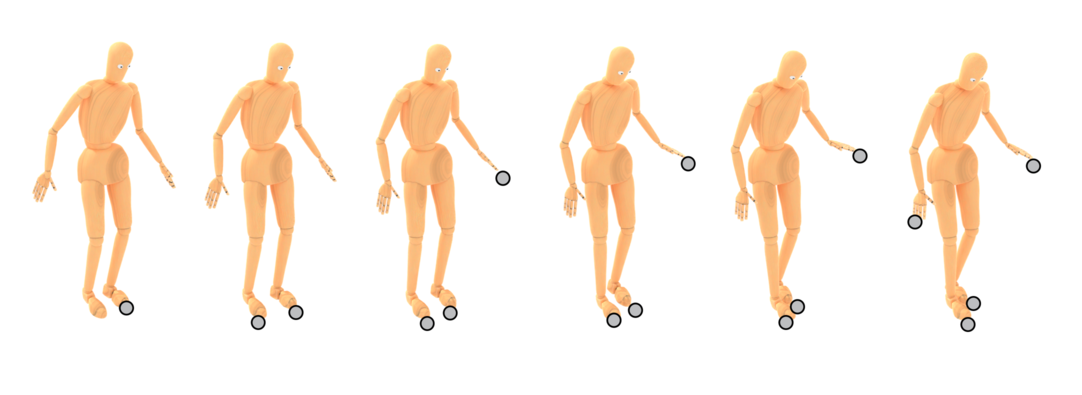}
        \includegraphics[trim={0.0cm 0.3cm 0.0cm 0.15cm},clip,width=\textwidth]{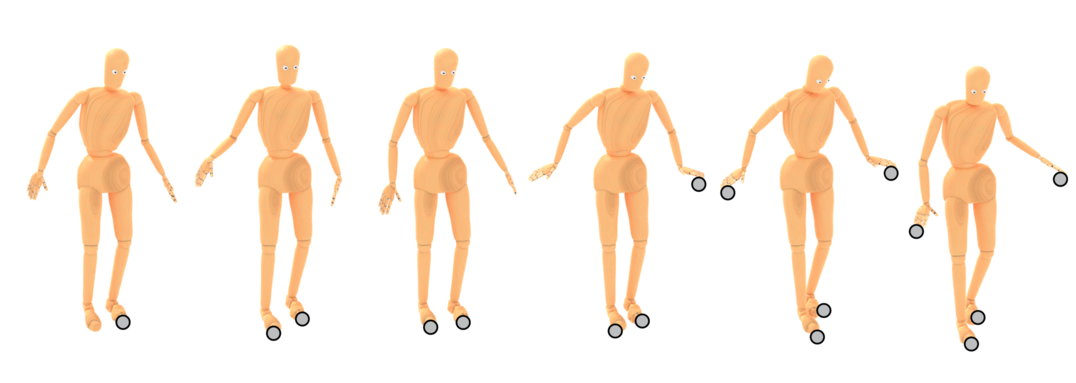}
        \includegraphics[trim={0.0cm 0.0cm 0.0cm 0.0cm},clip,width=\textwidth]{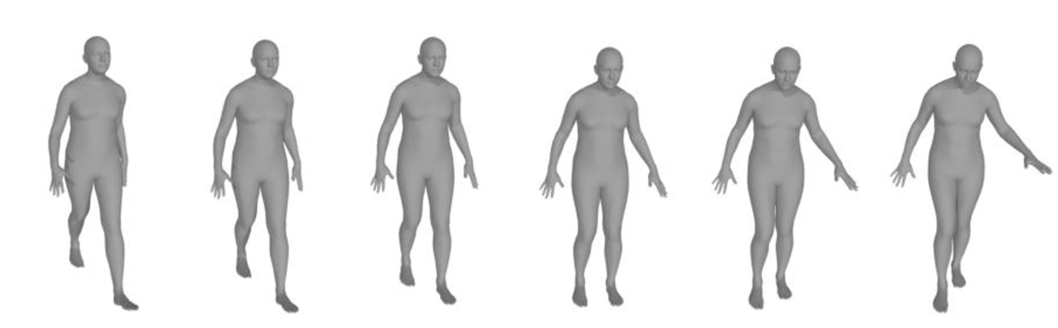}
		\caption{$\mathsf{LF}$ to $\mathsf{F_2} \mathsf{H}_2$}
		\label{fig:augmentedfeet_geodesic1}
	\end{subfigure}
    \vspace{-0.2cm}
	\caption{Generated motions for support poses. \emph{Top:} Motions obtained via geodesic interpolation in the latent space of the back-constrained GPHLVM trained on the support pose taxonomy (Fig.~\ref{fig:GPHLVM:backconstrained_and_stress}). \emph{Middle:} Motions obtained via linear interpolation in the latent space of the corresponding back-constrained GPLVM. \emph{Bottom:} Motions obtained via linear interpolation in the latent space of VPoser. Contacts are depicted by gray circles. }
	\label{fig:GPHLVM:trajectories_support_poses2}
	\vspace{-0.3cm}
\end{figure}

\begin{figure}[h!]
	\centering
	\begin{subfigure}[b]{0.65\textwidth}
		\centering
		\includegraphics[trim={0.0cm 0.25cm 0.0cm 0.25cm},clip,width=\textwidth]{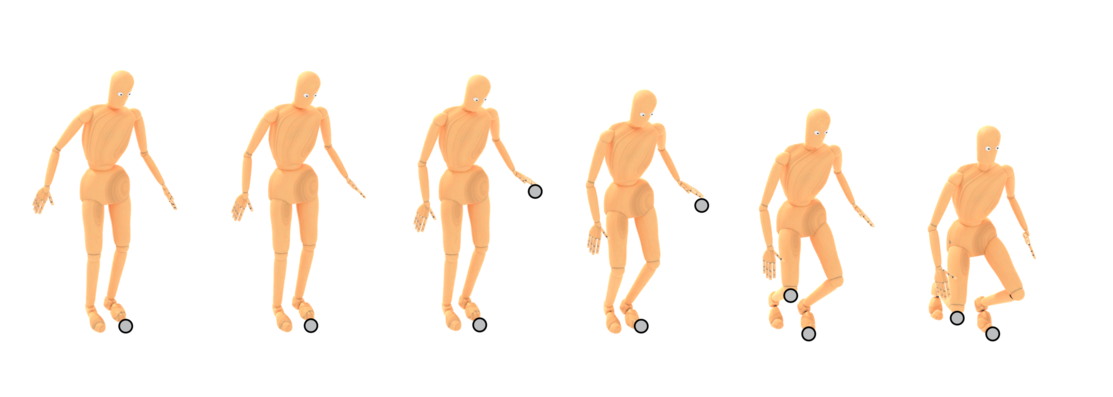}
        \includegraphics[trim={0.0cm 0.1cm 0.0cm 0.55cm},clip,width=\textwidth]{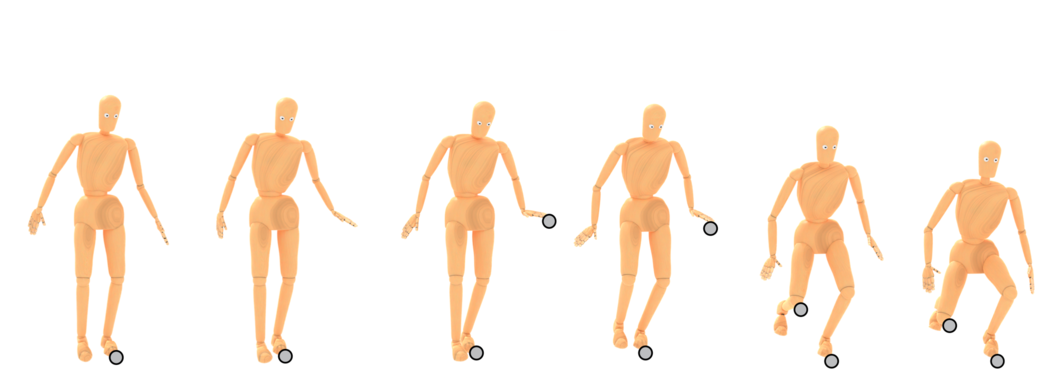}
        \includegraphics[trim={0.0cm 0.0cm 0.0cm 0.0cm},clip,width=\textwidth]{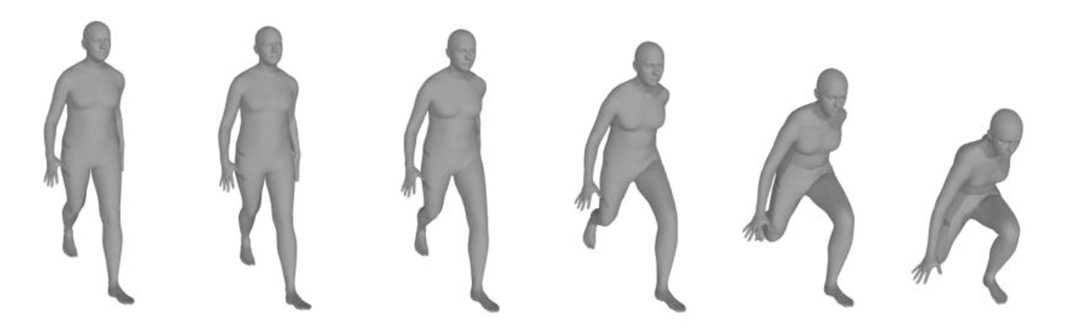}
		\caption{$\mathsf{LF}$ to $\mathsf{LF RK}$}
		\label{fig:augmentedfeet_geodesic2}
	\end{subfigure}
	\\
	\begin{subfigure}[b]{0.65\textwidth}
		\centering
		\includegraphics[trim={0.0cm 0.65cm 0.0cm 0.4cm},clip,width=\textwidth]{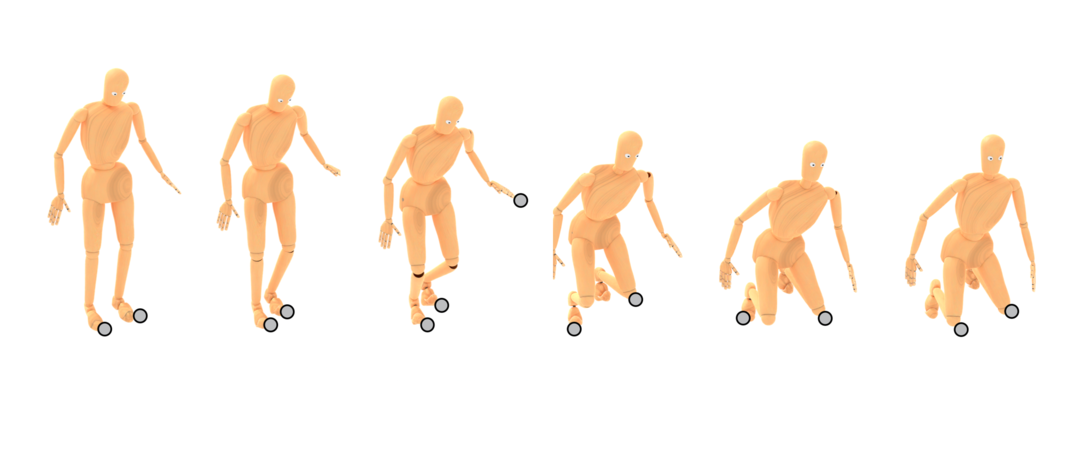}
        \includegraphics[trim={0.0cm 0.4cm 0.0cm 0.65cm},clip,width=\textwidth]{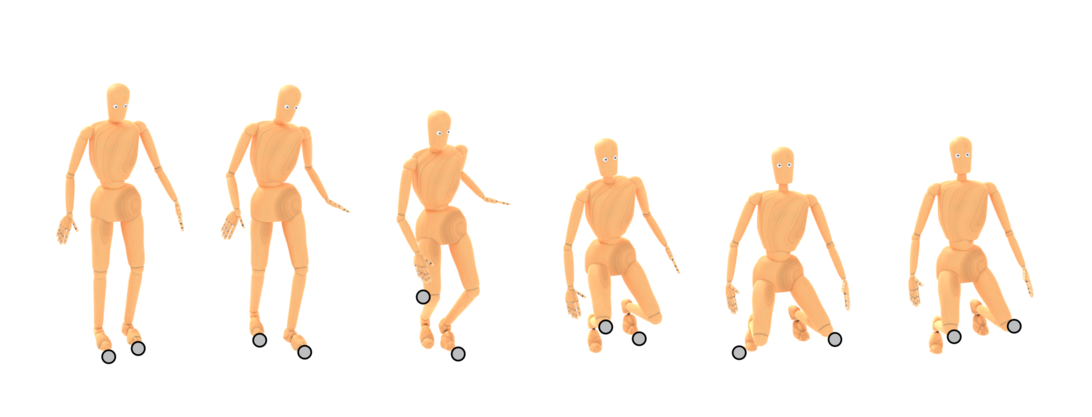}
        \includegraphics[trim={0.0cm 0.0cm 0.0cm 0.0cm},clip,width=\textwidth]{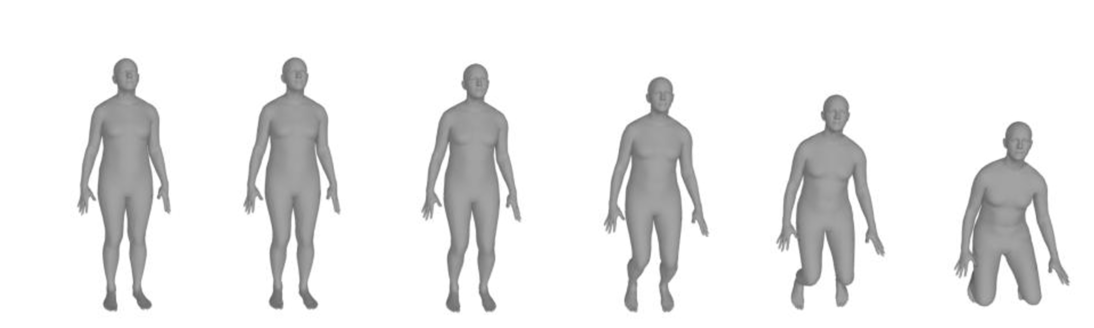}
		\caption{$\mathsf{F}_2$ to $\mathsf{K}_2$}
		\label{fig:augmentedfeet_geodesic3}
	\end{subfigure}
    \vspace{-0.2cm}
	\caption{Generated motions for support poses. \emph{Top:} Motions obtained via geodesic interpolation in the latent space of the back-constrained GPHLVM trained on the support pose taxonomy (Fig.~\ref{fig:GPHLVM:backconstrained_and_stress}). \emph{Middle:} Motions obtained via linear interpolation in the latent space of the corresponding back-constrained GPLVM. \emph{Bottom:} Motions obtained via linear interpolation in the latent space of VPoser. Contacts are depicted by gray circles. }
	\label{fig:GPHLVM:trajectories_support_poses3}
	\vspace{-0.3cm}
\end{figure}

Figs.~\ref{fig:GPHLVM:trajectories_support_poses1}-\ref{fig:GPHLVM:trajectories_support_poses3} show additional examples of motions obtained via geodesic interpolation between two embeddings of the whole-body support pose taxonomy in the latent space of the GPHLVM. The generated motions look realistic, smoothly interpolate between the given initial and final body poses, and are consistent with the transitions between classes encoded in the taxonomy. 
In comparison, motions obtained via linear interpolation between two embeddings in the Euclidean latent space of the GPLVM look less realistic and are less smooth (see Table~\ref{table:jerkiness_of_models}). In particular, the resulting kneeling poses often look unnatural (see Figs.~\ref{fig:GPHLVM:trajectories_support_poses1} and~\ref{fig:GPHLVM:trajectories_support_poses3}).

\begin{table}[h!]
    \caption{Average jerkiness (a.k.a. smoothness~\citep{Balasubramanian15:Smoothness}) of the motions obtained via linear and geodesic interpolation in the latent space of the back-constrained GPLVMs and GPHLVMs.}
    \vspace{-0.3cm}
    \label{table:jerkiness_of_models}
    \begin{center}
    \begin{small}
    \begin{sc}
		\begin{tabular}{lll}
			\toprule
			    \textbf{Taxonomy} & \textbf{Model} & \textbf{Jerkiness} \\
			\toprule
   			\multirow{2}{*}{Grasps} & GPLVM, $\mathbb{R}^2$ & $1377.05\pm 1721.44$\\
			& GPHLVM, $ \lorentz{2}$ & $\bm{108.65\pm 140.54}$ \\
            \toprule
            \multirow{2}{*}{\parbox{1.5cm}{Support poses}} & GPLVM, $\mathbb{R}^2$ &  $210.08\pm 228.97$\\
			& GPHLVM, $ \lorentz{2}$ &  $\bm{27.15\pm 27.58}$ \\
			\bottomrule
	\end{tabular}
    \end{sc}
    \end{small}
    \end{center}
\end{table}

We also compare the trajectories generated via geodesic interpolation with the trajectories generated in the latent space of VPoser~\citep[Sec. 3.3]{Pavlakos19:VPoser}, a state-of-the art human pose latent space obtained from a VAE trained on MoCap data and used to generate human motions. VPoser was introduced by~\citet{Pavlakos19:VPoser} as a body pose prior to address the problem of building a full 3D model of human gestures by learning a deep neural network that jointly models the human body, face and hands from RBG images.
\citet{Pavlakos19:VPoser} released the weights of their model under a non-commercial licence.\footnote{\url{https://smpl-x.is.tue.mpg.de/}} Of the two models available, we downloaded version 2, and followed the instructions on their repository for set-up.\footnote{\url{https://github.com/nghorbani/human_body_prior} (\texttt{vposer.ipynb}).} Since our human poses used a different number of joints, we searched inside the KIT dataset part of the AMASS dataset~\citep{AMASS:2019} for similar poses with the same contacts configuration. Table~\ref{tab:app:exact_poses_used_for_vposer} shows the exact poses used in the comparison. These poses were embedded into the latent space of VPoser.
The motions obtained via linear interpolation in the space of VPoser are displayed in the bottom rows of Figs.~\ref{fig:GPHLVM:trajectories_support_poses1}-\ref{fig:GPHLVM:trajectories_support_poses3}. We observe that the motions generated by our approach are as realistic as the ones obtained from VPoser.
It is worth noticing that VPoser is trained on full human motion trajectories and a large dataset of $1$M datapoints. Therefore, it is natural that it can retrieve realistic human motions. This is also the case for other models such as TEACH~\cite{Athanasiou22:TEACH} and text-conditioned human motion diffusion models~\citep{Shafir23:MDM}, which are trained on full human motion trajectories and conditioned on textual prompts to generate sequences of human motions. In contrast, the GPHLVM is not trained on full trajectories, but only on $100$ single human poses. Instead, GPLHVM leverages the robotic taxonomy and geodesic interpolation as a motion generation mechanism. Notice that the latent space of the GPHLVM is of low dimension compared to the 32-dimensional latent space of VPoser.

\begin{table}[h]
    \caption{Poses used when comparing with VPoser~\citep{Pavlakos19:VPoser}. In our notation, the files inside the KIT subset of AMASS~\citep{AMASS:2019} are structured into subfolders of name \texttt{entry\_id}; each \texttt{.npz} file contains an array of body poses, and the exact pose used in the comparison is specified by the index $t$.}
    \vspace{-0.3cm}
    \label{tab:app:exact_poses_used_for_vposer}
    \begin{center}
    \begin{small}
    \resizebox{\textwidth}{!}{
    \begin{tabular}{lll}
        \toprule
        \textbf{\textsc{Trajectory}} & \textbf{\textsc{File for source} (\texttt{entry\_id, $t=$index})} & \textbf{\textsc{File for target} (\texttt{entry\_id, $t=$index})} \\
        \midrule
        $\mathsf{LF}\mathsf{RH}$ to $\mathsf{K}_2\mathsf{RH}$ (Fig.~\ref{fig:GPHLVM:trajectories_support_poses1}) & Walk w. handrail table beam, left, Nr. 01 (675, $t=250$) & Kneel up w. right hand, Nr. 01 (3, $t=185$) \\
        $\mathsf{RF}$ to $\mathsf{RF H}_2$ (Fig.~\ref{fig:augmentedfeet_geodesic0}) & Walk w. handrail table beam, left, Nr. 01 (675, $t=100$) & Walk w. handrail table beam, left, Nr. 01 (675, $t=300$) \\
        $\mathsf{LF}$ to $\mathsf{F_2} \mathsf{H}_2$ (Fig.~\ref{fig:augmentedfeet_geodesic1}) & Walk at medium speed Nr. 01 (450, $t=320$) & Walk w. handrail table beam, left, Nr. 01 (675, $t=250$) \\
        $\mathsf{LF}$ to $\mathsf{LF RK}$ (Fig.~\ref{fig:augmentedfeet_geodesic2}) & Walk at medium speed Nr. 01 (450, $t=320$) & Kneel up w. right hand Nr. 09 (3, $t=150$) \\
        $\mathsf{F}_2$ to $\mathsf{K}_2$ (Fig.~\ref{fig:augmentedfeet_geodesic3}) & Walk at medium speed Nr. 01 (450, $t=10$) & Kneel up w. left hand Nr. 01 (3, $t=50$) \\
        \bottomrule
    \end{tabular}
    }
    \end{small}
    \end{center}
\end{table}

It is important to emphasize that augmenting the support pose taxonomy to explicitly distinguish between left and right contact is crucial for generating realistic motions with the GPHLVM. With the original taxonomy, poses with very different feet and hands positions may belong to the same class. For instance, \emph{a right foot contact with a left hand contact on the handrail} or a \emph{left foot contact with a right hand contact on the table} both belong to the same $\mathsf{F}\mathsf{H}$ node in the original taxonomy. 
In contrast, differentiating between left and right contacts allows very different poses to be placed far apart in the latent space. For instance, the two aforementioned poses are identified with the nodes $\mathsf{LF}\mathsf{RH}$ and $\mathsf{RF}\mathsf{LH}$ in the augmented taxonomy.

\section{Additional comparisons}
\label{app:comparisons}
\subsection{Comparison against Variational Autoencoders}
\label{app:comparison_VAEs}
\begin{figure}
	\centering
	\includegraphics[trim={5.8cm 2.2cm 4.3cm 2.2cm},clip,width=0.8\textwidth]{Figures/legend_bimanual.png}
    \begin{subfigure}[b]{0.15\textwidth}
		\centering
		\includegraphics[trim={2.5cm 2.5cm 2.5cm 2.5cm},clip,width=\textwidth]{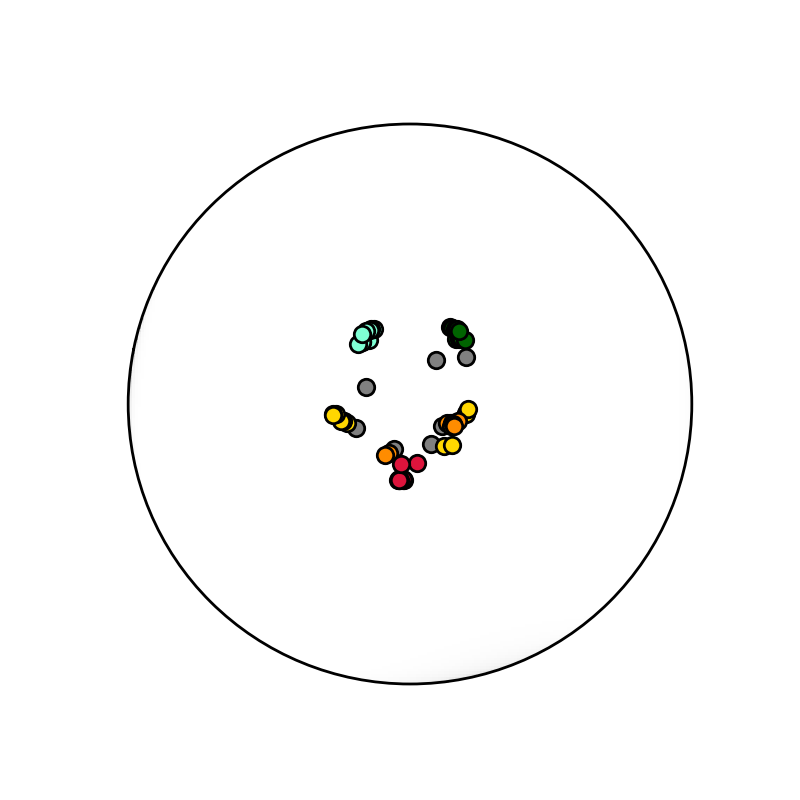}
		\includegraphics[trim={2.0cm 2.0cm 0.5cm 2.0cm},clip,width=0.9\textwidth]{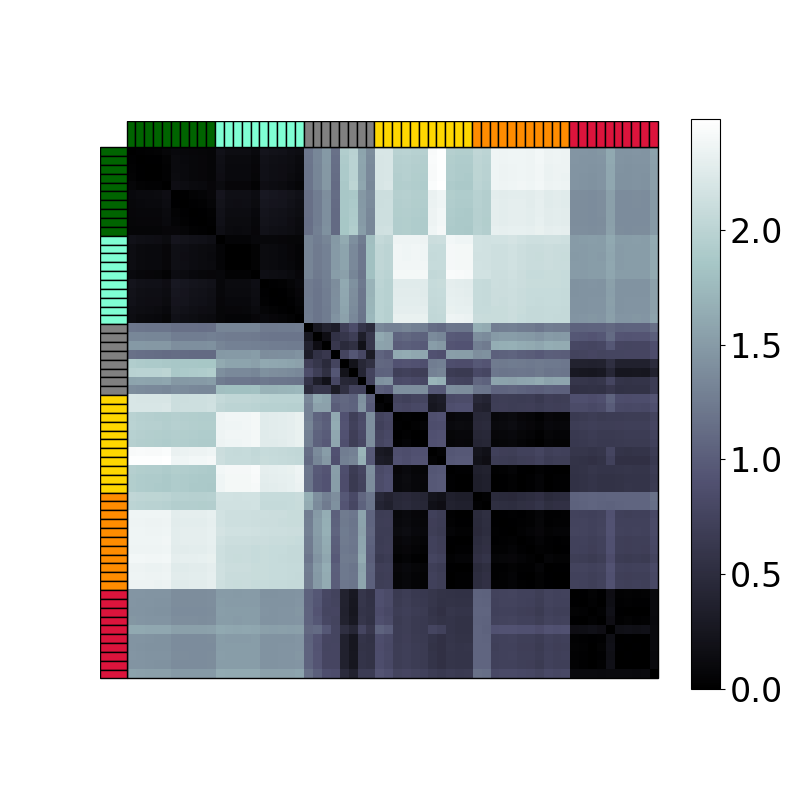}
		\includegraphics[trim={0.0cm 0.0cm 0.0cm 0.0cm},clip,width=\textwidth]{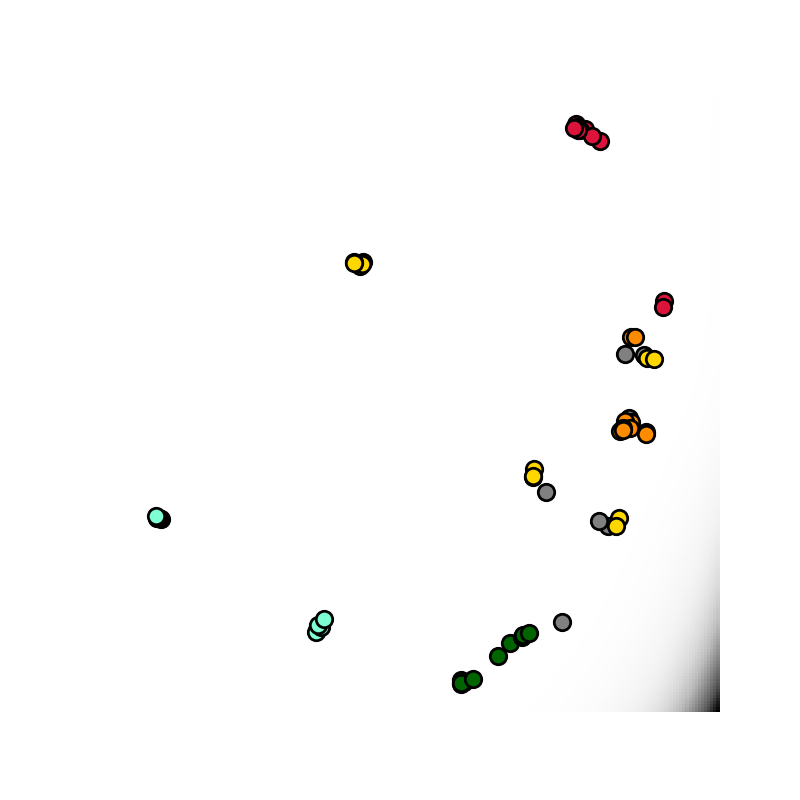}
		\includegraphics[trim={2.0cm 2.0cm 0.5cm 2.0cm},clip,width=0.9\textwidth]{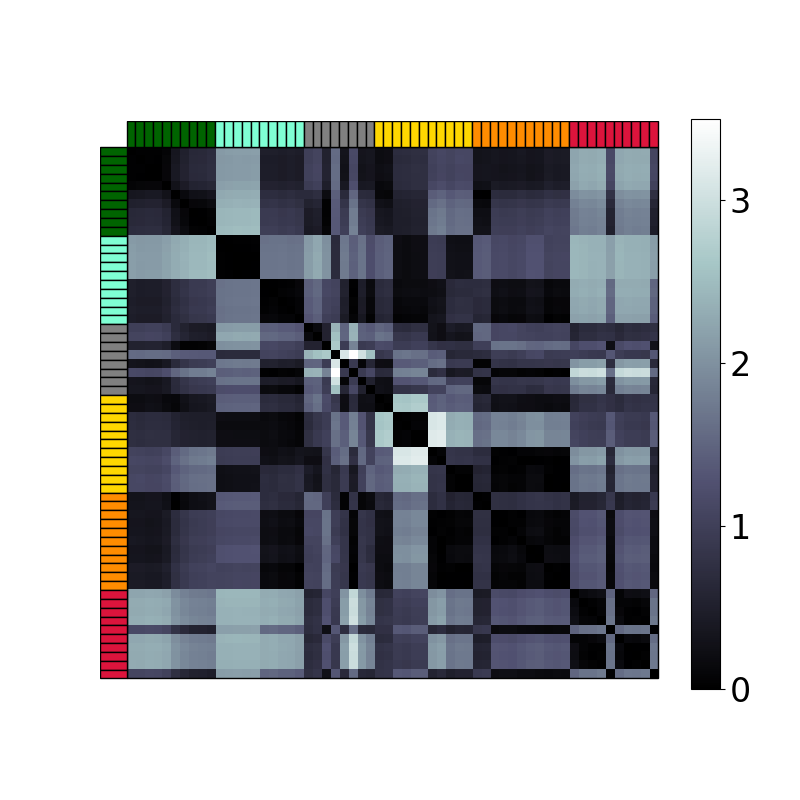}
		\caption{Vanilla (2d)}
		\label{fig:appendix:vae_baseline_2d:vanilla_bimanual}
	\end{subfigure}%
	\begin{subfigure}[b]{0.15\textwidth}
		\centering
		\includegraphics[trim={2.5cm 2.5cm 2.5cm 2.5cm},clip,width=\textwidth]{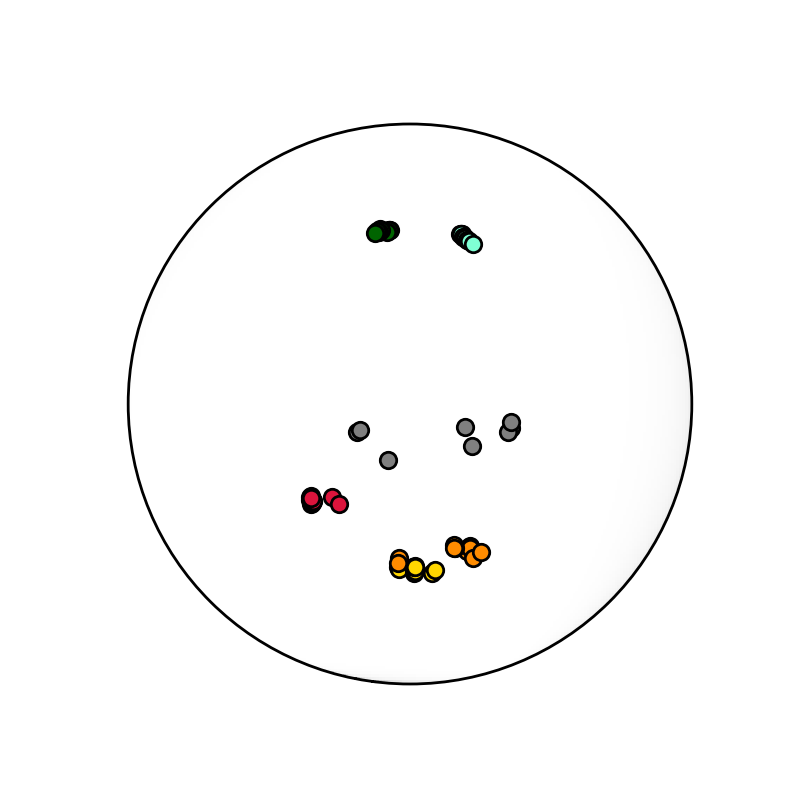}
		\includegraphics[trim={2.0cm 2.0cm 0.5cm 2.0cm},clip,width=0.9\textwidth]{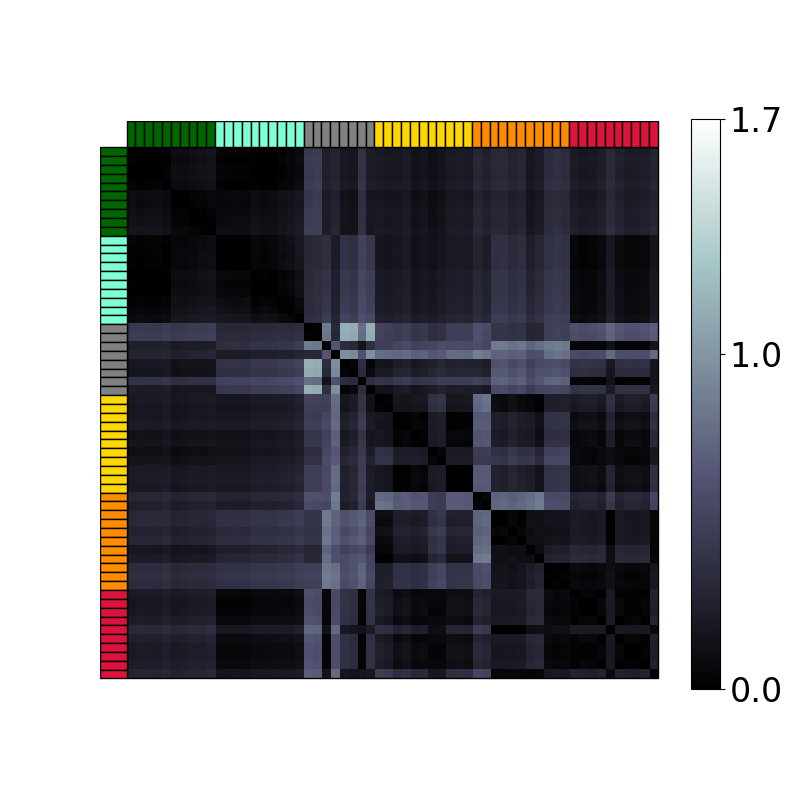}
		\includegraphics[trim={0.0cm 0.0cm 0.0cm 0.0cm},clip,width=\textwidth]{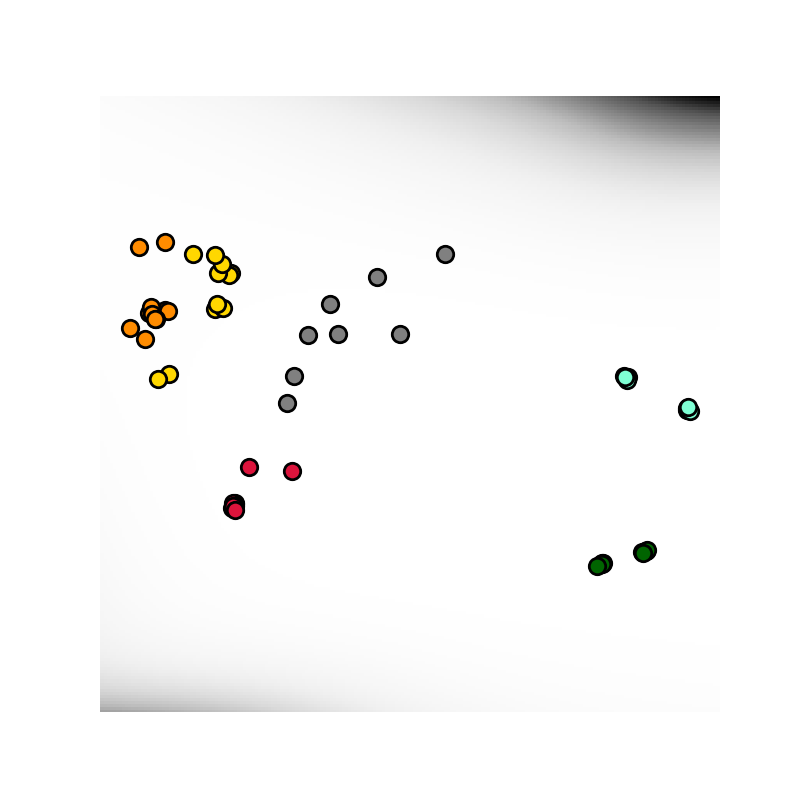}
		\includegraphics[trim={2.0cm 2.0cm 0.5cm 2.0cm},clip,width=0.9\textwidth]{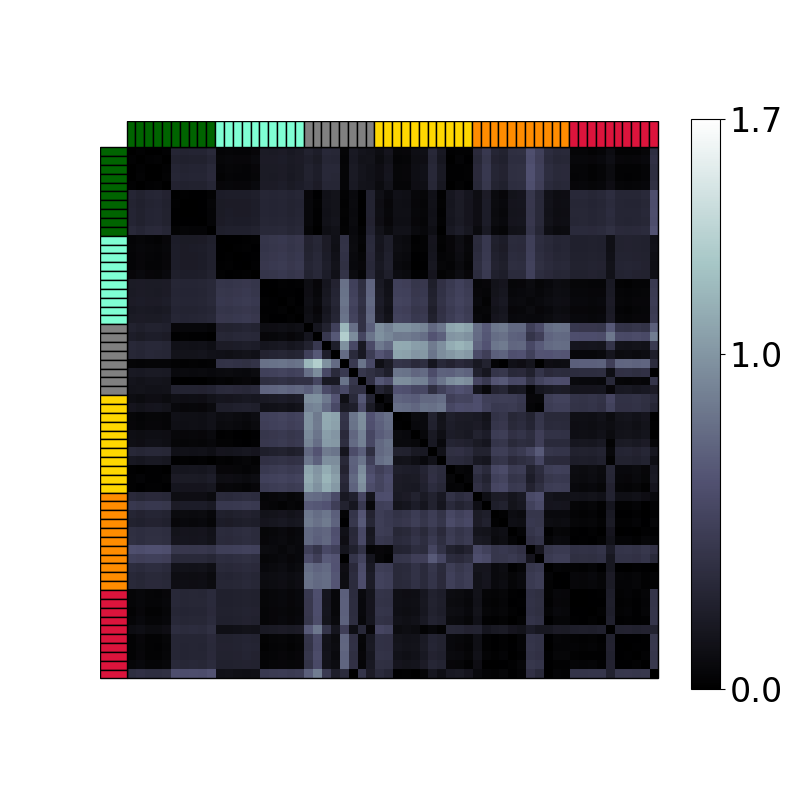}
		\caption{Stress (2d)}
		\label{fig:appendix:vae_baseline_2d:stress_bimanual}
	\end{subfigure}%
    \begin{subfigure}[b]{0.15\textwidth}
		\centering
		\includegraphics[trim={2.0cm 1.8cm 2.0cm 1.8cm},clip,width=\textwidth]{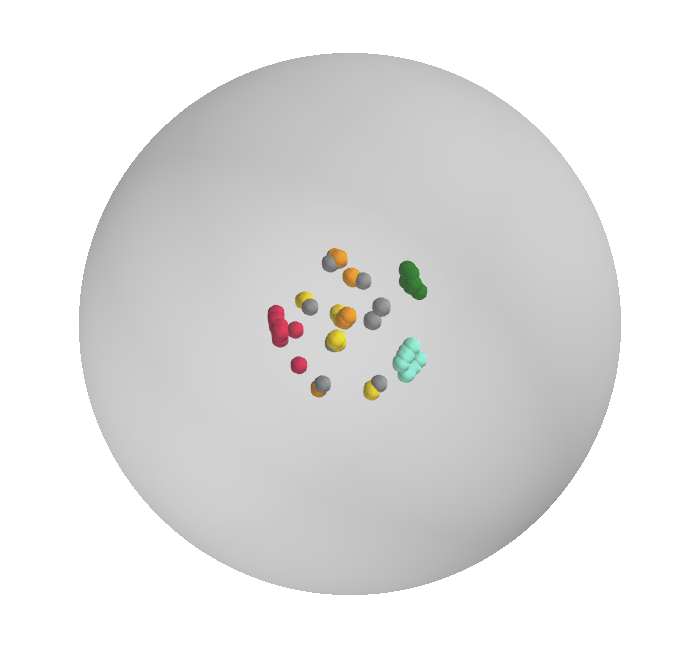}
		\includegraphics[trim={2.0cm 2.0cm 0.5cm 2.0cm},clip,width=0.9\textwidth]{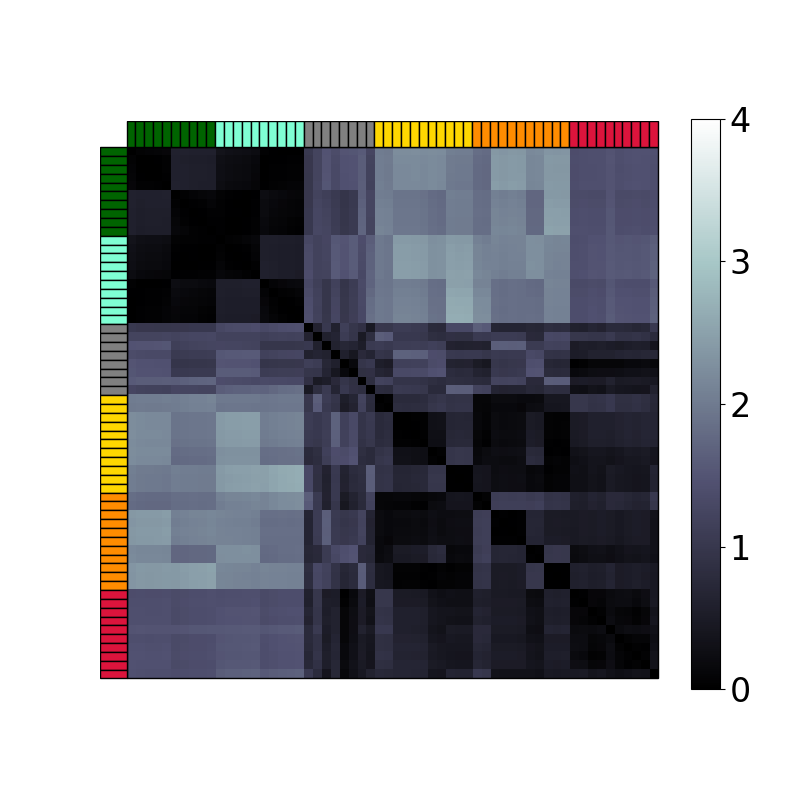}
  
        \vspace{0.2cm}
        
		\includegraphics[trim={0.0cm 0.0cm 0.0cm 0.0cm},clip,width=\textwidth]{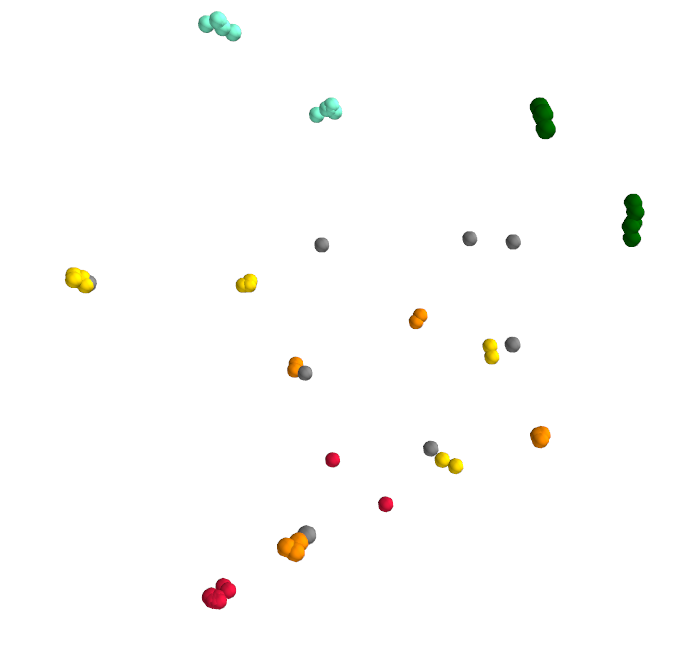}
		\includegraphics[trim={2.0cm 2.0cm 0.5cm 2.0cm},clip,width=0.9\textwidth]{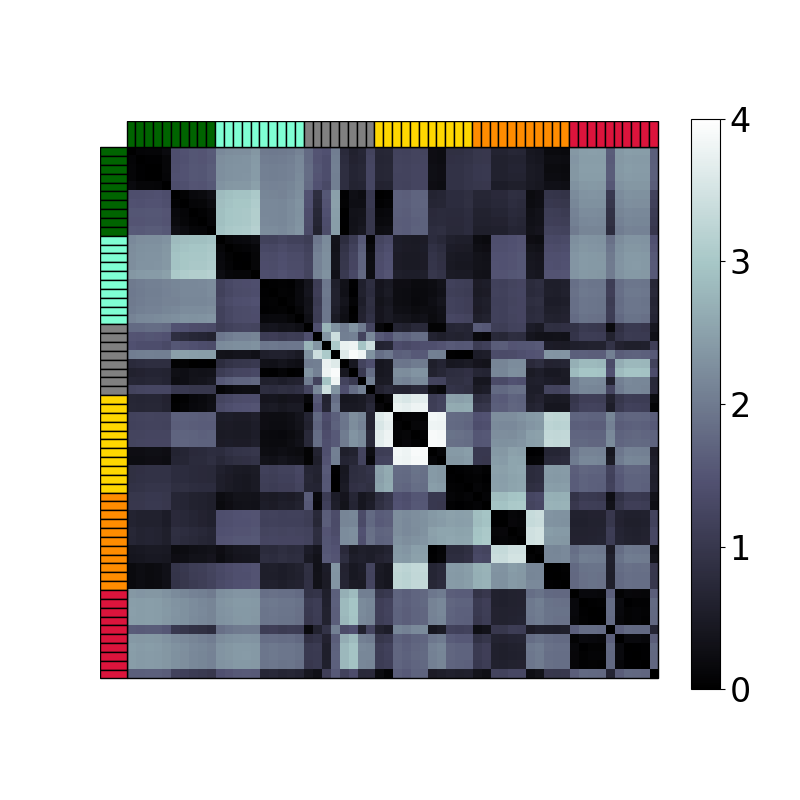}
		\caption{Vanilla (3d)}
		\label{fig:appendix:vae_baseline_3d:vanilla_bimanual}
	\end{subfigure}%
	\begin{subfigure}[b]{0.15\textwidth}
		\centering
		\includegraphics[trim={2.0cm 1.8cm 2.0cm 1.8cm},clip,width=\textwidth]{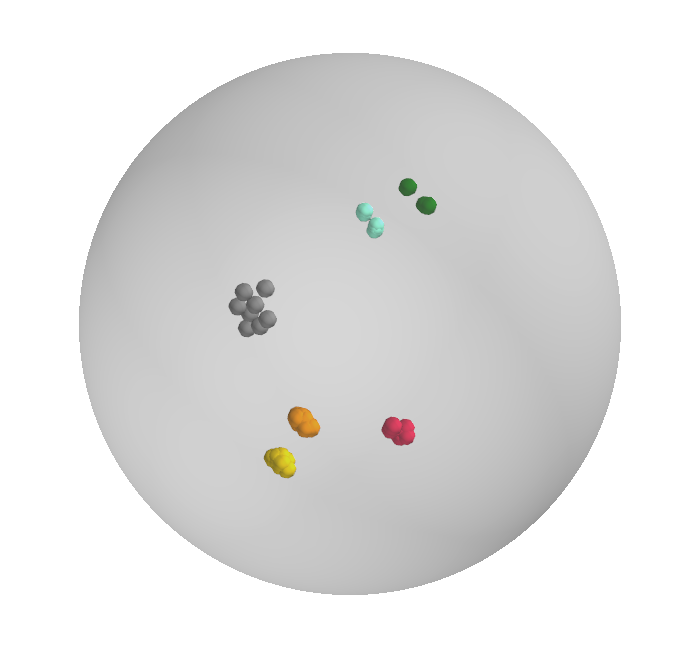}
		\includegraphics[trim={2.0cm 2.0cm 0.5cm 2.0cm},clip,width=0.9\textwidth]{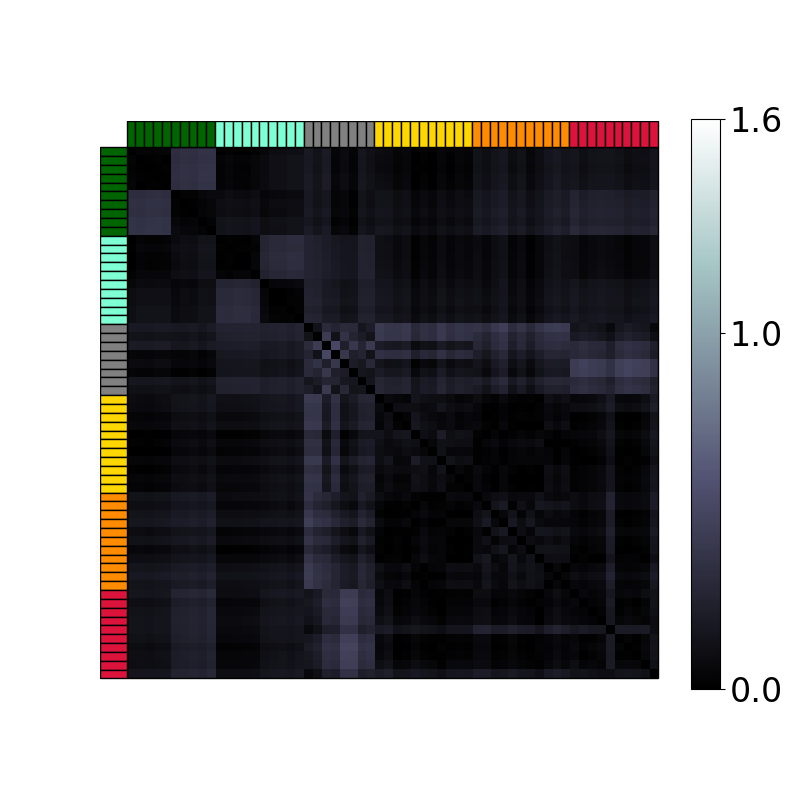}
  
        \vspace{0.2cm}
        
		\includegraphics[trim={0.0cm 0.0cm 0.0cm 0.0cm},clip,width=\textwidth]{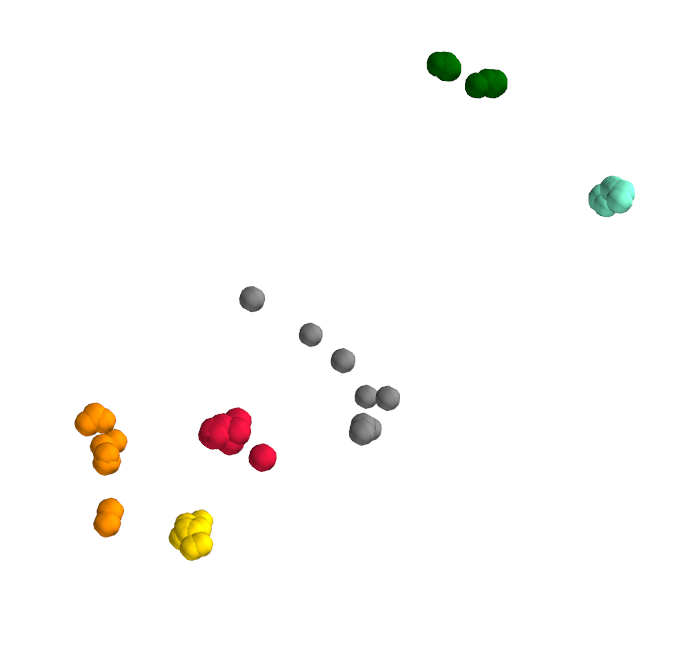}
		\includegraphics[trim={2.0cm 2.0cm 0.5cm 2.0cm},clip,width=0.9\textwidth]{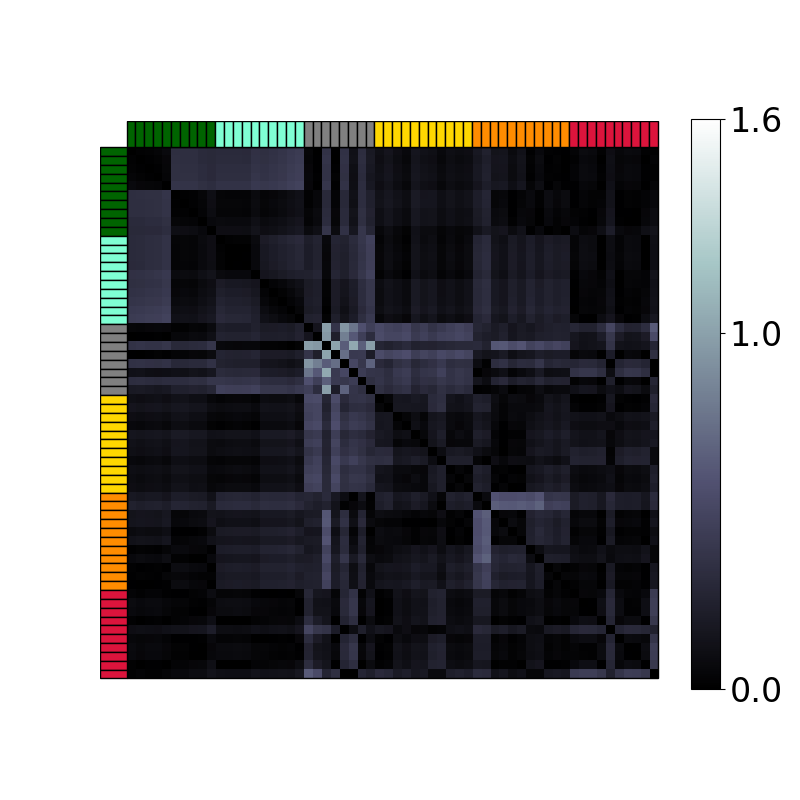}
		\caption{Stress (3d)}
		\label{fig:appendix:vae_baseline_3d:stress_bimanual}
	\end{subfigure}%
    \vspace{-0.2cm}
	\caption{Embeddings of bimanual manipulation categories with VAEs: The first and last two rows show the latent embeddings of the hyperbolic and Euclidean VAE in $\mathcal{P}^Q$ and $\mathbb{R}^Q$, followed by pairwise error matrices between geodesic and taxonomy graph distances.}
    \vspace{-0.3cm}
	\label{fig:appendix:vae_baselines_bimanual}
\end{figure}

\begin{figure}
	\centering
	\includegraphics[trim={5.8cm 2.2cm 4.3cm 2.2cm},clip,width=0.8\textwidth]{Figures/legend_grasps.png}
    \begin{subfigure}[b]{0.15\textwidth}
		\centering
		\includegraphics[trim={2.5cm 2.5cm 2.5cm 2.5cm},clip,width=\textwidth]{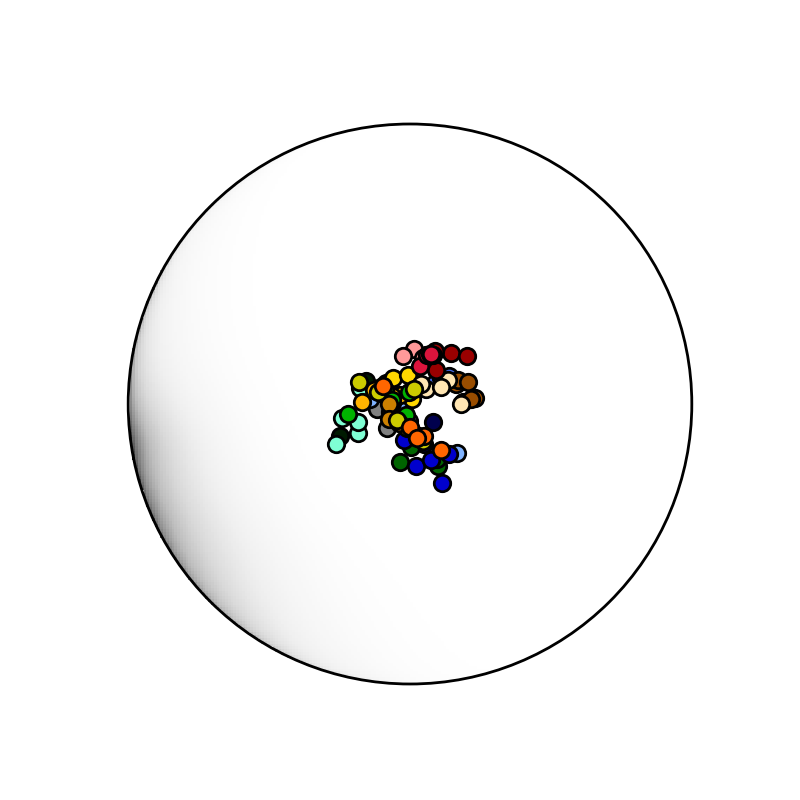}
		\includegraphics[trim={2.0cm 2.0cm 0.5cm 2.0cm},clip,width=0.9\textwidth]{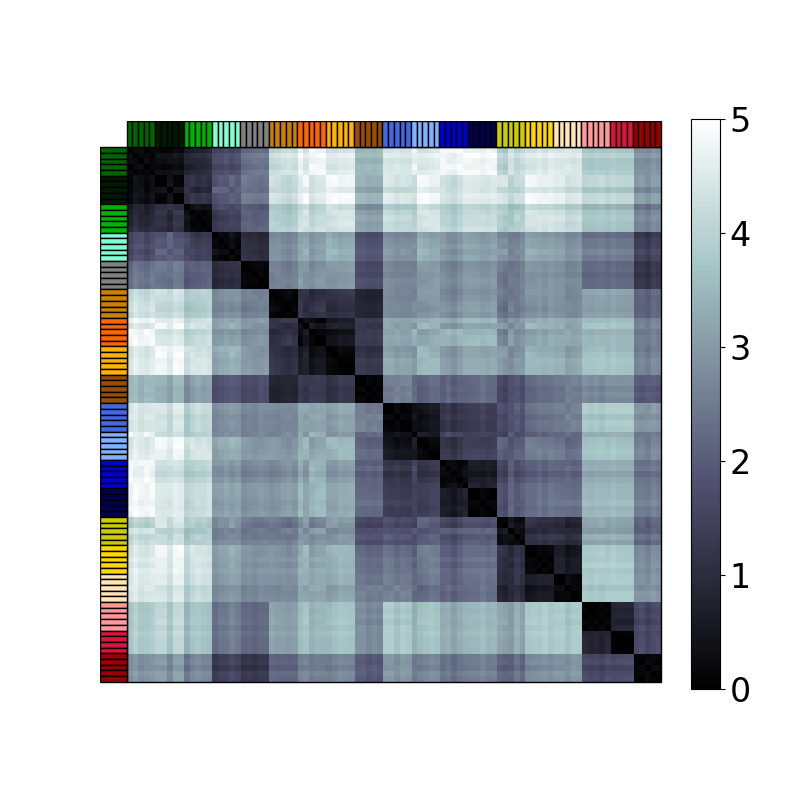}
		\includegraphics[trim={0.0cm 0.0cm 0.0cm 0.0cm},clip,width=\textwidth]{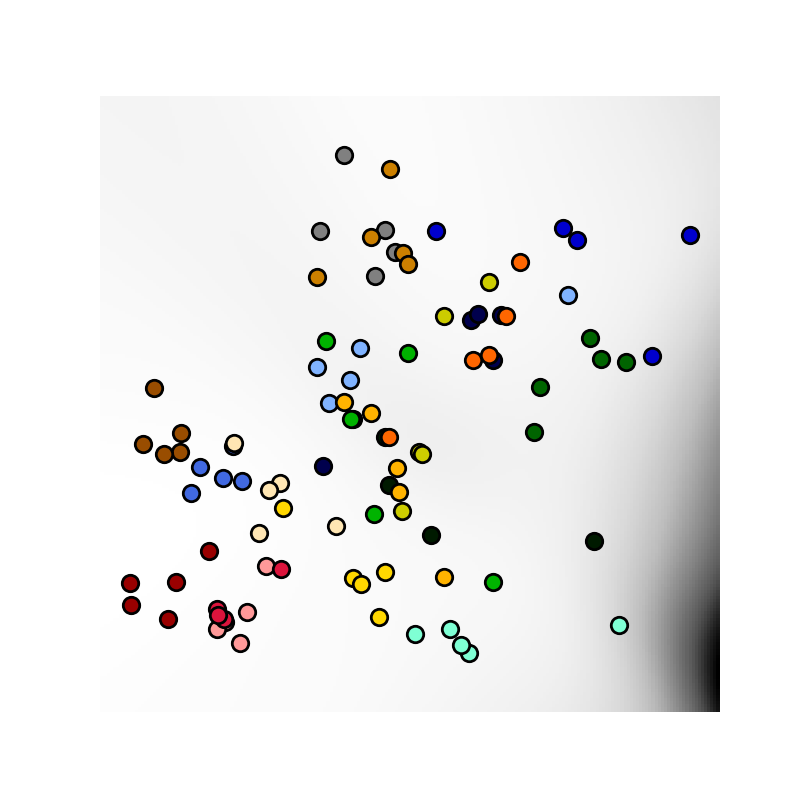}
		\includegraphics[trim={2.0cm 2.0cm 0.5cm 2.0cm},clip,width=0.9\textwidth]{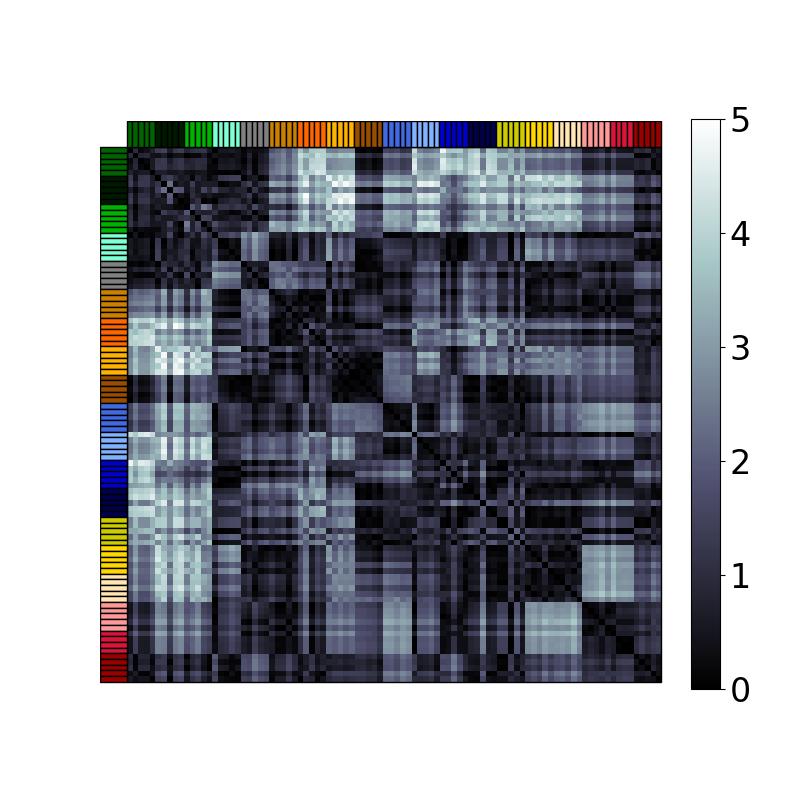}
		\caption{Vanilla (2d)}
		\label{fig:appendix:vae_baseline_2d:vanilla_grasps}
	\end{subfigure}%
	\begin{subfigure}[b]{0.15\textwidth}
		\centering
		\includegraphics[trim={2.5cm 2.5cm 2.5cm 2.5cm},clip,width=\textwidth]{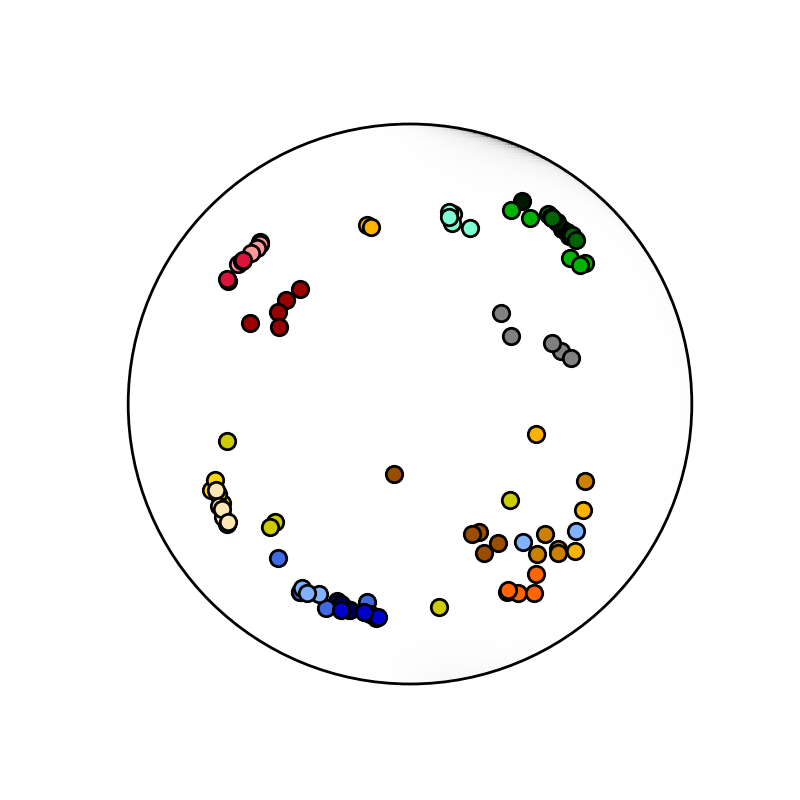}
		\includegraphics[trim={2.0cm 2.0cm 0.5cm 2.0cm},clip,width=0.9\textwidth]{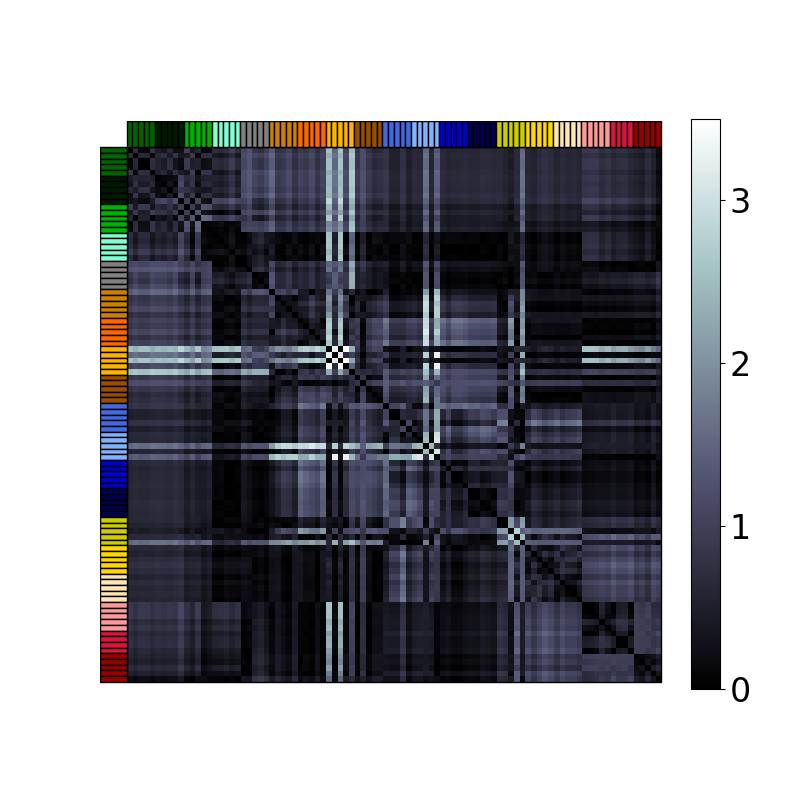}
		\includegraphics[trim={0.0cm 0.0cm 0.0cm 0.0cm},clip,width=\textwidth]{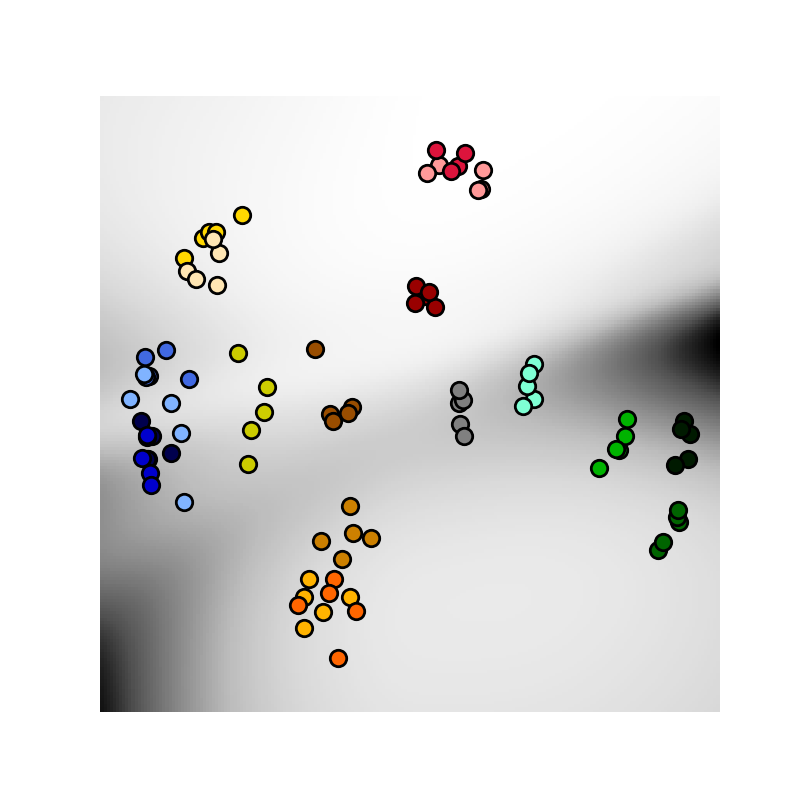}
		\includegraphics[trim={2.0cm 2.0cm 0.5cm 2.0cm},clip,width=0.9\textwidth]{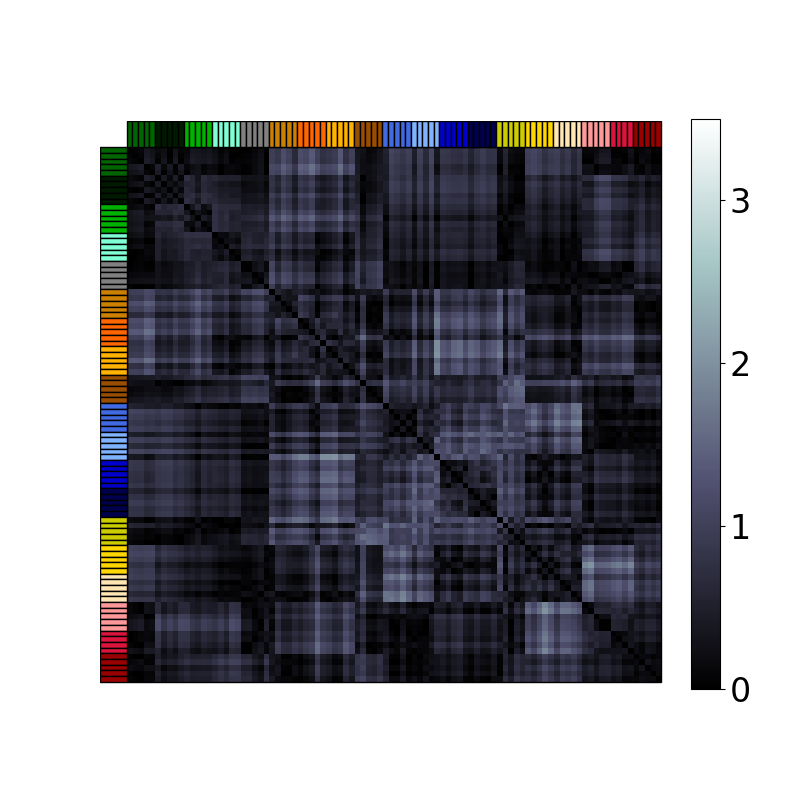}
		\caption{Stress (2d)}
		\label{fig:appendix:vae_baseline_2d:stress_grasps}
	\end{subfigure}%
    \begin{subfigure}[b]{0.15\textwidth}
		\centering
		\includegraphics[trim={2.0cm 1.8cm 2.0cm 1.8cm},clip,width=\textwidth]{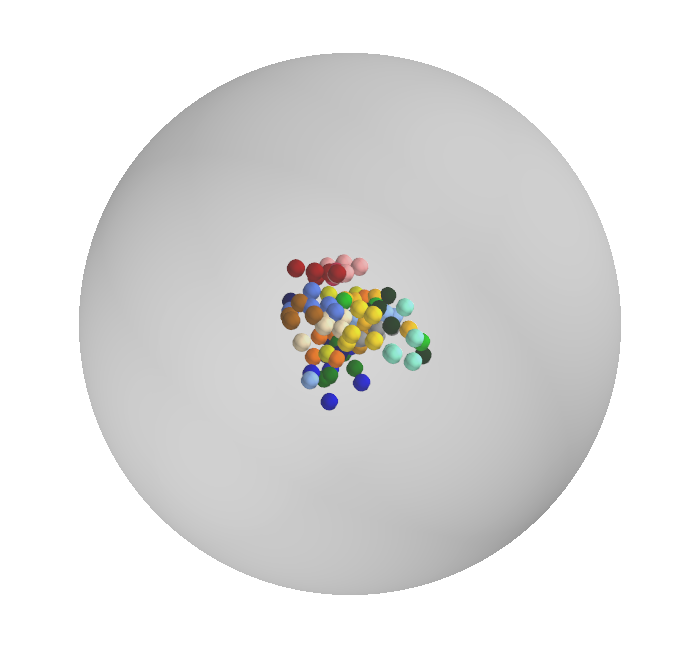}
		\includegraphics[trim={2.0cm 2.0cm 0.5cm 2.0cm},clip,width=0.9\textwidth]{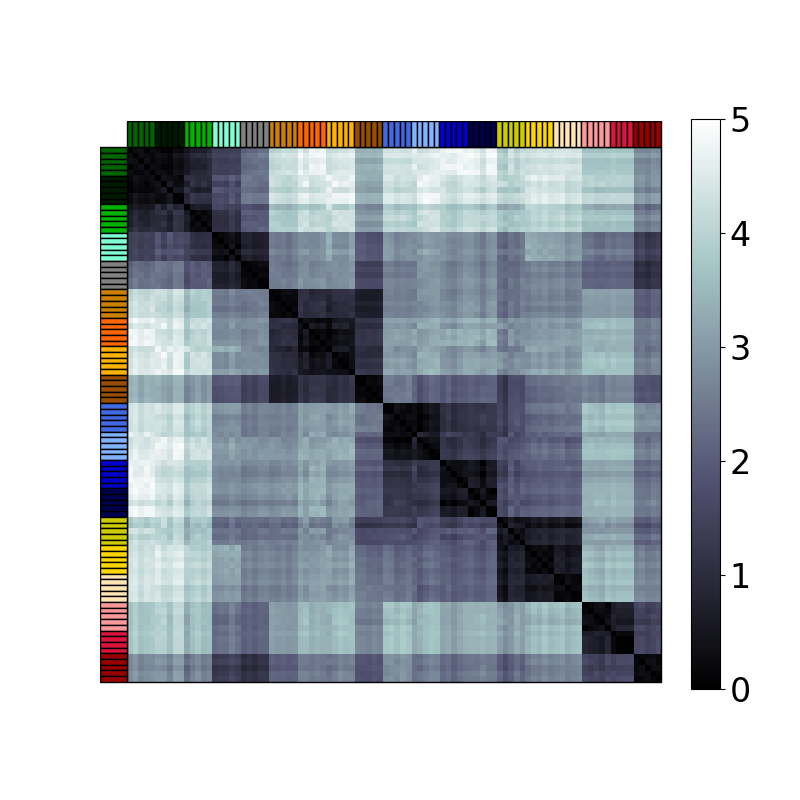}
  
        \vspace{0.2cm}
        
		\includegraphics[trim={0.0cm 0.0cm 0.0cm 0.0cm},clip,width=\textwidth]{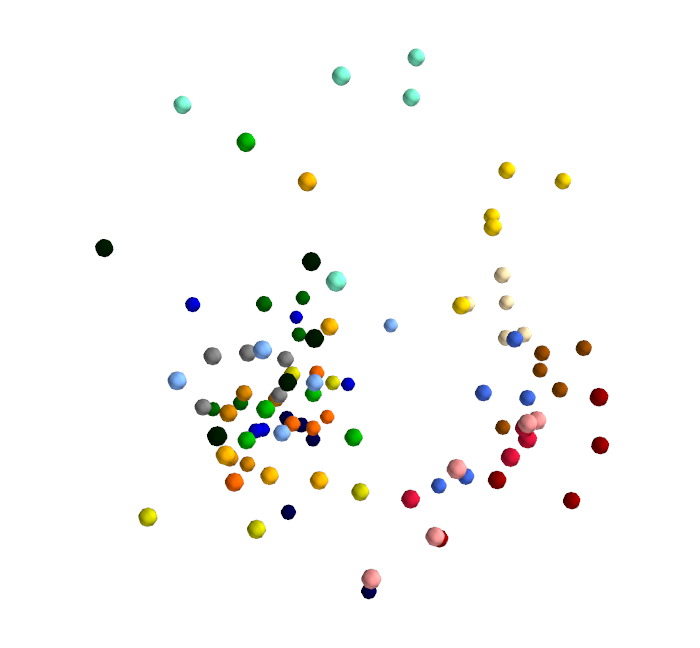}
		\includegraphics[trim={2.0cm 2.0cm 0.5cm 2.0cm},clip,width=0.9\textwidth]{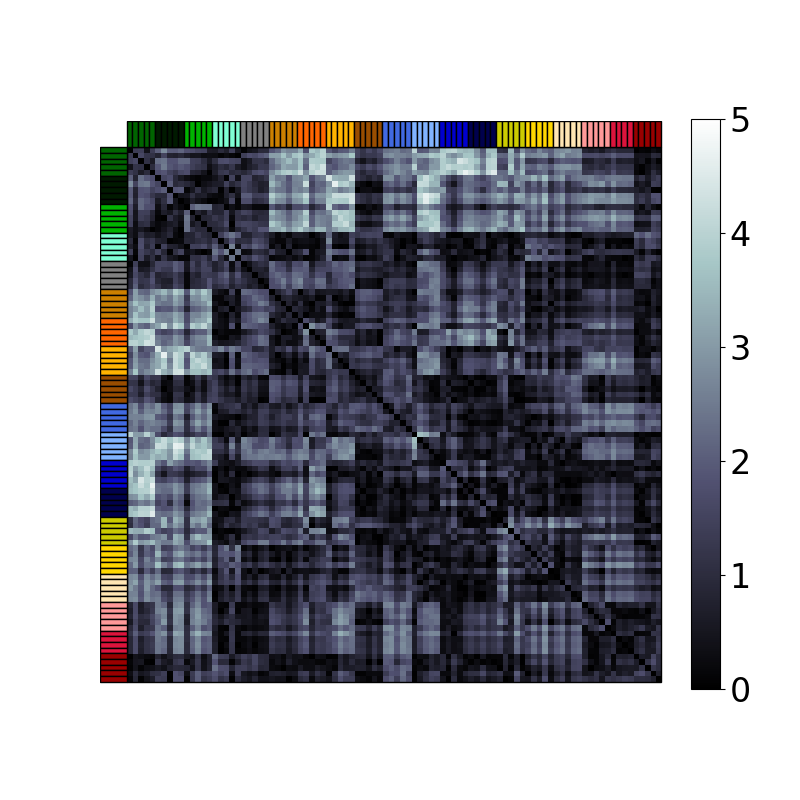}
		\caption{Vanilla (3d)}
		\label{fig:appendix:vae_baseline_3d:vanilla_grasps}
	\end{subfigure}%
	\begin{subfigure}[b]{0.15\textwidth}
		\centering
		\includegraphics[trim={2.0cm 1.8cm 2.0cm 1.8cm},clip,width=\textwidth]{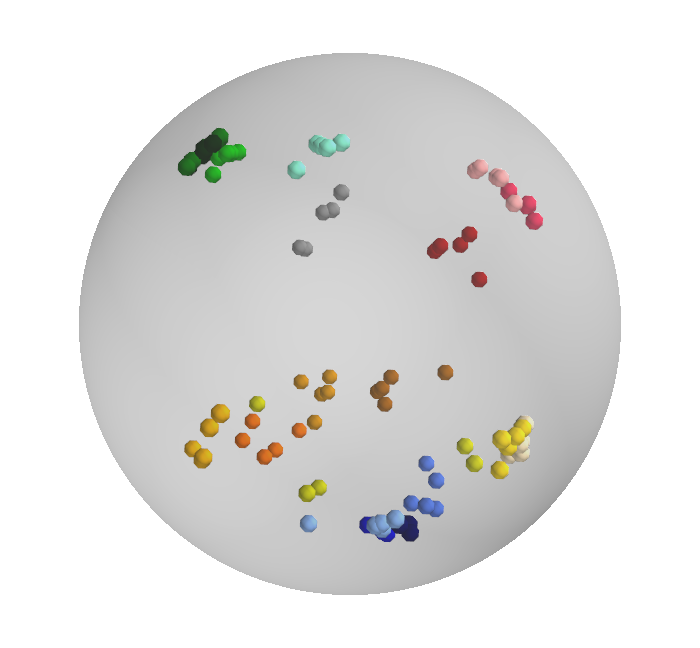}
		\includegraphics[trim={2.0cm 2.0cm 0.5cm 2.0cm},clip,width=0.9\textwidth]{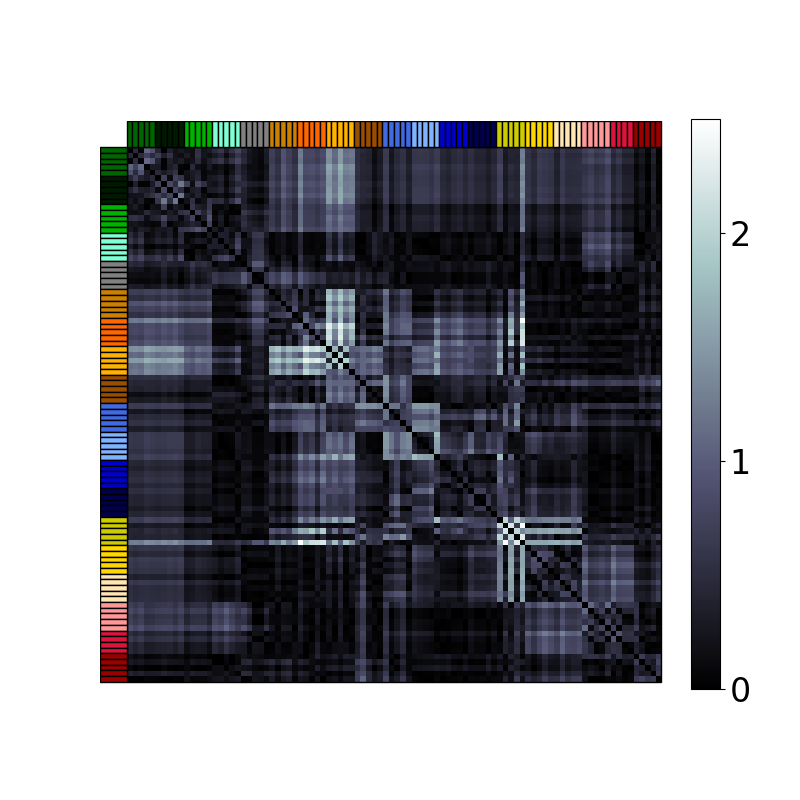}
  
        \vspace{0.2cm}
        
		\includegraphics[trim={0.0cm 0.0cm 0.0cm 0.0cm},clip,width=\textwidth]{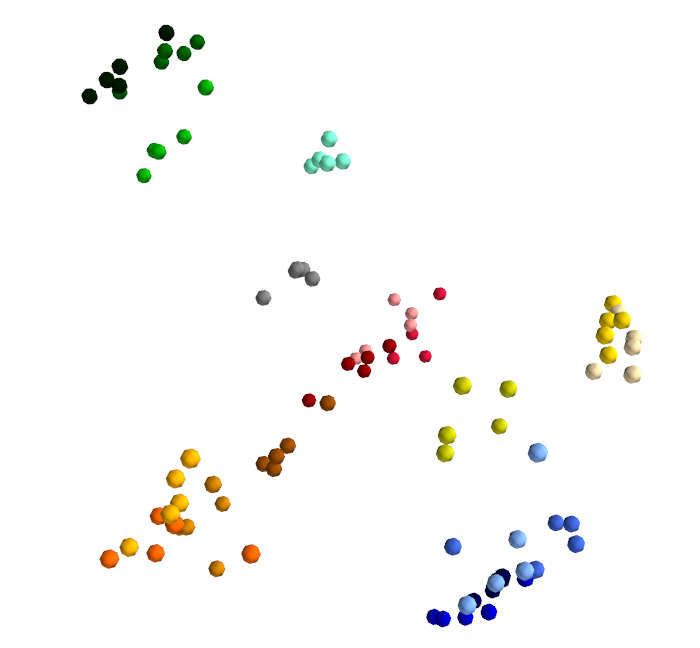}
		\includegraphics[trim={2.0cm 2.0cm 0.5cm 2.0cm},clip,width=0.9\textwidth]{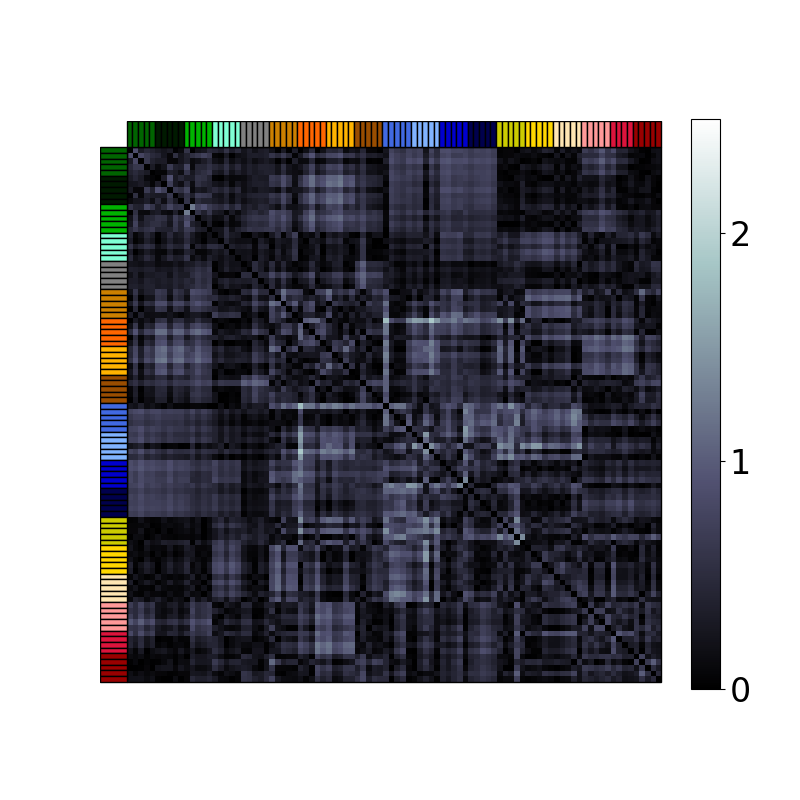}
		\caption{Stress (3d)}
		\label{fig:appendix:vae_baseline_3d:stress_grasps}
	\end{subfigure}%
    \vspace{-0.2cm}
	\caption{Embeddings of grasps with VAEs: The first and last two rows show the latent embeddings of the hyperbolic and Euclidean VAE in $\mathcal{P}^Q$ and $\mathbb{R}^Q$, followed by pairwise error matrices.}
    \vspace{-0.3cm}
	\label{fig:appendix:vae_baselines_grasps}
\end{figure}

\begin{figure}
	\centering
	\includegraphics[trim={5.3cm 2.2cm 4.3cm 2.2cm},clip,width=0.8\textwidth]{Figures/legend_semifull.png}
    \begin{subfigure}[b]{0.15\textwidth}
		\centering
		\includegraphics[trim={2.5cm 2.5cm 2.5cm 2.5cm},clip,width=\textwidth]{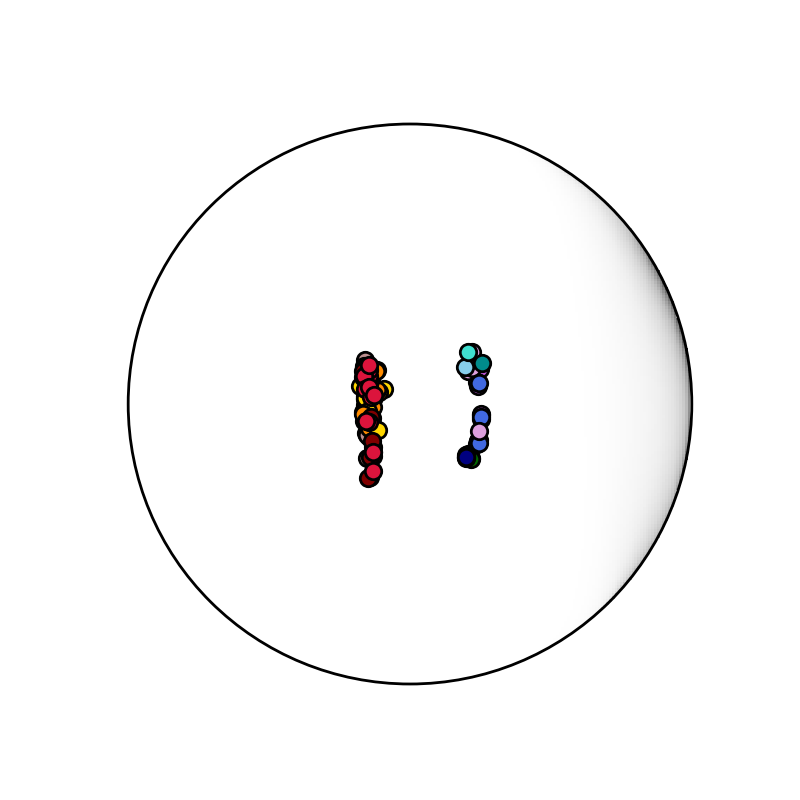}
		\includegraphics[trim={2.cm 2.0cm 1.0cm 2.0cm},clip,width=0.9\textwidth]{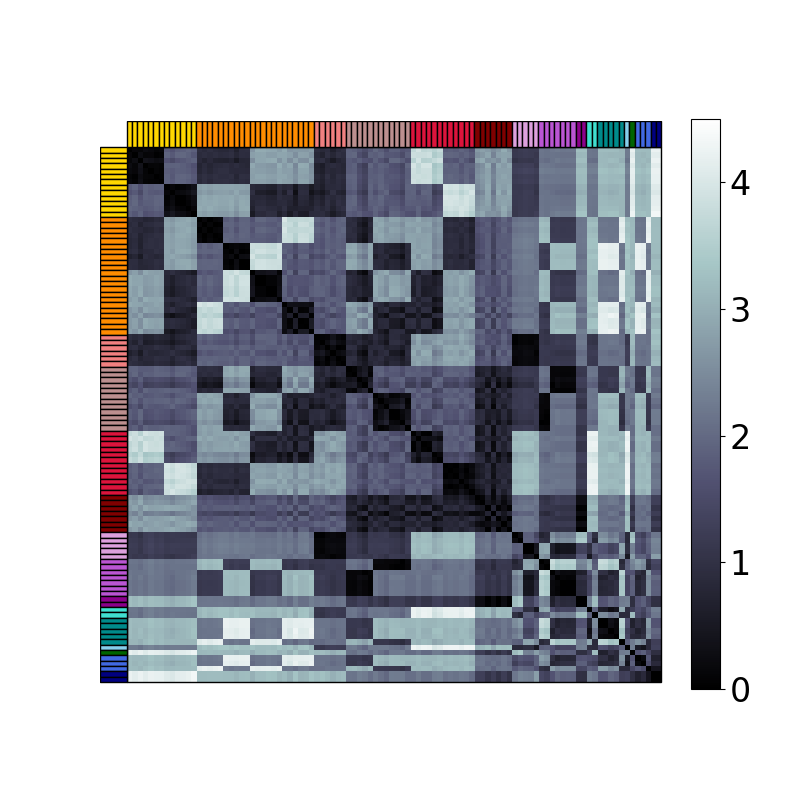}
		\includegraphics[trim={0.0cm 0.0cm 0.0cm 0.0cm},clip,width=\textwidth]{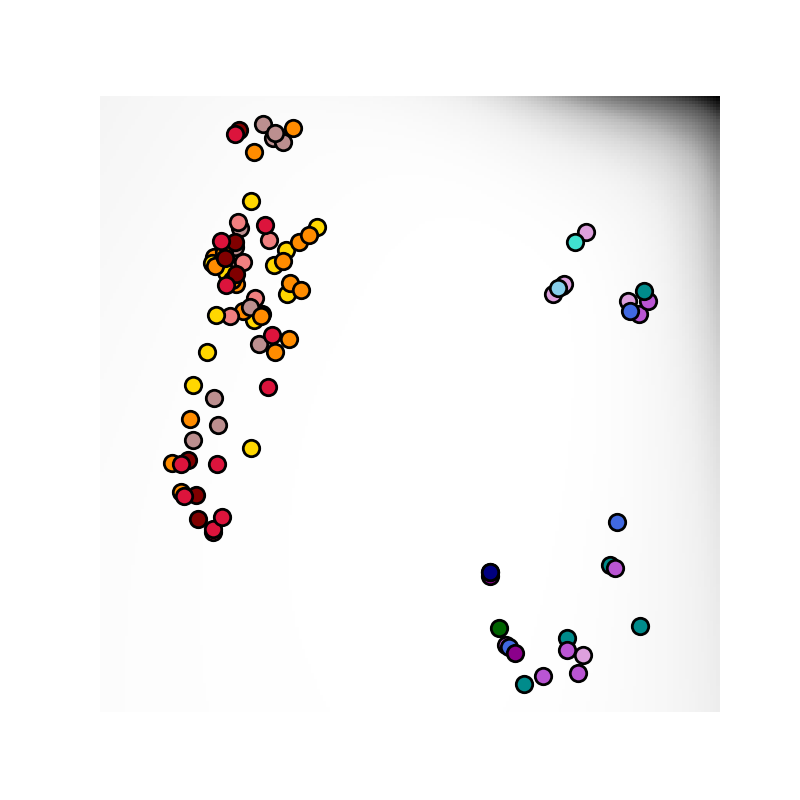}
		\includegraphics[trim={2.0cm 2.0cm 0.5cm 2.0cm},clip,width=0.9\textwidth]{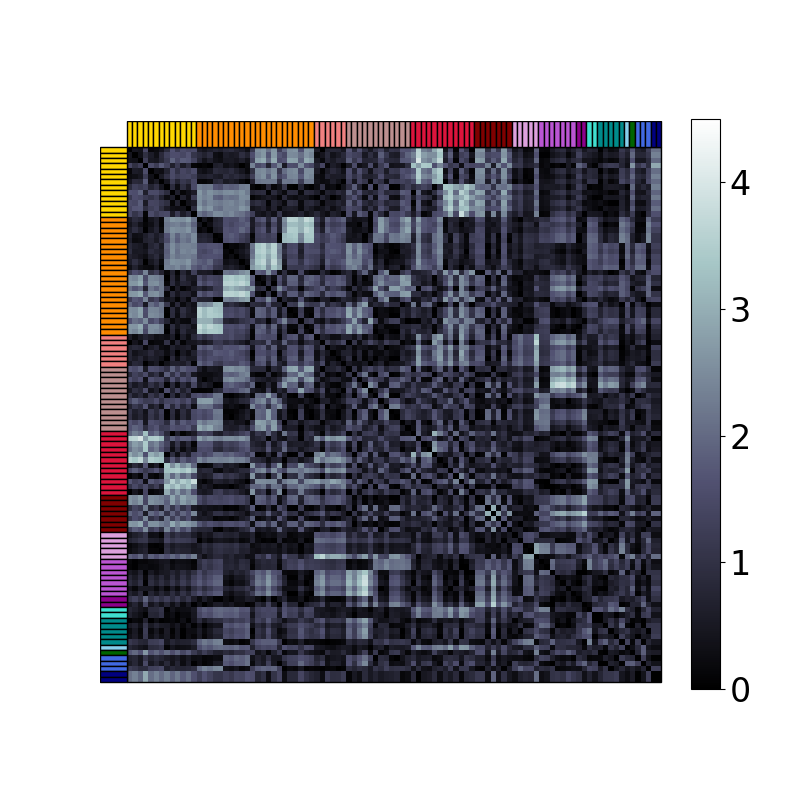}
		\caption{Vanilla (2d)}
		\label{fig:appendix:vae_baseline_2d:vanilla}
	\end{subfigure}%
	\begin{subfigure}[b]{0.15\textwidth}
		\centering
		\includegraphics[trim={2.5cm 2.5cm 2.5cm 2.5cm},clip,width=\textwidth]{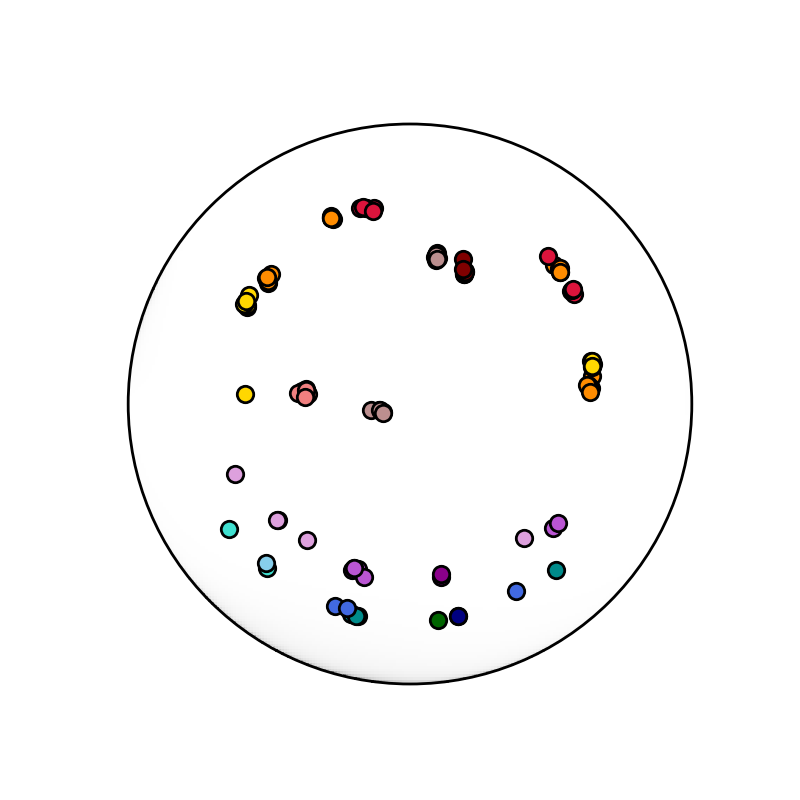}
		\includegraphics[trim={2.0cm 2.0cm 0.5cm 2.0cm},clip,width=0.9\textwidth]{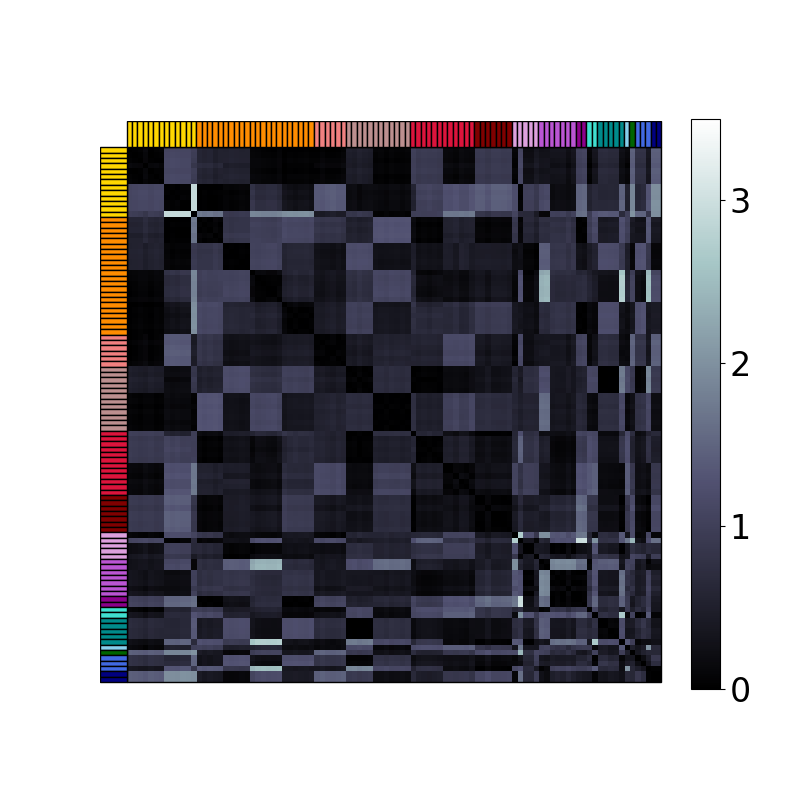}
		\includegraphics[trim={0.0cm 0.0cm 0.0cm 0.0cm},clip,width=\textwidth]{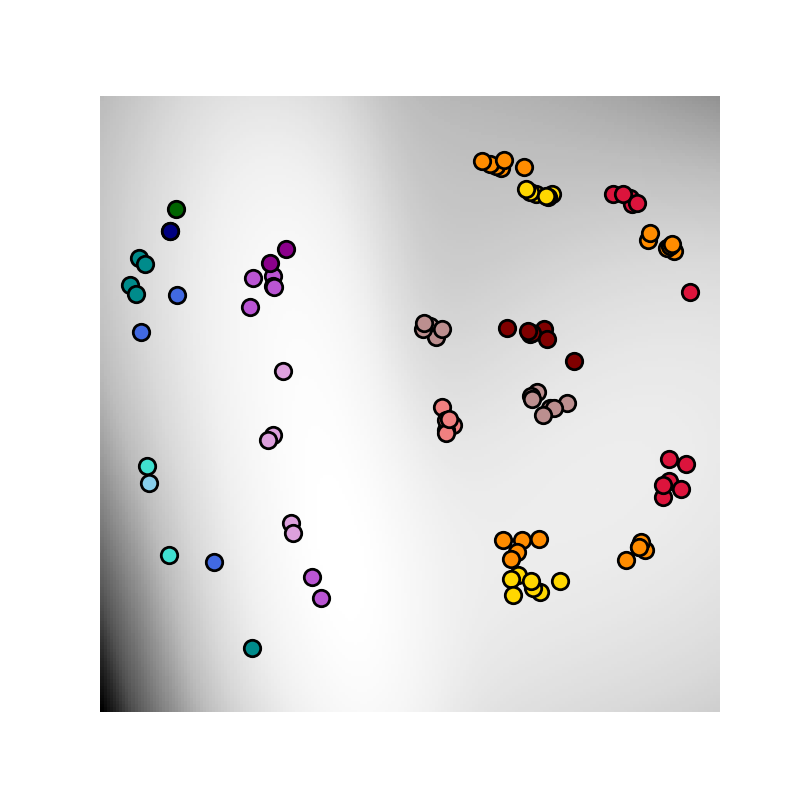}
		\includegraphics[trim={2.0cm 2.0cm 0.5cm 2.0cm},clip,width=0.9\textwidth]{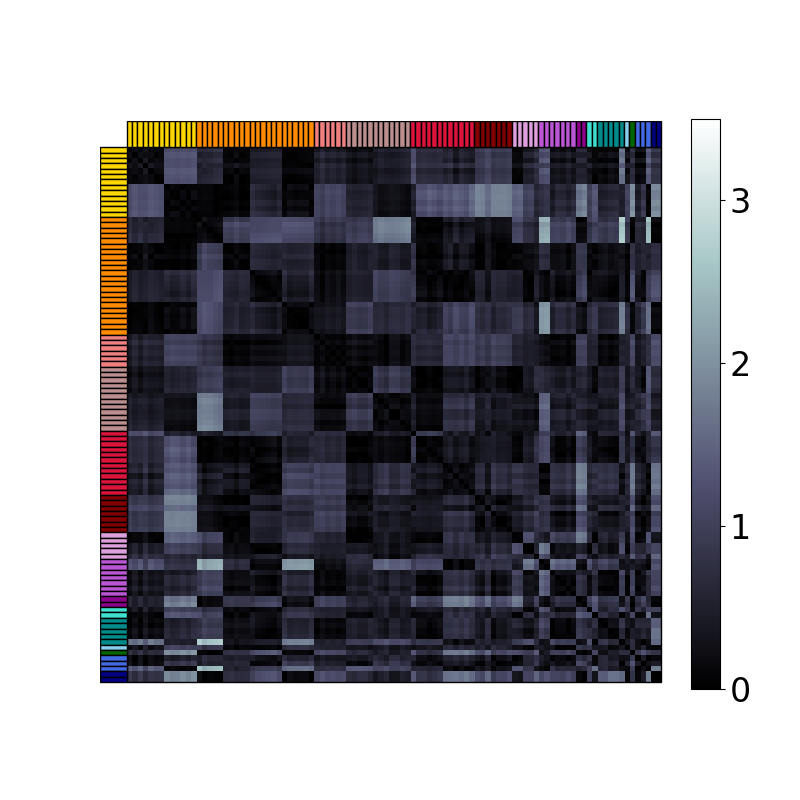}
		\caption{Stress (2d)}
		\label{fig:appendix:vae_baseline_2d:stress}
	\end{subfigure}%
    \begin{subfigure}[b]{0.15\textwidth}
		\centering
		\includegraphics[trim={2.0cm 1.8cm 2.0cm 1.8cm},clip,width=\textwidth]{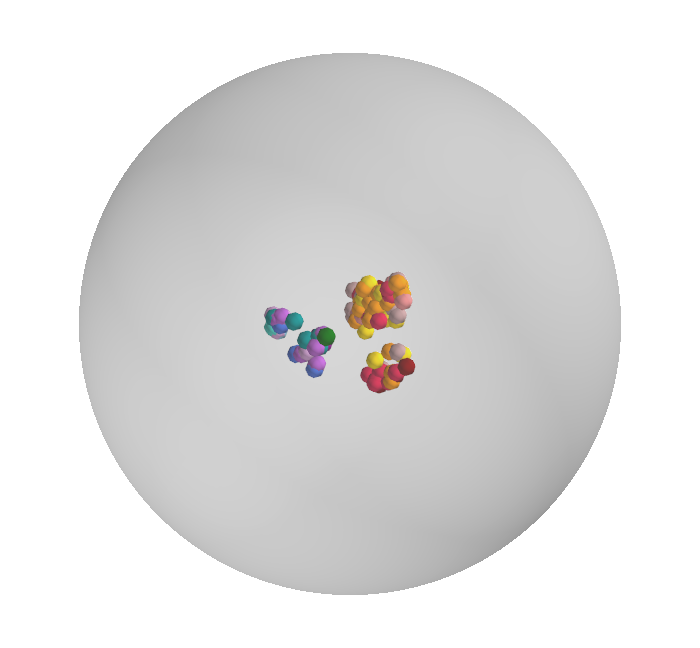}
		\includegraphics[trim={2.0cm 2.0cm 0.5cm 2.0cm},clip,width=0.9\textwidth]{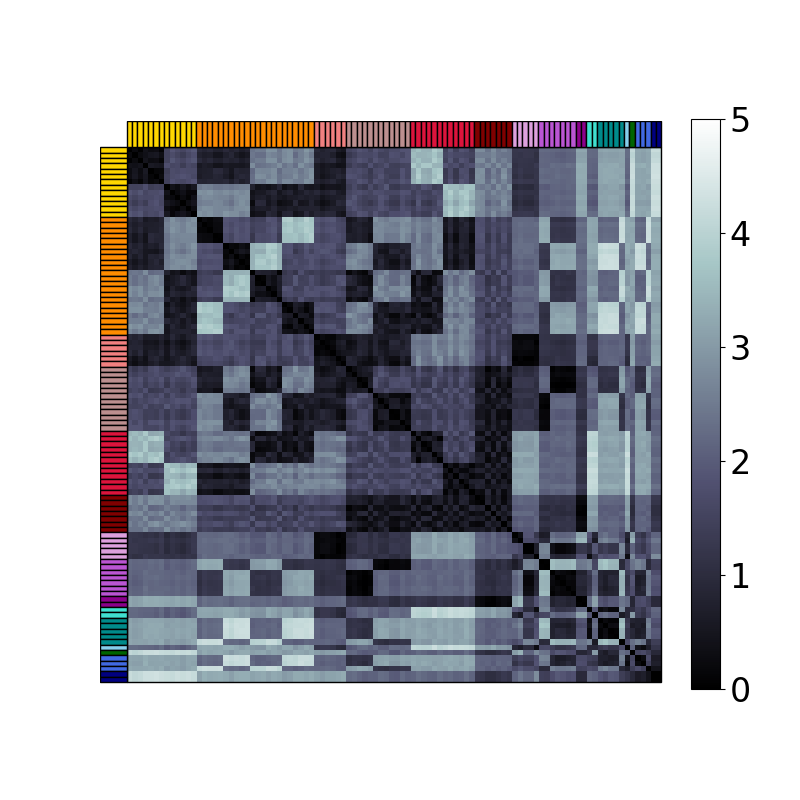}
    
        \vspace{0.2cm}
        
		\includegraphics[trim={0.0cm 0.0cm 0.0cm 0.0cm},clip,width=\textwidth]{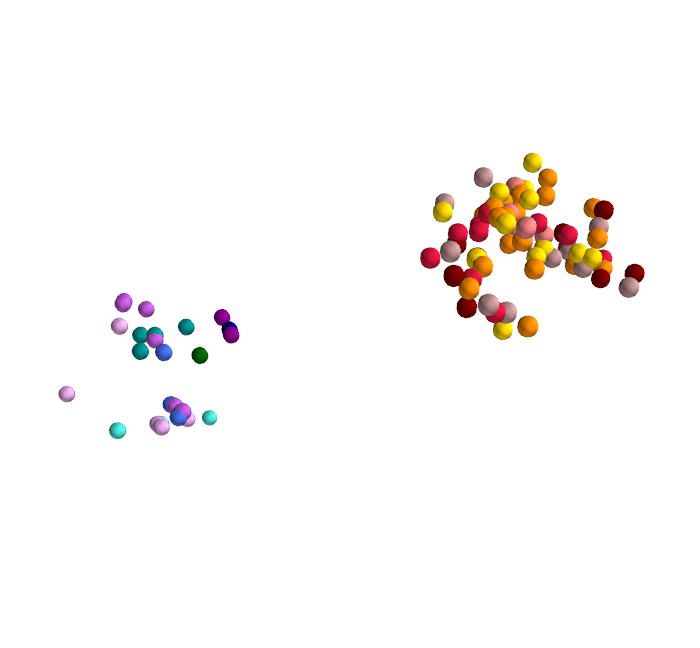}
		\includegraphics[trim={2.0cm 2.0cm 0.5cm 2.0cm},clip,width=0.9\textwidth]{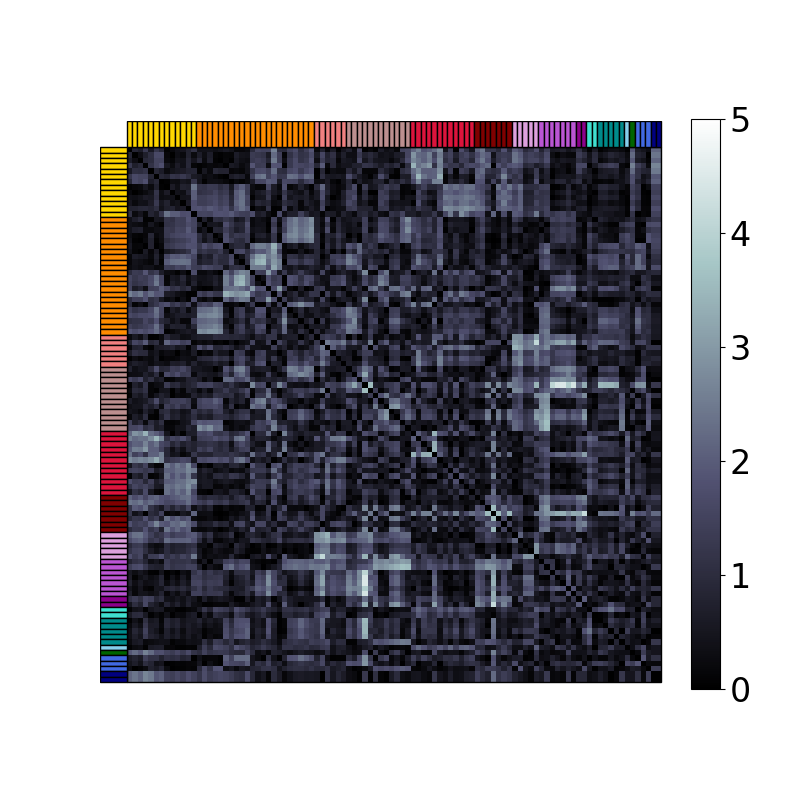}
		\caption{Vanilla (3d)}
		\label{fig:appendix:vae_baseline_3d:vanilla}
	\end{subfigure}%
	\begin{subfigure}[b]{0.15\textwidth}
		\centering
		\includegraphics[trim={2.0cm 1.8cm 2.0cm 1.8cm},clip,width=\textwidth]{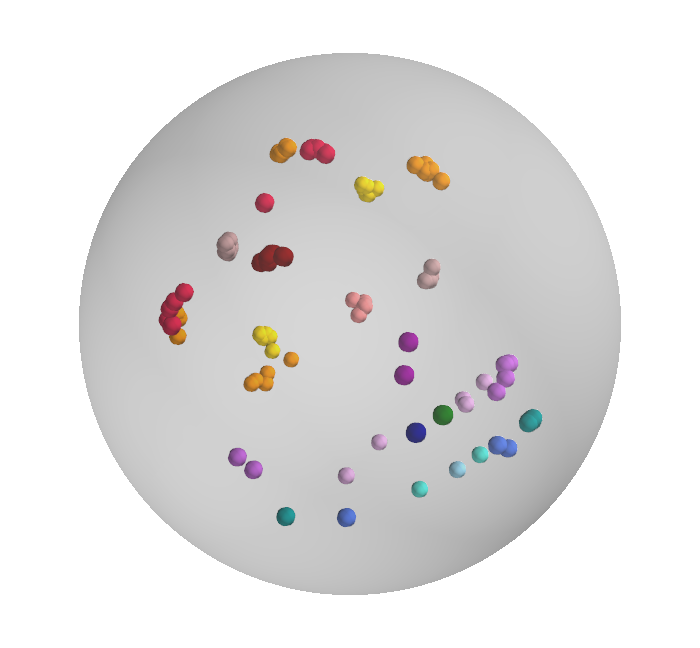}
		\includegraphics[trim={2.0cm 2.0cm 0.5cm 2.0cm},clip,width=0.9\textwidth]{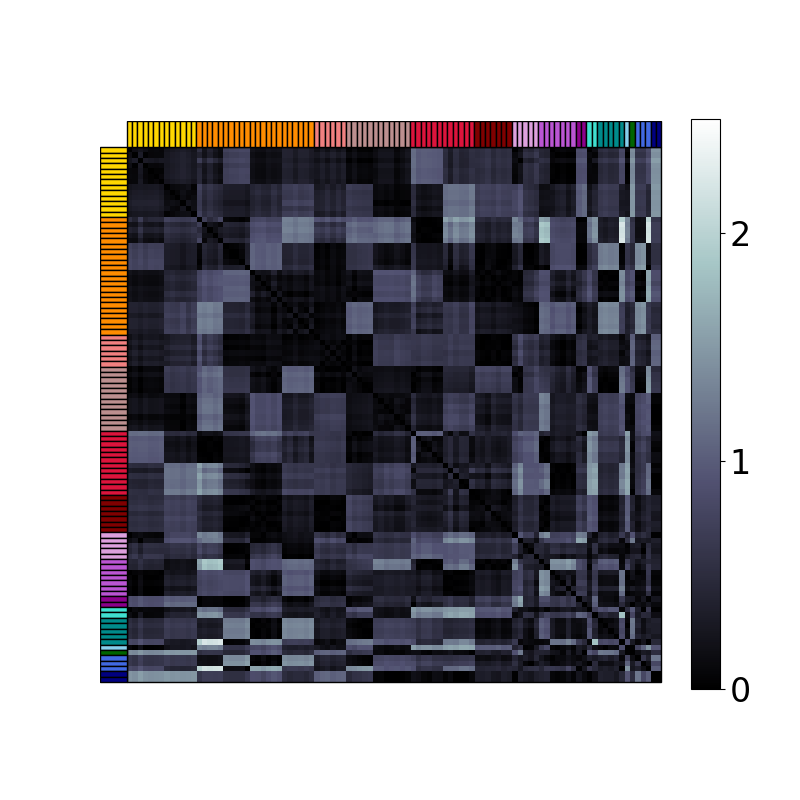}
    
        \vspace{0.2cm}
        
		\includegraphics[trim={0.0cm 0.0cm 0.0cm 0.0cm},clip,width=\textwidth]{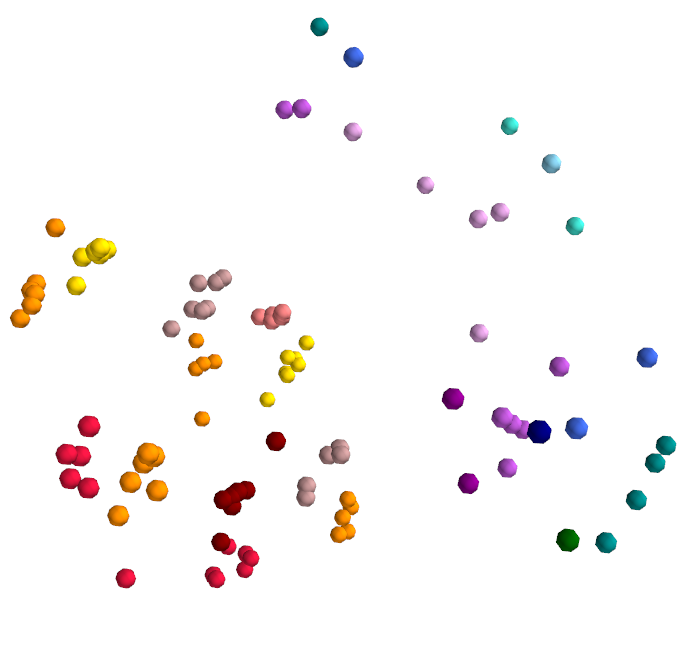}
		\includegraphics[trim={2.0cm 2.0cm 0.5cm 2.0cm},clip,width=0.9\textwidth]{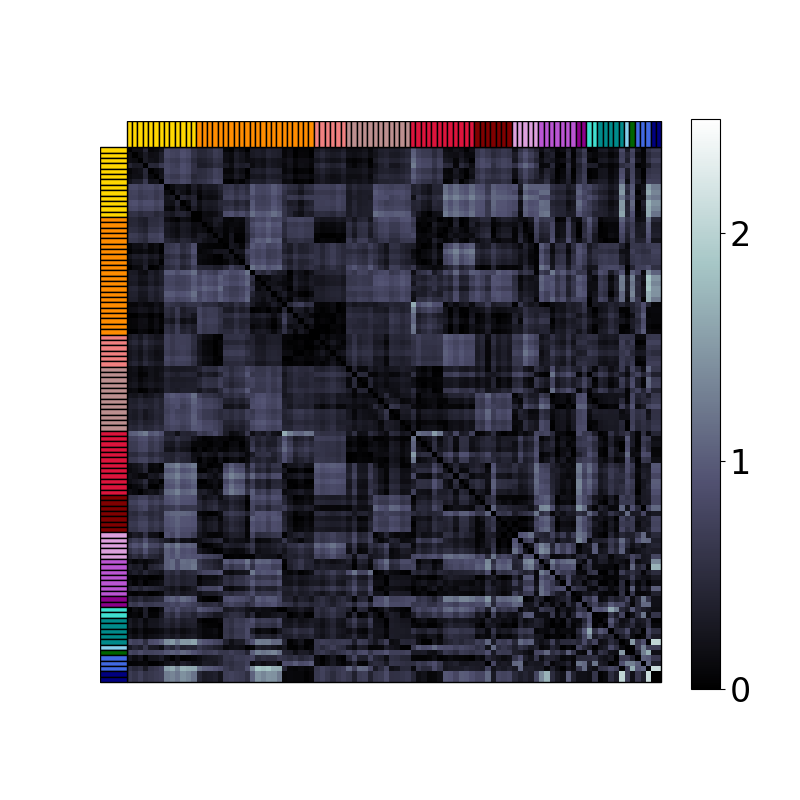}
		\caption{Stress (3d)}
		\label{fig:appendix:vae_baseline_3d:stress}
	\end{subfigure}%
    \vspace{-0.2cm}
	\caption{Embeddings of support poses with VAEs: The first and last two rows respectively show the latent embeddings of the hyperbolic and Euclidean VAE in $\mathcal{P}^Q$ and $\mathbb{R}^Q$, followed by pairwise error matrices.}
    \vspace{-0.3cm}
	\label{fig:appendix:vae_baselines_augmented_support_poses}
\end{figure}

In this section, we compare the trained GPHLVMs of Figs.~\ref{fig:GPHLVM_trained_models-bimanual},~\ref{fig:GPHLVM_trained_models-grasps}, and~\ref{fig:GPHLVM_trained_models} with two additional baselines: a vanilla variational autoencoder (VAE) and a hyperbolic variant of this VAE in which the latent space is the Lorentz model of hyperbolic geometry (akin to \cite{Mathieu19:HyperbolicVAE}). Both VAEs are designed with $12$ input nodes, $6$ hidden nodes, a $2$-dimensional latent space, and a symmetric decoder. Their encoder specifies the mean and standard deviation of a normal distribution (resp. wrapped normal for the hyperbolic VAE), and their decoder specifies the mean and standard deviation of the normal distribution that governs the reconstructions. Both models are trained by maximizing an Evidence Lower Bound (ELBO) under similar regimes as the GPHLVMs, i.e., $1000$ epochs with a learning rate of $0.05$. The KL divergence for the hyperbolic VAE is computed using Monte Carlo estimates.

Figs.~\ref{fig:appendix:vae_baselines_bimanual},~\ref{fig:appendix:vae_baselines_grasps}, and~\ref{fig:appendix:vae_baselines_augmented_support_poses} show the learned embeddings of the Euclidean and hyperbolic VAE with $2$ and $3$-dimensional latent spaces alongside the corresponding error matrices between geodesic and taxonomy graph distances for the bimanual manipulation, hand grasps and  support pose taxonomies. 
Although adding a stress regularization as for the GPHLVM helps preserve the graph distance structure, the embeddings of different classes are not as well separated as in our GPHLVM models (see Fig.~\ref{fig:appendix:vae_baselines_bimanual} vs~\ref{fig:GPHLVM_trained_models-bimanual}, Fig.~\ref{fig:appendix:vae_baselines_grasps} vs~\ref{fig:GPHLVM_trained_models-grasps}, and Fig.~\ref{fig:appendix:vae_baselines_augmented_support_poses} vs~\ref{fig:GPHLVM_trained_models}). Moreover, when compared to our proposed GPHLVM, all VAE models provide a subpar uncertainty modeling in their latent spaces. 

Table~\ref{table:mean_stress_and_reconstruction_of_models} shows that the VAE baselines result in higher average stress than the GPLVMs.  In other words, our proposed GPHLVM consistently outperforms all VAEs to encode meaningful taxonomy information in the latent space. Moreover, the GPLVMs consistently achieve a lower reconstruction error than the VAE baseline. We argue that VAEs are not the right tool for our target applications. When training VAEs, the Kullback-Leibler term in the ELBO tries to regularize the latent space to match a unit Gaussian. This regularization is in stark contrast with our goal of separating the embeddings to preserve the taxonomy graph distances.

\begin{table}[t]
    \caption{Average stress and reconstruction error per model, geometry, and regularization.}
    \vspace{-0.3cm}
    \label{table:mean_stress_and_reconstruction_of_models}
    \begin{center}
    \begin{small}
    \begin{sc}
    \resizebox{\textwidth}{!}{
		\begin{tabular}{llllllll}
			%\toprule
             &  & \multicolumn{3}{c}{\textbf{Stress}} & \multicolumn{3}{c}{\textbf{Reconstruction error}} \\
             \cmidrule(lr){3-5} \cmidrule(lr){6-8}
			     &  & \textbf{No reg.} & \textbf{Stress} & \textbf{BC + Stress} & \textbf{No reg.} & \textbf{Stress} & \textbf{BC + Stress} \\
			\toprule
   			\multirow{8}{*}{\parbox{2.2cm}{Bimanual manipulation categories}} & GPLVM, $\mathbb{R}^2$ &  $2.03\pm2.15$ &  $0.13\pm0.33$ &  $0.15\pm0.31$&  $0.01\pm0.02$ &  $0.01\pm0.01$ & $0.02\pm0.02$ \\
            & VAE, $\mathbb{R}^2$ & $1.70\pm1.97$ & \multicolumn{2}{c}{$0.12\pm0.20$} & $0.11\pm0.18$ & \multicolumn{2}{c}{$0.12\pm0.17$} \\
            & VAE, $\lorentz{2}$ & $1.89\pm1.85$ &  \multicolumn{2}{c}{$0.10\pm0.15$} & $0.12\pm0.18$ &  \multicolumn{2}{c}{$0.12\pm0.17$} \\
			& GPHLVM, $ \lorentz{2}$ &  $\bm{0.98\pm1.26}$ &  $\bm{0.11\pm0.33}$ &  $\bm{0.09\pm0.12}$&  $0.04\pm0.04$ &  $ 0.03\pm0.04$ & $0.04\pm0.04$ \\
            \cmidrule(lr){2-8}
			& GPLVM, $ \mathbb{R}^3$ &  $2.39\pm2.36$ &  $\bm{0.01\pm0.01}$ &  $0.20\pm0.38$&  $0.01\pm0.01$ &  $\bm{0.01\pm0.01}$ & $0.01\pm0.01$ \\
            & VAE, $\mathbb{R}^3$ & $2.58\pm2.76$ & \multicolumn{2}{c}{$0.05\pm0.09$} & $0.08\pm0.15$ & \multicolumn{2}{c}{$0.12\pm0.18$} \\
            & VAE, $\lorentz{3}$ & $1.76\pm1.84$ & \multicolumn{2}{c}{$0.11\pm0.17$} & $0.03\pm0.04$ & \multicolumn{2}{c}{$0.12\pm0.18$} \\
            & GPHLVM, $ \lorentz{3}$ &  $\bm{1.18\pm1.35}$ &  $\bm{0.01\pm0.03}$ &  $\bm{0.04\pm0.08}$&  $\bm{0.00\pm0.01}$ &  $\bm{0.01\pm0.01}$ & $\bm{0.00\pm0.01}$ \\
            \toprule
			\multirow{8}{*}{Grasps} & GPLVM, $\mathbb{R}^2$ &  $7.25\pm 5.40$ &  $0.39\pm 0.41$ &  $0.40\pm 0.44$& $\bm{0.04\pm0.04}$ & $\bm{0.03\pm0.03}$ & $\bm{0.03\pm0.03}$ \\   
            & VAE, $\mathbb{R}^2$ & $3.52\pm4.31$ & \multicolumn{2}{c}{$0.48\pm0.55$} & $0.11\pm0.12$ & \multicolumn{2}{c}{$0.13\pm0.15$} \\
            & VAE, $\lorentz{2}$ & $8.99\pm6.20$ & \multicolumn{2}{c}{$0.70\pm1.28$} & $0.13\pm0.16$ & \multicolumn{2}{c}{$0.14\pm0.15$} \\
			& GPHLVM, $ \lorentz{2}$ &  $\bm{5.47\pm 4.07}$ &  $\bm{0.14\pm 0.16}$ & $\bm{0.18 \pm 0.29}$ & $0.05\pm0.05$ & $0.08\pm0.07$ & $0.09\pm0.09$ \\
            \cmidrule(lr){2-8}
			& GPLVM, $ \mathbb{R}^3$ &  $\bm{8.15\pm 5.85}$ &  $0.14\pm 0.18$ &  $0.15\pm 0.19$& $0.03\pm0.03$ & $0.14\pm0.18$ & $0.15\pm0.19$\\
            & VAE, $\mathbb{R}^3$ & $2.71\pm3.47$ & \multicolumn{2}{c}{$0.25\pm0.32$} & $0.10\pm0.13$ & \multicolumn{2}{c}{$0.14\pm0.16$} \\
            & VAE, $\lorentz{3}$ & $8.28\pm5.94$ & \multicolumn{2}{c}{$0.33\pm0.59$} & $0.11\pm0.14$& \multicolumn{2}{c}{$0.12\pm0.14$} \\
            & GPHLVM, $ \lorentz{3}$ &  $8.37\pm 5.71$ &  $\bm{0.04\pm 0.08}$ &  $\bm{0.07\pm 0.18}$& $\bm{0.03\pm0.02}$ & $\bm{0.01\pm0.01}$ & $\bm{0.02\pm0.02}$ \\
            \toprule
            \multirow{8}{*}{\parbox{1.5cm}{Support poses}} & GPLVM, $\mathbb{R}^2$ &  $3.93\pm 3.97$ &  $0.58\pm 0.94$ &  $0.63\pm 0.94$& $\bm{0.05\pm0.05}$ & $0.17\pm0.18$ & $\bm{0.11\pm0.12}$ \\
            & VAE, $\mathbb{R}^2$ & $1.75\pm2.29$ & \multicolumn{2}{c}{$0.54\pm0.80$} & $0.15\pm0.18$ & \multicolumn{2}{c}{$0.18\pm0.20$} \\
            & VAE, $\lorentz{2}$ & $4.81\pm4.29$ & \multicolumn{2}{c}{$0.57\pm0.85$} & $0.18\pm0.21$ & \multicolumn{2}{c}{$0.18\pm0.20$} \\
			& GPHLVM, $ \lorentz{2}$ &  $\bm{2.05\pm 2.50}$ &  $\bm{0.51\pm 0.82}$ &  $\bm{0.53\pm 0.83}$& $0.07\pm0.07$ & $\bm{0.16\pm0.17}$ & $0.15\pm0.16$\\
            \cmidrule(lr){2-8}
			& GPLVM, $ \mathbb{R}^3$ &  $\bm{3.76\pm 3.74}$ &  $\bm{0.24\pm 0.40}$ &  $\bm{0.29\pm 0.39}$& $\bm{0.03\pm0.03}$ & $0.17\pm0.18$ & $\bm{0.08\pm0.09}$ \\
            & VAE, $\mathbb{R}^3$ & $2.10\pm2.64$ & \multicolumn{2}{c}{$0.31\pm0.40$} & $0.38\pm0.47$ & \multicolumn{2}{c}{$0.16\pm0.19$} \\
            & VAE, $\lorentz{3}$ & $4.53\pm4.23$ & \multicolumn{2}{c}{$0.38\pm0.55$} & $0.17\pm0.21$ & \multicolumn{2}{c}{$0.17\pm0.20$} \\
            & GPHLVM, $\lorentz{3}$ &  $3.78\pm 3.71$ &  $0.30\pm 0.38$ &  $0.35\pm 0.45$& $\bm{0.03\pm0.03}$ & $\bm{0.16\pm0.17}$ & $\bm{0.08\pm0.09}$\\
			\bottomrule
	\end{tabular}
 }
    \end{sc}
    \end{small}
    \end{center}
\end{table}

\subsection{Comparison against learned manifolds}
\label{app:comparison_Tosi}

We compare the proposed GPHLVM to a GPLVM that learns a Riemannian manifold from data~\citep{Tosi14:RiemannianGPLVM}. 
Fig.~\ref{fig:appendix:tosi_baselines} shows the learned latent space including the embeddings and the volume of the Riemannian metric of the learned manifold, alongside distance matrices for the three considered robotics taxonomies.
Overall, the model is unable to capture the local and global taxonomy structure. This is due to the fact that the learned Riemannian metric is designed to be high in regions with high uncertainty, thus leading to shortest paths, i.e., geodesics, avoiding these regions. As such, this model was not designed for hierarchical discrete data and does not embed any knowledge about the taxonomy. This is further reflected by the resulting high stress values (see Fig.~\ref{table:appendix:tosi_stresses}).

\begin{figure}
	\centering
	\includegraphics[trim={5.3cm 2.2cm 4.3cm 2.2cm},clip,width=0.8\textwidth]{Figures/legend_semifull.png}
  	\begin{subfigure}[b]{0.17\textwidth}
		\centering
		\includegraphics[trim={1.0cm 2.0cm 0.0cm 2.0cm},clip,width=\textwidth]{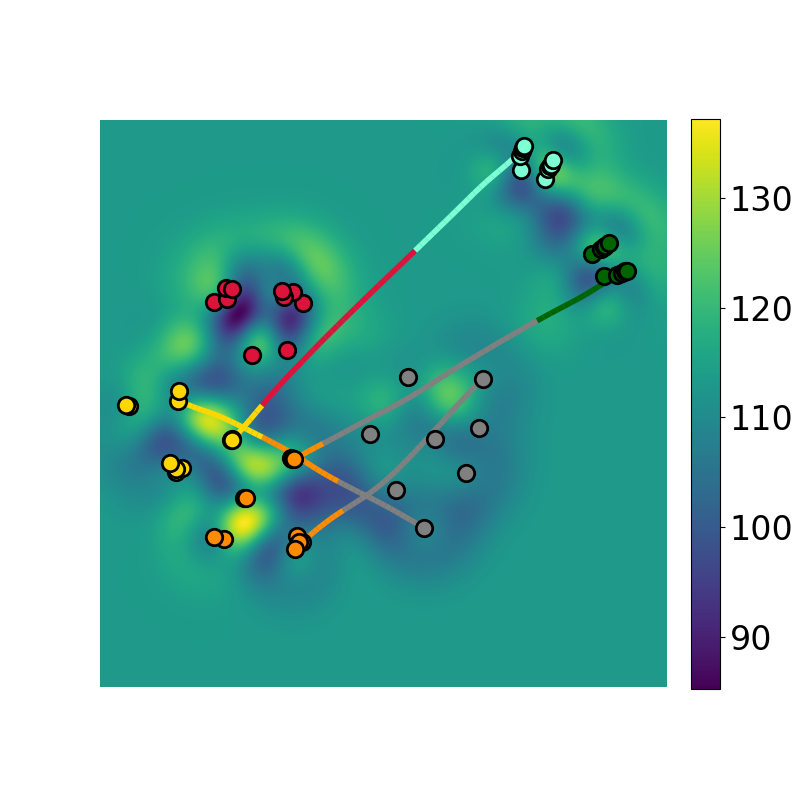}
		\includegraphics[trim={2.0cm 2.0cm 1.0cm 2.0cm},clip,width=0.9\textwidth]{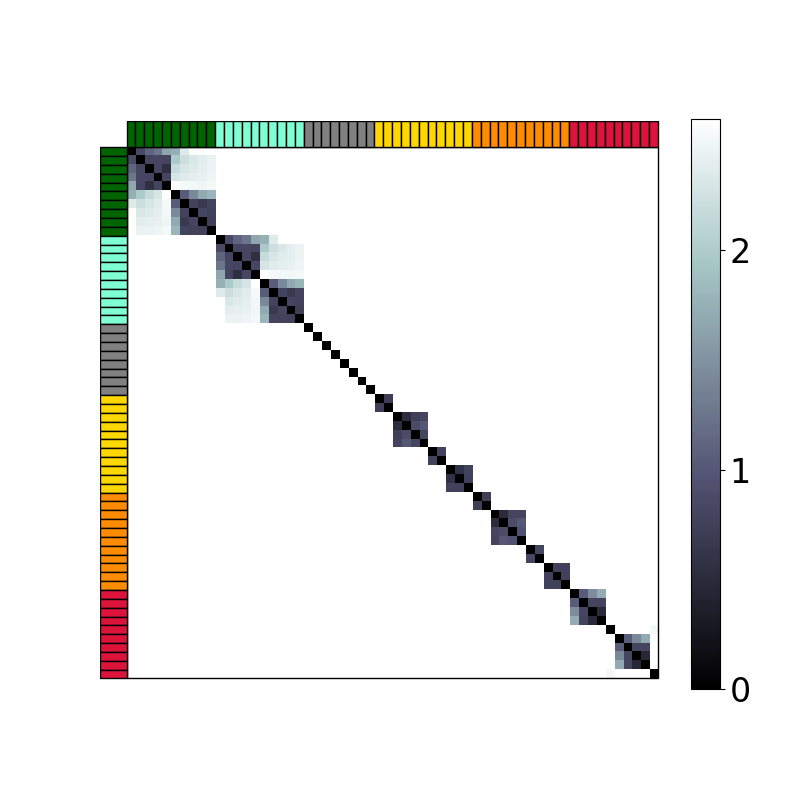}
		\caption{Bimanual categories}
		\label{fig:appendix:tosi-bimanual}
	\end{subfigure}%
	\begin{subfigure}[b]{0.17\textwidth}
		\centering
		\includegraphics[trim={1.0cm 2.0cm 0.0cm 2.0cm},clip,width=\textwidth]{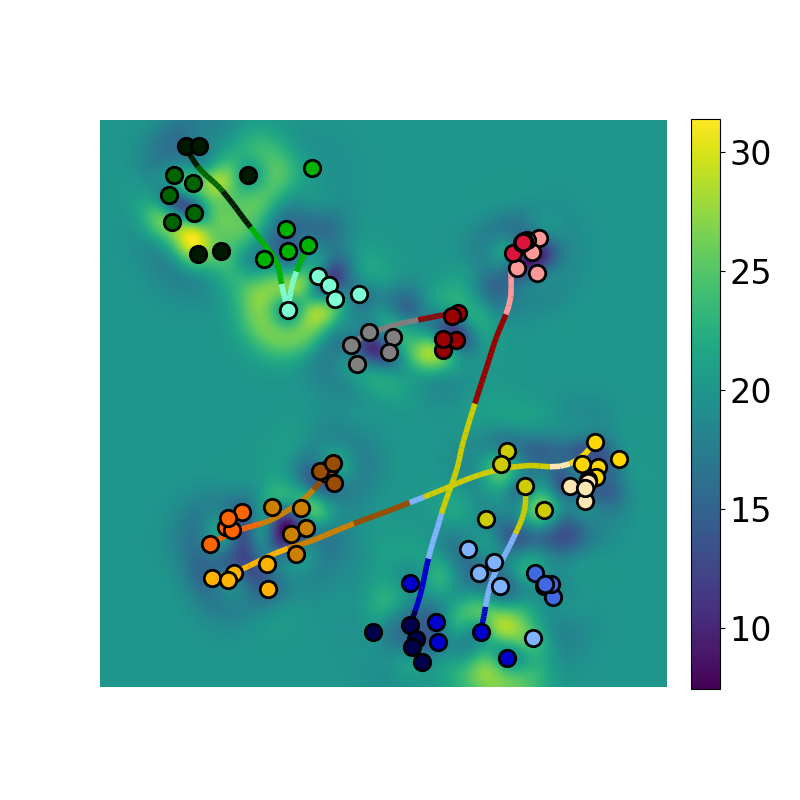}
		\includegraphics[trim={2.0cm 2.0cm 1.0cm 2.0cm},clip,width=0.9\textwidth]{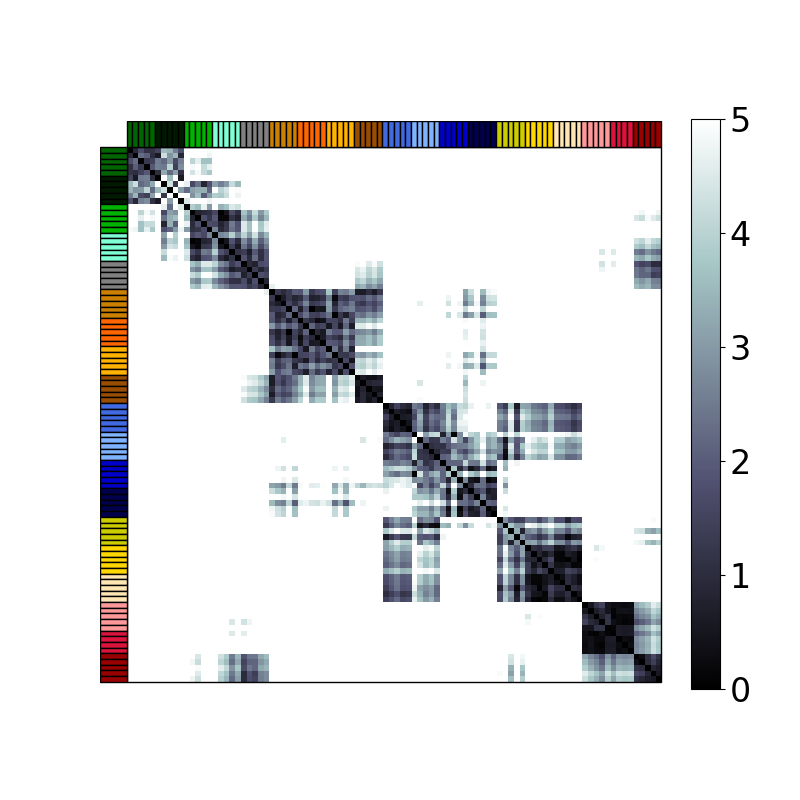}
		\caption{Grasps}
		\label{fig:appendix:tosi-grasps}
	\end{subfigure}%
    \begin{subfigure}[b]{0.17\textwidth}
		\centering
		\includegraphics[trim={1.0cm 2.0cm 0.0cm 2.0cm},clip,width=\textwidth]{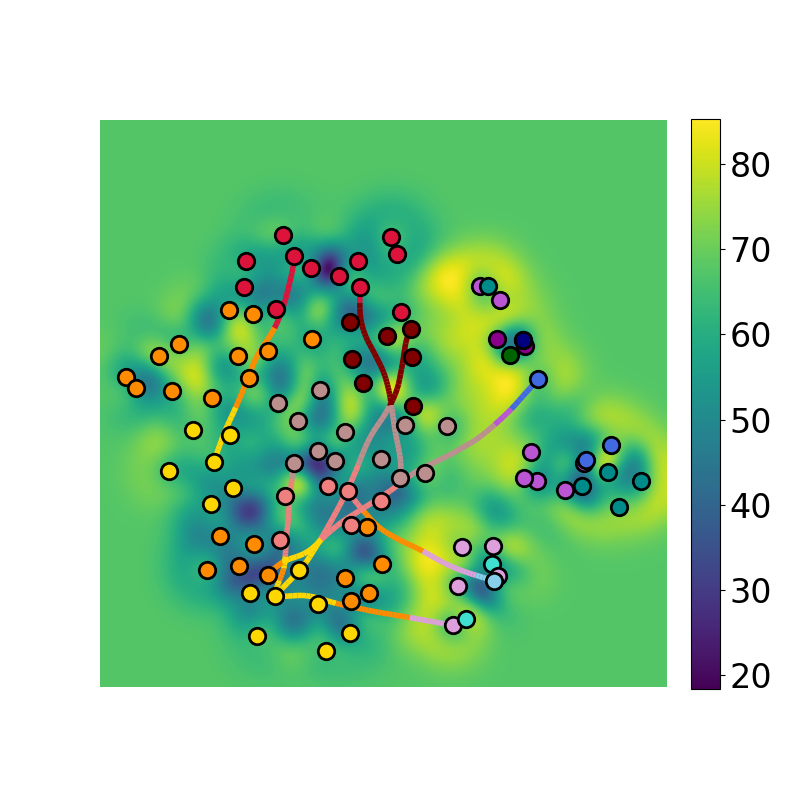}
		\includegraphics[trim={2.0cm 2.0cm 1.0cm 2.0cm},clip,width=0.9\textwidth]{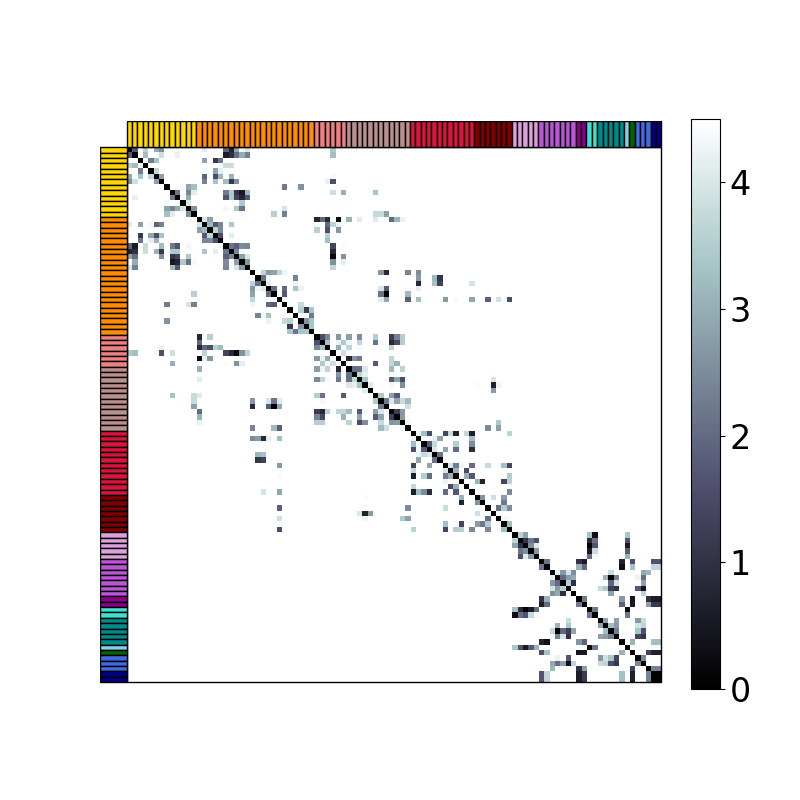}
		\caption{Support poses}
		\label{fig:appendix:tosi-support-poses}
	\end{subfigure}%
    \begin{subfigure}[b]{0.28\textwidth}
		\centering
		\resizebox{\textwidth}{!}{
		\begin{tabular}{ll}
			\toprule
			\textbf{Taxonomy}&  \textbf{Stress} \\
			\toprule
            \parbox{1.9cm}{Bimanual manipulation categories} &  $395.90\pm366.01$ \\
            \midrule
   			Grasps &  $108.56\pm106.11$ \\
            \midrule
			\parbox{1.9cm}{Support poses} &  $277.59\pm240.37$ \\
			\bottomrule
		\end{tabular}
	}
 
    \vspace{0.1cm}
 
	\caption{Average stress}
	\label{table:appendix:tosi_stresses}
	\end{subfigure}%
    \vspace{-0.2cm}
	\caption{Embeddings of taxonomy data on learned manifolds: The first row shows the latent spaces of the GPLVM. The background color is proportional to volume of the learned Riemannian metric. The second row displays the error matrix between the geodesic and taxonomy graph distances.}
    \vspace{-0.3cm}
	\label{fig:appendix:tosi_baselines}
\end{figure}

\clearpage

%%%%%%%%%%%%%%%%%%%%%%%%%%%%%%%%%%%%%%%%%%%%%%%%%%%%%%%%%%%%%%%%%%%%%%%%%%%%%%%
%%%%%%%%%%%%%%%%%%%%%%%%%%%%%%%%%%%%%%%%%%%%%%%%%%%%%%%%%%%%%%%%%%%%%%%%%%%%%%%

\end{document}